\documentclass[sn-mathphys,icol]{sn-jnl}

\usepackage{array}
\usepackage{enumitem}
\usepackage{float}
\usepackage{verbatim}
\usepackage{amsfonts}
\usepackage{makecell}
\usepackage{hhline}
\usepackage{xcolor}
\usepackage{caption}
\usepackage{afterpage}
\usepackage{mathtools}
\usepackage{amsmath}
\usepackage{flafter}
\usepackage{adjustbox}
\usepackage{setspace}
\usepackage{float}
\usepackage{subcaption}
\usepackage{bbm}
\usepackage{mathrsfs} 
\usepackage{graphicx}
\usepackage{etoolbox}
\usepackage{color}
\usepackage{rotating}
\usepackage{color,soul}

 \normalbaroutside
\jyear{2021}%

\theoremstyle{thmstyleone}%
\newtheorem{theorem}{Theorem}[section]
\newtheorem{proposition}{Proposition}[section]%
\newtheorem{lem}{Lemma}[section]

\theoremstyle{thmstyletwo}%

\theoremstyle{thmstylethree}%
\usepackage[linesnumbered,ruled,vlined]{algorithm2e}

\SetKwInput{KwInput}{Input}                
\SetKwInput{KwOutput}{Output} 

\DeclareMathOperator{\sign}{sign}
\DeclareMathOperator*{\argmin}{arg\,min}
\DeclareMathOperator*{\argmax}{arg\,max}

\begin{document}

\title[Article Title]{An Efficient Smoothing and Thresholding Image Segmentation Framework with Weighted Anisotropic--Isotropic Total Variation}


\author*[1]{\fnm{Kevin} \sur{Bui}}\email{kevinb3@uci.edu}

\author[2]{\fnm{Yifei} \sur{Lou}}\email{yflou@unc.edu}

\author[3]{\fnm{Fredrick} \sur{Park}}\email{fpark@whittier.edu}
\author[1]{\fnm{Jack} \sur{Xin}}\email{jxin@math.uci.edu}

\affil*[1]{\orgdiv{Department of Mathematics}, \orgname{University of California, Irvine}, \orgaddress{\city{Irvine}, \postcode{92697-3875}, \state{CA}, \country{United States}}}

\affil[2]{\orgdiv{Department of Mathematics}, \orgname{University of North Carolina, Chapel Hill}, \orgaddress{ \city{Chapel Hill}, \postcode{27599}, \state{NC}, \country{United States}}}

\affil[3]{\orgdiv{Department of Mathematics \& Computer Science}, \orgname{Whittier College}, \orgaddress{\city{Whittier}, \postcode{90602}, \state{CA}, \country{United States}}}


\abstract{In this paper, we design an efficient, multi-stage image segmentation framework that incorporates a weighted difference of anisotropic and isotropic total variation (AITV). The segmentation framework generally consists of two stages: smoothing and thresholding, thus referred to as SaT. In the first stage, a smoothed image is obtained by 
an AITV-regularized Mumford-Shah (MS) model, which can be solved efficiently by the alternating direction method of multipliers (ADMM) with a closed-form solution of a proximal operator of the $\ell_1 -\alpha \ell_2$ regularizer. Convergence of the ADMM algorithm is analyzed. In the second stage, we threshold the smoothed image by $K$-means clustering to obtain the final segmentation result. Numerical experiments demonstrate that the proposed segmentation framework is versatile for both grayscale and color images, efficient in
producing high-quality segmentation results within a few seconds, and robust to input images that  are corrupted with noise, blur, or both. We compare the AITV method with its original convex TV and nonconvex TV$^p (0<p<1)$ counterparts, showcasing the qualitative and quantitative advantages of our proposed method. }


\maketitle

\section{Introduction}
Image segmentation is a prevalent, challenging problem in computer vision, aiming to partition an image into several regions that represent specific objects of interest. Each partitioned region has similar features such as edges, colors, and intensities. One  segmentation method is
the Mumford-Shah (MS) model \cite{mumford1989optimal} well-known for its robustness to noise. It finds the optimal piecewise-smooth approximation of an input image that incorporates region and boundary information to facilitate segmentation. Given  a bounded, open set $\Omega \subset \mathbb{R}^2$  with Lipschitz boundary and an observed image $f: \Omega \rightarrow [0,1]$, the MS model can be expressed as an energy minimization problem, 
{\small
\begin{align} \label{eq:MS_model}
\begin{split}
    \min_{u, \Gamma} E_{MS}(u, \Gamma) \coloneqq &\frac{\lambda}{2} \int_{\Omega} (f-u)^2 \;dx + \frac{\mu}{2} \int_{\Omega \setminus \Gamma} |\nabla u|^2 \;dx + \text{Length}(\Gamma), 
    \end{split}
\end{align}}%
where $\lambda, \mu >0$ are weighing parameters, $\Gamma \subset \Omega$ is a compact curve representing the boundaries separating disparate objects, and $u: \Omega \rightarrow \mathbb{R}$ is an approximation of $f$ that is smooth in $\Omega \setminus \Gamma$ but possibly discontinuous across $\Gamma$. The middle term $\int_{\Omega \setminus \Gamma} |\nabla u|^2 \;dx$ ensures that $u$ is piecewise smooth, or more specifically differentiable on $\Omega \setminus \Gamma$. The last term ``$\text{Length}(\Gamma)$'' measures the perimeter of $\Gamma$ that can be mathematically expressed as $\mathcal{H}^1(\Gamma)$, which is the 1-dimensional Hausdorff measure in $\mathbb{R}^2$ \cite{bar2011mumford}. It is  challenging to solve for the minimization problem \eqref{eq:MS_model} due to its nonconvex nature and difficulties in discretizing the unknown set of boundaries.
Pock et al.~\cite{pock2009algorithm} proposed a convex relaxation of \eqref{eq:MS_model}  together with  an efficient primal-dual algorithm. For the boundary issue, one early attempt involved  a sequence of (local) elliptic variational problems \cite{ambrosio1990approximation} to approximate the energy functional \eqref{eq:MS_model}. Later, nonlocal approximations were adopted in 
\cite{gobbino1998finite, chambolle1999finite} and a finite element approximation was developed in \cite{chambolle1999discrete}. 

By relaxing $u$ from piecewise smooth to piecewise constant, Chan and Vese (CV) \cite{chan-vese-2001} proposed a two-phase model to  segment the image domain $\Omega$ into two regions that  are inside and outside of the curve $\Gamma$.
The curve can be represented by a level-set function $\phi$ that is Lipschitz continuous and satisfies 
\begin{align*}
    \begin{cases}
    \phi(x) > 0 &\text{ if } x \text{ is inside } \Gamma, \\
    \phi(x) = 0 &\text{ if } x \text{ is at } \Gamma, \\
    \phi(x) <0 & \text{ if } x \text{ is outside } \Gamma.
    \end{cases}
\end{align*}
The Heaviside function $H(\phi)$  is defined by $H(\phi) = 1$ if $\phi \geq 0$ and $H(\phi) = 0$ otherwise. The CV model is given by
{\small
\begin{align} \label{eq:CV}
\begin{split}
    \min_{c_1, c_2, \phi} E_{CV}(c_1, c_2, \phi) &\coloneqq \lambda \int_{\Omega} |f-c_1|^2 H(\phi) \;dx + \lambda \int_{\Omega} |f-c_2|^2 (1-H(\phi))\;dx\\ &+ \nu \int_{\Omega} |\nabla H(\phi)| \;dx,
    \end{split}
\end{align}}%
where $\lambda, \nu $ are two positive parameters and $c_1, c_2 \in \mathbb{R}$ are mean intensity values 
of the two regions. Originally, the CV model \eqref{eq:CV} was solved by finite difference methods \cite{chan2000active, getreuer2012chan}. Later Chan et al.~\cite{chan-esedoglu-nikolova-2004} formulated a convex relaxation of  CV  so that it can be solved by convex optimization techniques such as split Bregman \cite{goldstein2009split, goldstein2010geometric}, alternating direction method of multipliers (ADMM) \cite{boyd2011distributed}, and primal-dual hybrid gradient (PDHG) \cite{esser2010general, chambolle-pock-2011}. As an alternative to the level-set formulation \eqref{eq:CV}, a diffuse-interface approximation to the CV model  was considered in \cite{esedog2006threshold}, which can be solved efficiently  by the Merrimen-Bence-Osher scheme \cite{merriman1994motion}. The (two-phase) CV model can be naively extended to the multiphase segmentation \cite{vese2002multiphase} but with a limitation that it can only deal with power-two number of segmentation regions. 
The multiphase CV model was later combined with fuzzy membership functions \cite{li2010multiphase} in order to segment arbitrary number of regions. 

Another approach of finding a piecewise-constant solution to the MS model is the smoothing-and-thresholding (SaT) framework \cite{cai2013two}. In SaT, one first finds a smoothed image $u$ by solving a convex variant of the MS model:
\begin{align} \label{eq:convex_MS}
\begin{split}
    \min_u &\frac{\lambda}{2} \int_{\Omega} (f-Au)^2 \;dx + \frac{\mu}{2} \int_{\Omega} |\nabla u|^2 \;dx + \int_{\Omega} |\nabla u|\;dx, \; 
    \end{split}
\end{align}
where $\lambda>0, \mu > 0,$ and $A$ is a linear operator. Specifically, $A$ is the identity operator if one wants to segment a noisy image $f$, while it can be a blurring operator for the desire of segmenting a blurry and noisy image $f$. 
The middle term $\int_{\Omega}|\nabla u|^2 \;dx$ extends the piecewise-smooth regularization $\int_{\Omega \setminus \Gamma} |\nabla u|^2 \;dx$ in \eqref{eq:MS_model} to the entire image domain $\Omega$. 
The last term $\int_{\Omega} |\nabla u| \;dx$ is the total variation (TV) that approximates the length term in \eqref{eq:MS_model} based on the coarea formula \cite{chan-esedoglu-nikolova-2004}. 
 After obtaining a piecewise-smooth approximation, one segments the image domain into $K$ regions by thresholding $u$ with $K-1$ appropriately selected values. SaT has several advantages over the MS model \eqref{eq:MS_model} and the CV model \eqref{eq:CV}. First, the smoothing stage involves  a strictly convex problem \eqref{eq:convex_MS} to guarantee 
 a unique solution that can be found by numerous convex optimization algorithms. Second, the thresholding stage allows for segmenting any number of regions via a clustering algorithm such as $K$-means clustering \cite{hartigan1979ak, arthur07}. Lastly, thresholding is independent of smoothing; in other words, thresholding can be adjusted to obtain a visually appealing segmentation without going back to smoothing again. 
 SaT was  adapted to segment images corrupted by Poisson or multiplicative Gamma noise \cite{chan2014two}. For color images, SaT extended to quaternion space {\cite{wu2022efficient}} or evolved into the ``smoothing, lifting, and thresholding" (SLaT) framework \cite{cai2017three}. The additional lifting stage in SLaT adds 
 the Lab (perceived lightness, red-green and yellow-blue) color space
 to provide more discriminatory information than the conventional RGB color space with correlated color channels. 
 The idea of lifting can also improve image segmentation of grayscale images whose pixel intensities vary dramatically, referred to as \textit{intensity inhomogeneity}. Traditional methods that deal with inhomogeneity include  preprocessing   \cite{hou2006review} and intensity correction \cite{li2008minimization, wang2010efficient}. By generating an additional image channel \cite{li2020three}, SaT/SLaT yields better segmentation results for grayscale images that suffer from  intensity inhomogeneity. 

Note that the convex approximation of the length term in \eqref{eq:MS_model} by  $\int_{\Omega} |\nabla u| \;dx$ in \eqref{eq:convex_MS} is not optimal, since the Hausdorff measure is nonconvex. For a better approximation, Wu et al.~\cite{wu2021two} adopted a nonconvex  term $\int_{\Omega} |\nabla u|^p \;dx$ for $(0<p<1),$ referred to as $\text{TV}^p$,  leading to a nonconvex problem,
\begin{align}\label{eq:TVp_MS}
\begin{split}
        \min_u &\frac{\lambda}{2} \int_{\Omega} (f-Au)^2 \;dx + \frac{\mu}{2} \int_{\Omega} |\nabla u|^2 \;dx + \int_{\Omega} |\nabla u|^p\;dx.
        \end{split}
\end{align}
If $p=1,$  $\text{TV}^p$ becomes the TV model. Generally, $\text{TV}^p$ outperforms TV in image restoration and segmentation \cite{chen2012non,hintermuller2013nonconvex, lanza2016constrained, zeng2018edge,li2020tv}. The $\text{TV}^p$ regularization originated from the $\ell_p$ quasinorm, which is more effective than the convex $\ell_1$ norm in recovering sparse signals from an underdetermined linear system \cite{chartrand2008iteratively, xu2012l_}. Recently, a series of work \cite{lou2015computational,lou-2015-cs, yin2015minimization} has demonstrated through experiments that the nonconvex regularizer $\ell_1-\ell_2$ outperforms $\ell_1$ and $\ell_p$ when the linear system is highly coherent.  The $\ell_1 - \ell_2$ model can be generalized to $\ell_1 - \alpha \ell_2$ for $\alpha \in [0,1]$ to allow for sparsity control via the parameter $\alpha$. Theoretical analyses of the $\ell_1 - \alpha \ell_2$ family have been investigated in \cite{ding2019regularization,yin2015minimization,ge2021new,li2020} that justify its superior performances. When applying  $\ell_1 -\alpha \ell_2$  on the image gradient, Lou et al.~\cite{lou-2015} proposed a weighted difference of anisotropic and isotropic TV (AITV) that yields better results over TV and $\text{TV}^p $ for image denoising and  deconvolution. AITV is robust against impulsive noise for image reconstruction \cite{li2020}, and it yields satisfactory segmentation results in the CV model and the fuzzy region competition model \cite{bui2020weighted}. Recently, an AITV-based segmentation model was discussed in \cite{wu2022image}. However, these models are solved by a difference-of-convex algorithm (DCA) \cite{le2018dc, tao-1997, tao-1998}
that requires solving a TV-type subproblem iteratively, thus being computationally expensive.

In this paper, we propose an efficient ADMM framework to solve the AITV variant of \eqref{eq:convex_MS} and demonstrate its efficiency and effectiveness in the SaT/SLaT framework through various numerical experiments.
The efficiency lies in the closed-form solution \cite{louY18}  of
the proximal operator for $\ell_1 - \alpha \ell_2$  to avoid nested loops  in DCA as considered in \cite{bui2020weighted,wu2022image}. The main contributions of this paper are summarized as follows:
\begin{enumerate}
    \item We provide model analysis such as  coerciveness and the existence of  global minimizers for the AITV-regularized variant of \eqref{eq:convex_MS}. 
    \item We develop an efficient ADMM algorithm for minimizing the AITV-based MS model  based on the proximal operator of $\ell_1-\alpha \ell_2$ with a convergence guarantee.
    \item We conduct extensive numerical experiments to showcase that the SaT/SLaT framework with AITV regularization is a competitive segmentation method, especially using our proposed ADMM algorithm. The segmentation framework is robust to  noise, blur, and intensity inhomogeneity.
    \item We  demonstrate experimentally that the proposed ADMM framework is significantly more efficient than DCA used in \cite{bui2020weighted,wu2022image} in producing segmentation results of comparable or even better quality.
\end{enumerate}

The paper is organized as follows.  Section \ref{sec:notation} summarizes mathematical notations and reviews the SaT/SLaT framework. Section \ref{sec:AITV_SAT} provides analysis of the AITV-regularized MS model that can be solved by ADMM. Convergence analysis of the algorithm subsequently follows. Section \ref{sec:experiments} presents extensive experiments on various grayscale and color images,  comparing the AITV SaT/SLaT framework  to other state-of-the-art segmentation methods to  demonstrate the effectiveness of the proposed approaches. Lastly, we conclude the paper in Section \ref{sec:conclusion}.

\section{ Preliminaries}\label{sec:notation}
\subsection{Notations}
For simplicity, we adopt the discrete notations for images and mathematical models. Without loss of generality, an image is represented as an $M \times N$ matrix, so the image domain is $\Omega = \{1, 2, \ldots, M\} \times \{1,2,\ldots, N\}$. Then we denote $X \coloneqq \mathbb{R}^{M \times N}$. We adopt the linear index for 2D image, where for $u \in X$, we have $u_{i,j} \in \mathbb{R}$ be the $((i-1)M+j)$th component of $u$.  The gradient operator $\nabla: X \rightarrow X \times X$ is denoted by $\nabla u = (\nabla_x u, \nabla_y u)$ with 
 $\nabla_x$ and $\nabla_y$ being the horizontal and vertical forward difference operators, respectively, with the periodic boundary condition.
 Specifically, the $(i,j)$th entry of $\nabla u$ is defined by
\begin{align*}
    (\nabla u)_{i,j} = \begin{bmatrix}(\nabla_x u)_{i,j}\\
    (\nabla_y u)_{i,j}
    \end{bmatrix},
\end{align*}
where
	\begin{align*}
	(\nabla_x u)_{i,j} = \begin{cases}
	u_{i,j} - u_{i,j-1} &\text{ if } 2 \leq j \leq N, \\
	u_{i,1} - u_{i,N} &\text{ if } j = 1
	\end{cases}
\end{align*}
and
\begin{align*}
	(\nabla_y u)_{i,j} = \begin{cases}
	u_{i,j} - u_{i-1,j} &\text{ if } 2 \leq i \leq M, \\
	u_{1,j} - u_{M,j} &\text{ if } i = 1.
	\end{cases}
	\end{align*}
 For $p = (p_x, p_y) \in X \times X$, its $((i-1)M+j)$th component is $p_{i,j} =
\begin{bmatrix}
(p_x)_{i,j}\\
(p_y)_{i,j}
\end{bmatrix} \in \mathbb{R}^2$. We define the following norms on $X \times X$:
\begin{align*}
    \|p\|_1 &= \sum_{i=1}^M \sum_{j=1}^N \left(|(p_x)_{i,j}| + |(p_y)_{i,j}| \right),\\
    \|p\|_2 &= \sqrt{\sum_{i=1}^M \sum_{j=1}^N |(p_x)_{i,j}|^2 + |(p_y)_{i,j}|^2},\\
    \|p\|_{2,1}  &= \displaystyle \sum_{i=1}^M \sum_{j=1}^N \sqrt{(p_x)_{i,j}^2+(p_y)_{i,j}^2}.
\end{align*}
Lastly, the proximal operator for a function $f: \mathbb{R}^n \rightarrow \mathbb{R} \cup \{+\infty\}$ at $y \in \mathbb{R}^n$  is given by 
\begin{align*}
\text{prox}_f(y) = \argmin_{x \in \mathbb{R}^n} f(x) + \frac{1}{2} \|x-y\|_2^2.
\end{align*}
\subsection{Review of SaT/SLaT}
Both SaT and SLaT frameworks consist of two general steps:
(1) smoothing  to extract a piecewise-smooth approximation of a given image and (2) thresholding to segment the regions via $K$-means clustering. SLaT has an intermediate stage called lifting, which generates additional color channels as opposed to the RGB color space for the smoothed image.  More details for each stage are described below.

\subsubsection{First Stage: Smoothing}
Let $f = (f_1, \ldots, f_d) \in X^d$, where $d$ represents the number of channels in the image $f$. For example, when the image $f$ is grayscale, we have $d=1$, and when it is color, we have $d=3$. In general, $f$ can be a multichannel image. Some of its channels could be generated from the original image to provide more information for segmentation. For example, the intensity inhomogeneity image {\cite{li2020three}} is generated as an additional channel that quantifies the amount of intensity inhomogeneity in the original image.

The discretized model of \eqref{eq:convex_MS} for each channel $\ell = 1, \ldots, d$ can be expressed as
\begin{align} \label{eq:discretized_convex_MS}
\begin{split}
    \min_{u_\ell} \frac{\lambda}{2} \|f_\ell-Au_\ell\|_2^2 + \frac{\mu}{2} \|\nabla u_\ell\|_2^2 + &\|\nabla u_\ell\|_{2,1},
    \end{split}
\end{align}
where $\lambda >0, \mu > 0$
and $\|\nabla u_{\ell}\|_2^2$ is a  smoothing term to reduce the staircase effects caused by the isotropic TV $\|\nabla u_\ell \|_{2,1}$. We assume the same pair of parameters $(\lambda,\mu)$ across channels.  In summary, we obtain a smooth approximation $u_\ell$ for each channel $f_\ell$ by solving \eqref{eq:discretized_convex_MS}. 

\subsubsection{Intermediate Stage: Lifting}
For a color image $f = (f_1, f_2, f_3) \in X^3$, where $f_1$, $f_2$, and $f_3$ are the red, green, and blue channels, respectively,  we can obtain $(u_1,u_2,u_3)$ by applying the smoothing stage to each channel of $f$.  Instead of using $(u_1,u_2,u_3)$, SLaT transforms $(u_1,u_2,u_3)$ into $(\bar{u}_1, \bar{u}_2, \bar{u}_3)$ in the Lab space (perceived lightness, red-green, and yellow-blue) \cite{luong1993color} and operates on a new vector-valued image $(u_1, u_2, u_3, \bar{u}_1, \bar{u}_2, \bar{u}_3)$. The rationale is that RGB channels are highly correlated, while the Lab space relies on numerical color differences to approximate the color differences perceived by the human eye. As a result, $(u_1, u_2, u_3, \bar{u}_1, \bar{u}_2, \bar{u}_3)$ leads to better segmentation results compared to $(u_1,u_2,u_3)$.

\subsubsection{Final Stage: Thresholding }
After rescaling the image obtained after smoothing and/or lifting, we denote the resultant image by $u^* \in [0,1]^D$. For example, we have $D=1$ when applying SaT to a grayscale image, and we have $D=6$ when applying SLaT to a color image. Suppose the number of segmented regions is given and denoted by $K$. The thresholding stage applies $K$-means clustering to the vector-valued image $u^*$, providing $K$ centroids $c_1, c_2, \ldots, c_K$ as constant vectors. These centroids are used to form the regions
{\small
\begin{align*}
    \Omega_{k} =
    \left \{(i,j) \in \Omega : \| u^*_{i,j} -c_{k}\|_2 = \min_{1 \leq \kappa \leq K} \|u_{i,j}^* - c_{\kappa}\|_2 \right\},
\end{align*}}%
for $k=1, \ldots, K$ such that $\Omega_{k}$'s are disjoint and $\bigcup_{k=1}^K \Omega_{k} = \Omega$.  Using the centroids and regions, we can obtain a piecewise-constant approximation of $f$, denoted by
{\small
\begin{align}\label{eq:pc_constant_approx}
    \tilde{f} = (\tilde{f}_1, \ldots, \tilde{f}_d) \text{ such that } \tilde{f}_{\ell} = \sum_{k=1}^K c_{k, \ell} \mathbbm{1}_{\Omega_{k}}\; \forall \ell =1, \ldots, d,
    \end{align}
    where $c_{k, \ell}$ is the $\ell$th entry of $c_k$ and 
    \begin{align*}\mathbbm{1}_{\Omega_{k}} = \begin{cases}
    1 &\text{ if } (i,j) \in \Omega_{k}, \\
    0 &\text{ if } (i,j) \not \in \Omega_{k}.
    \end{cases} 
\end{align*}
Recall that $d=1$ when the image $f$ is grayscale and $d=3$ when it is color.

\section{Smoothing with AITV Regularization} \label{sec:AITV_SAT}

\begin{sidewaysfigure}
     \centering
     \begin{subfigure}[b]{\textwidth}
         \centering
         \includegraphics[scale=0.25]{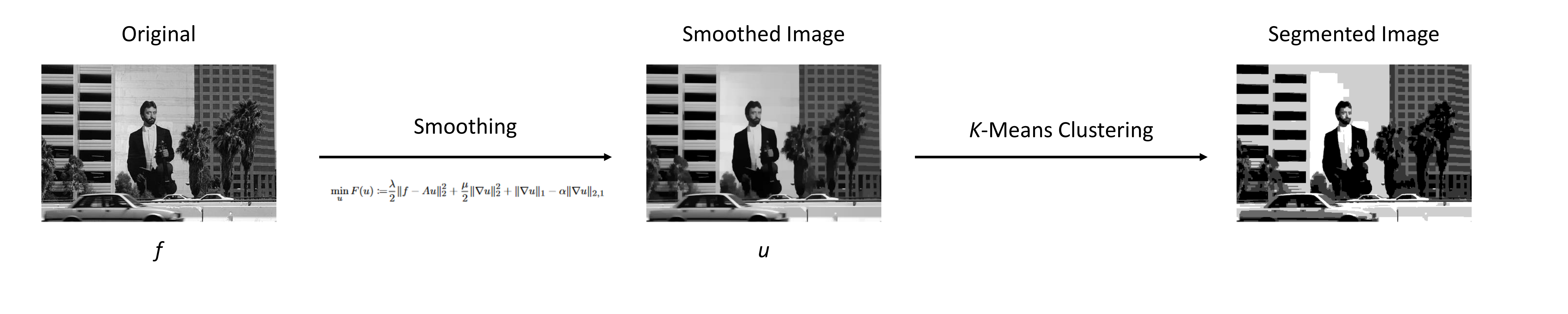}
         \caption{AITV SaT framework for grayscale image.}
     \end{subfigure}\\
     \begin{subfigure}[b]{\textwidth}
         \centering
         \includegraphics[scale=0.175]{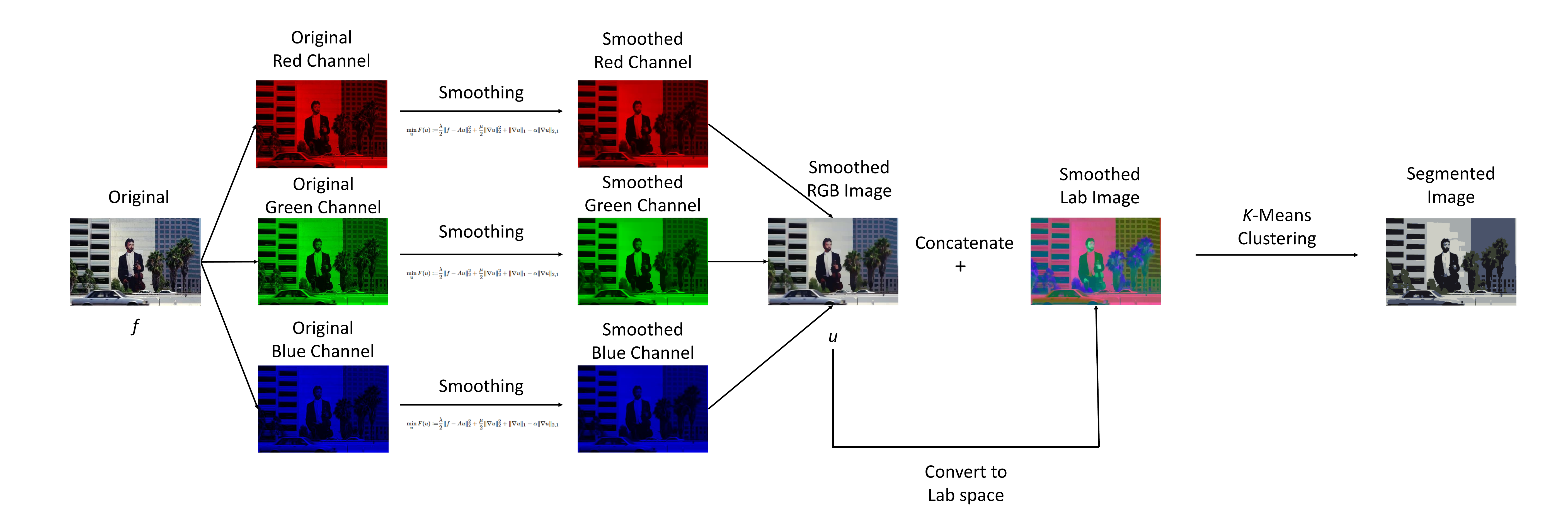}
         \caption{AITV SLaT framework for color image.}
     \end{subfigure}
        \caption{AITV SaT/SLaT framework visualized.}
        \label{fig:aitv_sat_slat_framework}
\end{sidewaysfigure}

\begin{algorithm}[t]
  \KwInput{\begin{itemize}
      \item image $f= (f_1, \ldots, f_d)$
      \item blurring operator $A$
      \item fidelity parameter $\lambda >0$
      \item smoothing parameter $\mu > 0$
      \item AITV parameter $\alpha \in [0,1]$
      \item the number of regions in the image $K$
  \end{itemize}}
  \KwOutput{Segmentation $\tilde{f}$}
   Stage one: Compute $u_{\ell}$ by solving \eqref{eq:AITV_MS} for $\ell = 1, \ldots, d$.
   
   Stage two:
   
    \If{$f$ is a color image, i.e, $d=3$}{Transfer $u=(u_1, u_2,u_3)$ into Lab space to obtain $(\bar{u}_1, \bar{u}_2, \bar{u}_3)$ and concatenate to form $(u_1, u_2, u_3, \bar{u}_1, \bar{u}_2, \bar{u}_3)$.}
   \Else{Go to stage three. }
  
   Stage three: Apply $K$-means to obtain $\{(c_{l}, \Omega_{k})\}_{k=1}^{K}$ and compute $\tilde{f}$ by \eqref{eq:pc_constant_approx}.
\caption{AITV SaT/SLaT}
\label{alg:sat_slat}
\end{algorithm}

We replace the isotropic TV in
\eqref{eq:discretized_convex_MS} by a weighted difference of anisotropic and isotropic TV, i.e.,
\begin{align} \label{eq:AITV_MS}
\begin{split}
    \min_{u} F(u) \coloneqq &\frac{\lambda}{2} \|f-Au\|_2^2 + \frac{\mu}{2} \|\nabla u\|_2^2 + \|\nabla u\|_1 - \alpha \|\nabla u\|_{2,1}, 
    \end{split}
\end{align}
with $\lambda >0, \mu > 0, \alpha \in [0,1].$ AITV is a more suitable alternative to TV (no matter whether it is anisotropic or isotropic) since TV typically fails to recover oblique edges \cite{birkholz2011unifying, condat2017discrete}, which can be preserved by AITV  \cite{bui2020weighted, lou-2015}. To simplify notations, we omit the subscript $\ell$ in \eqref{eq:discretized_convex_MS} because the smoothing model is applied channel by channel independently. We show that  our model \eqref{eq:AITV_MS} admits a global solution in Section~\ref{sect:analysis}. 
 To find a solution to \eqref{eq:AITV_MS},
 we describe in Section~\ref{sec:algorithm} the  ADMM scheme with its convergence analysis conducted in Section~\ref{sect:convergence}.  The overall AITV SaT/SLaT framework for segmentation is visualized in Figure \ref{fig:aitv_sat_slat_framework} and summarized  in Algorithm \ref{alg:sat_slat}.

\subsection{Model Analysis}\label{sect:analysis}

In Theorem~\ref{thm:global} we establish the existence of a global solution to \eqref{eq:AITV_MS} by showing that its objective function $F$ is coercive in Lemma~\ref{lemma:coercive}.

\begin{lem}\label{lemma:coercive}
If $\lambda>0, \mu > 0, \alpha \in [0,1],$ and $\text{ker}(A) \cap \text{ker}(\nabla) = \{0\}$, then $F$ defined in \eqref{eq:AITV_MS} is coercive.
\end{lem}
\begin{proof}
We prove by  contradiction. Suppose  there exists a sequence $\{u_n\}_{n=1}^{\infty}$ and a constant $C>0$ such that $\|u_n\|_2 \rightarrow \infty$ and $F(u_n) < C$ for all $n \in \mathbb{N}$. We define a sequence $\{v_n\}_{n=1}^{\infty}$ where $v_n = \displaystyle \frac{u_n}{\|u_n\|_2}$ and thereby satisfies $\|v_n\|_2 = 1$ for all $n \in \mathbb{N}$. Since $\{v_n\}_{n=1}^{\infty}$ is bounded, there exists a convergent subsequence $\{v_{n_k}\}_{k=1}^{\infty}$ such that $v_{n_k} \rightarrow v^*$ and $\|v^*\|_2=1$. 

It follows from $\|\nabla u\|_{2,1} \leq \|\nabla u\|_1$ that
\begin{align*}
    F(u) &\geq \frac{ \lambda}{2} \|Au-f\|_2^2 + \frac{\mu}{2} \|\nabla u\|_2^2  \geq \frac{\lambda}{2} (\|Au\|_2 - \|f\|_2)^2 + \frac{\mu}{2} \|\nabla u\|_2^2 .  
\end{align*}
Since $F(u_n) < C$, we have $\|\nabla u_n \|_2 <\sqrt{\frac{2C}{\mu}}$ and $\|Au_n\|_2 < \sqrt{\frac{2C}{\lambda}}+ \|f\|_2$.
As a result, we have
\begin{align*}
    \|Av_{n_k}\|_2  &= \frac{\|Au_{n_k}\|_2}{\|u_{n_k}\|_2} < \frac{\sqrt{\frac{2C}{\lambda}}+ \|f\|_2}{\|u_{n_k}\|_2}\\
    \|\nabla v_{n_k}\|_2 &= \frac{ \|\nabla u_{n_k}\|_2}{\|u_{n_k}\|_2} < \frac{\sqrt{2C}}{\sqrt{\mu}\|u_{n_k}\|_2}.
\end{align*}
After taking the limit $n_k\rightarrow \infty,$ we get
     $\|Av^{*}\|_2  = 0  \text{ and } \|\nabla v^*\|_2  = 0$, which implies that  $v^* = 0$ due to the assumption that $ \text{ker}(A) \cap \text{ker}(\nabla) = \{0\}$. However, it contradicts with $\|v^*\|_2 = 1$,  and hence $F$ is coercive.
\end{proof}
\begin{theorem}\label{thm:global}
If $\lambda>0, \mu > 0, \alpha \in [0,1],$ and  $\text{ker}(A) \cap \text{ker}(\nabla) = \{0\}$, then $F$ has a global minimizer. 
\end{theorem}
\begin{proof}
As $F$ is lower bounded by 0, it has a minimizing sequence $\{u_n\}_{n=1}^{\infty}$. Without loss of generality, we assume $u_1 = 0$. Since $F$ is coercive by Lemma \ref{lemma:coercive}, we have $F(u_n) \leq F(0) < \infty$, showing that $\{\|\nabla u_n\|_1\}_{n=1}^{\infty}$ and $\{\|Au_n\|_2\}_{n=1}^{\infty}$ are bounded. As $\text{ker}(A) \cap \text{ker}(\nabla) = \{0\}$, we have $\{u_n\}_{n=1}^{\infty}$ shall be bounded.  Then there exists a convergent subsequence $\{u_{n_k}\}_{k=1}^{\infty}$  such that $u_{n_k} \rightarrow u^*$. Since $A$ and $\nabla$ are both bounded, linear operators, we have $Au_{n_k} \rightarrow Au^*$ and $\nabla u_{n_k} \rightarrow \nabla u^*$. Since norms are continuous and thereby lower semi-continuous, we have
{
\begin{align*}
  &  \|\nabla u^* \|_1 - \alpha \|\nabla u^*\|_{2,1} \leq \liminf_{k \rightarrow \infty}  \left(\|\nabla u_{n_k} \|_1 - \alpha \|\nabla u_{n_k}\|_{2,1}\right), \\
   & \|\nabla u^*\|_2^2 \leq \liminf_{k\rightarrow \infty} \|\nabla u_{n_k}\|_2^2,\\
&    \|Au^* -f \|_2^2 \leq \liminf_{k \rightarrow \infty} \|Au_{n_k} - f\|_2^2. 
\end{align*}}%
Altogether, we obtain $F(u^*) \leq \displaystyle \liminf_{k \rightarrow \infty} F(u_{n_k})$, which implies that $u^*$ minimizes $F(u)$. 
\end{proof}

\subsection{Numerical Scheme} \label{sec:algorithm}

We describe an efficient algorithm to minimize \eqref{eq:AITV_MS} via ADMM. In particular, we introduce an auxiliary variable $w=(w_x, w_y) \in X \times X$ and rewrite  \eqref{eq:AITV_MS} into an equivalent  constrained optimization problem
\begin{equation}
\begin{aligned}\label{eq:constrained_opt}
    \min_{u,w}  & \quad \frac{\lambda}{2} \|f-Au\|_2^2 + \frac{\mu}{2} \|\nabla u \|_2^2 + \|w\|_1 - \alpha \|w\|_{2,1} \\
    \text{s.t.} & \quad \nabla u = w,
\end{aligned}
\end{equation}
where $w_x = \nabla_x u$ and $w_y = \nabla_y u$.
Then the corresponding  augmented Lagrangian is expressed by
\begin{align}\label{eq:lagrange10}
\begin{split}
     \mathcal{L}_{\delta}(u,w, z) \coloneqq&\frac{\lambda}{2}  \|f-Au\|_2^2 + \frac{\mu}{2} \|\nabla u \|_2^2 + \|w\|_1 - \alpha \|w\|_{2,1}\\ &+ \langle z , \nabla u - w \rangle + \frac{\delta}{2} \|\nabla u - w\|_2^2\\
    =& \frac{\lambda}{2}  \|f - Au\|_2^2 + \frac{\mu}{2} \|\nabla u \|_2^2 + \|w\|_1 - \alpha \|w\|_{2,1}\\ &+ \frac{\delta}{2} \left \|\nabla u - w + \frac{z}{\delta} \right \|_2^2 - \frac{1}{2 \delta} \|z\|_2^2, \end{split}
\end{align}
where $\delta > 0$ is a penalty parameter and  $z= ( z_x,  z_y ) \in X\times X$ is a dual variable. The ADMM iterations proceed as follows:
\begin{subequations}
\begin{align}
\label{eq:u_update}
   u_{t+1} &\in \displaystyle \argmin_{u} \mathcal{L}_{\delta_t}(u, w_t,z_t)\\
w_{t+1} &\in  \displaystyle \argmin_{w} \mathcal{L}_{\delta_t}(u_{t+1}, w,z_t)\label{eq:w_update}\\
       \label{eq:z_update}
    z_{t+1} &= z_t + \delta_t (\nabla u_{t+1} - w_{t+1}) \\ \label{eq:delta_update}
    \delta_{t+1} &= \sigma \delta_t,\; \sigma \geq 1.
\end{align}
\end{subequations}
Note that $\sigma = 1$ reduces to the original ADMM framework \cite{boyd2011distributed}. 
We consider an adaptive penalty parameter $\delta_t$ by choosing
$\sigma >1$. In fact, the parameter $\sigma >1$ controls the numerical convergence speed of the algorithm in the sense that a larger $\sigma$ leads to 
a fewer number of iterations the algorithm needs to run before satisfying a stopping criterion. However, if $\delta_t$ increases too quickly, the ADMM algorithm will numerically converge within a few iterations, which may yield a low-quality  solution. Thus, a small $\sigma$ is recommended and we discuss its choice  in experiments (Section \ref{sec:experiments}).

Next we  elaborate on how to solve the two subproblems \eqref{eq:u_update} and \eqref{eq:w_update}. The subproblem \eqref{eq:u_update} is   written as
\begin{align*}
\begin{split}
      u_{t+1} &\in \displaystyle \argmin_{u} \frac{\lambda}{2} \|f- Au\|_2^2+ \frac{\mu}{2} \|\nabla u\|_2^2 + \langle z_t, \nabla u - w_t \rangle + \frac{\delta_t}{2} \|\nabla u - w_t\|_2^2.
      \end{split}
\end{align*}
The first-order optimality condition of \eqref{eq:u_update} is given by
{
\begin{align*}
\begin{split}
        \left[\lambda A^{\top}A - (\mu + \delta_t) \Delta \right] u_{t+1} = \lambda A^{\top} f+ \delta_t \nabla^{\top} \left( w_t - \frac{z_t}{\delta_t} \right),
        \end{split}
\end{align*}}%
where $\Delta = - \nabla^{\top} \nabla$ is the Laplacian operator. 
If $\text{ker}(A) \cap \text{ker}(\nabla) = \{0\}$, then $\lambda A^{\top}A - (\mu+\delta_t) \Delta$ is positive definite. By assuming the periodic boundary condition, $A^{\top}A$ and $\Delta$ are block circulant, so we can solve for $u_{t+1}$ via the fast Fourier transform $\mathcal{F}$~\cite{chan1996conjugate, ng1999fast, wang2008new}.  By the Convolution Theorem, the closed-form solution for $u_{t+1}$ is
{\
\begin{align*}
\begin{split}u_{t+1} = \mathcal{F}^{-1} \left(\frac{\lambda \mathcal{F}(A)^* \circ \mathcal{F}(f) + \delta_t \mathcal{F}(\nabla)^* \circ\mathcal{F} \left(w_t - \displaystyle \frac{z_t}{\delta_t}\right)}{\lambda\mathcal{F}( A)^* \circ \mathcal{F}(A) - (\mu +\delta_t) \mathcal{F}(\Delta)} \right),
\end{split}
\end{align*}}%
where $\mathcal{F}^{-1}$ is the inverse Fourier transform, $*$ denotes complex conjugate,    $\circ$ denotes componentwise multiplication, and division  is also componentwise.

Denote $w_{i,j} = \begin{bmatrix} (w_x)_{i,j} \\
(w_y)_{i,j}
\end{bmatrix} \in \mathbb{R}^2$ as the $(i,j)$th entry of $w$. The subproblem \eqref{eq:w_update} can be expressed as
{\small
\begin{align*}
    w_{t+1} \in \displaystyle \argmin_{w} \|w\|_1 - \alpha \|w\|_{2,1}+ \frac{\delta_t}{2} \left \|\nabla u_{t+1}  + \frac{z_t}{\delta_t}- w \right \|_2^2.
\end{align*}}
Expanding \eqref{eq:w_update}, we get
{\small
\begin{align}\label{eq:w-update-ij}
\begin{split}
    \argmin_{w} &\sum_{(i,j) \in \Omega} \Bigg( \| w_{i,j} \|_1 - \alpha \|w_{i,j}\|_2 + \frac{\delta_t}{2} \left\| (\nabla u_{t+1})_{i,j} + \frac{(z_t)_{i,j}}{\delta_t} - w_{i,j} \right\|_2^2 \Bigg),
    \end{split}
\end{align}}%
which shows that $w_{i,j}$ can be solved elementwise. Specifically, the optimal solution of $w_{i,j}\in\mathbb R^2$ is related to the proximal operator for $\ell_1 - \alpha \ell_2$ defined by
{\small
\begin{align}\label{eq:prox_l1_l2}
    \text{prox}(y; \alpha, \beta) \coloneqq \text{prox}_{\beta\left(\|\cdot\|_1 - \alpha \|\cdot\|_2\right)}(y) = \argmin_x \|x\|_1 -\alpha \|x\|_2 + \frac{1}{2 \beta} \|x-y\|_2^2.
\end{align}}%
The closed-form solution for \eqref{eq:prox_l1_l2} is given in Lemma~\ref{lemma:prox} \cite{louY18}. By comparing \eqref{eq:w-update-ij} and \eqref{eq:prox_l1_l2}, the $w$-update is given by $\forall (i,j) \in \Omega,$
\begin{align*}
\begin{split}
    (w_{t+1})_{i,j} =\text{prox}\left( (\nabla u_{t+1})_{i,j} + \frac{(z_t)_{i,j}}{\delta_t}; \alpha, \frac{1}{\delta_t} \right).
    \end{split}
\end{align*}

\begin{lem}[\cite{louY18}]\label{lemma:prox}
Given $y \in \mathbb{R}^n$, $\beta >0$, and $\alpha \geq 0$, the optimal solution to \eqref{eq:prox_l1_l2} can be discussed separately into the following cases:
\begin{enumerate}
    \item When $\|y\|_{\infty} > \beta$, we have
    \begin{align*}
        x^* = (\|\xi\|_2 + \alpha \beta) \frac{\xi}{\|\xi\|_2},
    \end{align*}
     where $\xi= \sign(y)\circ\max(|y|-\beta,0)$. 
    \item When $(1-\alpha) \beta < \|y\|_{\infty} \leq \beta$, then $x^*$ is a 1-sparse vector such that one chooses $i \in \displaystyle \argmax_j(|y_j|)$ and defines $x^*_i=\left(|y_i| + (\alpha-1)\beta\right)\sign(y_i)$ and the rest of the  elements equal to 0.
\item When $\|y\|_{\infty} \leq (1- \alpha)\beta$, then $x^* = 0$. 
\end{enumerate}
\end{lem}

In summary, the ADMM scheme that minimizes \eqref{eq:AITV_MS} is presented in Algorithm \ref{alg:admm}.

\begin{algorithm*}[t!!!]
  \KwInput{\begin{itemize}
      \item image $f$
      \item blurring operator $A$
      \item fidelity parameter $\lambda >0$
      \item smoothing parameter $\mu > 0$
      \item AITV parameter $\alpha \in [0,1]$
      \item penalty parameter $\delta_0> 0$
      \item penalty multiplier $\sigma \geq 1$
      \item relative error $\epsilon >0 $ 
  \end{itemize}}
  \KwOutput{$u_{t}$}
    Initialize $u_0, w_0, z_0$.\\
    Set $t=0$.\\
   \While{$\frac{\|u_{t}-u_{t-1}\|_2}{\|u_{t}\|_2} > \epsilon$}
   {
   \begin{align*}
         u_{t+1} &= \mathcal{F}^{-1} \left(\frac{\lambda \mathcal{F}(A)^* \circ \mathcal{F}(f) + \delta_t \mathcal{F}(\nabla)^* \circ\mathcal{F} \left(w_t - \displaystyle \frac{z_t}{\delta_t}\right)}{\lambda\mathcal{F}( A)^* \circ \mathcal{F}(A) - (\mu +\delta_t) \mathcal{F}(\Delta)} \right)\\
               (w_{t+1})_{i,j} &= \text{prox}\left( (\nabla u_{t+1})_{i,j} + \frac{(z_t)_{i,j}}{\delta_t}; \alpha, \frac{1}{\delta_t} \right) \quad \forall (i,j) \in \Omega\\
                   z_{t+1} &= z_t + \delta_t (\nabla u_{t+1} - w_{t+1}) \\
    \delta_{t+1} &= \sigma \delta_t\\
    t &\coloneqq t+1
   \end{align*}
   }

\caption{ADMM for minimizing the AITV-Regularized smoothing model}
\label{alg:admm}

\end{algorithm*}

\subsection{Convergence Analysis}
\label{sect:convergence}
We aim to analyze the convergence for Algorithm \ref{alg:admm}. It is true that global convergence of ADMM has been established in \cite{deng2016global} for certain classes of nonconvex optimization problems, but unfortunately it cannot be applied to our problem \eqref{eq:constrained_opt} since the gradient operator $\nabla$ is not surjective. 
Instead of global convergence, we manage to achieve weaker subsequential convergence  for two cases: $\sigma = 1$ and $\sigma >1$. The proof of $\sigma >1$ is
 adapted from \cite{gu2017weighted,you2019nonconvex}. 

Before providing convergence results for ADMM, we provide a definition of subdifferential for general functions. 
For a function $h: \mathbb{R}^n \rightarrow \mathbb{R} \cup \{\infty\}$, we denote the (limiting) subdifferential by $\partial{h(x)}$ \cite[Definition 11.10]{rockafellar2009variational}, which is defined as a set
{\small
\begin{align*}
\begin{split}
    \partial h(x) = \{v \in \mathbb{R}^n: \exists \{(x_t, v_t)\}_{t=1}^{\infty}\text{ s.t. } x_t \rightarrow x, \; h(x_t) \rightarrow h(x),\; v_t \in \hat{\partial}{h}(x_t), \text{and } v_t \rightarrow v\}, 
    \end{split}
\end{align*}}%
with
{\
\begin{align*}
    \hat{\partial}h(x) =\left\{v \in \mathbb{R}^n: \liminf_{z \rightarrow x, z \neq x} \frac{h(z) - h(x)- \langle v, z-x \rangle}{\|z-x\|_2} \geq 0\right\}.
\end{align*}}%
Since $\hat{\partial}h(x) \subset \partial{h}(x)$ where $h$ is finite on $x$, the graph $x  \mapsto \partial h(x)$ is closed \cite{clarke2013functional, rockafellar2009variational} by definition:
{\small
\begin{align*}
    v_t \in \partial h(x_t), \; x_t \rightarrow x, \; h(x_t) \rightarrow h(x), \; &v_t \rightarrow v\implies v \in \partial h(x).
\end{align*}
}%

First, we present a lemma and a proposition whose proofs are delayed to the appendix.
\begin{lem} \label{lemma:strong_convexity_ineq} Suppose that $\text{ker}(A) \cap \text{ker}(\nabla) = \{0\}$. Let $\{(u_t, w_t, z_t)\}_{t=1}^{\infty}$ be generated by \eqref{eq:u_update}-\eqref{eq:delta_update} with $\sigma \geq 1$. The following inequality holds:
\begin{align}\label{eq:strong_convexity_ineq}
\begin{split}
        &\mathcal{L}_{\delta_{t+1}} (u_{t+1}, w_{t+1} ,z_{t+1}) - \mathcal{L}_{\delta_{t}} (u_t, w_t ,z_t) \leq \frac{ \sigma+1}{2\sigma^t\delta_0} \|z_{t+1}- z_t\|_2^2- \frac{\zeta}{2} \|u_{t+1} - u_t\|_2^2 ,
        \end{split}
\end{align}
where $\zeta>0$ is  the smallest eigenvalue  of $\lambda A^{\top}A +(\mu+\delta_0)\nabla^{\top} \nabla.$
\end{lem}

\begin{proposition} \label{prop:partial_conv1}
Suppose that $\text{ker}(A) \cap \text{ker}(\nabla) = \{0\}$. Let $\{(u_t, w_t, z_t)\}_{t=1}^{\infty}$ be generated by \eqref{eq:u_update}-\eqref{eq:delta_update}.
Assume one of the conditions holds:
\begin{itemize}
    \item $\sigma = 1$ and $\displaystyle \sum_{i=0}^{\infty} \|z_{i+1} - z_i\|_2^2 < \infty$.
    \item $\sigma >1$.
\end{itemize}
Then we have the following statements:
\begin{enumerate}[label=(\alph*)]
    \item The sequence $\{(u_t, w_t, z_t)\}_{t=1}^{\infty}$ is bounded.
    \item $u_{t+1}-u_t \rightarrow 0$ as $t\rightarrow \infty$.
\end{enumerate}
\end{proposition}

Proposition \ref{prop:partial_conv1} reveals an advantage of using the adaptive penality parameter with $\sigma >1$. For $\sigma = 1$, we require $\displaystyle \sum_{i=0}^{\infty} \|z_{i+1} - z_i\|_2^2 < \infty$ in order for the iterates $\{(u_t, w_t, z_t)\}_{t=1}^{\infty}$ of Algorithm \ref{alg:admm} to be bounded and to satisfy the relative stopping criterion $\frac{\|u_t-u_{t-1}\|_2}{\|u_t\|_2} < \epsilon$. The requirement  $\displaystyle \sum_{i=0}^{\infty} \|z_{i+1} - z_i\|_2^2 < \infty$ is no longer necessary if $\sigma >1$. 

Finally, we establish the subsequential convergence in Theorem~\ref{thm:admm_convergence} under stronger conditions compared to the ones in Proposition \ref{prop:partial_conv1}. These conditions are 
motivated by a series of works \cite{chang2016phase, chang2018total, jung2014variational, jung2017piecewise, li2016multiphase, li2020tv} that proved the theoretical convergence of ADMM in solving TV-based inverse problems.

\begin{theorem} \label{thm:admm_convergence}
Let $\{(u_t, w_t, z_t)\}_{t=1}^{\infty}$ be generated by \eqref{eq:u_update}-\eqref{eq:delta_update}. Assume one set of the following conditions holds:
\begin{itemize}
    \item $\sigma = 1$ and $\displaystyle \sum_{i=0}^{\infty} \|z_{i+1} - z_i\|_2^2 < \infty$.
    \item $\sigma >1$,  $\delta_t( w_{t+1} - w_t) \rightarrow 0$, and $z_{t+1} - z_t \rightarrow 0$.
\end{itemize}
Then there exists a subsequence of $\{(u_t, w_t, z_t)\}_{t=1}^{\infty}$ whose limit point $(u^*, w^*, z^*)$ is a KKT point of \eqref{eq:constrained_opt} that satisfies 
\begin{subequations}
\begin{align}
 0 &= \lambda A^{\top}(Au^*-f) - \mu \Delta u^* + \nabla^{\top}z^* \label{eq:first_KKT_opt}\\
        z^* &\in \partial \left( \|w^{*}\|_1 - \alpha \|w^{*}\|_{2,1} \right)\\
        \nabla u^* &=w^*. \label{eq:kkt_equality_opt}
        \end{align}
\end{subequations}
\end{theorem}
\begin{proof}
By Proposition \ref{prop:partial_conv1}, $\{(u_t, w_t, z_t)\}_{t=1}^{\infty}$ is bounded, and hence there exists a  subsequence that converges to a point $(u^*, w^*, z^*)$, denoted by $
    (u_{t_k}, w_{t_k}, z_{t_k}) \rightarrow (u^*, w^*, z^*).$ Proposition \ref{prop:partial_conv1} also
establishes  $\displaystyle \lim_{t \rightarrow \infty} u_{t+1} - u_t = 0,$ which implies that $\displaystyle \lim_{k \rightarrow \infty} u_{t_k+1} = \lim_{k \rightarrow \infty} u_{t_k} = u^*.$ Either set of assumptions establishes $ \displaystyle \lim_{k \rightarrow \infty} z_{t_k+1} = \lim_{k \rightarrow \infty} z_{t_k} = z^*.$ The optimality conditions at iteration $t_k$ are
\begin{subequations}
\begin{align}
\label{eq:t_j_first_opt}
\begin{split}
        0 &= \lambda A^{\top}(Au_{t_k+1}-f) - \mu \Delta u_{t_k+1}+ \delta_{t_k} \nabla^{\top} ( \nabla u_{t_k+1} - w_{t_k}) + \nabla^{\top} z_{t_k}
        \end{split}\\
        \label{eq:w_t_j_opt}
        \begin{split}
    0 &\in \partial \left( \|w_{t_k+1}\|_1 - \alpha \|w_{t_k+1}\|_{2,1} \right)- \delta_{t_k} \left( \nabla u_{t_k+1} -w_{t_k+1}\right) - z_{t_k}\end{split} \\
    \label{eq:z_t_j_opt}
    z_{t_k+1} &= z_{t_k} + \delta_{t_k} (\nabla u_{t_k+1} - w_{t_k+1}).
\end{align}
\end{subequations}
Next we discuss two sets of assumptions individually. 

If $\sigma = 1$, then $\delta_{t_k} = \delta_0$ for each iteration $t_k$. Together with $\displaystyle \lim_{t \rightarrow \infty} z_{t+1} - z_t = 0$,  we have $\displaystyle \lim_{t \rightarrow \infty}\nabla u_{t} - w_{t} = 0$ by \eqref{eq:z_update} and 
\begin{align*}
    \nabla u^* &= \lim_{k\rightarrow \infty} \nabla u_{t_k} =  \lim_{k\rightarrow \infty} (\nabla u_{t_{k}} - w_{t_{k}}) + \lim_{k \rightarrow \infty} w_{t_k} =w^*,
\end{align*}
leading to \eqref{eq:kkt_equality_opt}.
According to \eqref{eq:t_j_first_opt}, the point $u_{t_{k}+1}$ satisfies
\begin{align*}
    0 =&\lambda A^{\top}(Au_{t_k+1}-f) - \mu \Delta u_{t_k+1} + \delta_0\nabla^{\top} ( \nabla u_{t_k+1} - w_{t_k}) + \nabla^{\top} z_{t_k} \\
    =& \lambda A^{\top}(Au_{t_k+1}-f) - \mu \Delta u_{t_k+1} + \delta_0\nabla^{\top} ( \nabla u_{t_k+1} -  \nabla u_{t_k}) + \delta_0 \nabla^{\top}(\nabla u_{t_k} - w_{t_k})\\&+ \nabla^{\top} z_{t_k}.
\end{align*}
Then \eqref{eq:first_KKT_opt} holds after taking $k \rightarrow \infty$. 
Finally, we have
\begin{align*}
    \lim_{k\rightarrow \infty} w_{t_k+1} &= \lim_{k \rightarrow \infty} (w_{t_k+1} - \nabla u_{t_k+1}) + \lim_{k \rightarrow \infty} \nabla u_{t_k+1} = 
    \lim_{k \rightarrow \infty} \nabla u_{t_k} = w^*.
\end{align*}

If $\sigma >1$ and $\delta_t( w_{t+1} - w_t) \rightarrow 0$, we
 substitute \eqref{eq:z_t_j_opt} into \eqref{eq:t_j_first_opt} and simplify it to obtain
\begin{align*}
    0 &= \lim_{k \rightarrow \infty} \lambda A^{\top}(Au_{t_k+1}-f) - \mu \Delta u_{t_k} + \delta_{t_k} \nabla^{\top} ( w_{t_k+1} - w_{t_k}) + \nabla^{\top} z_{t_k+1}\\
    &=\lambda A^{\top}(Au^{*}-f) - \mu \Delta u^{*} + \nabla^{\top} z^{*}.
\end{align*}
We need to prove $\displaystyle \lim_{k \rightarrow \infty} w_{t_k+1} = w^*$.  Since $\{z_t\}_{t=1}^{\infty}$ is bounded in this case, there exists $C > 0$ such that $ \|z_t\|_2 \leq C$. By {\eqref{eq:z_update}}, we have
\begin{align*}
\|w_{t+1} - w_t\|_2 &\leq\|w_{t+1} - \nabla u_{t+1}\|_2 + \|\nabla u_{t+1} - \nabla u_t \|_2 + \|\nabla u_t - w_t \|_2\\ &=\left \|\frac{z_{t+1}-z_t}{\delta_t} \right\|_2+\|\nabla u_{t+1} - \nabla u_t \|_2 + \left \|\frac{z_{t}-z_{t-1}}{\delta_{t-1}} \right\|_2 \\&\leq \frac{  4C}{\delta_{t-1}}+\|\nabla u_{t+1} - \nabla u_t \|_2 .
\end{align*}
Taking the limit $t \rightarrow \infty$, we obtain $\|w_{t+1} - w_t\|_2 \rightarrow 0$ and $w_{t+1}-w_t \rightarrow 0$. It follows that
\begin{align*}
    \lim_{k \rightarrow \infty} w_{t_k+1} - w_{t_k} = 0 \implies \lim_{k \rightarrow \infty} w_{t_k+1} = \lim_{k \rightarrow \infty} w_{t_k} = w^*.
\end{align*}
 Then \eqref{eq:z_t_j_opt} implies
\begin{align*}
    \|\nabla u^* - w^*\|_2 &= \lim_{k \rightarrow \infty} \|\nabla u_{t_k+1} - w_{t_k+1} \|_2=\lim_{k \rightarrow \infty} \frac{1}{\delta_{t_k}}\left\|z_{t_k+1}-z_{t_k}\right\|_2\leq \lim_{k \rightarrow \infty} \frac{2C}{\delta_{t_k}} = 0.
\end{align*}
As a result, we have $\nabla u^* = w^*$.

By substituting \eqref{eq:z_t_j_opt} into \eqref{eq:w_t_j_opt}, we have
\begin{align*}
    z_{t_k+1} \in \partial \left( \|w_{t_k+1}\|_1 - \alpha \|w_{t_k+1}\|_{2,1} \right)\quad \forall k \in \mathbb{N}.
\end{align*}
By continuity, we have $\|w_{t_k+1}\|_1 - \alpha \|w_{t_k+1}\|_{2,1} \rightarrow \|w^*\|_1 - \alpha \|w^*\|_{2,1}$. Together with the fact that $(w_{t_k+1}, z_{t_k+1}) \rightarrow (w^*,z^*), $ we obtain
$z^{*} \in \partial \left( \|w^{*}\|_1 - \alpha \|w^{*}\|_{2,1} \right)$. 

Therefore,  if either set of assumptions hold, then $(u^*, w^*, z^*)$ is a KKT point of \eqref{eq:constrained_opt}. 
\end{proof}
\section{Experimental Results} \label{sec:experiments}
We examine the SaT/SLaT framework by comparing the isotropic TV\footnote{MATLAB code is available at \url{https://xiaohaocai.netlify.app/download/}.} \cite{cai2013two, cai2017three}, the $\text{TV}^p (0 < p < 1)$ \cite{wu2021two}, and  the AITV. The experiment comparison also includes the AITV-regularized CV and fuzzy region (FR) models \cite{bui2020weighted}, the iterative convolution-thresholding method (ICTM) {\cite{wang2022iterative}}, and the TV$^p$-regularized Mumford-Shah  (TV$^p$ MS) model without the bias term {\cite{li2020tv}} together with the Potts model \cite{potts-1952} solved  by either a primal-dual algorithm\footnote{Python code is available at \url{https://github.com/VLOGroup/pgmo-lecture/blob/master/notebooks/tv-potts.ipynb} and a translated MATLAB code is available at \url{https://github.com/kbui1993/MATLAB_Potts}.} \cite{pock-2009}  or ADMM\footnote{Code is available at \url{https://github.com/mstorath/Pottslab}.} \cite{storath2014fast}. 
In particular, the primal-dual algorithm solves a convex relaxation of the Potts model \cite{pock-2009}:
{\small
\begin{align}\label{eq:convex_potts}
\begin{split}
 U^*= \argmin_{U \in S} \sum_{k=1}^{K}  \Bigg[\lambda\sum_{(i,j) \in \Omega} (u_{k})_{i,j} |(u_{k})_{i,j}-c_{{k}}|^2+  \| \nabla u_{{k}} \|_{2,1} \Bigg],
  \end{split}
\end{align}}%
where $K$ is the number of regions specified in an image, $\{c_{k}\}_{{k}=1}^{K} \subset \mathbb{R}$ are constant values, and 
{\small
\begin{align*}
    S = \Bigg \{&U=(u_1, u_2, \ldots, u_{K}) \in X^{K}:\forall \; (i,j) \in \Omega, \sum_{k=1}^{K} (u_{k})_{i,j} = 1;\;\\ &(u_{k})_{i,j} \in [0,1],k = 1, \ldots,K  \Bigg  \}.
\end{align*}}%
Once getting $U^*$ from \eqref{eq:convex_potts}, the regions of an image can be approximated by
\begin{align*}
    \Omega_{\kappa} = \left \{(i,j) \in \Omega : \kappa = \argmax_{1 \leq k \leq K} (u^*_{k})_{i,j} \right\},
\end{align*}
with 
$\kappa = 1, \ldots, K$.
For short, we refer \eqref{eq:convex_potts} as the convex Potts model. To apply ADMM,
Storath and Weinmann \cite{storath2014fast} considered the following version of the Potts model: 
\begin{align}\label{eq:admm_potts}
    \min_{u} \lambda \|u-f\|_2^2 + \|\nabla u\|_0.
\end{align}
Since it does not admit a segmentation result with a chosen number of regions, we develop its SaT version called SaT-Potts that solves \eqref{eq:admm_potts}, followed by the $K$-means clustering for segmentation.
Both \eqref{eq:convex_potts} and \eqref{eq:admm_potts} can deal with multichannel input; please refer to \cite{pock-2009,storath2014fast} for more details.

\begin{figure*}[t!]
	\centering
	\begin{tabular}{c@{}c@{}}
		\subcaptionbox{\label{fig:synthetic_grayscale}}{\includegraphics[width = 0.75in]{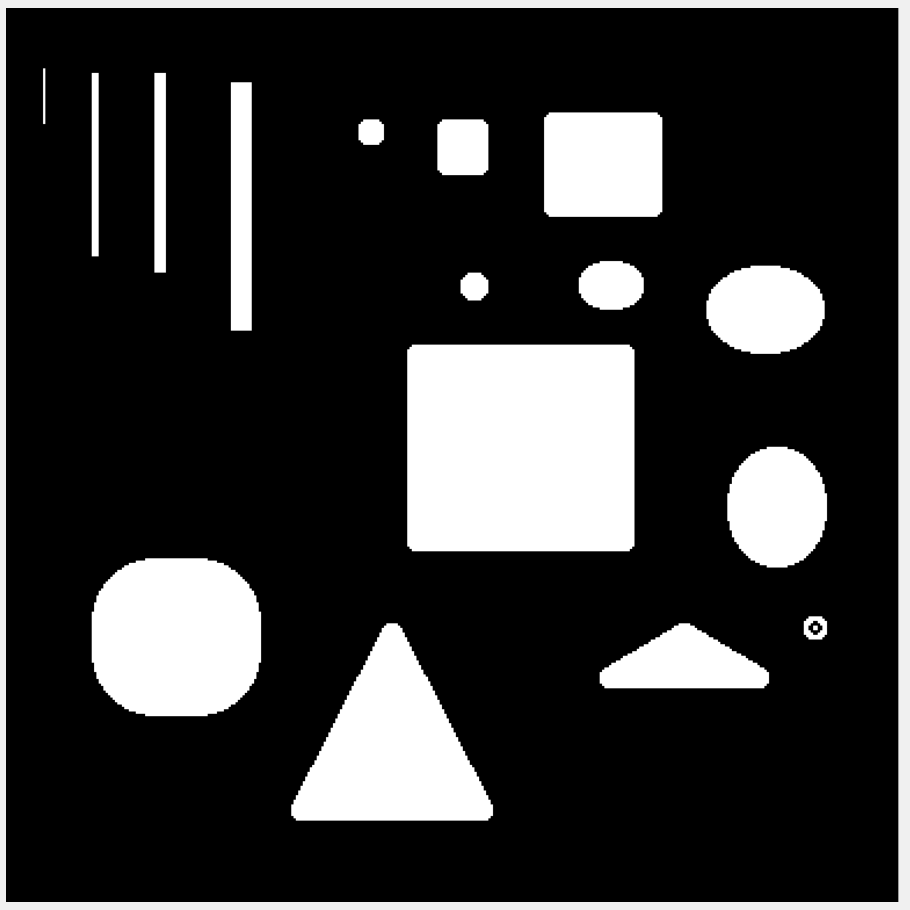}} &
		 \subcaptionbox{\label{fig:synthetic_2phase}}{\includegraphics[width = 0.75in]{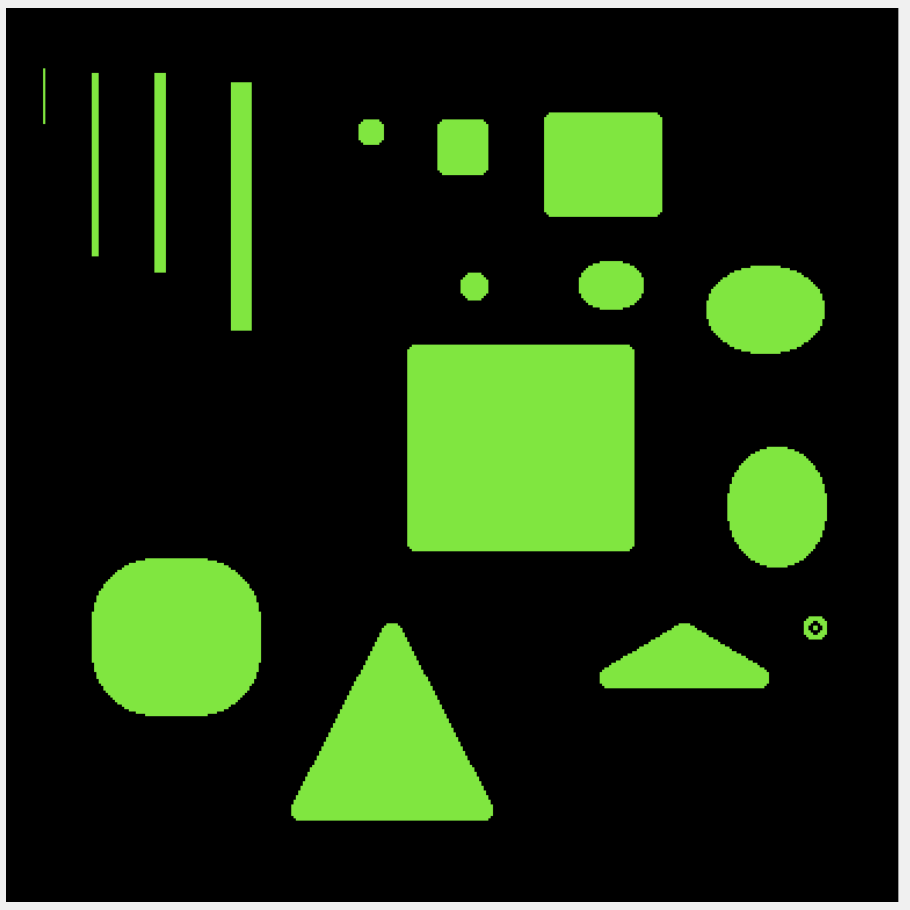}} 
	\end{tabular}
	\caption{Synthetic images for two-phase segmentation. (a) Grayscale image. (b) Color image whose regions have pixel value $(128, 230, 64)$. The size of both images is $385 \times 385.$  }
	\label{fig:synthetic}
\end{figure*}

To ease the parameter tuning, we scale the pixel intensity of all the testing images in our experiments to $[0,1]$. Stage 1 of the isotropic TV SaT/SLaT is solved using the authors' official code that is implemented by a similar ADMM algorithm  to Algorithm \ref{alg:admm} with $\sigma=1$. Stage 1 of TV$^p$ and AITV SaT/SLaT is solved by Algorithm \ref{alg:admm} with $\sigma =1.25$ using the appropriate proximal operators. We set the penalty parameter in Algorithm \ref{alg:admm} to be $\delta_0 = 1.0, 2.0$ for grayscale and multichannel images, respectively.  The stopping criterion for the ADMM algorithms are until $  \frac{\|u_{t+1}-u_t\|_2}{\|u_{t+1}\|_2} < 10^{-4}$
with a maximum number of 300 iterations. We compare the proposed ADMM algorithm with our own DCA implementation for AITV SaT/SLaT as described in \cite{wu2022image}. Note that its  inner minimization subproblem is solved by semi-proximal ADMM \cite{han2018linear}, which has more parameters than  ADMM. We use the default parameter setting as suggested in \cite{wu2022image}.

To quantitatively evaluate the segmentation performance, we use two metrics: DICE index \cite{dice1945measures} when the ground truth is available and PSNR when the ground truth is unavailable. The DICE index is
given by \begin{align*}
\text{DICE} = 2\frac{\#\{R(i) \cap R'(i)\}}{\#\{R(i)\} +\#\{R'(i)\}},
\end{align*}
where $R(i)$ is the set of pixels with label $i$ in the ground-truth image $f$, $R'(i)$ is the set of pixels with label $i$ in the segmented image $\tilde{f}$, and $\#\{R\}$ refers to the number of pixels in the set $R$. Following the works of \cite{jung2017piecewise, li2021smoothing, ono2017, storath2014fast}, we use PSNR to determine how well the segmented image $\tilde{f}$ approximates the original image $f$. It is computed by $10\log_{10}(1/\text{MSE})$, where $\text{MSE}$ is the mean square error between $f$ and $\tilde{f}$.

We tune various parameters in the investigated algorithms to achieve the best DICE indices or PSNRs for synthetic or real images, respectively. The fidelity parameter $\lambda$ and the smoothing parameter $\mu$ are tuned for each image, which will be specified later. For TV$^p$ SaT/SLaT, we only consider  $p=1/2,2/3$ because they are the only values that have closed-form solutions \cite{cao2013fast, xu2012l_} for their proximal operators.
For the AITV related algorithms, we tune $\alpha \in \{0.2, 0.4, 0.6, 0.8\}$. For
the SaT-Potts model \cite{storath2014fast}, we use a default setting for the other parameters. For the convex Potts model \cite{pock-2009}, we run the algorithm for up to 150 iterations with the same stopping criterion as AITV does. 

\begin{table}[t!]
\caption{Comparison of the DICE indices and computational times (seconds) between the segmentation methods applied to Figure \ref{fig:synthetic_grayscale}
corrupted in four cases. Number in \textbf{bold} indicates either the highest DICE index or the fastest time among the segmentation methods for a given corrupted image.}\label{tab:grayscale_results}
\resizebox{\textwidth}{!}{ \begin{tabular}{l|cc||cc||cc||cc|}
\hhline{~|--------}
       & \multicolumn{2}{c||}{65\% RV}    & \multicolumn{2}{c||}{65\% SP}    & \multicolumn{2}{c||}{Blur and 50\% RV}    & \multicolumn{2}{c|}{Blur and 50\% SP}    \\ \hhline{~|--------}

                       & \multicolumn{1}{c|}{DICE} & \multicolumn{1}{c||}{Time (s) } & \multicolumn{1}{c|}{DICE} & \multicolumn{1}{c||}{Time (s) }  & \multicolumn{1}{c|}{DICE} & \multicolumn{1}{c||}{Time (s) }  & \multicolumn{1}{c|}{DICE} & \multicolumn{1}{c|}{Time (s) }   \\ \hline
\multicolumn{1}{|l|}{(Original) SaT} & \multicolumn{1}{c|}{0.9748} & \multicolumn{1}{c||}{4.71} & \multicolumn{1}{c|}{0.9641} & \multicolumn{1}{c||}{5.18} & \multicolumn{1}{c|}{0.9557} & \multicolumn{1}{c||}{6.15} & \multicolumn{1}{c|}{0.9498} & \multicolumn{1}{c|}{8.19} \\ \hline
\multicolumn{1}{|l|}{TV$^{p}$ SaT} & \multicolumn{1}{c|}{0.9751} & \multicolumn{1}{c||}{2.14} & \multicolumn{1}{c|}{0.9647} & \multicolumn{1}{c||}{2.33} & \multicolumn{1}{c|}{0.9539} & \multicolumn{1}{c||}{2.74} & \multicolumn{1}{c|}{0.9475} & \multicolumn{1}{c|}{3.76}  \\ \hline
\multicolumn{1}{|l|}{AITV SaT (ADMM)} & \multicolumn{1}{c|}{\textbf{0.9793}} & \multicolumn{1}{c||}{2.32} &  \multicolumn{1}{c|}{\textbf{0.9658}} & \multicolumn{1}{c||}{2.04} & \multicolumn{1}{c|}{\textbf{0.9581}} &\multicolumn{1}{c||}{2.43}  & \multicolumn{1}{c|}{\textbf{0.9522}} & \multicolumn{1}{c|}{2.54} \\ \hline
\multicolumn{1}{|l|}{AITV SaT (DCA)} & \multicolumn{1}{c|}{0.9783} & \multicolumn{1}{c||}{23.22} &  \multicolumn{1}{c|}{0.9644} & \multicolumn{1}{c||}{24.65} & \multicolumn{1}{c|}{0.9488} &\multicolumn{1}{c||}{40.93}  & \multicolumn{1}{c|}{0.9434} & \multicolumn{1}{c|}{35.23} \\ \hline
\multicolumn{1}{|l|}{AITV CV} & \multicolumn{1}{c|}{0.9786} & \multicolumn{1}{c||}{91.26} & \multicolumn{1}{c|}{0.9655} & \multicolumn{1}{c||}{121.57} & \multicolumn{1}{c|}{0.9328} & \multicolumn{1}{c||}{121.11} & \multicolumn{1}{c|}{0.9190} & \multicolumn{1}{c|}{151.68}  \\ \hline
\multicolumn{1}{|l|}{ICTM} & \multicolumn{1}{c|}{0.4322} & \multicolumn{1}{c||}{\textbf{0.50}} & \multicolumn{1}{c|}{0.4321} & \multicolumn{1}{c||}{\textbf{0.18}} & \multicolumn{1}{c|}{0.5319} & \multicolumn{1}{c||}{\textbf{0.82}} & \multicolumn{1}{c|}{0.5065} & \multicolumn{1}{c|}{\textbf{0.22}} \\ \hline
\multicolumn{1}{|l|}{TV$^p$ MS} & \multicolumn{1}{c|}{0.9681} & \multicolumn{1}{c||}{4.96} & \multicolumn{1}{c|}{0.9533} & \multicolumn{1}{c||}{10.14} & \multicolumn{1}{c|}{0.9369} & \multicolumn{1}{c||}{3.20} & \multicolumn{1}{c|}{0.9271} & \multicolumn{1}{c|}{6.08} \\ \hline
\multicolumn{1}{|l|}{Convex Potts} & \multicolumn{1}{c|}{0.9755} & \multicolumn{1}{c||}{8.01} & \multicolumn{1}{c|}{0.9637} & \multicolumn{1}{c||}{6.81} & \multicolumn{1}{c|}{0.9101} & \multicolumn{1}{c||}{7.99} & \multicolumn{1}{c|}{0.9132} & \multicolumn{1}{c|}{6.85} \\ \hline
\multicolumn{1}{|l|}{SaT-Potts} & \multicolumn{1}{c|}{0.9714} & \multicolumn{1}{c||}{4.67} & \multicolumn{1}{c|}{0.9559} & \multicolumn{1}{c||}{4.24} & \multicolumn{1}{c|}{0.9305} & \multicolumn{1}{c||}{4.39} & \multicolumn{1}{c|}{0.9180} & \multicolumn{1}{c|}{4.11} \\ \hline
\end{tabular}}
\end{table}
\begin{figure}[t!]

\begin{tabular}{ccccc}
		\captionsetup[subfigure]{justification=centering}
\subcaptionbox{Average blur and RV noise}{\includegraphics[width = 0.75in]{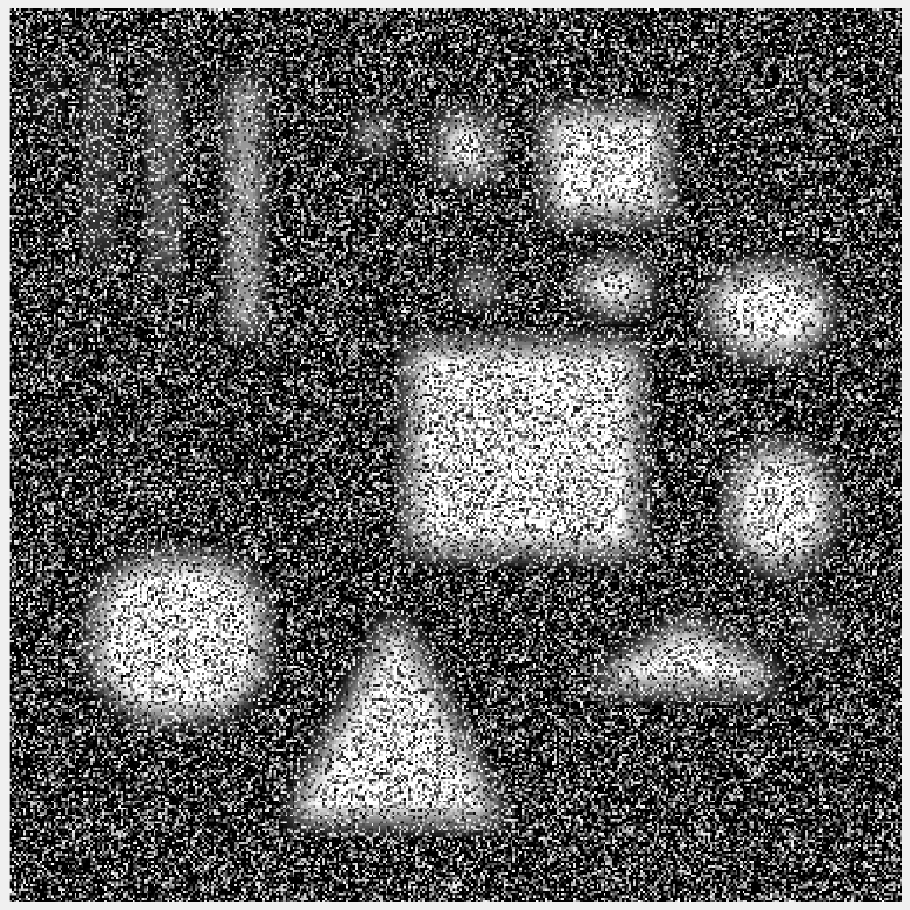}} &
		\captionsetup[subfigure]{justification=centering}
\subcaptionbox{(original) SaT\\
DICE: 0.9557}{\includegraphics[width = 0.75in]{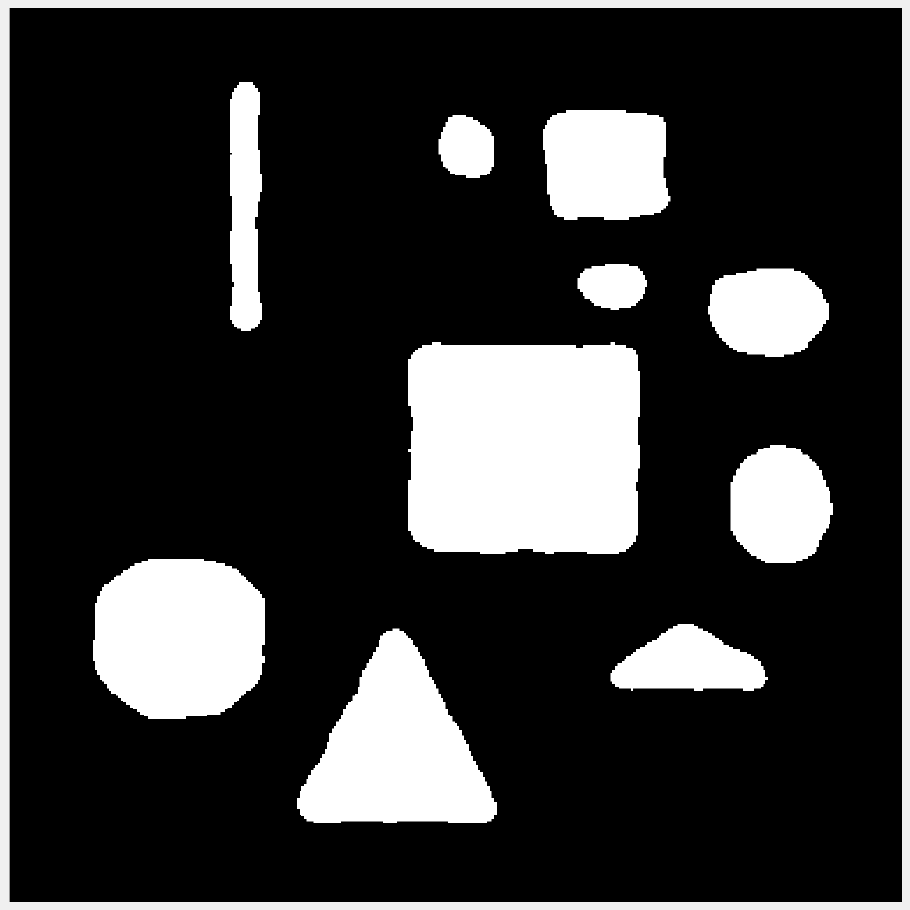}} & 		\captionsetup[subfigure]{justification=centering}
\subcaptionbox{TV$^{p}$ SaT\\
DICE: 0.9539}{\includegraphics[ width = 0.75in]{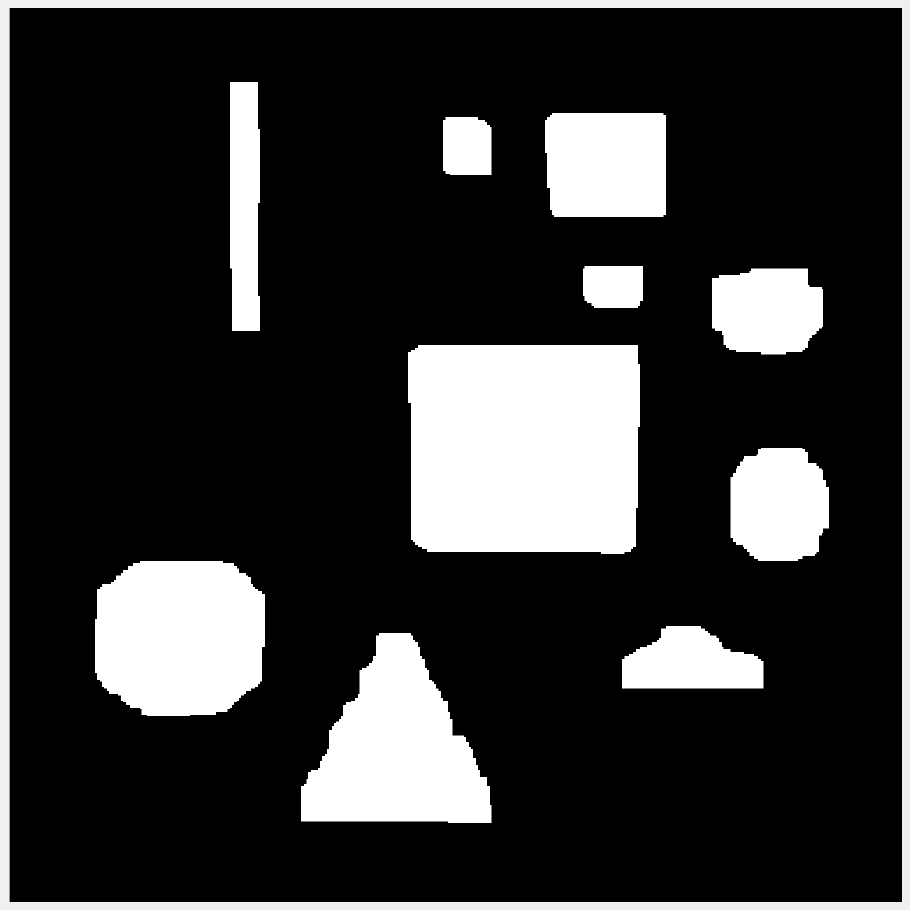}} &		\captionsetup[subfigure]{justification=centering}
\subcaptionbox{AITV  SaT (ADMM)\\
DICE: 0.9581}{\includegraphics[width = 0.75in]{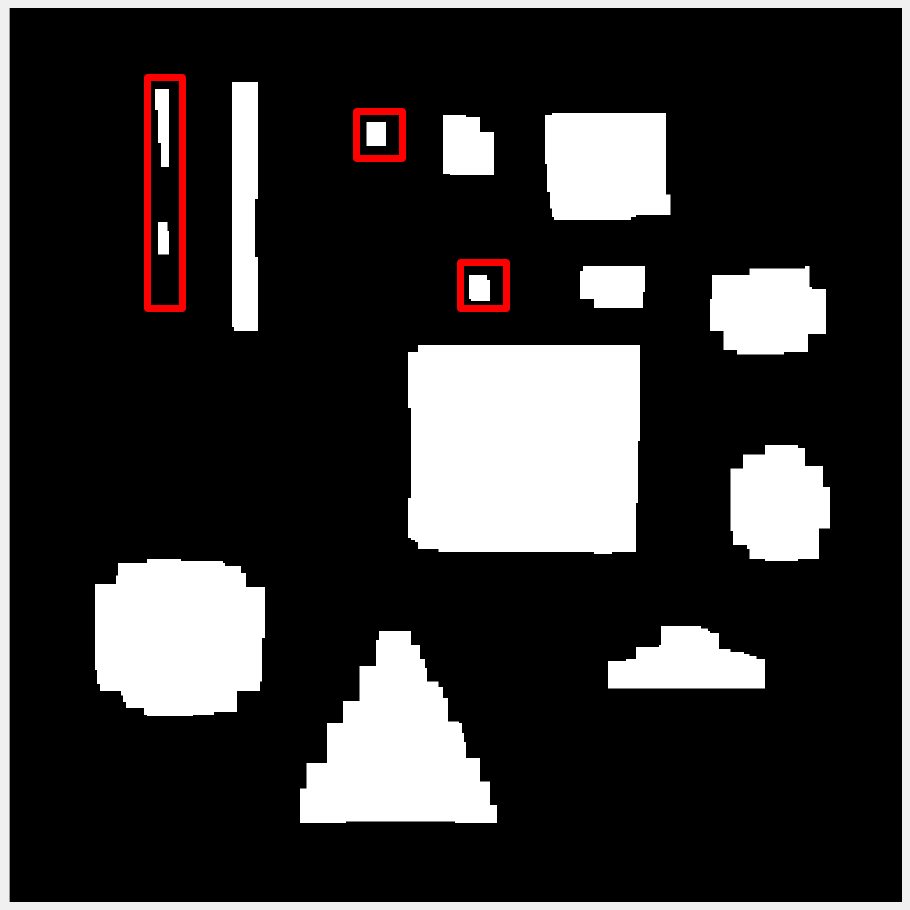}} &
\captionsetup[subfigure]{justification=centering}
\subcaptionbox{AITV  SaT (DCA)
DICE: 0.9488}{\includegraphics[width = 0.75in]{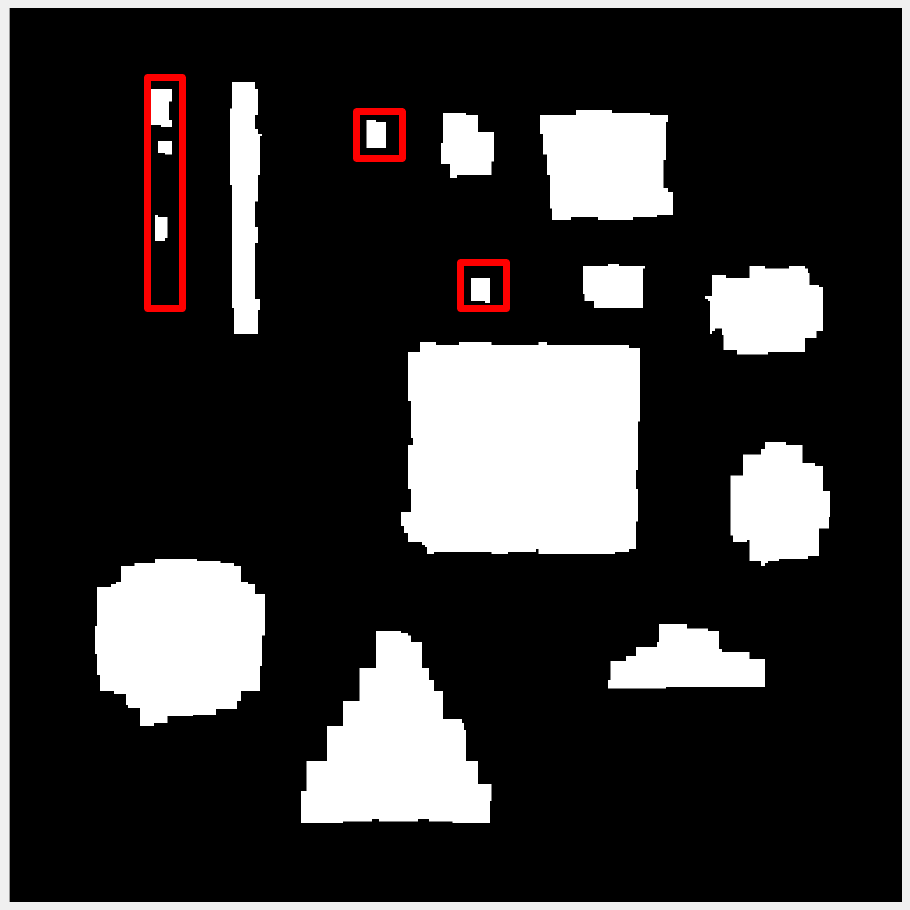}}\\
		\captionsetup[subfigure]{justification=centering} \subcaptionbox{AITV  CV  \\ DICE: 0.9328}{\includegraphics[width = 0.75in]{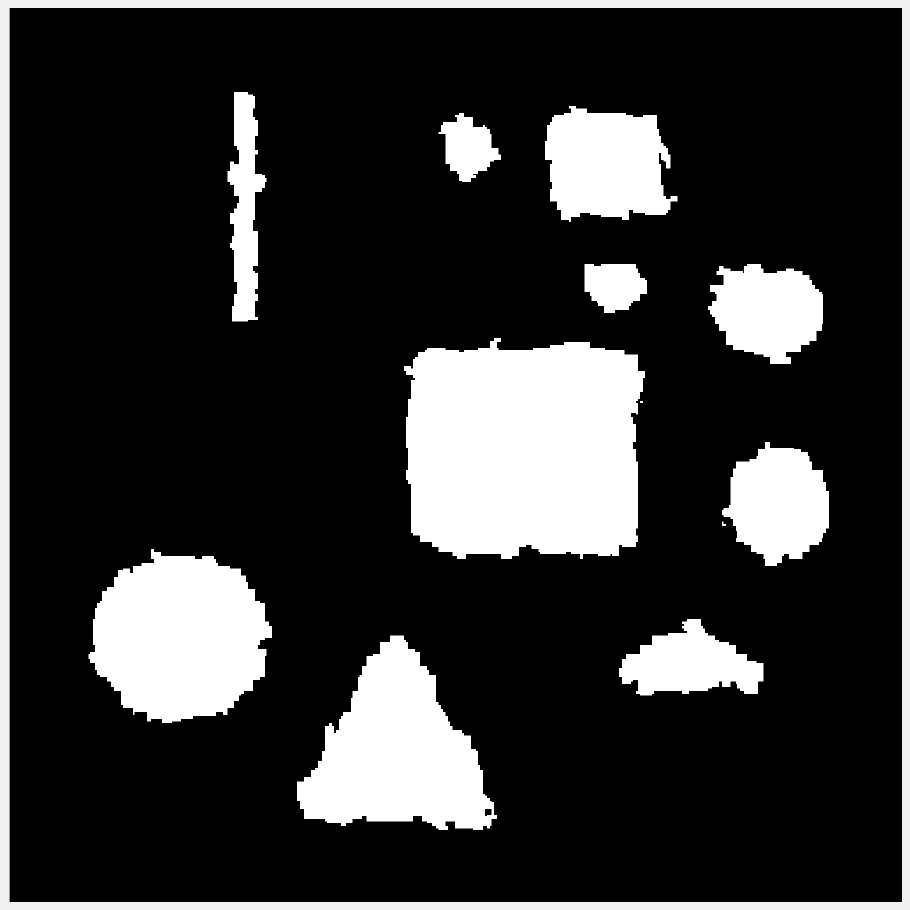}} & \captionsetup[subfigure]{justification=centering} \subcaptionbox{ICTM \\ DICE: 0.5319}{\includegraphics[width = 0.75in]{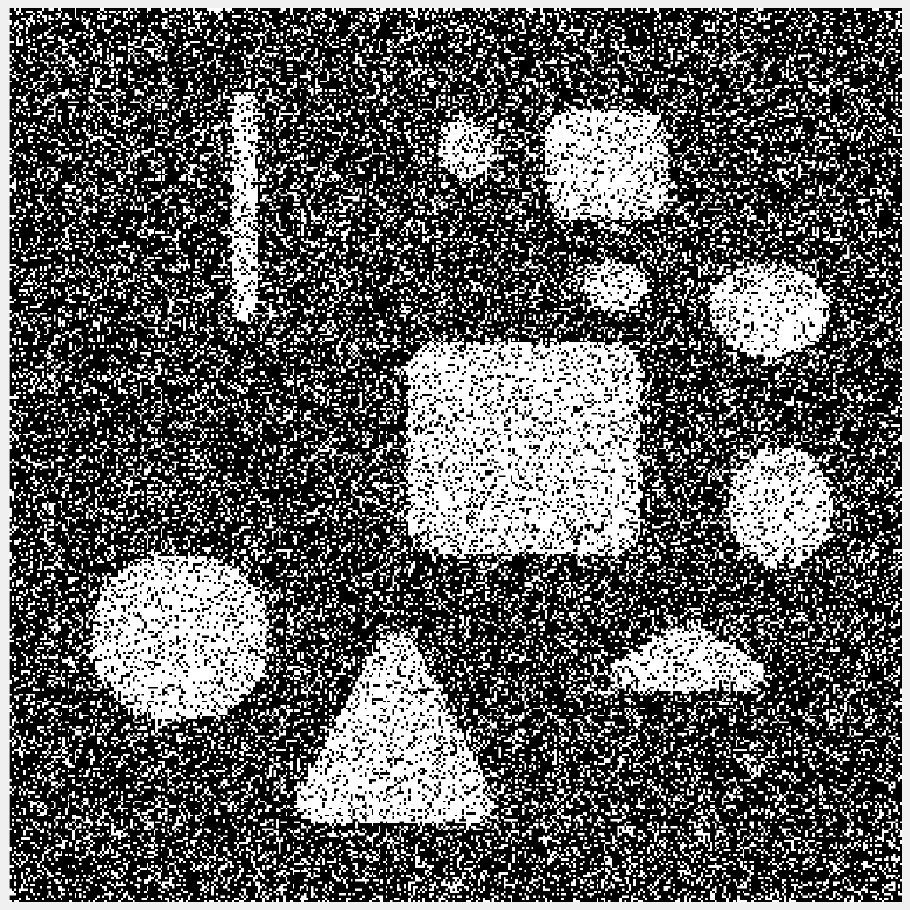}} &
  \captionsetup[subfigure]{justification=centering} \subcaptionbox{TV$^p$ MS \\ DICE: 0.9369}{\includegraphics[width = 0.75in]{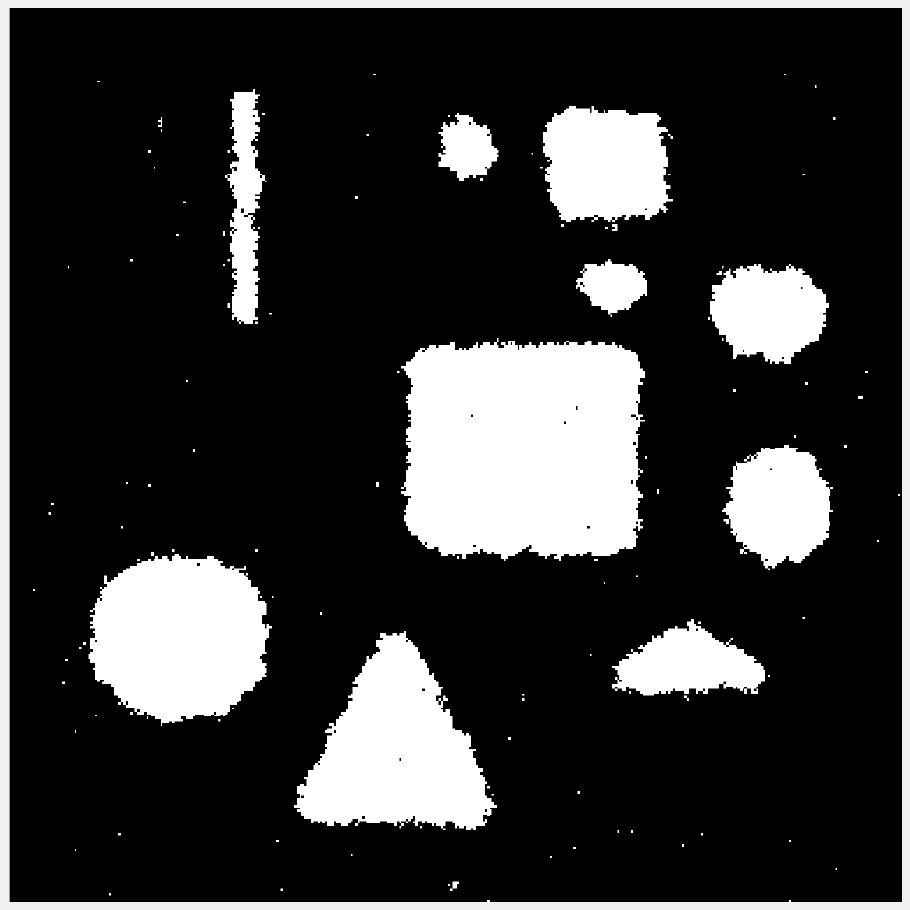}} &  \captionsetup[subfigure]{justification=centering} \subcaptionbox{Convex Potts \\DICE: 0.9101}{\includegraphics[width = 0.75in]{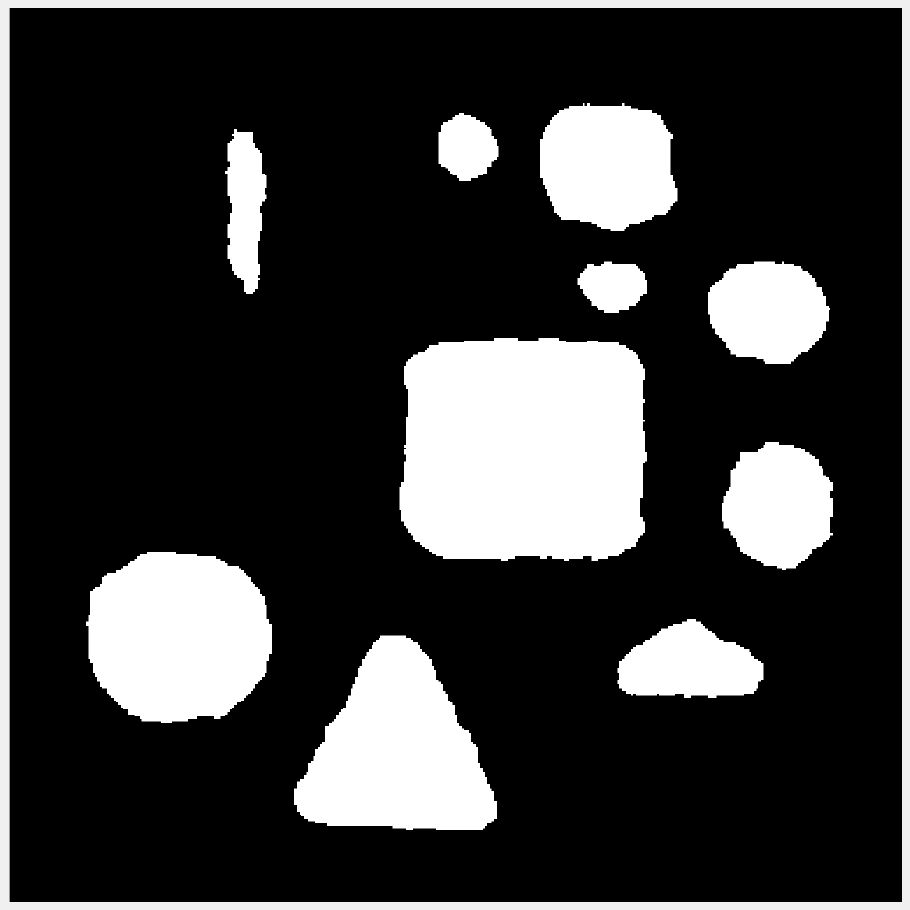}} & \captionsetup[subfigure]{justification=centering} \subcaptionbox{SaT-Potts \\ DICE: 0.9305}{\includegraphics[width = 0.75in]{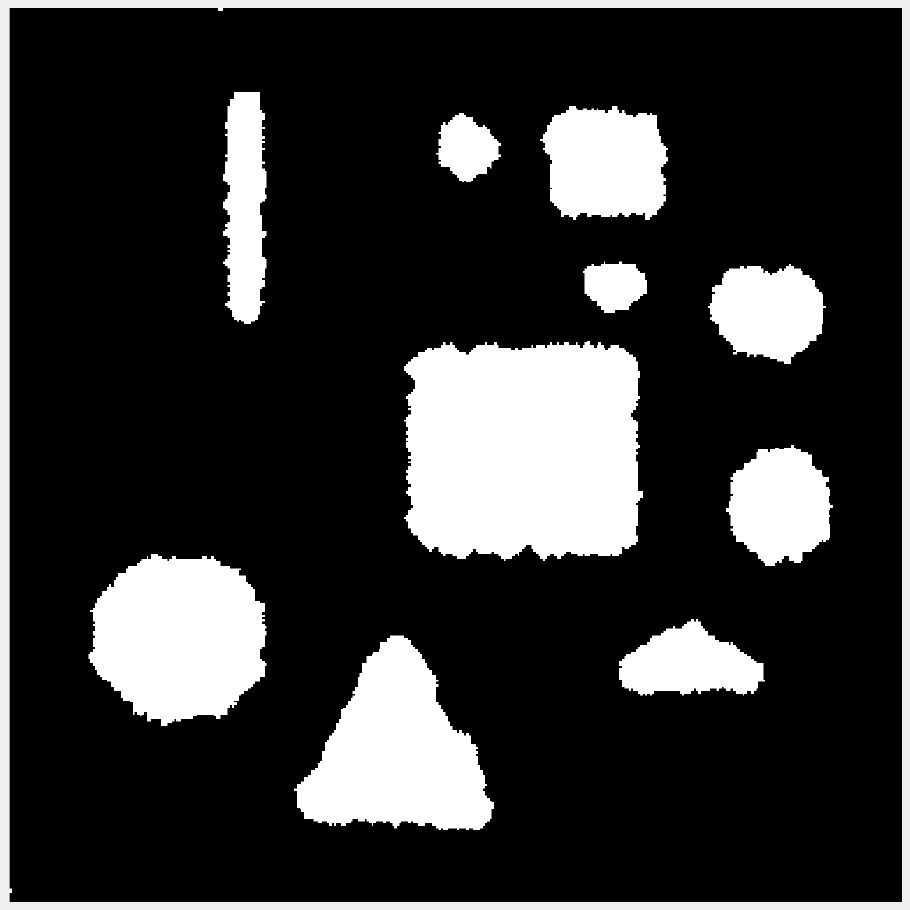}} 
		\end{tabular}
		\caption{Segmentation results of Figure \ref{fig:synthetic_grayscale} corrupted with average blur followed by 50\% RV  noise. Regions boxed in red are only  identified by AITV SaT.}
		\label{fig:grayscale_50_rvin_blur}		
\end{figure}

All experiments are performed in MATLAB R2022b on a Dell laptop with a 1.80 GHz Intel Core i7-8565U processor and 16.0 GB of RAM. In the general SaT/SLaT framework, we use some MATLAB built-in functions. In Stage 2, \texttt{makecform(`srgb2lab')} is used to convert RGB to Lab. In Stage 3, \texttt{kmeans} performs $K$-means++ clustering \cite{arthur07} for up to 100 iterations five times with different initialization and selects the best arrangement among the five solutions. We also parallelize Stage 1 for color, or generally multichannel, images to speed up the computation. To compute DICE and PSNR, we use the MATLAB functions \texttt{dice} and \texttt{psnr}. The AITV SaT/SLaT codes are available at \url{https://github.com/kbui1993/Official_AITV_SaT_SLaT}.

\subsection{Two-Phase Segmentation on Synthetic Images}\label{sec:synthetic_image}
We compare the proposed ADMM algorithm of AITV SaT/SLaT with the other SaT/SLaT methods, the Potts models, ICTM, TV$^p$ MS, and the AITV CV model on the synthetic images presented in Figure \ref{fig:synthetic}. We corrupt the images with either random-valued (RV) or salt-and-pepper (SP) impulsive noises. Additionally, we consider  blurring the image before  adding impulsive noises. Specifically, we use an average blur \texttt{fspecial(`average', 15)} for Figure \ref{fig:synthetic_grayscale} and a motion blur \texttt{fspecial(`motion', 5, 45)} for Figure \ref{fig:synthetic_2phase}. For the SaT/SLaT methods applied to both images in Figure {\ref{fig:synthetic}}, we tune the parameters $\lambda \in [1,10]$ and $\mu \in [0.2,6]$.

\begin{table}[t!]
\centering
\scriptsize
\captionof{table}{Comparison of the DICE indices and computational times (seconds) between the segmentation methods applied to Figure \ref{fig:synthetic_2phase} corrupted in four cases. Number in \textbf{bold} indicates either the highest DICE index or the fastest time among the segmentation methods for a given corrupted image.}
\label{tab:color_results}
\resizebox{\textwidth}{!}{\begin{tabular}{l|cc||cc||cc||cc|}
\hhline{~|--------}
                       & \multicolumn{2}{c||}{60\% RV}    & \multicolumn{2}{c||}{60\% SP}    & \multicolumn{2}{c||}{Blur and 45\% RV}    & \multicolumn{2}{c|}{Blur and 45\% SP}    \\ \hhline{~|--------}
                       & \multicolumn{1}{c|}{DICE} & \multicolumn{1}{c||}{Time (s) } & \multicolumn{1}{c|}{DICE} & \multicolumn{1}{c||}{Time (s) }  & \multicolumn{1}{c|}{DICE} & \multicolumn{1}{c||}{Time (s) }  & \multicolumn{1}{c|}{DICE} & \multicolumn{1}{c|}{Time (s) }   \\ \hline
\multicolumn{1}{|l|}{(Original) SLaT} & \multicolumn{1}{c|}{0.9814} & \multicolumn{1}{c||}{11.74} & \multicolumn{1}{c|}{0.9637} & \multicolumn{1}{c||}{12.61} & \multicolumn{1}{c|}{0.9845} & \multicolumn{1}{c||}{11.07} & \multicolumn{1}{c|}{0.9749} & \multicolumn{1}{c|}{12.17} \\ \hline
\multicolumn{1}{|l|}{TV$^{p}$ SLaT} &  \multicolumn{1}{c|}{0.9822} & \multicolumn{1}{c||}{4.54} &\multicolumn{1}{c|}{0.9731} & \multicolumn{1}{c||}{4.77}  & \multicolumn{1}{c|}{0.9863} & \multicolumn{1}{c||}{8.05} & \multicolumn{1}{c|}{0.9772} & \multicolumn{1}{c|}{6.27}  \\ \hline
\multicolumn{1}{|l|}{AITV SLaT (ADMM)} &  \multicolumn{1}{c|}{0.9839} & \multicolumn{1}{c||}{3.31} &\multicolumn{1}{c|}{0.9748} & \multicolumn{1}{c||}{4.67} & \multicolumn{1}{c|}{\textbf{0.9872}} & \multicolumn{1}{c||}{6.01} & \multicolumn{1}{c|}{\textbf{0.9780}} & \multicolumn{1}{c|}{6.45} \\ \hline 
\multicolumn{1}{|l|}{AITV SLaT (DCA)} &  \multicolumn{1}{c|}{0.9849} & \multicolumn{1}{c||}{41.27} &\multicolumn{1}{c|}{0.9753} & \multicolumn{1}{c||}{47.09} & \multicolumn{1}{c|}{0.9866} & \multicolumn{1}{c||}{44.54} & \multicolumn{1}{c|}{0.9776} & \multicolumn{1}{c|}{61.64} \\ \hline 
\multicolumn{1}{|l|}{AITV CV} & \multicolumn{1}{c|}{\textbf{0.9893}} & \multicolumn{1}{c||}{84.56} &  \multicolumn{1}{c|}{\textbf{0.9806}} & \multicolumn{1}{c||}{113.08}&  \multicolumn{1}{c|}{0.9771} & \multicolumn{1}{c||}{92.41} & \multicolumn{1}{c|}{0.9702} & \multicolumn{1}{c|}{103.99}  \\ \hline
\multicolumn{1}{|l|}{ICTM} & \multicolumn{1}{c|}{0.4788} & \multicolumn{1}{c||}{\textbf{1.02}} & \multicolumn{1}{c|}{0.4589} & \multicolumn{1}{c||}{\textbf{0.25}} & \multicolumn{1}{c|}{0.5782} & \multicolumn{1}{c||}{\textbf{1.35}} & \multicolumn{1}{c|}{0.5565} & \multicolumn{1}{c|}{\textbf{0.35}} \\ \hline
\multicolumn{1}{|l|}{TV$^p$ MS} & \multicolumn{1}{c|}{0.9799} & \multicolumn{1}{c||}{3.71} & \multicolumn{1}{c|}{0.9688} & \multicolumn{1}{c||}{54.10} & \multicolumn{1}{c|}{0.9791} & \multicolumn{1}{c||}{3.71} & \multicolumn{1}{c|}{0.9719} & \multicolumn{1}{c|}{3.52} \\ \hline
\multicolumn{1}{|l|}{Convex Potts} & \multicolumn{1}{c|}{0.9629} & \multicolumn{1}{c||}{9.20} & \multicolumn{1}{c|}{0.9614} & \multicolumn{1}{c||}{7.50} & \multicolumn{1}{c|}{0.9728} & \multicolumn{1}{c||}{7.31} & \multicolumn{1}{c|}{0.9573} & \multicolumn{1}{c|}{8.10} \\ \hline
\multicolumn{1}{|l|}{SaT-Potts} & \multicolumn{1}{c|}{0.9806} & \multicolumn{1}{c||}{7.66} & \multicolumn{1}{c|}{0.9672} & \multicolumn{1}{c||}{7.70} & \multicolumn{1}{c|}{0.9760} & \multicolumn{1}{c||}{6.24} & \multicolumn{1}{c|}{0.9643} & \multicolumn{1}{c|}{6.53} \\ \hline
\end{tabular}}
\end{table}
\begin{figure*}

\centering
\begin{tabular}{ccccc}
		\captionsetup[subfigure]{justification=centering}
\subcaptionbox{Motion blur with SP noise}{\includegraphics[width = 0.75in]{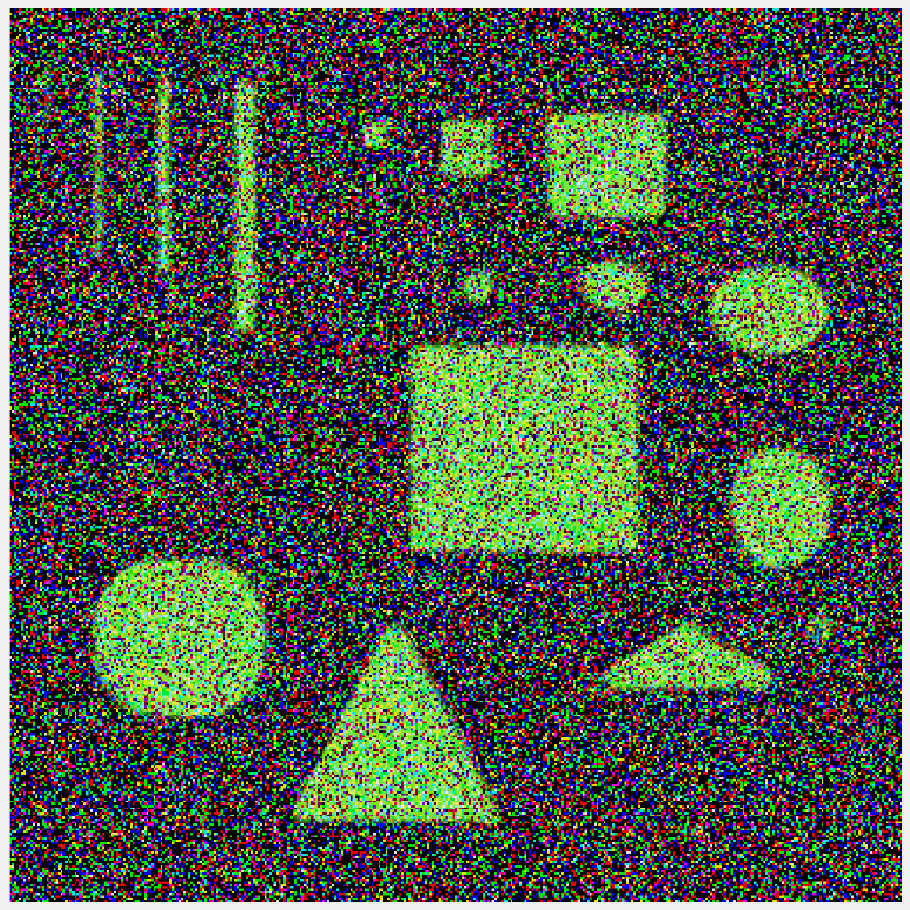}} &
		\captionsetup[subfigure]{justification=centering}
\subcaptionbox{(original) SLaT\\
DICE: 0.9749}{\includegraphics[width = 0.75in]{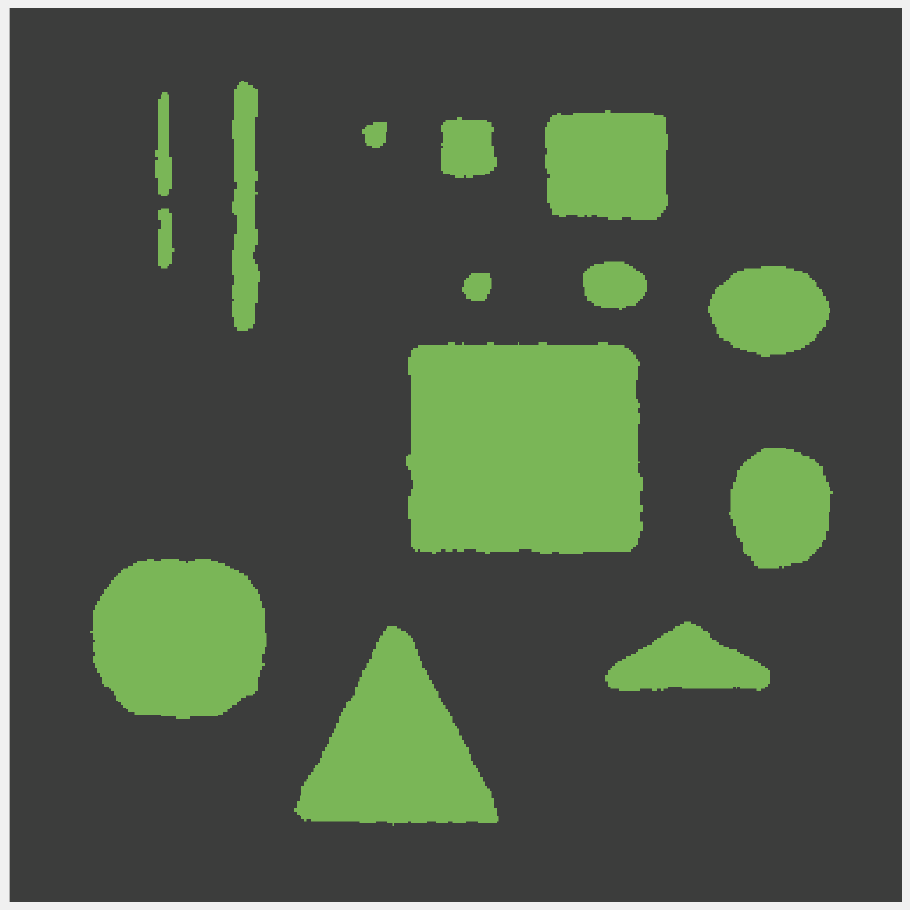}} & 		\captionsetup[subfigure]{justification=centering}
\subcaptionbox{TV$^{p}$ SLaT\\
DICE: 0.9772}{\includegraphics[ width = 0.75in]{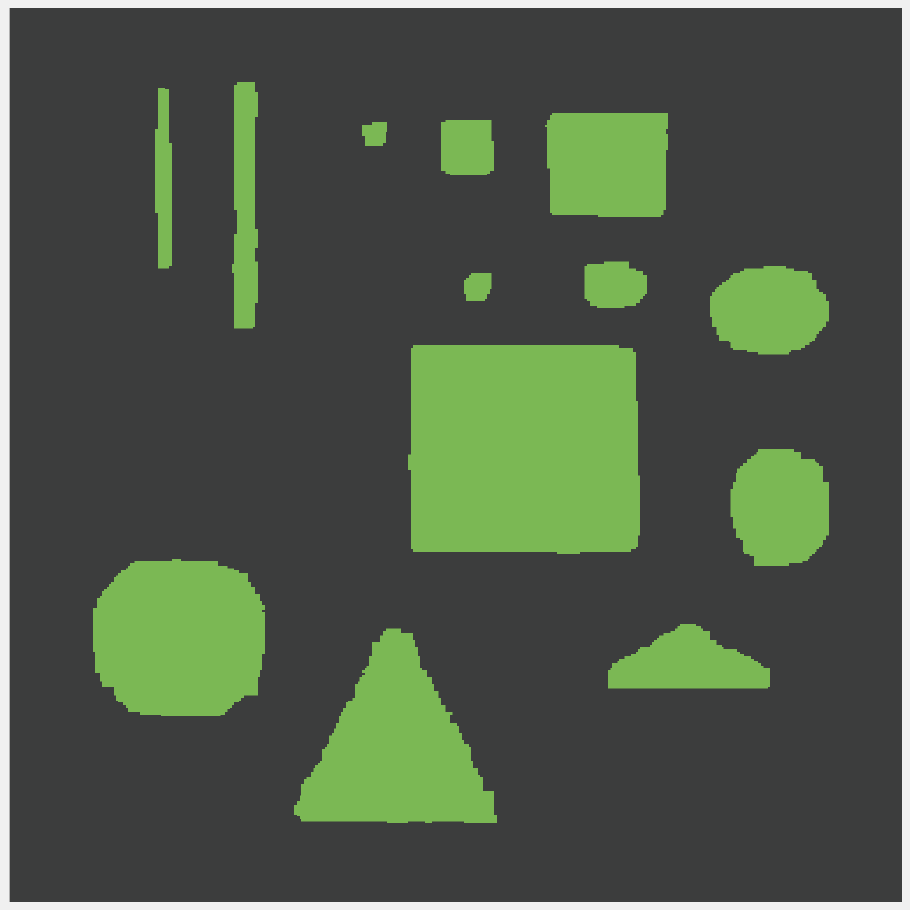}} &		\captionsetup[subfigure]{justification=centering}
\subcaptionbox{AITV SLaT (ADMM)\\
DICE: 0.9780}{\includegraphics[width = 0.75in]{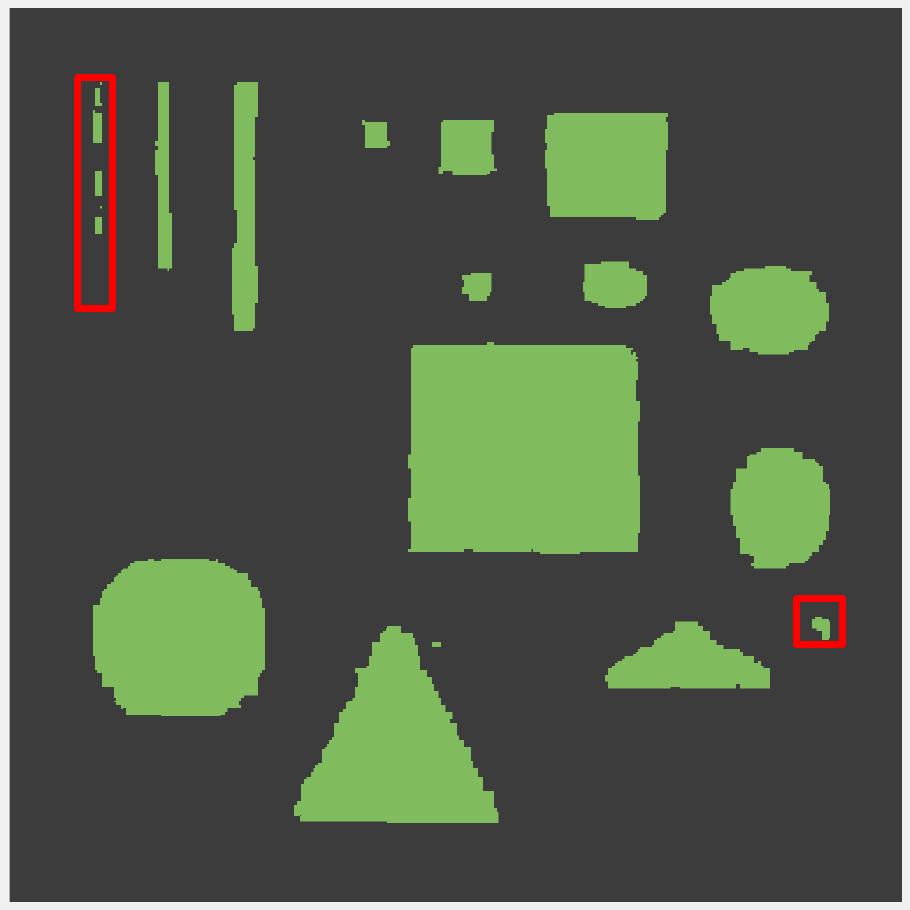}} & \captionsetup[subfigure]{justification=centering} 
	\subcaptionbox{AITV SLaT (DCA)\\
DICE: 0.9776}{\includegraphics[width = 0.75in]{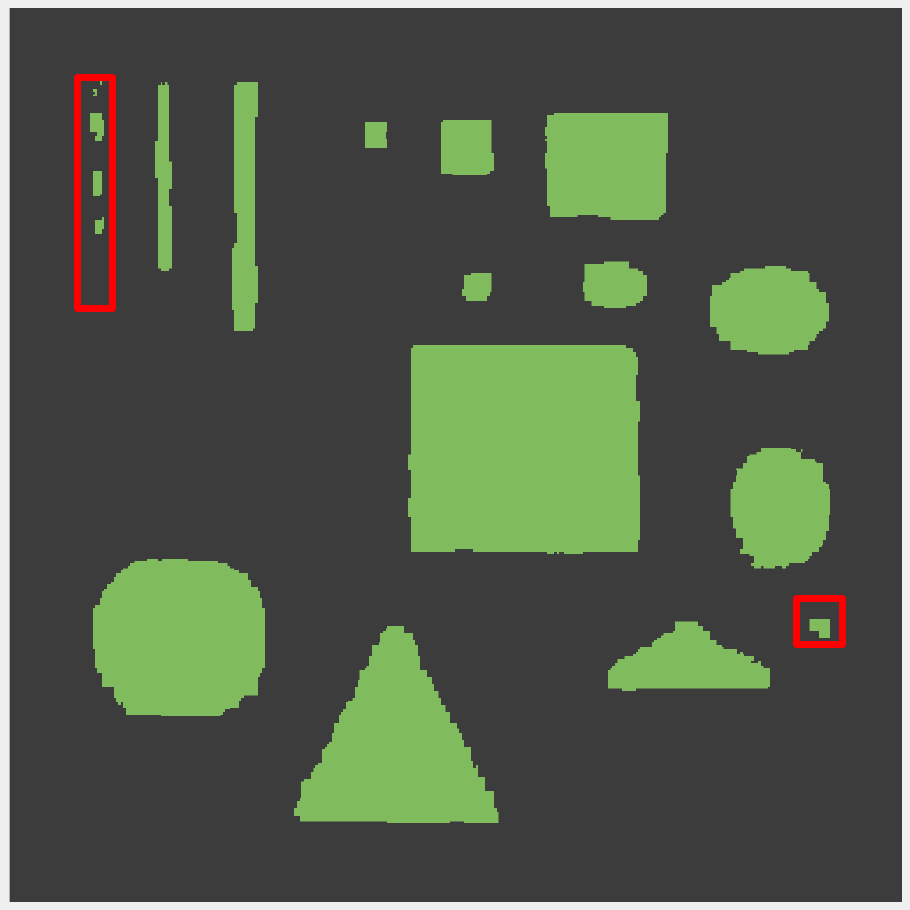}}\\	\captionsetup[subfigure]{justification=centering} \subcaptionbox{AITV   CV   \\ DICE: 0.9702}{\includegraphics[width = 0.75in]{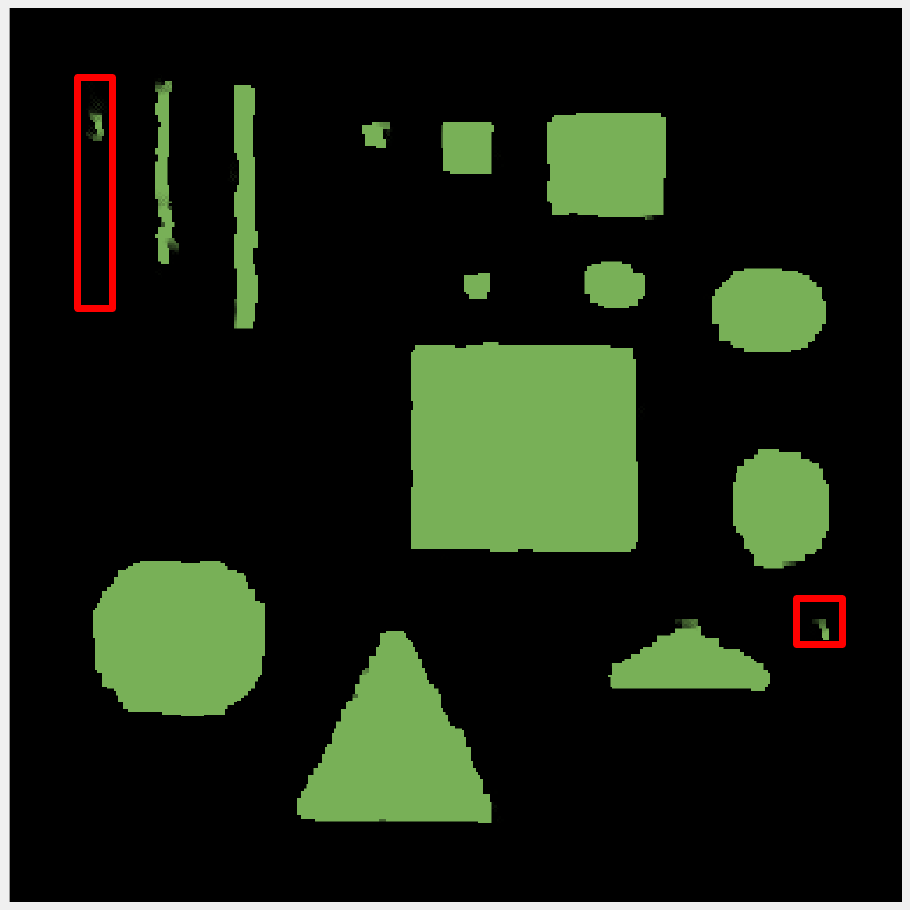}}&  \captionsetup[subfigure]{justification=centering} \subcaptionbox{ICTM \\DICE: 0.5565}{\includegraphics[width = 0.75in]{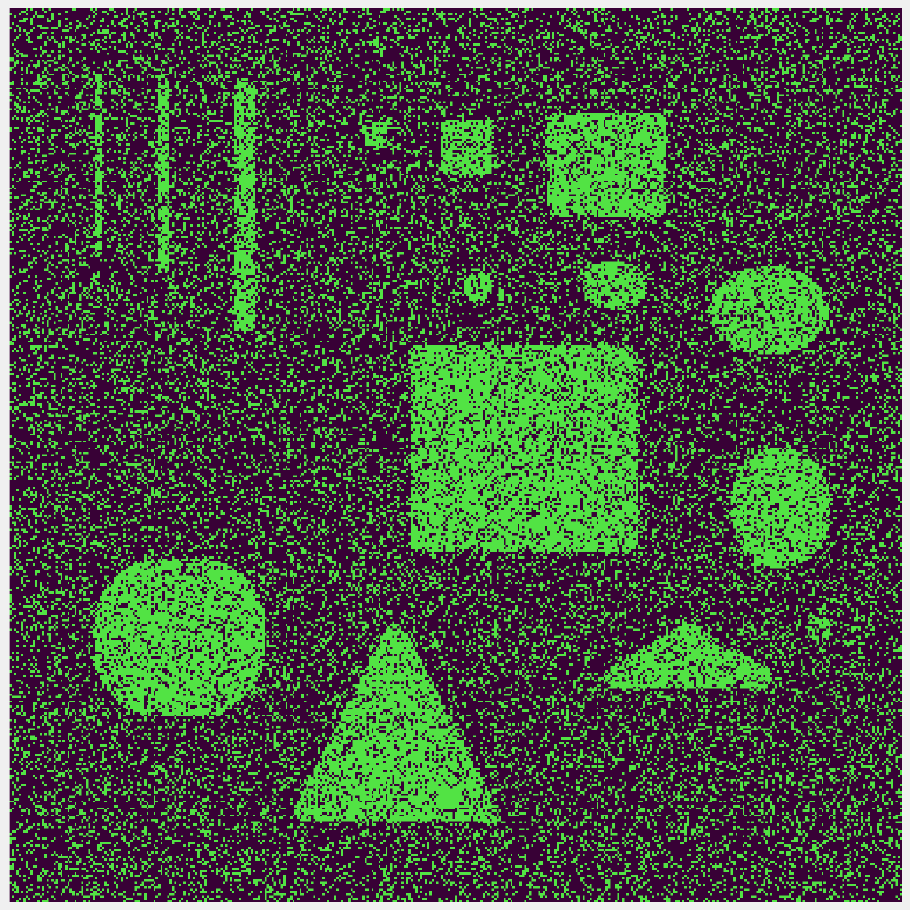}} & \captionsetup[subfigure]{justification=centering} \subcaptionbox{TV$^p$ MS \\ DICE: 0.9719}{\includegraphics[width = 0.75in]{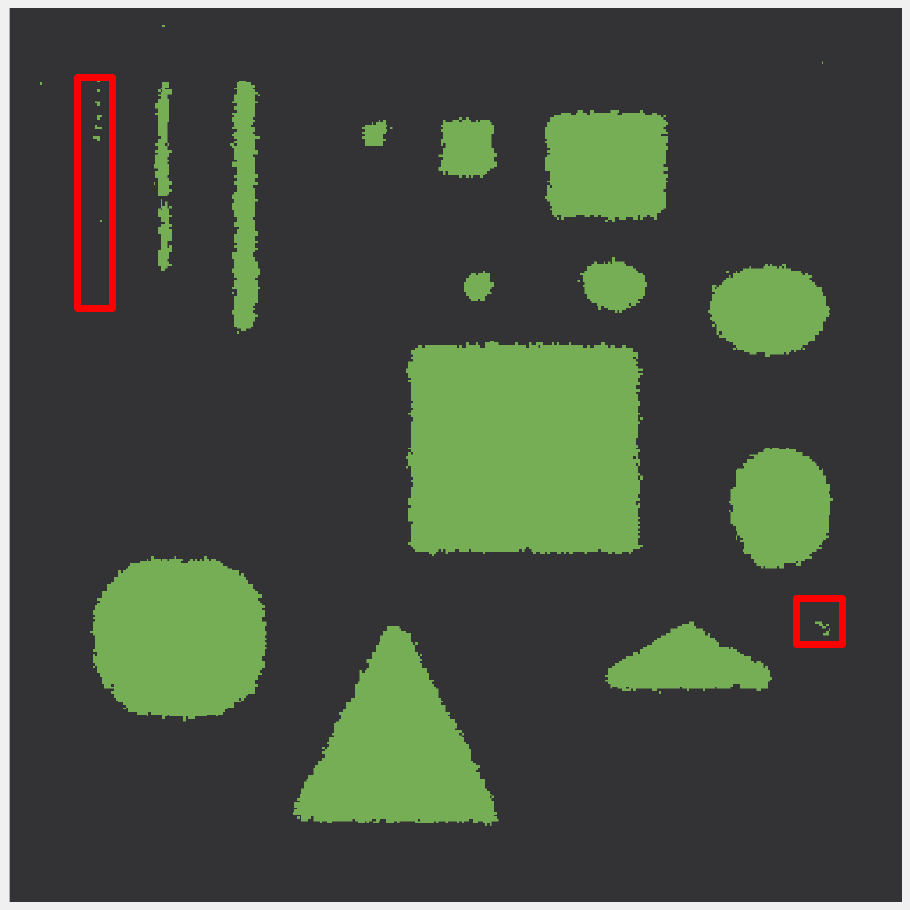}}  &  \captionsetup[subfigure]{justification=centering} \subcaptionbox{Convex Potts \\DICE: 0.9573}{\includegraphics[width = 0.75in]{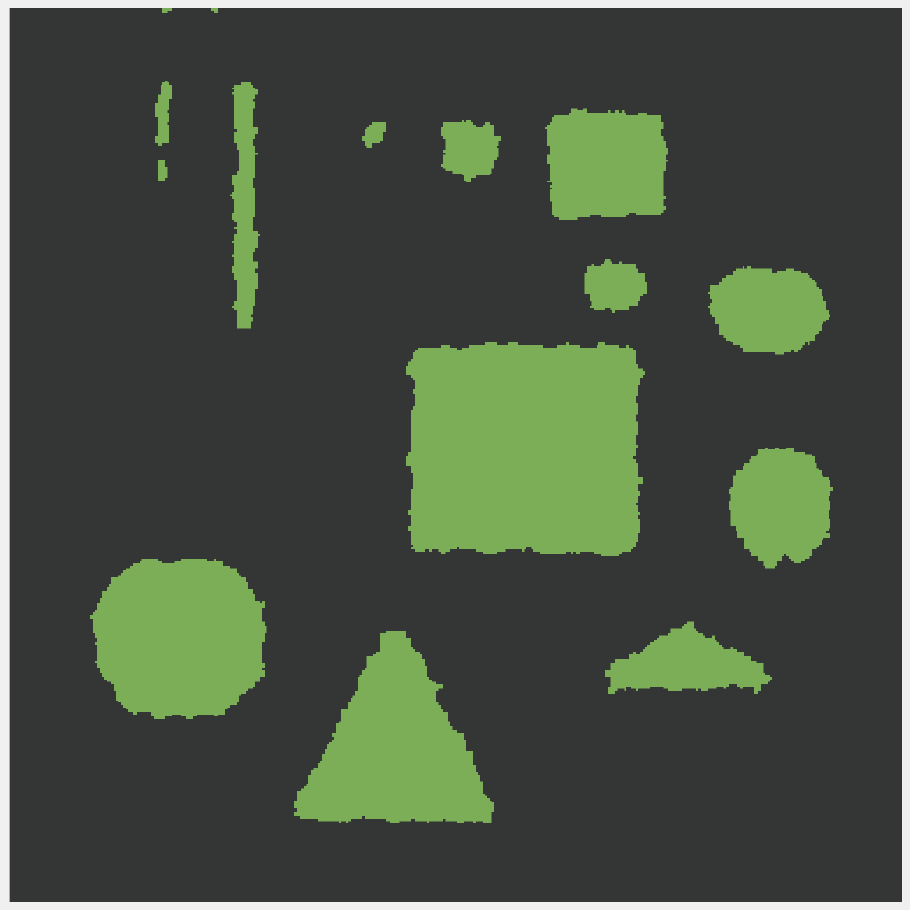}} & \captionsetup[subfigure]{justification=centering} \subcaptionbox{SaT-Potts \\ DICE: 0.9643}{\includegraphics[width = 0.75in]{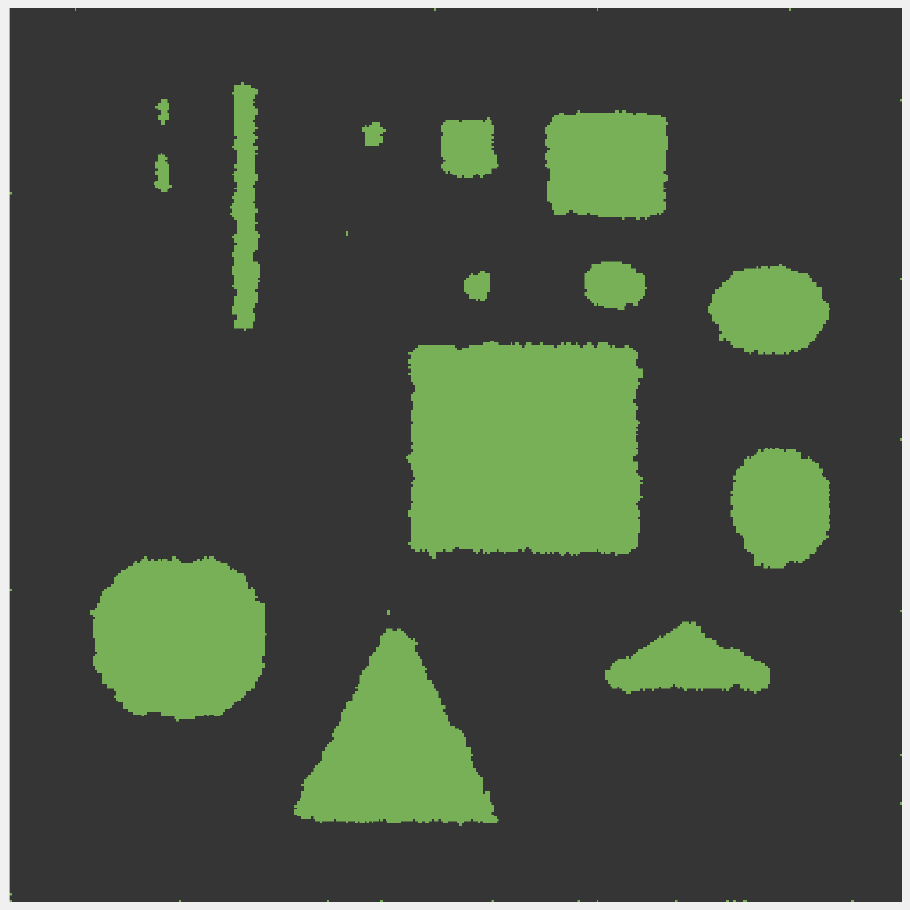}} 
		\end{tabular}
		\caption{Segmentation results of Figure \ref{fig:synthetic_2phase} corrupted with motion blur followed by 45\% SP  noise. Regions boxed in red are only identified by the AITV SLaT, AITV CV, and TV$^p$ MS.} 
		\label{fig:color_45_spin_blur}
\end{figure*}

\subsubsection{Synthetic Grayscale  Images}
We apply the competing segmentation methods on four types of input data based on Figure \ref{fig:synthetic_grayscale}, i.e., 65\% RV noise, 65\% SP noise, average blur followed by 50\% RV, and average blur followed by 50\% SP. The resulting DICE indices together with the computational times are recorded in Table \ref{tab:grayscale_results}. For all four cases, our proposed AITV SaT (ADMM) achieves the highest DICE indices with generally the second fastest times. The fastest time is attained by ICTM, but it yields the worst results, indicating that it performs poorly on images corrupted by impulsive noise. The AITV CV model, ICTM, and TV$^p$ MS model, and the Potts models perform worse than the SaT methods on blurry images because, unlike the SaT methods, they do not account for blurring. Lastly, we point out that solving {\eqref{eq:AITV_MS}} in the AITV SaT model by ADMM yields higher DICE in significantly less time than by DCA.
 
Visual segmentation results are presented in
Figures \ref{fig:grayscale_50_rvin_blur} under the RV noise with  average blur. Both
AITV SaT methods identify the middle rectangle of the three rectangles at the top left corner  and the two smallest circles above the middle square in Figure {\ref{fig:grayscale_50_rvin_blur}}. These regions are enclosed in red boxes. As a result, identifying more regions than the other methods and having a smoother segmentation than its DCA counterpart, AITV SaT (ADMM) has the highest DICE index for this case.

\subsubsection{ Synthetic Color Images}
The (original) color image,
Figure \ref{fig:synthetic_2phase}, is corrupted by either 60\% impulsive noise or motion blur followed by 45\% noise. Table \ref{tab:color_results} records the DICE indices and the computational times of various segmentation methods applied on all the four cases. 
For the noisy images without blur, AITV SLaT (ADMM) attains comparable DICE indices as the best AITV CV method and its DCA counterpart but with significantly less computational time. For the blurry, noisy inputs, AITV SLaT (ADMM) attains the highest DICE indices. In general, as an alternative to its DCA counterpart, AITV SLaT (ADMM) gives satisfactory segmentation results under a reasonable amount of time.

Figure \ref{fig:color_45_spin_blur} illustrates the visual results under the SP noise with motion blur case. AITV SLaT, AITV CV, and TV$^p$ MS are able to partially segment the leftmost rectangle in the upper left corner and the small circular region right of the triangular region. These regions are boxed in red to showcase the main differences in the results outputted by the segmentation methods. By taking account for blur, AITV SLaT (ADMM) has the highest DICE index for this case while having a significantly faster time thans its DCA counterpart.

\begin{figure*}[t!]
	\centering
 
	\resizebox{\textwidth}{!}{\begin{tabular}{c@{}c@{}c@{}c@{}}
   		\captionsetup[subfigure]{justification=centering}\subcaptionbox{Caterpillar. Size: $200 \times 300$.\label{fig:caterpiller}}{\includegraphics[width=1.00in]{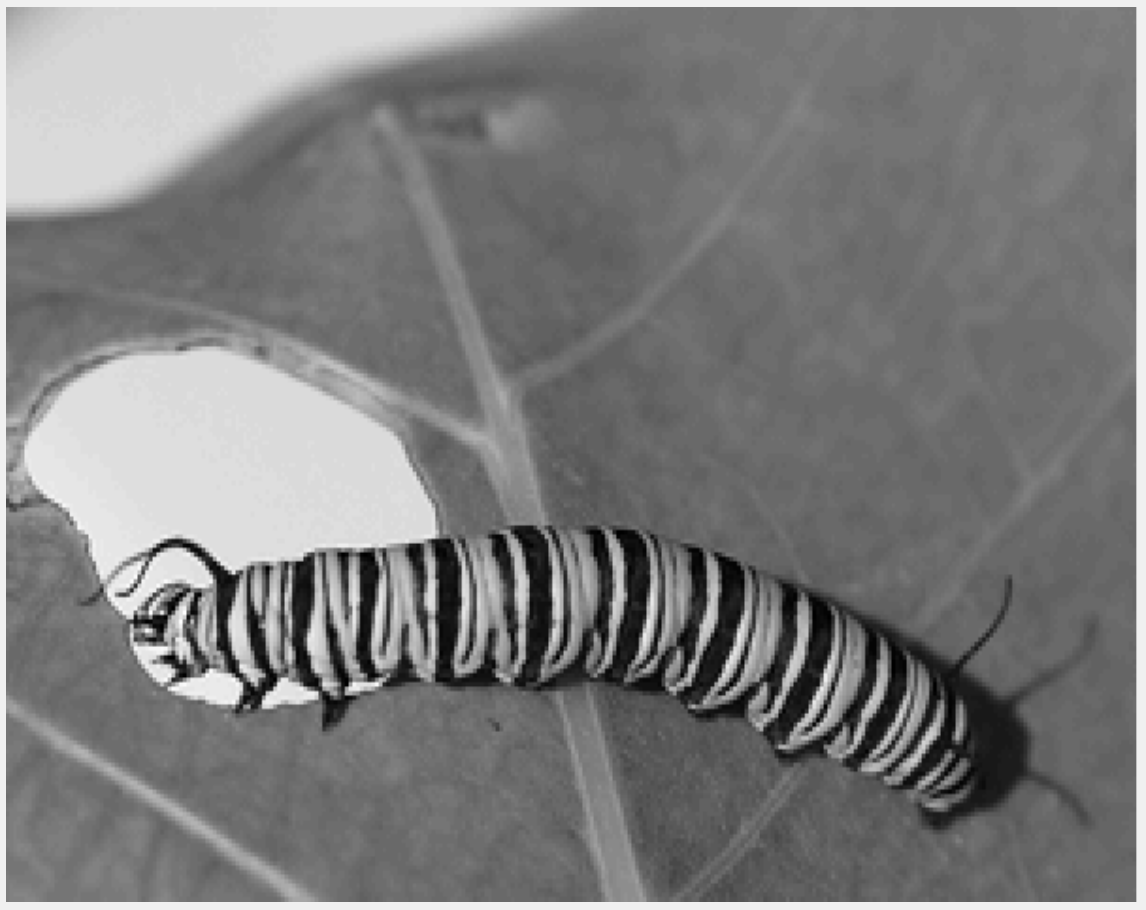}} &\captionsetup[subfigure]{justification=centering}
		 \subcaptionbox{ Egret.\\ Size: $200 \times 300.$ \label{fig:egret}}{\includegraphics[width=1.00in]{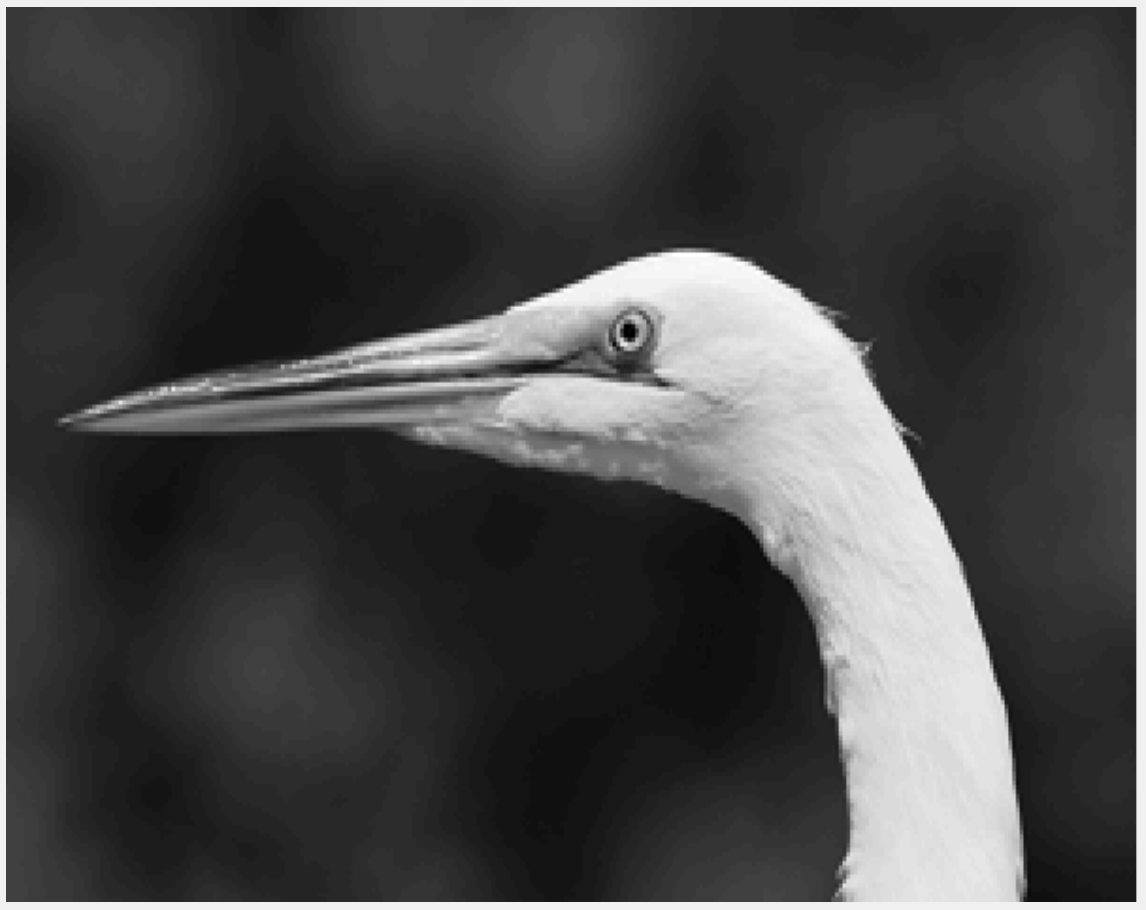}} &\captionsetup[subfigure]{justification=centering}
		 \subcaptionbox{Swan.\\ Size: $225 \times 300$.\label{fig:swan}}{\includegraphics[width=1.00in]{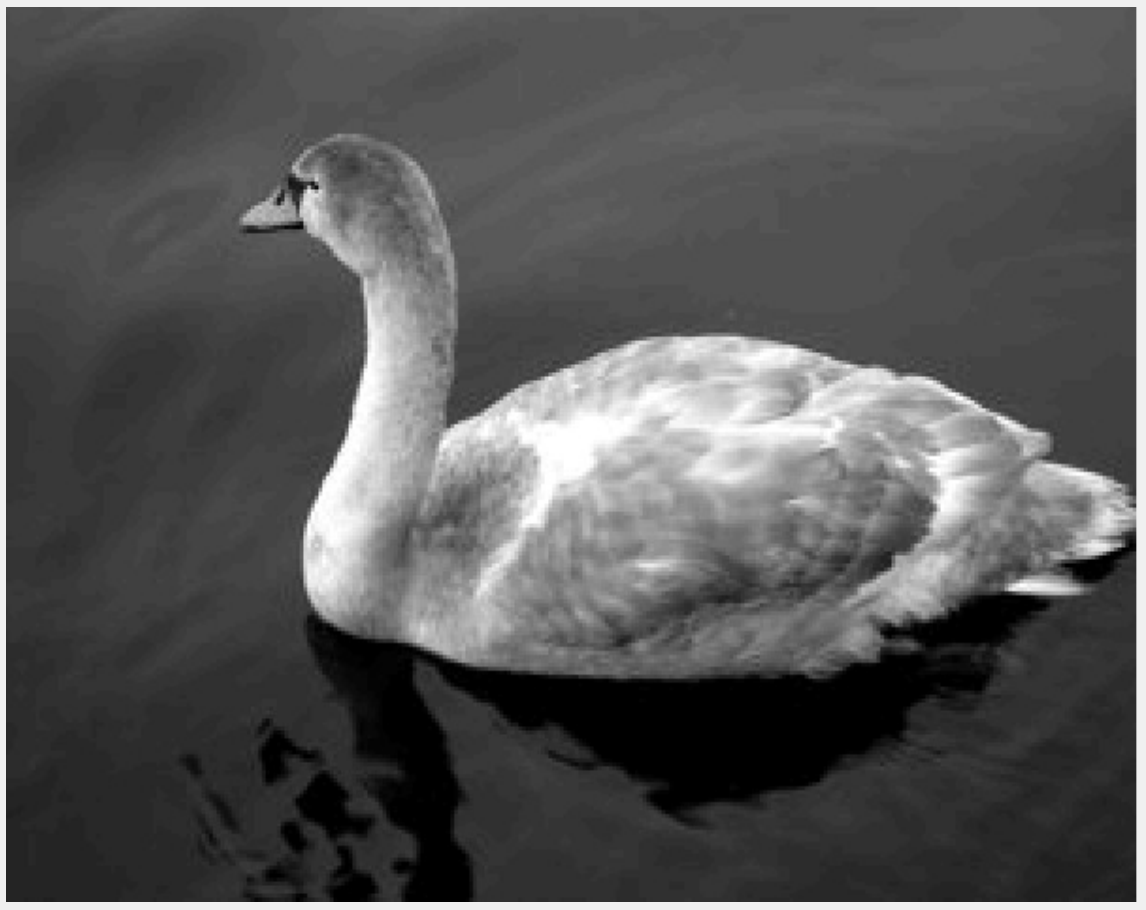}} &\captionsetup[subfigure]{justification=centering}
		 \subcaptionbox{Leaf.\\ Size: $203 \times 300$.\label{fig:leaf}}{\includegraphics[width=1.00in]{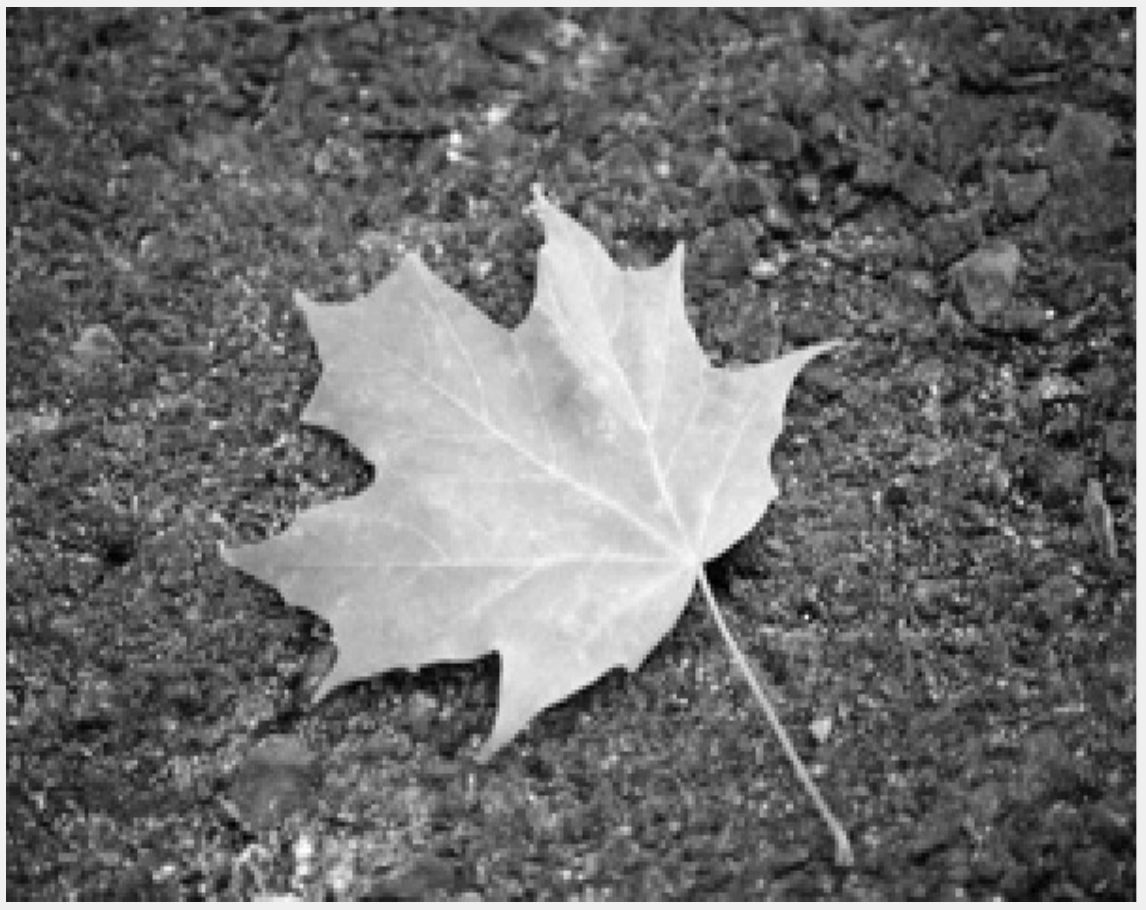}}
	\end{tabular}}
	\caption{Real, grayscale images for image segmentation.  }
	\label{fig:real_grayscale}
\end{figure*}

\begin{figure}[t!]
	\centering
	\begin{tabular}{c@{}c@{}c@{}}
		\subcaptionbox{\label{fig:non_iih_caterpiller}}{\includegraphics[width = 1.50in,height=1.00in]{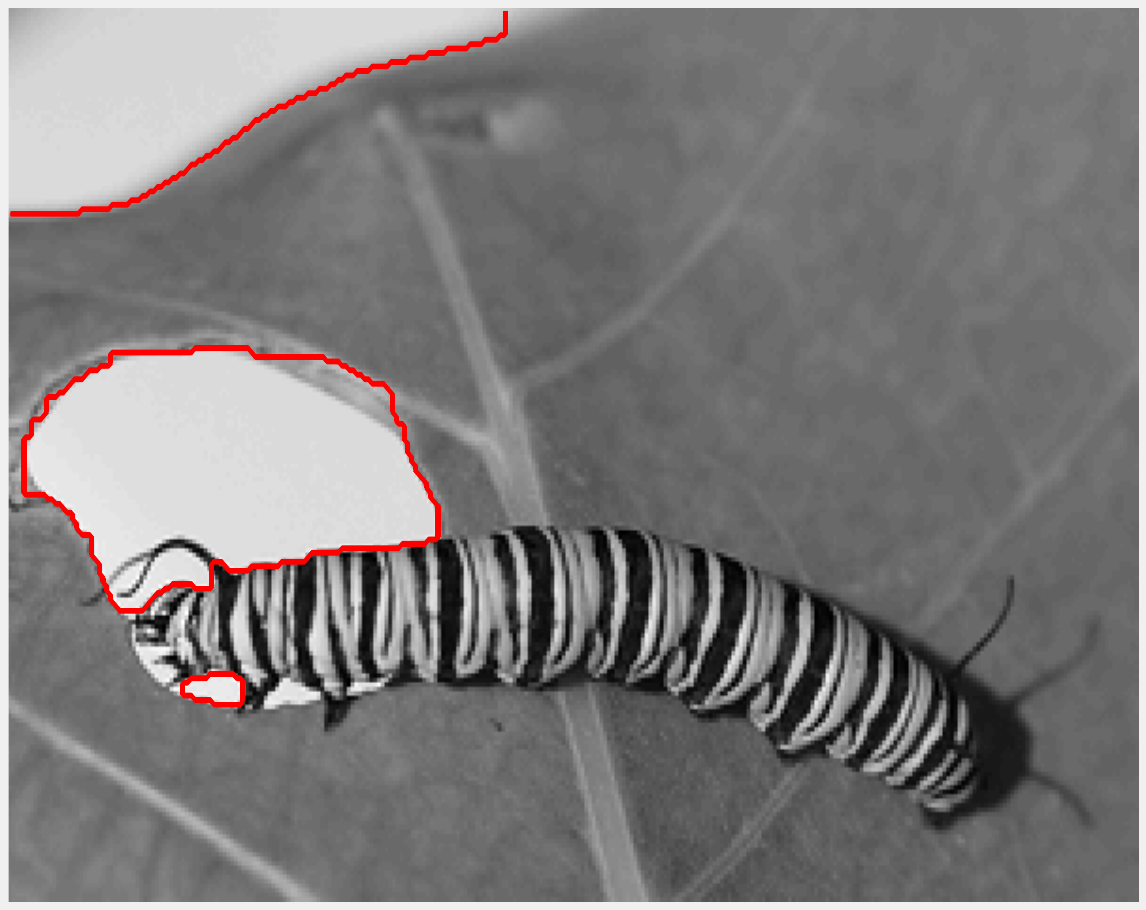}} &\subcaptionbox{\label{fig:iih_caterpiller}}{\includegraphics[width = 1.50in,height=1.00in]{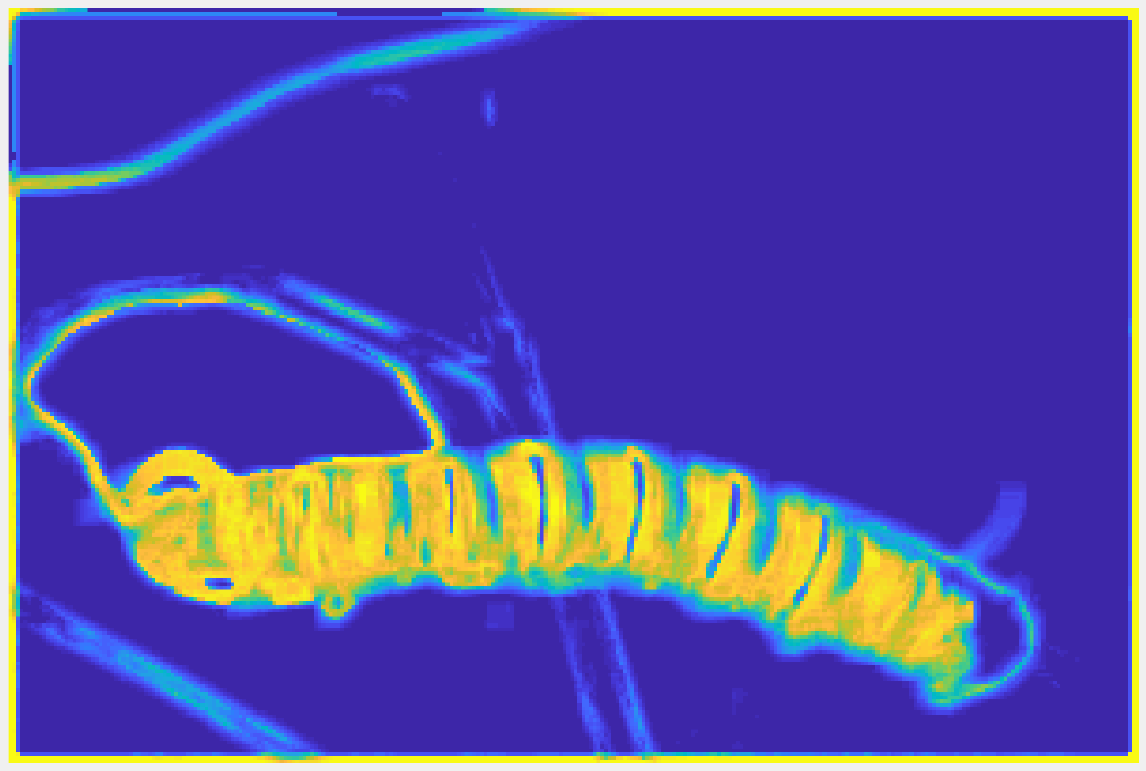}}&\subcaptionbox{\label{fig:aitv_iih_caterpiller}}{\includegraphics[width = 1.50in,height=1.00in]{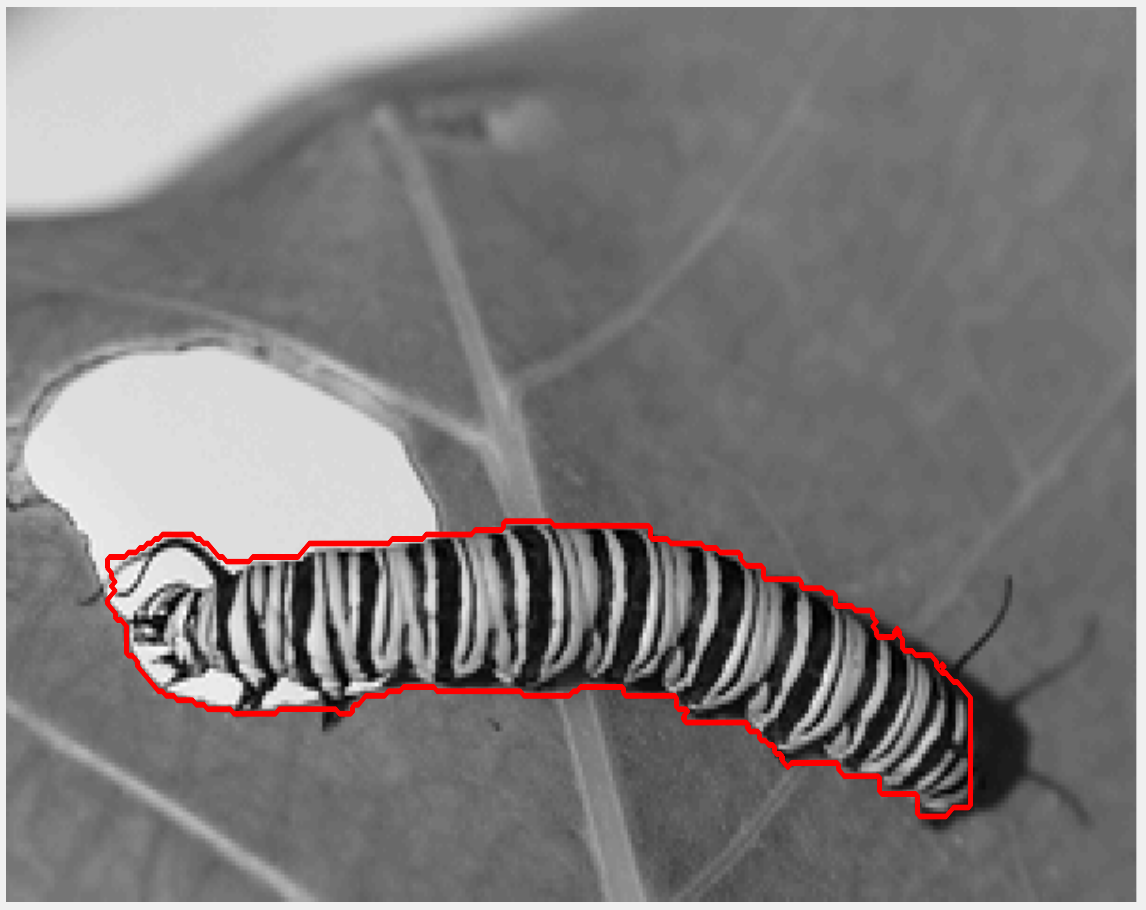}}\\
		 \subcaptionbox{\label{fig:non_iih_egret}}{\includegraphics[width = 1.50in,height=1.00in]{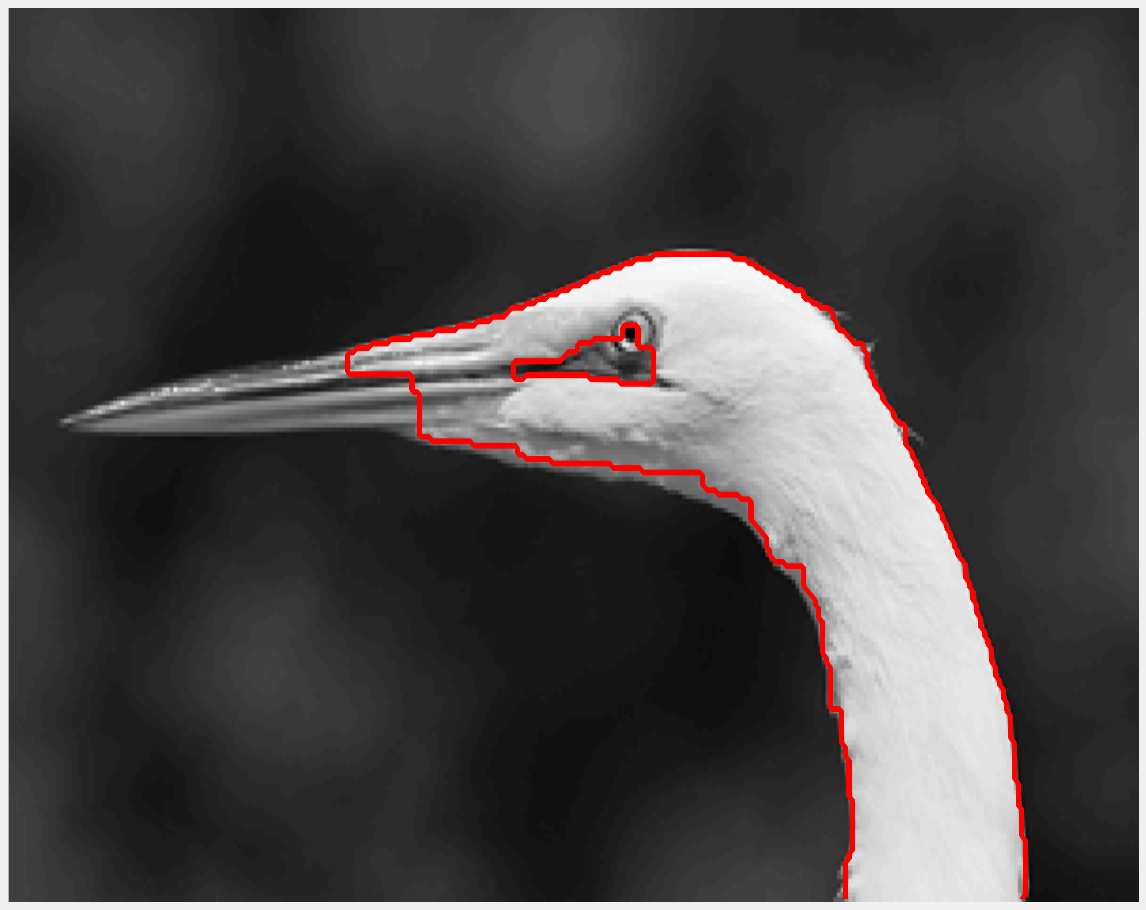}} &\subcaptionbox{\label{fig:iih_egret}}{\includegraphics[width = 1.50in,height=1.00in]{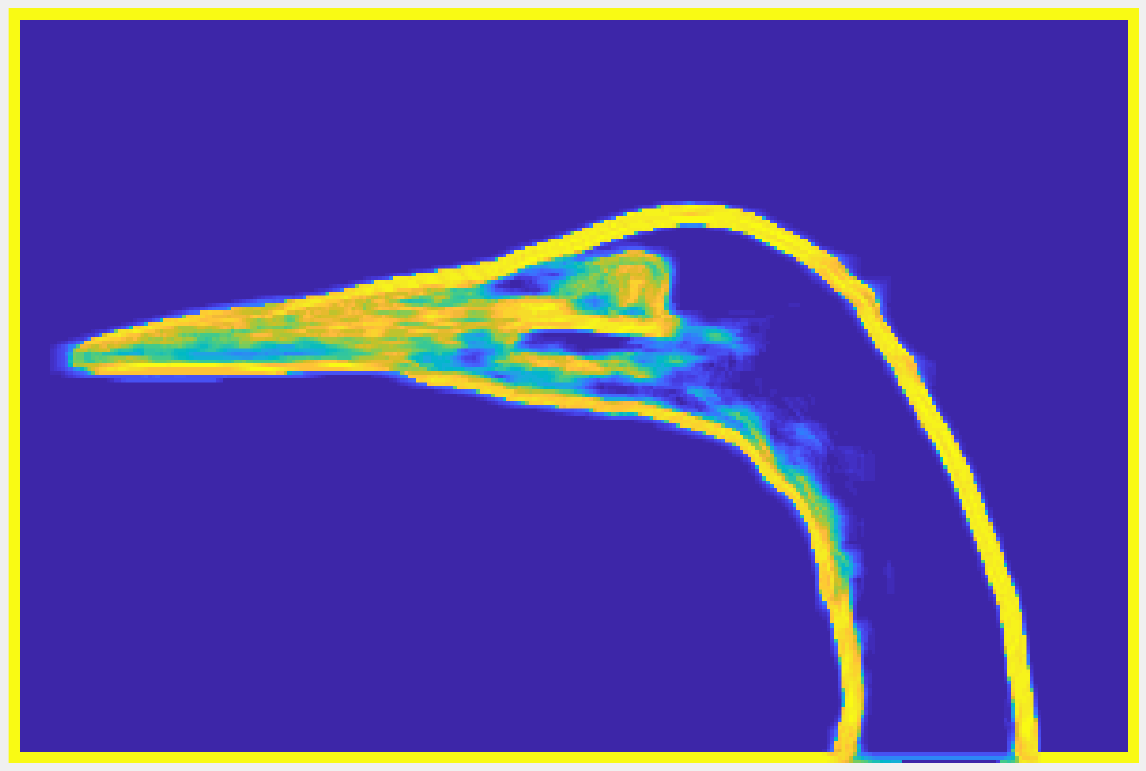}}&\subcaptionbox{\label{fig:aitv_iih_egret}}{\includegraphics[width = 1.50in, height=1.00in]{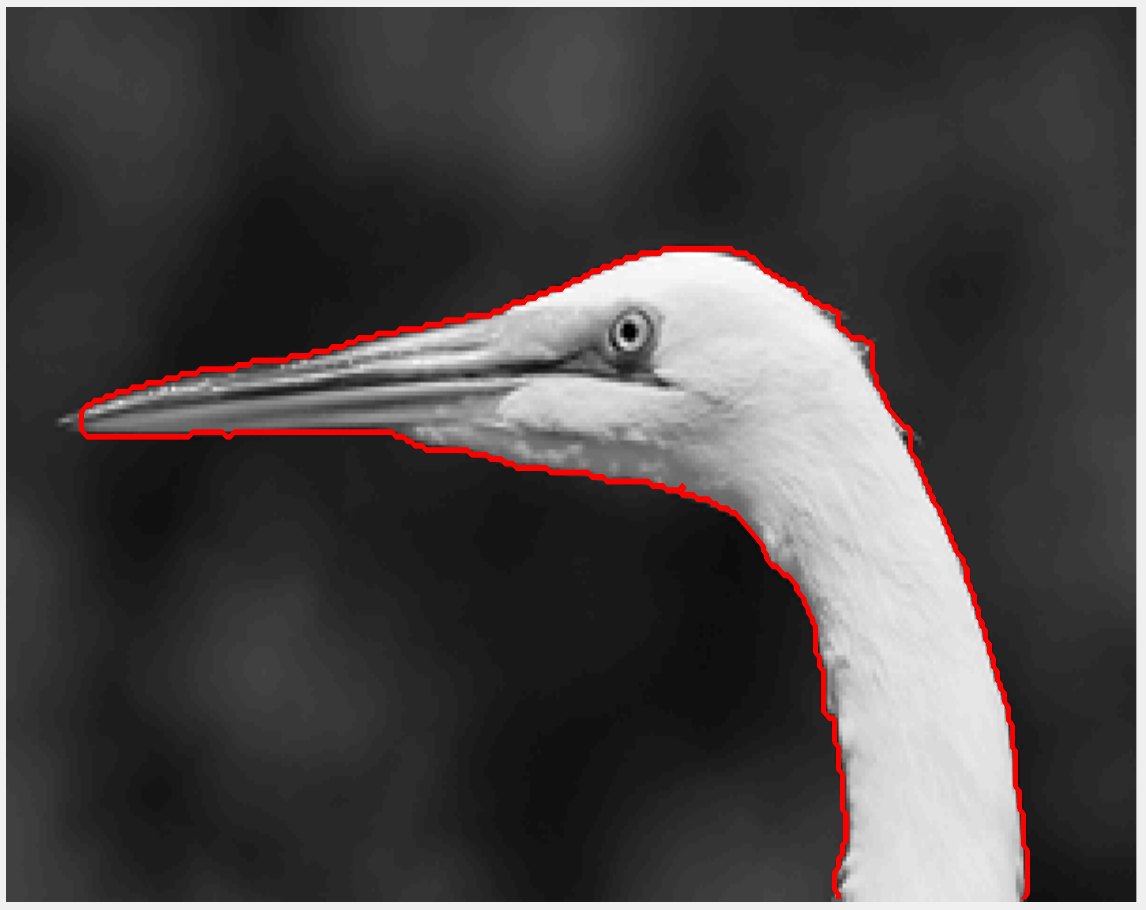}}
	\end{tabular}
	\caption{AITV SaT results on real grayscale images with and without IIH images. Left column: AITV SaT results without   IIH images. Middle column: IIH images.  Right column: AITV SaT results with IIH images.}
	\label{fig:non_iih_grayscale}
\end{figure}

\subsection{Real Grayscale Images with Intensity Inhomogeneities} \label{sec:real_grayscale}

\begin{table}[t!]
\captionof{table}{Comparison of the DICE indices and computational times (seconds) between the segmentation methods applied to Figure \ref{fig:real_grayscale}. Number in \textbf{bold} indicates either the highest DICE index or the fastest time among the segmentation methods for a given image.}
\label{tab:grayscale_real_results}\resizebox{\textwidth}{!}{
\begin{tabular}{l|cc||cc||cc||cc|}
\hhline{~|--------}
                       & \multicolumn{2}{c||}{Figure \ref{fig:caterpiller}}    & \multicolumn{2}{c||}{Figure \ref{fig:egret}}    & \multicolumn{2}{c||}{Figure \ref{fig:swan}}    & \multicolumn{2}{c|}{Figure \ref{fig:leaf}}    \\ \hhline{~|--------} 
                       & \multicolumn{1}{c|}{DICE} & \multicolumn{1}{c||}{Time (s) } & \multicolumn{1}{c|}{DICE} & \multicolumn{1}{c||}{Time (s) }  & \multicolumn{1}{c|}{DICE} & \multicolumn{1}{c||}{Time (s) }  & \multicolumn{1}{c|}{DICE} & \multicolumn{1}{c|}{Time (s) }   \\ \hline
\multicolumn{1}{|l|}{(Original) SaT} & \multicolumn{1}{c|}{0.8818} & \multicolumn{1}{c||}{3.10} & \multicolumn{1}{c|}{0.9677} & \multicolumn{1}{c||}{3.42} & \multicolumn{1}{c|}{0.9191} & \multicolumn{1}{c||}{5.09} & \multicolumn{1}{c|}{0.9288} & \multicolumn{1}{c|}{4.01} \\ \hline
\multicolumn{1}{|l|}{TV$^{p}$ SaT} & \multicolumn{1}{c|}{0.8899} & \multicolumn{1}{c||}{1.97} & \multicolumn{1}{c|}{0.8475} & \multicolumn{1}{c||}{1.89} & \multicolumn{1}{c|}{\textbf{0.9227}} & \multicolumn{1}{c||}{2.35} & \multicolumn{1}{c|}{\textbf{0.9368}} & \multicolumn{1}{c|}{2.06}  \\ \hline
\multicolumn{1}{|l|}{AITV SaT (ADMM)} & \multicolumn{1}{c|}{0.8888} & \multicolumn{1}{c||}{1.76} & \multicolumn{1}{c|}{\textbf{0.9686}} & \multicolumn{1}{c||}{1.45} & \multicolumn{1}{c|}{0.9173} & \multicolumn{1}{c||}{\textbf{2.20}} & \multicolumn{1}{c|}{0.9269} & \multicolumn{1}{c|}{1.53}  \\ \hline
\multicolumn{1}{|l|}{AITV SaT (DCA)} & \multicolumn{1}{c|}{0.8795} & \multicolumn{1}{c||}{19.52} & \multicolumn{1}{c|}{0.8435} & \multicolumn{1}{c||}{12.28} & \multicolumn{1}{c|}{0.9053} & \multicolumn{1}{c||}{25.05} & \multicolumn{1}{c|}{0.9269} & \multicolumn{1}{c|}{18.21}  \\ \hline
\multicolumn{1}{|l|}{AITV CV} & \multicolumn{1}{c|}{0.7568} & \multicolumn{1}{c||}{50.86} & \multicolumn{1}{c|}{0.9423} & \multicolumn{1}{c||}{46.57} & \multicolumn{1}{c|}{0.8913} & \multicolumn{1}{c||}{107.67} & \multicolumn{1}{c|}{0.9141} & \multicolumn{1}{c|}{15.89} \\ \hline
\multicolumn{1}{|l|}{ICTM} & \multicolumn{1}{c|}{0.6230} & \multicolumn{1}{c||}{\textbf{0.25}} & \multicolumn{1}{c|}{0.9516} & \multicolumn{1}{c||}{\textbf{1.39}} & \multicolumn{1}{c|}{0.8688} & \multicolumn{1}{c||}{2.90} & \multicolumn{1}{c|}{0.9129} & \multicolumn{1}{c|}{\textbf{0.16}} \\ \hline
\multicolumn{1}{|l|}{TV$^p$ MS} & \multicolumn{1}{c|}{0.6782} & \multicolumn{1}{c||}{8.33} & \multicolumn{1}{c|}{0.9346} & \multicolumn{1}{c||}{8.50} & \multicolumn{1}{c|}{0.7846} & \multicolumn{1}{c||}{10.34} & \multicolumn{1}{c|}{0.9179} & \multicolumn{1}{c|}{6.24} \\ \hline
\multicolumn{1}{|l|}{Convex Potts} & \multicolumn{1}{c|}{\textbf{0.8902}} & \multicolumn{1}{c||}{2.23} & \multicolumn{1}{c|}{0.5257} & \multicolumn{1}{c||}{2.27} & \multicolumn{1}{c|}{0.8131} & \multicolumn{1}{c||}{4.69} & \multicolumn{1}{c|}{0.9173} & \multicolumn{1}{c|}{1.96} \\ \hline
\multicolumn{1}{|l|}{SaT-Potts} & \multicolumn{1}{c|}{0.8769} & \multicolumn{1}{c||}{2.10} & \multicolumn{1}{c|}{0.9613} & \multicolumn{1}{c||}{2.08} & \multicolumn{1}{c|}{0.9165} & \multicolumn{1}{c||}{2.53} & \multicolumn{1}{c|}{0.9120} & \multicolumn{1}{c|}{2.37} \\ \hline
\end{tabular}}
\end{table}

\begin{figure*}[t!]
\resizebox{\textwidth}{!}{
\begin{tabular}{cccccc}
		\stepcounter{figure}
        \setcounter{caption@flags}{4}
        \setcounter{subfigure}{0}
\captionsetup[subfigure]{justification=centering}
\subcaptionbox{Noisy and blurry.}{\includegraphics[width = 1.50in]{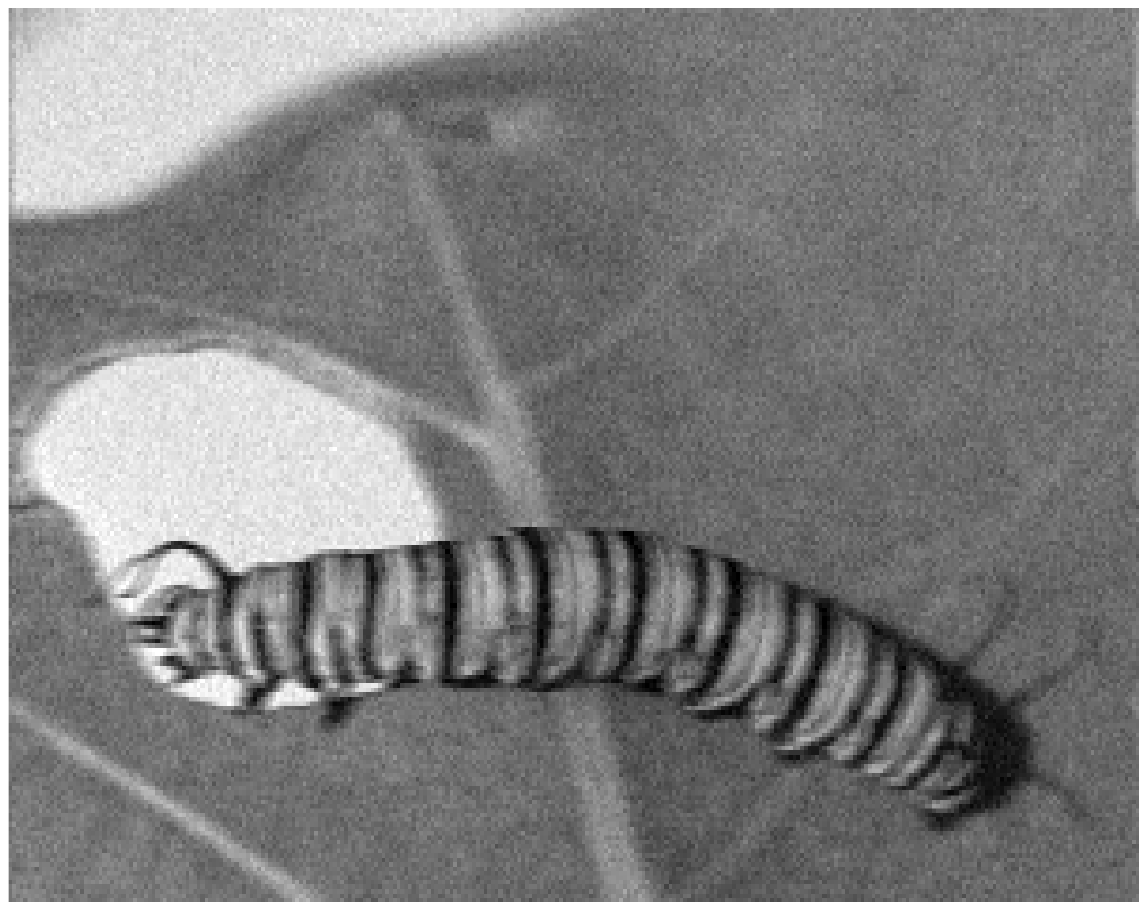}} & \captionsetup[subfigure]{justification=centering}
\subcaptionbox{Ground truth.}{\includegraphics[width = 1.50in]{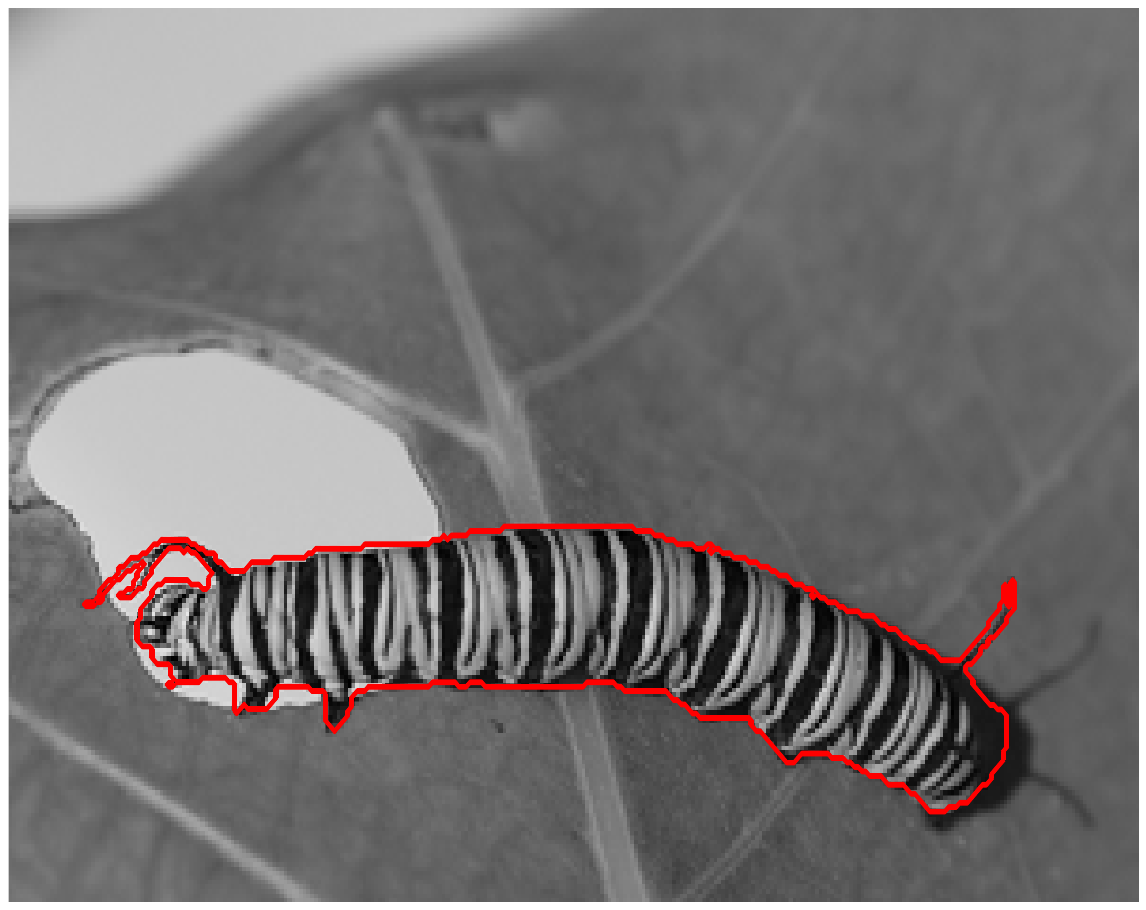}} &
\captionsetup[subfigure]{justification=centering}
\subcaptionbox{(original) SaT}{\includegraphics[width = 1.50in]{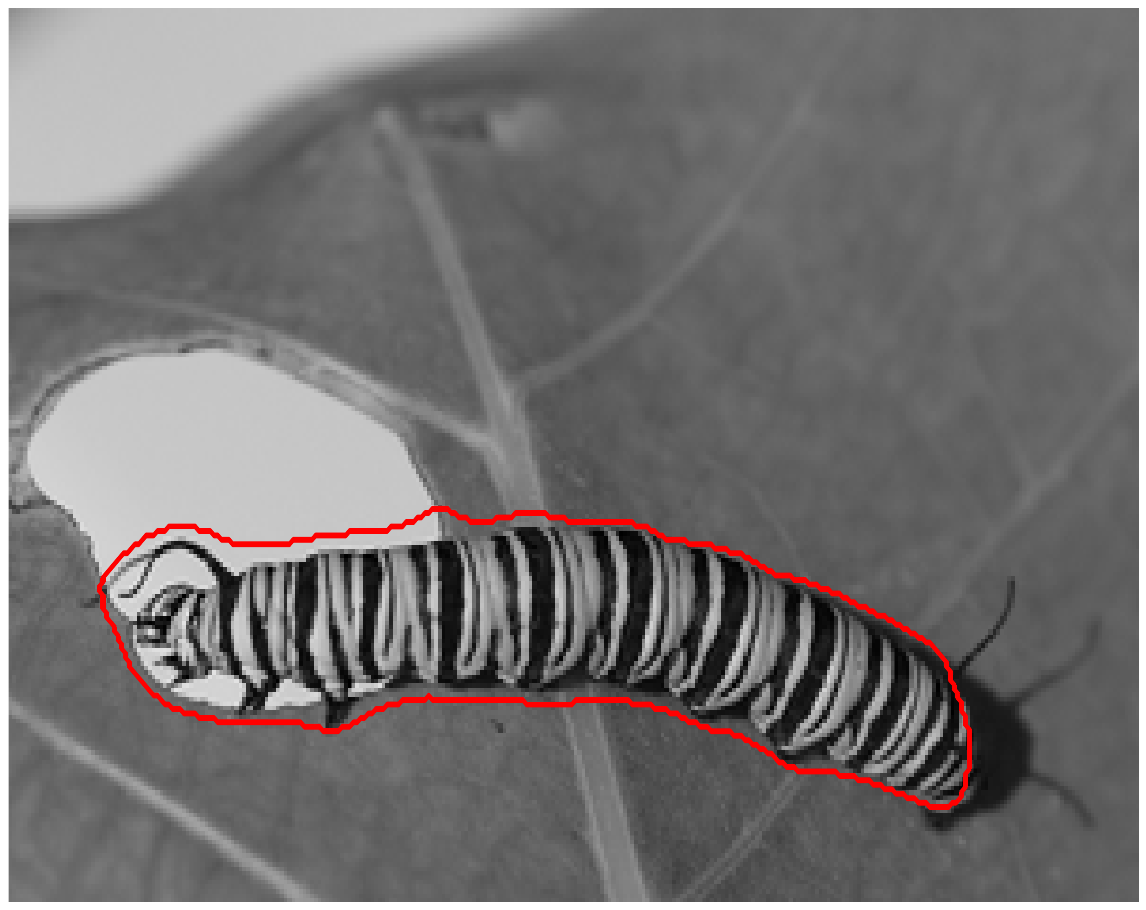}} & 		\captionsetup[subfigure]{justification=centering}
\subcaptionbox{TV$^{p}$ SaT}{\includegraphics[width = 1.50in]{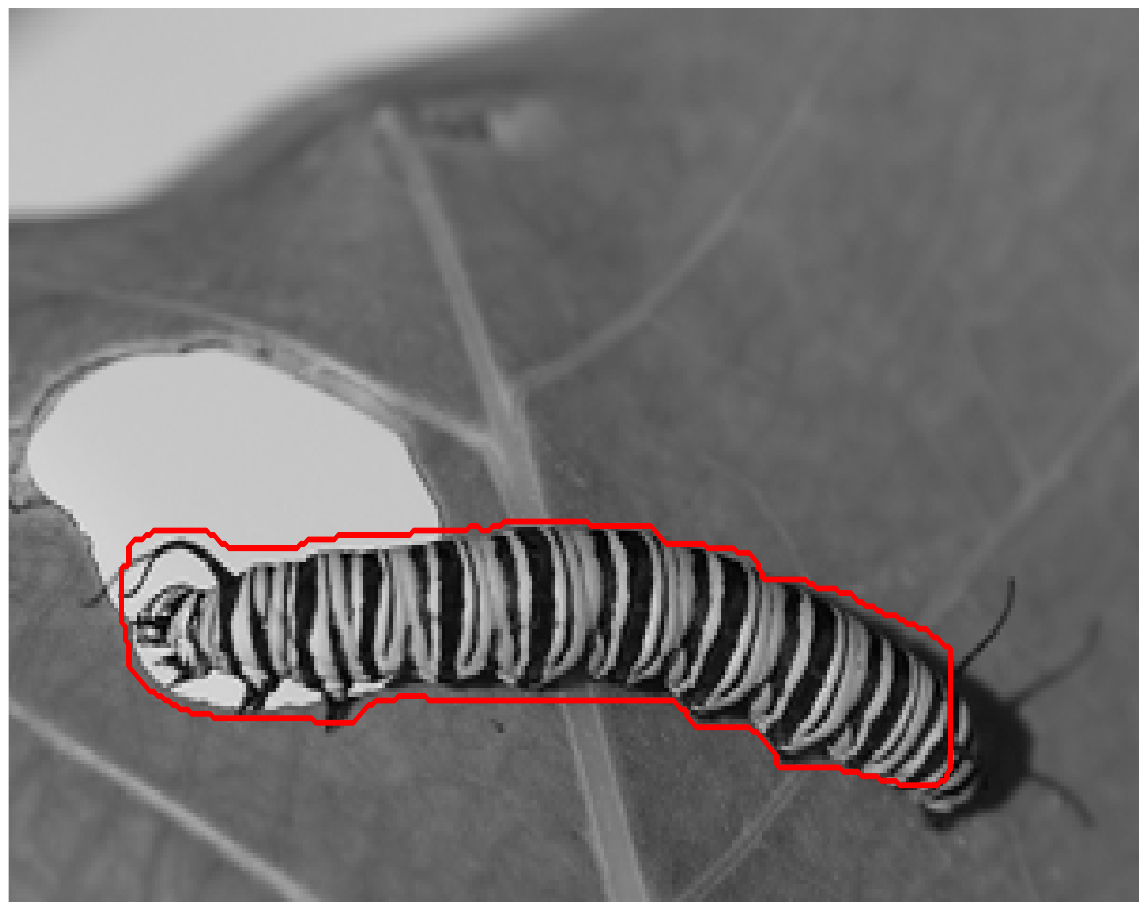}} &		\captionsetup[subfigure]{justification=centering}
\subcaptionbox{AITV   SaT (ADMM)}{\includegraphics[width = 1.50in]{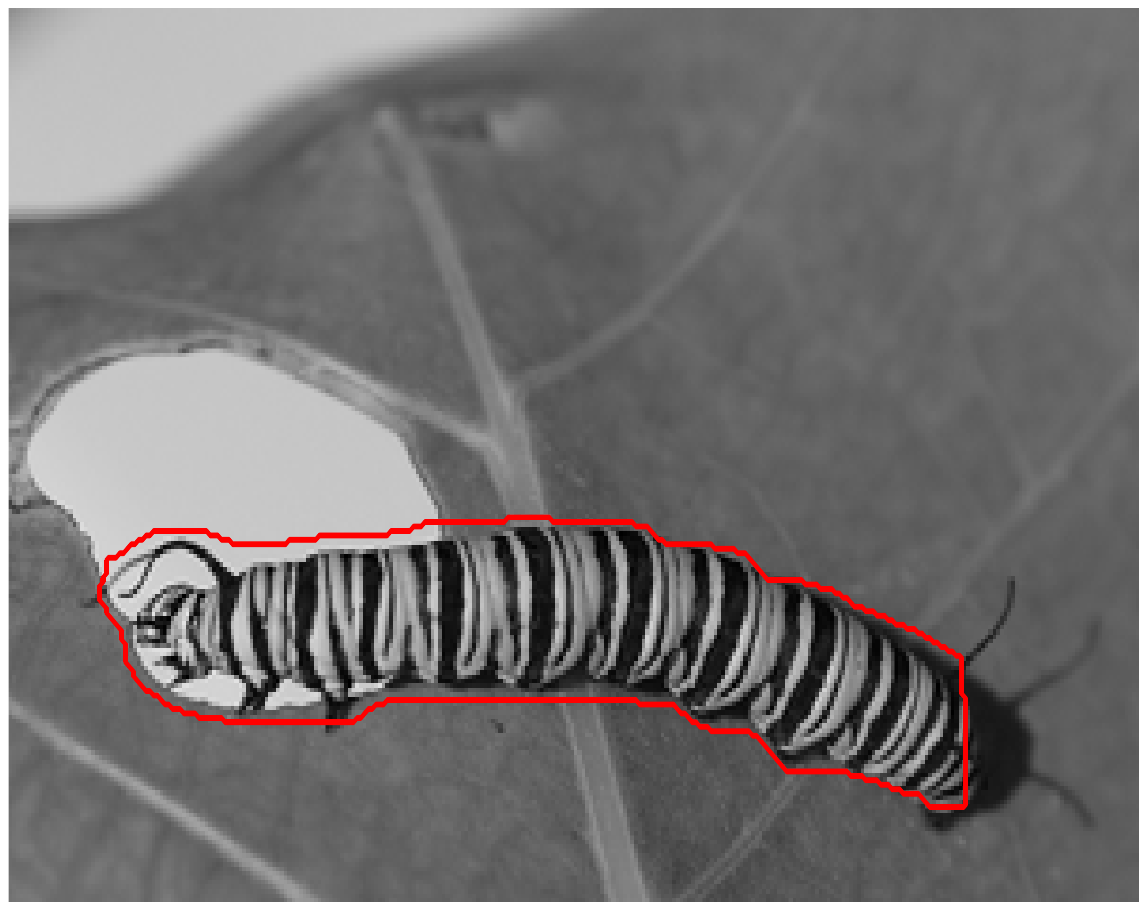}} 
& \captionsetup[subfigure]{justification=centering}
\subcaptionbox{AITV   SaT (DCA)}{\includegraphics[width = 1.50in]{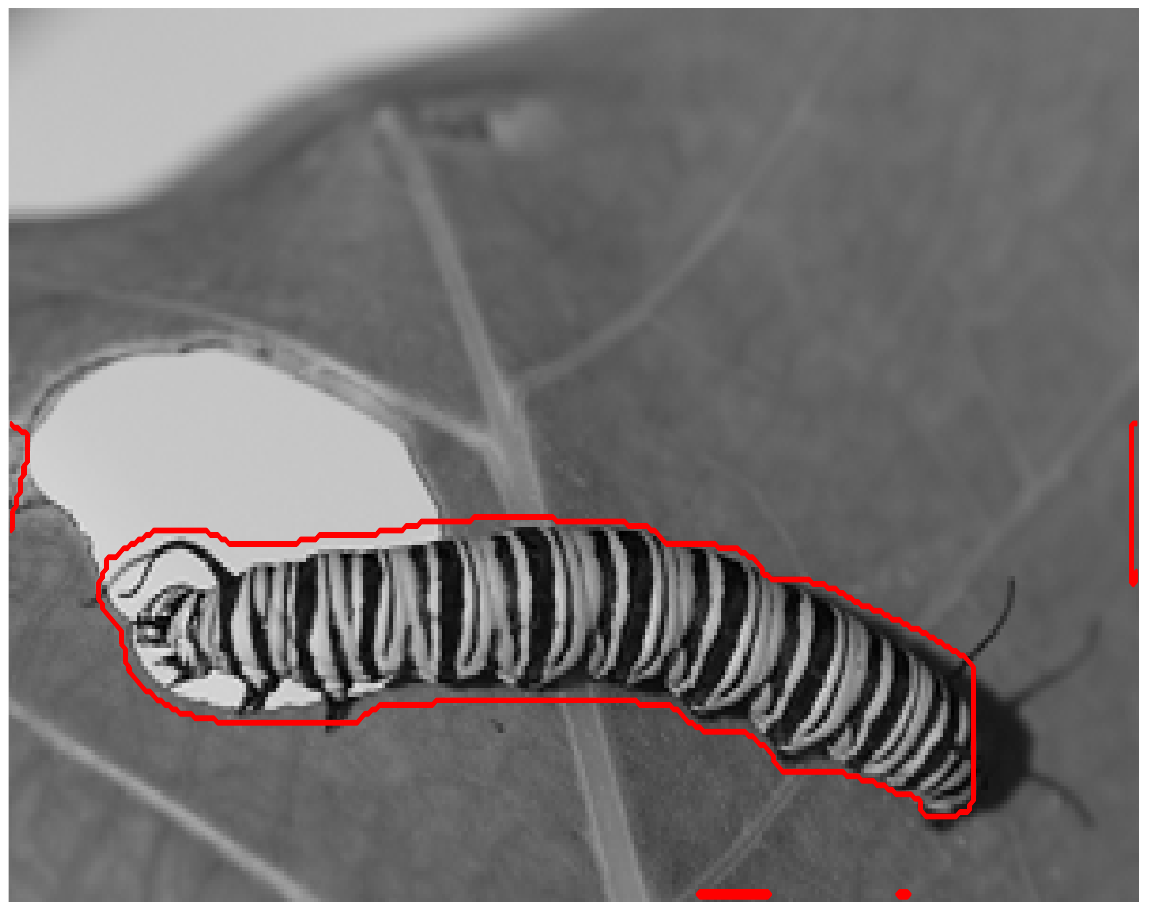}} \\
&
		\captionsetup[subfigure]{justification=centering}
\subcaptionbox{AITV CV}{\includegraphics[width = 1.50in]{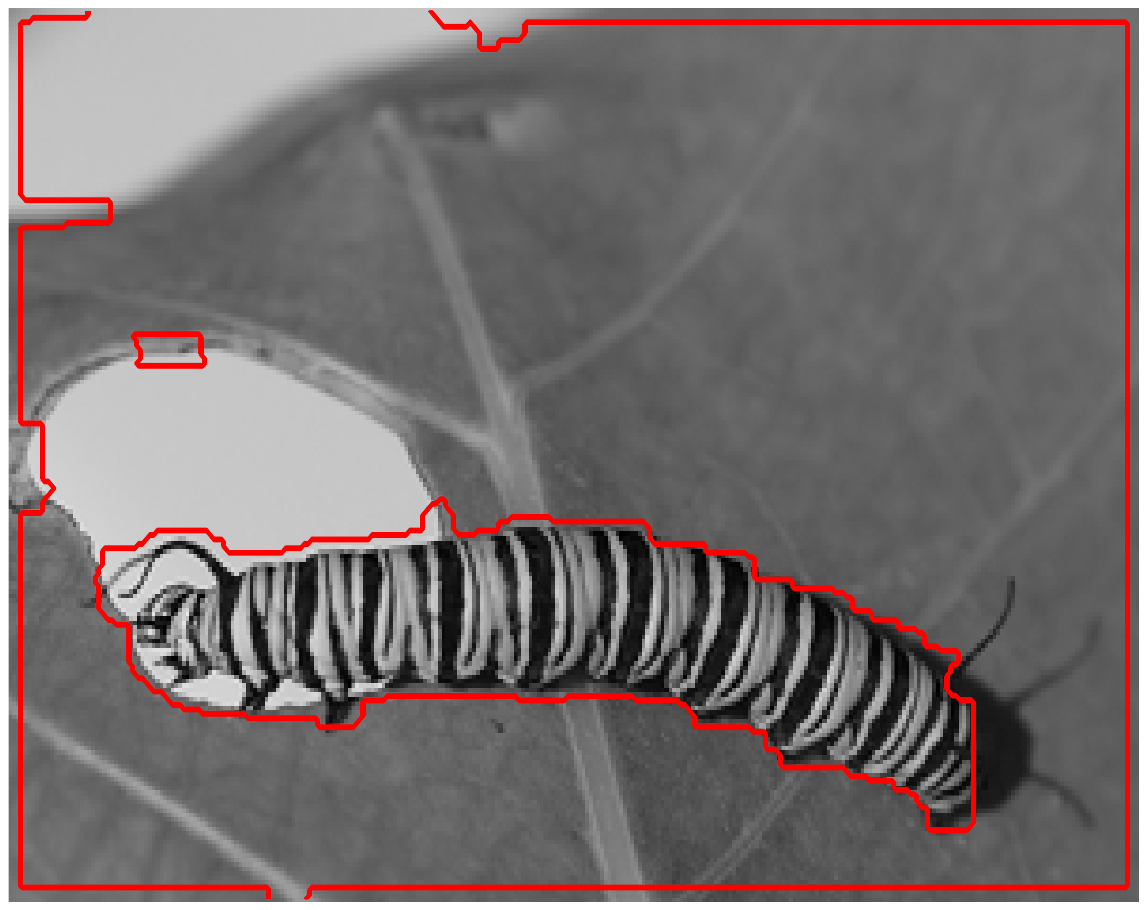}}&
		\captionsetup[subfigure]{justification=centering}
\subcaptionbox{TV$^p$ MS}{\includegraphics[width = 1.50in]{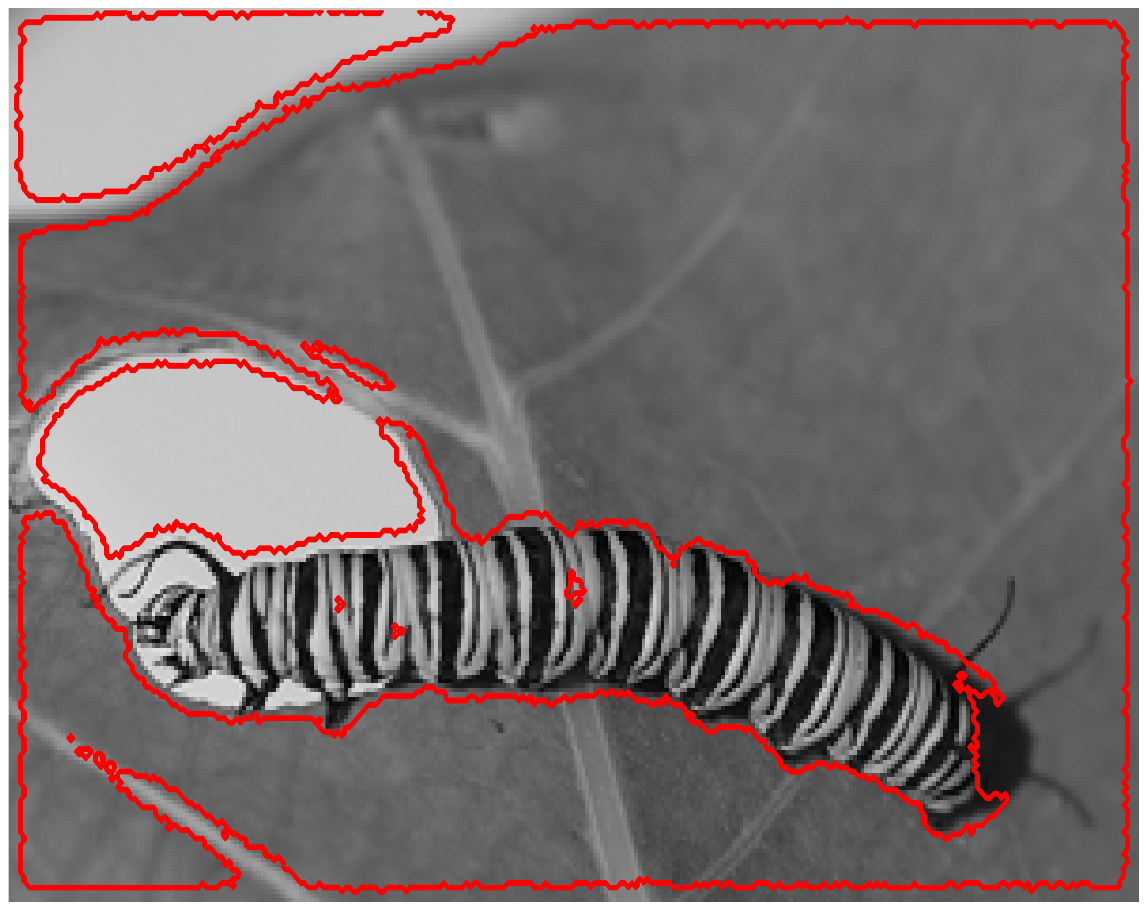}} &
		\captionsetup[subfigure]{justification=centering}
\subcaptionbox{ICTM}{\includegraphics[width = 1.50in]{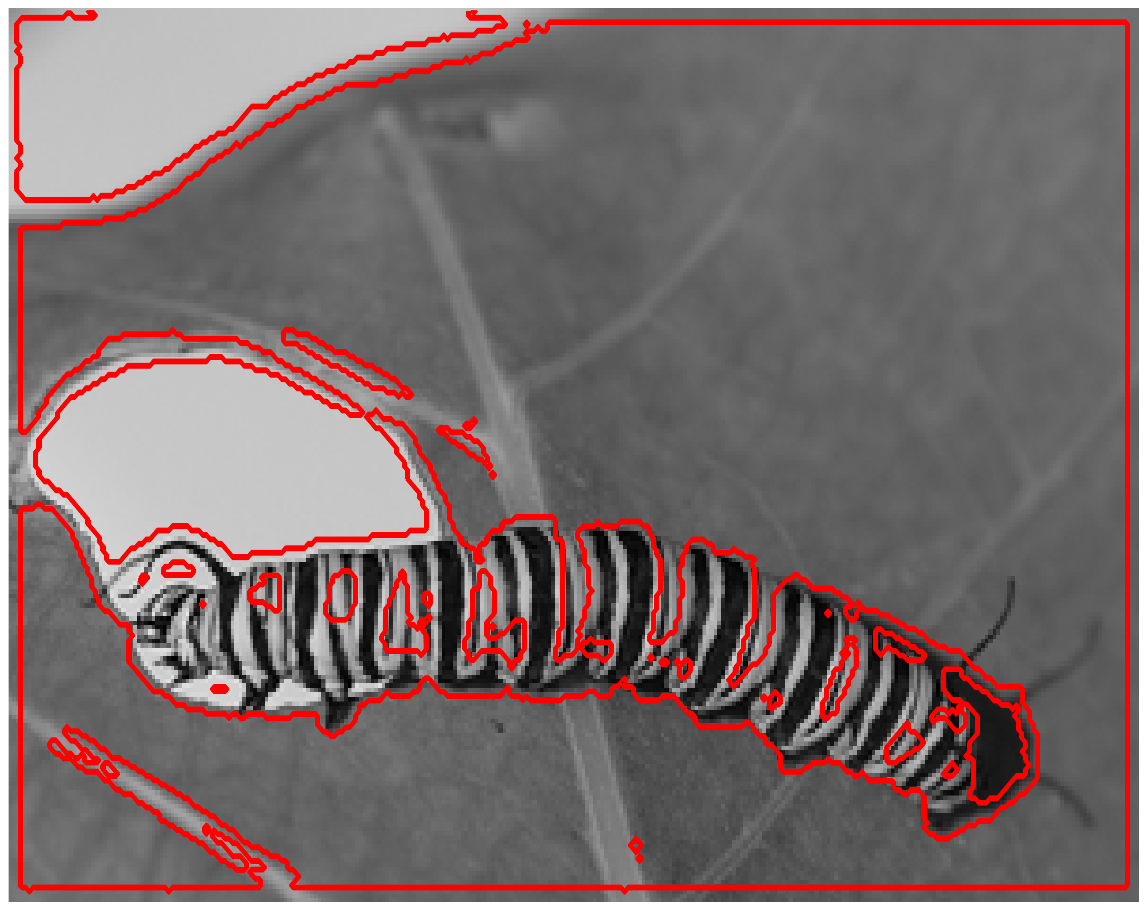}}&		\captionsetup[subfigure]{justification=centering}
\subcaptionbox{Convex Potts}{\includegraphics[width = 1.50in]{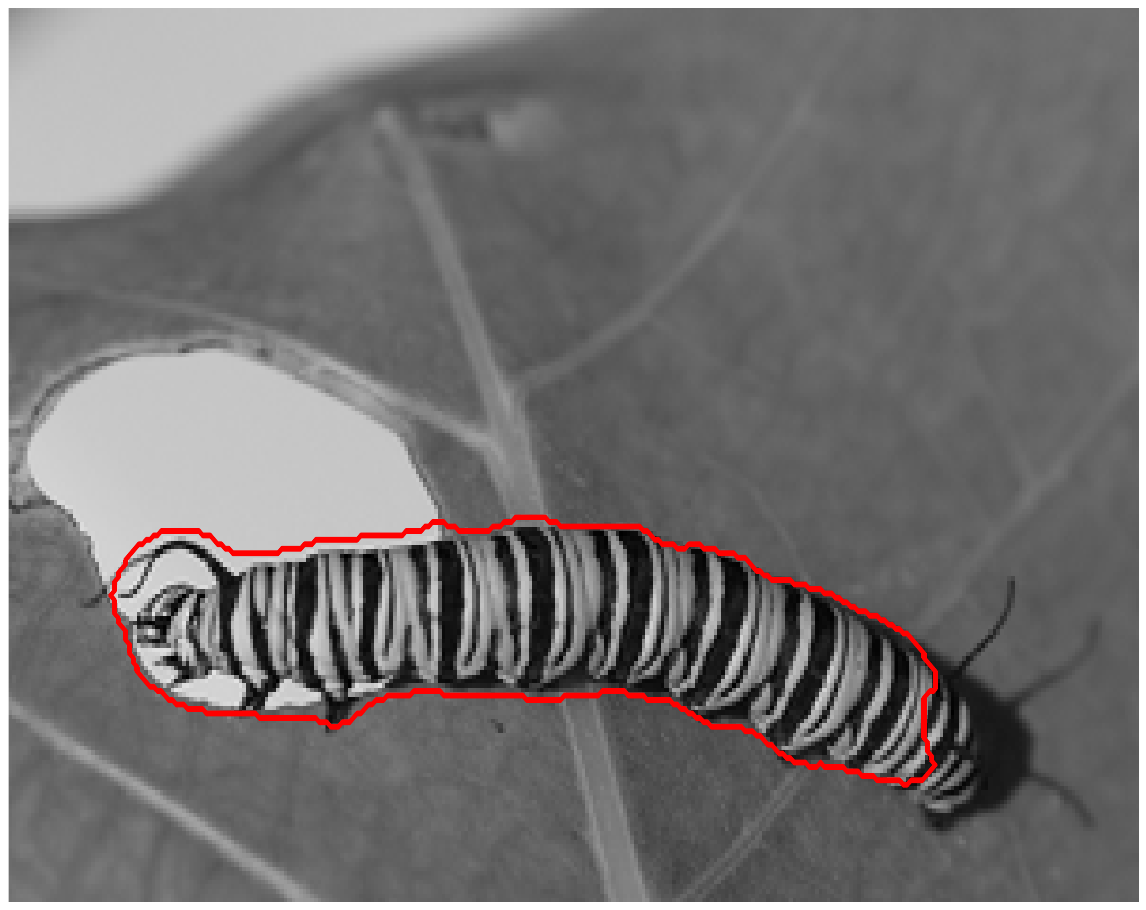}} &		\captionsetup[subfigure]{justification=centering}

\subcaptionbox{SaT-Potts}{\includegraphics[width = 1.50in]{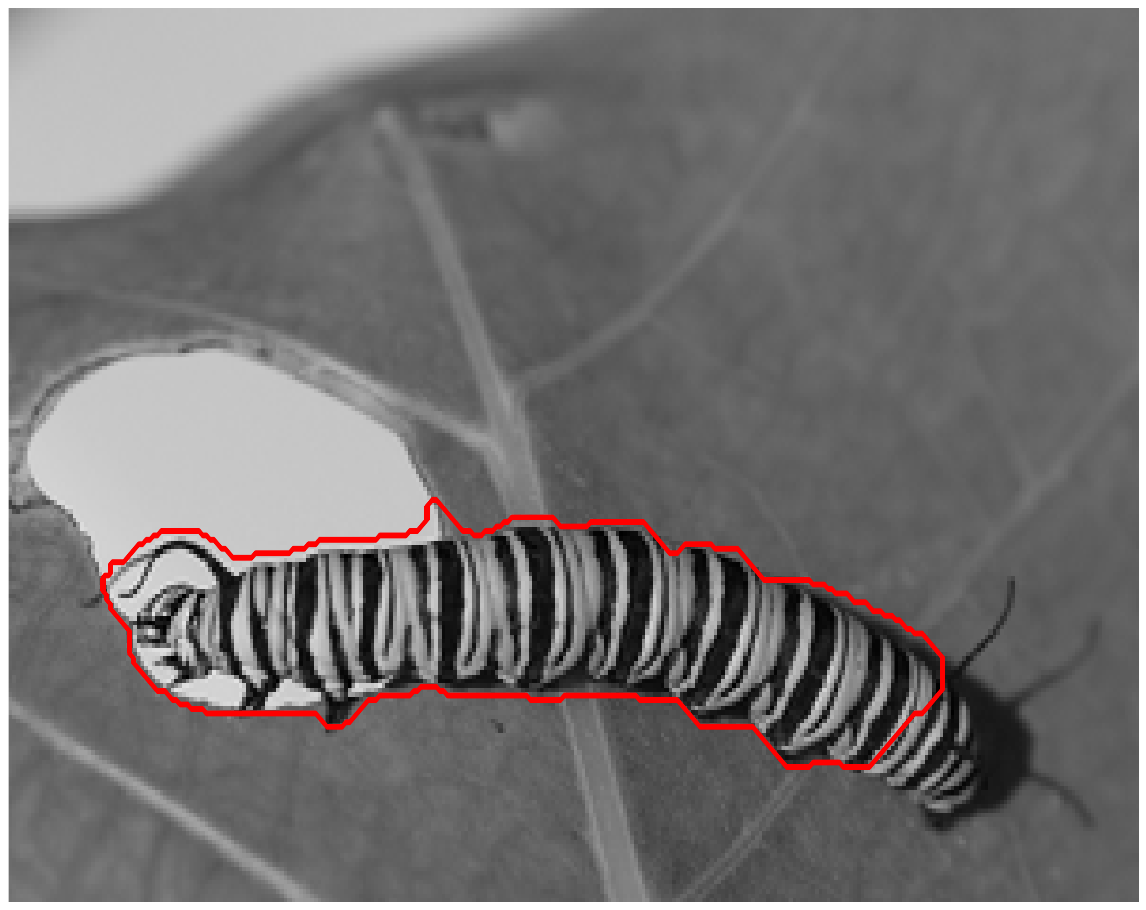}} 
		\end{tabular}}
  		\caption{Segmentation results of Figures {\ref{fig:caterpiller}} corrupted by motion blur followed by Gaussian noise. }
		\label{fig:caterpillar_result}
\end{figure*}

We examine real images with intensity inhomogeneities \cite{AlpertGBB07}, as shown in Figure \ref{fig:real_grayscale}. Intensity inhomogeneities can be problematic for image segmentation because of the dramatically varying pixel intensities in the local regions of an image. For example, we apply AITV SaT (ADMM) to Figures \ref{fig:caterpiller}-\ref{fig:egret} to exemplify the challenges of segmenting the object of interest. In Figure \ref{fig:non_iih_caterpiller}, no part of the caterpillar is segmented while in Figure \ref{fig:non_iih_egret}, most of the egret's beak is not segmented.However, by incorporating the intensity inhomogeneity (IIH) images {\cite{li2020three}} shown in Figures {\ref{fig:iih_caterpiller}},{\ref{fig:iih_egret}} as additional channels, AITV SaT accounts for intensity inhomogeneity and is able to segment the caterpillar in Figure {\ref{fig:aitv_iih_caterpiller}} and the egret and its beak in Figure {\ref{fig:aitv_iih_egret}}.

Following the work of \cite{li2020three}, we incorporate an IIH image by appending it as an additional channel to the original image to facilitate segmentation.
To generate the IIH image, one calculates an IIH-indicator $D$
{\small
\begin{align*}
    D= \frac{1}{|\Omega|} \sum_{(i,j) \in \Omega } \left( \frac{1}{|\Omega_{(i,j)}|} \sum_{(i',j') \in \Omega_{(i,j)}} |u_{i',j'} - \bar{u}_{i,j}|^2 \right),
\end{align*}}%
where $\Omega_{(i,j)}$ is a neighborhood  centered at pixel $(i,j)$ and $\bar{u}_{i,j}$ is the average pixel intensity in the neighborhood $\Omega_{(i,j)}$. Using the IIH-indicator $D$, the IIH-image is calculated by
\begin{align*}
    u^{\text{IIH}}_{i,j} = \frac{1}{|\Omega_{(i,j)}|} \sum_{(i',j') \in \Omega_{(i,j)}} \mathbbm{1}_{\Omega_{(i,j)}}(i',j'),
\end{align*}
where
\begin{align*}
     \mathbbm{1}_{\Omega_{(i,j)}}(i',j') = \begin{cases}
    1 &\text{ if } |\bar{u}_{i,j} - u_{i',j'}|^2 \geq D,\\
    0 &\text{ if } |\bar{u}_{i,j} - u_{i',j'}|^2 < D.
     \end{cases}
\end{align*}
For our experiments, $\Omega_{(i,j)}$ is a $7 \times 7$ patch centered at pixel $(i,j)$.

\begin{figure*}[t!]
\resizebox{\textwidth}{!}{
\begin{tabular}{cccccc}
		\stepcounter{figure}
        \setcounter{caption@flags}{4}
        \setcounter{subfigure}{0}
\captionsetup[subfigure]{justification=centering}
\subcaptionbox{Noisy and blurry.}{\includegraphics[width = 1.50in]{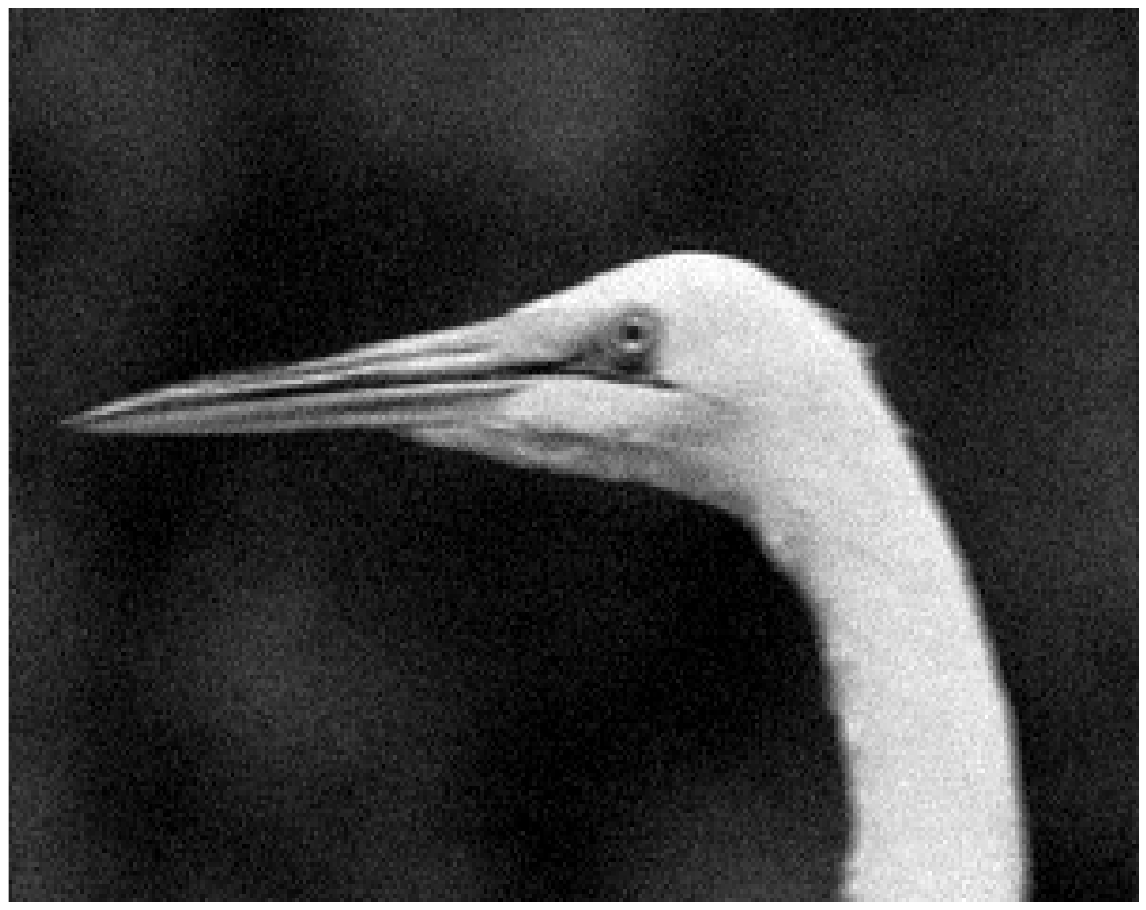}} & \captionsetup[subfigure]{justification=centering}
\subcaptionbox{Ground truth.}{\includegraphics[width = 1.50in]{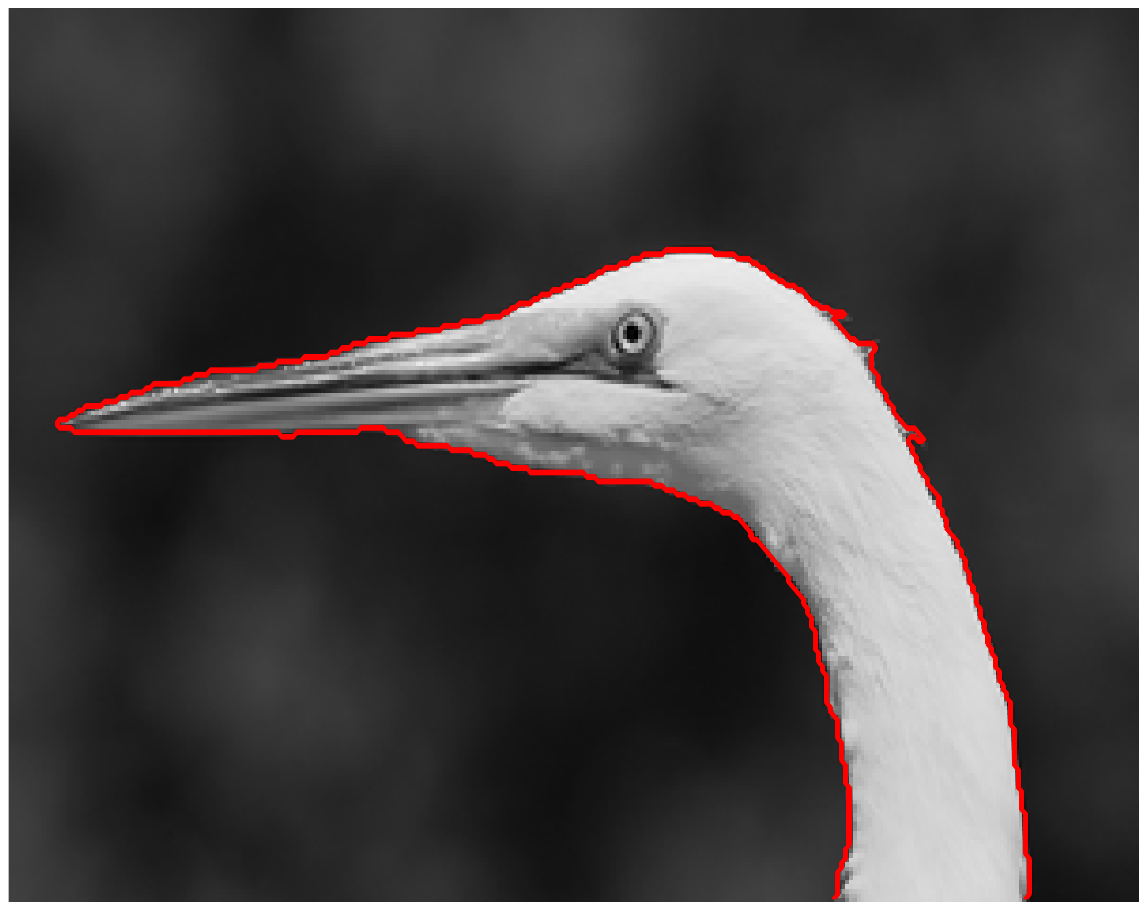}} &
\captionsetup[subfigure]{justification=centering}
\subcaptionbox{(original) SaT}{\includegraphics[width = 1.50in]{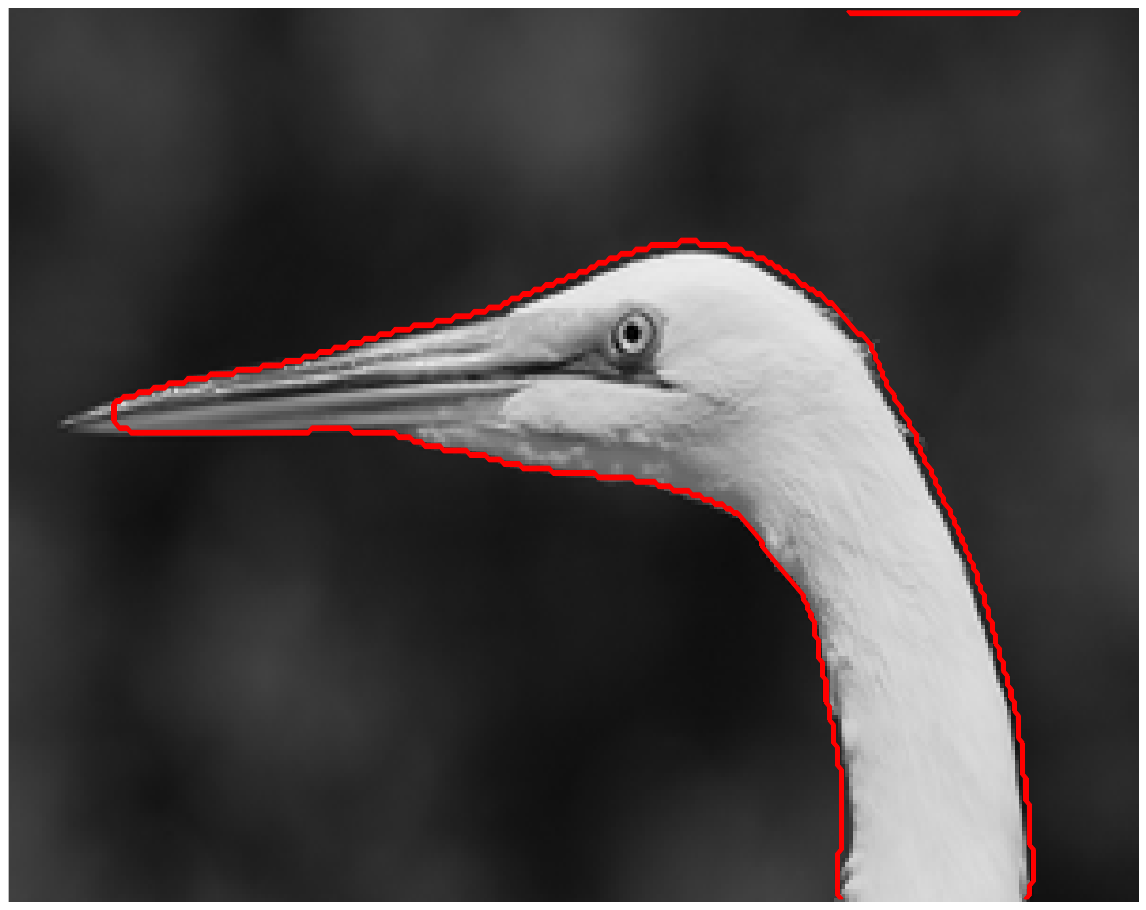}} & 		\captionsetup[subfigure]{justification=centering}
\subcaptionbox{TV$^{p}$ SaT}{\includegraphics[width = 1.50in]{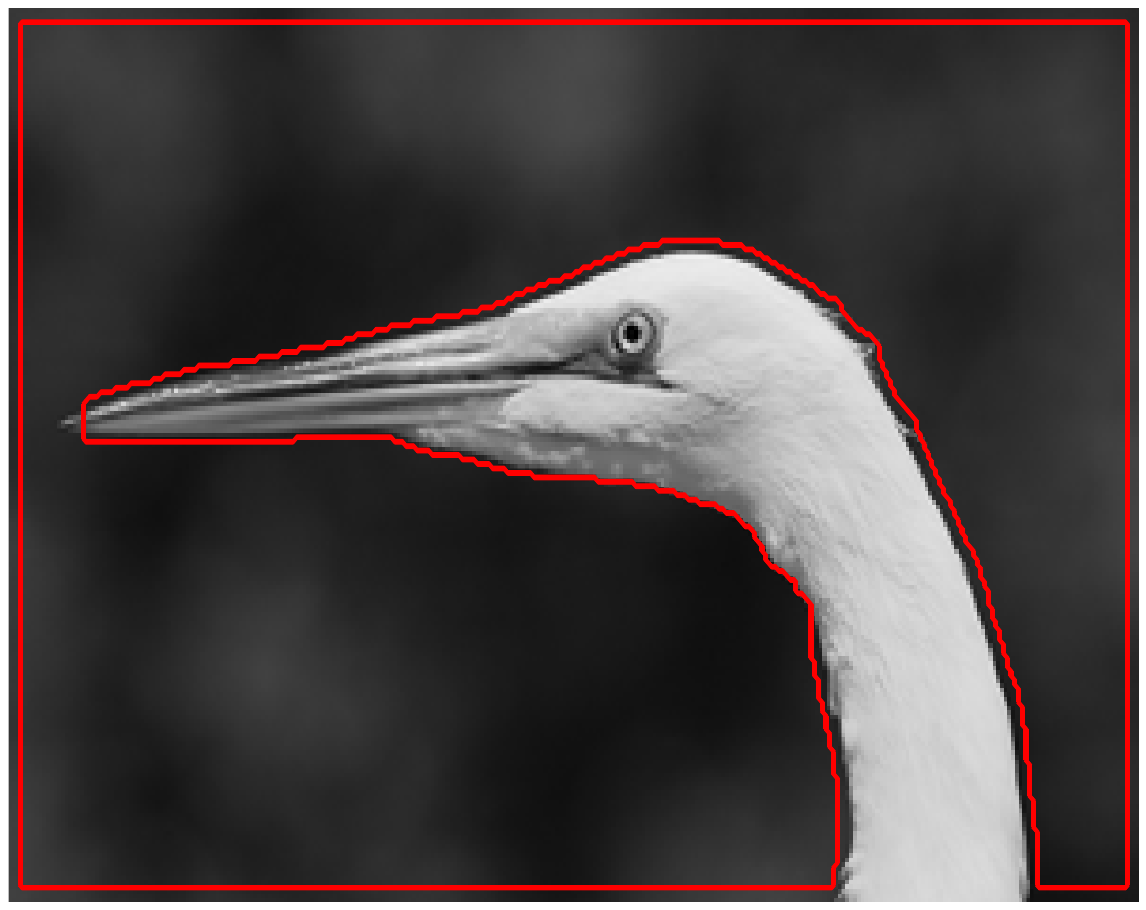}} &		\captionsetup[subfigure]{justification=centering}
\subcaptionbox{AITV   SaT (ADMM)}{\includegraphics[width = 1.50in]{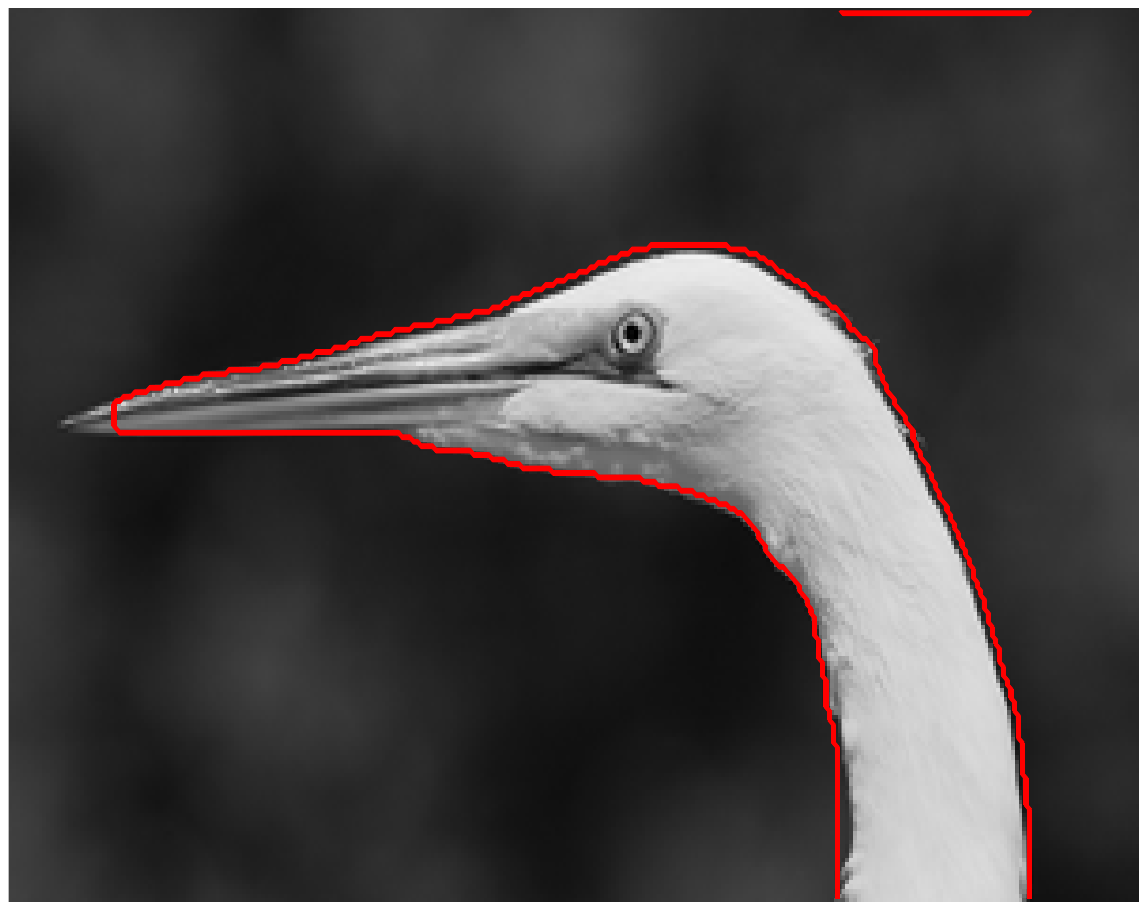}} 
& \captionsetup[subfigure]{justification=centering}
\subcaptionbox{AITV   SaT (DCA)}{\includegraphics[width = 1.50in]{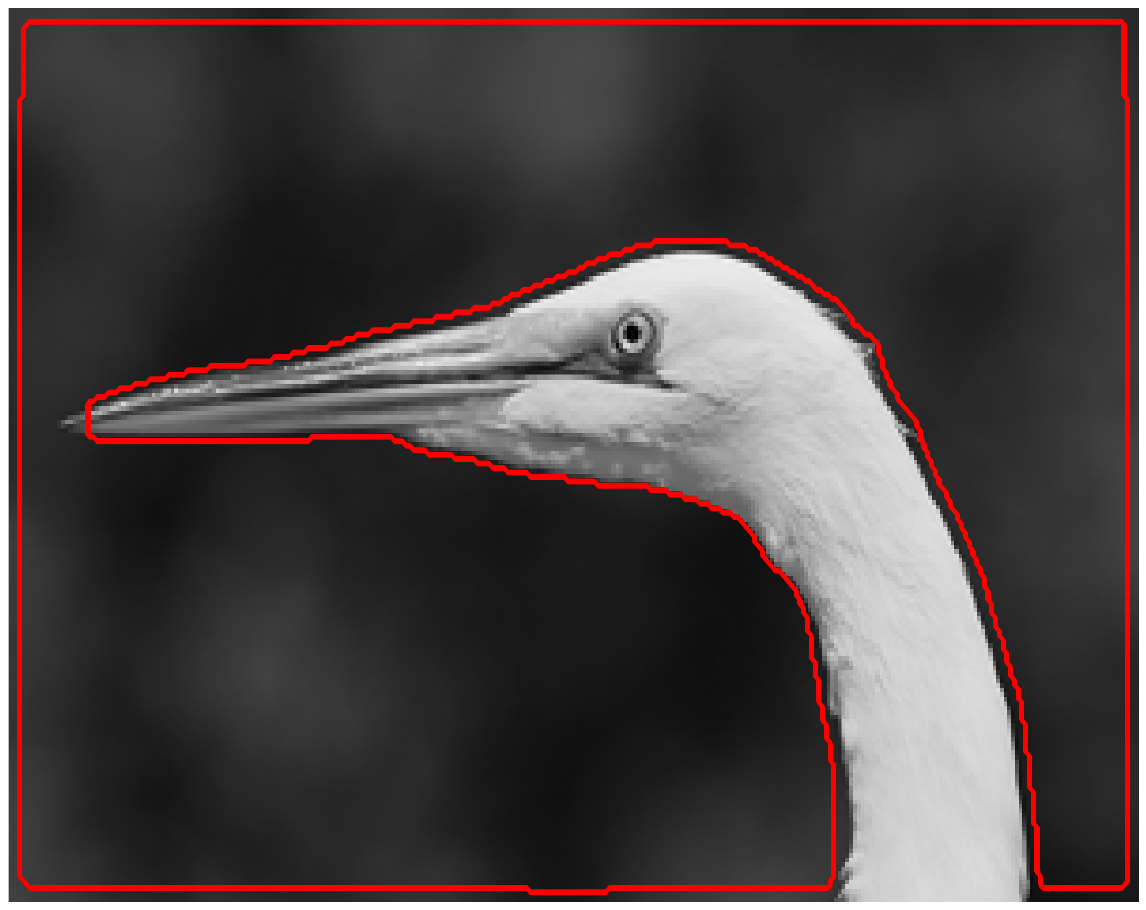}} \\
&
		\captionsetup[subfigure]{justification=centering}
\subcaptionbox{AITV CV}{\includegraphics[width = 1.50in]{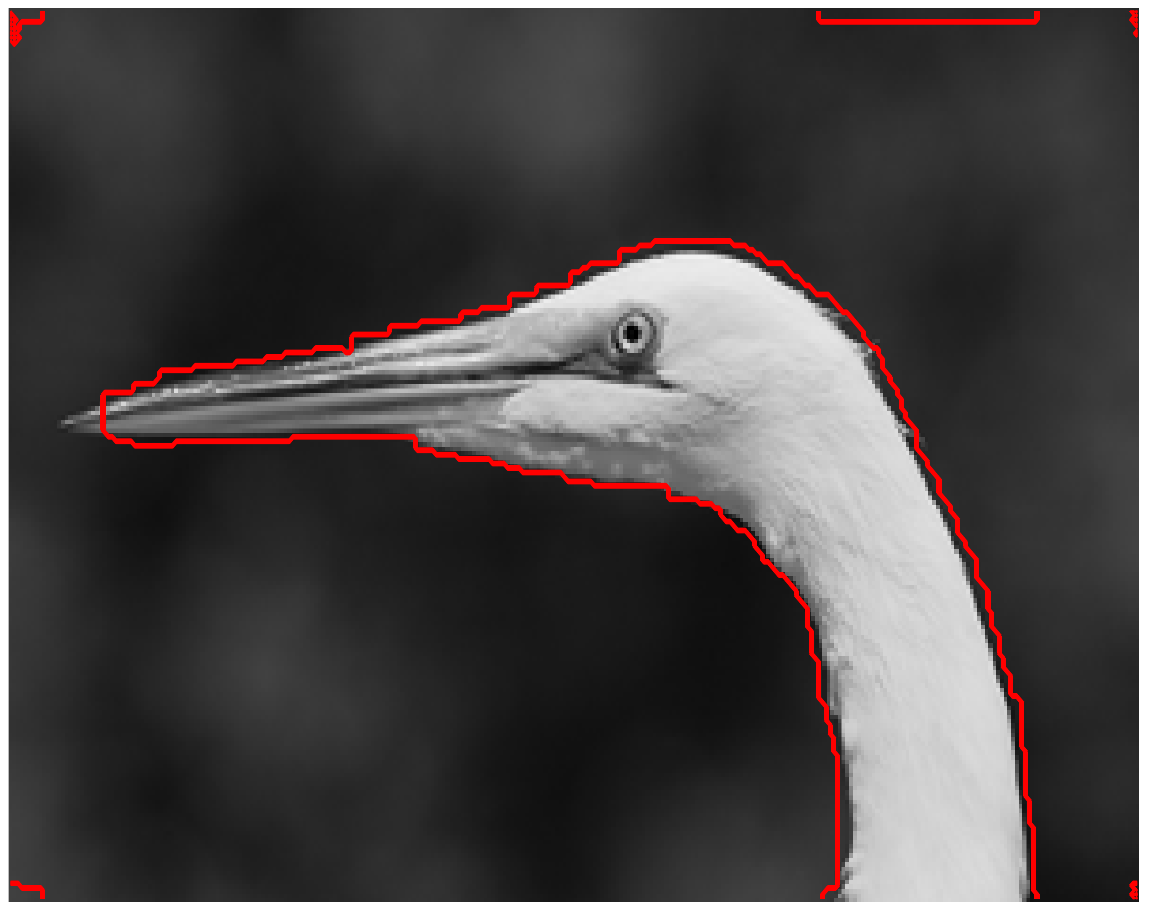}}&
		\captionsetup[subfigure]{justification=centering}
\subcaptionbox{TV$^p$ MS}{\includegraphics[width = 1.50in]{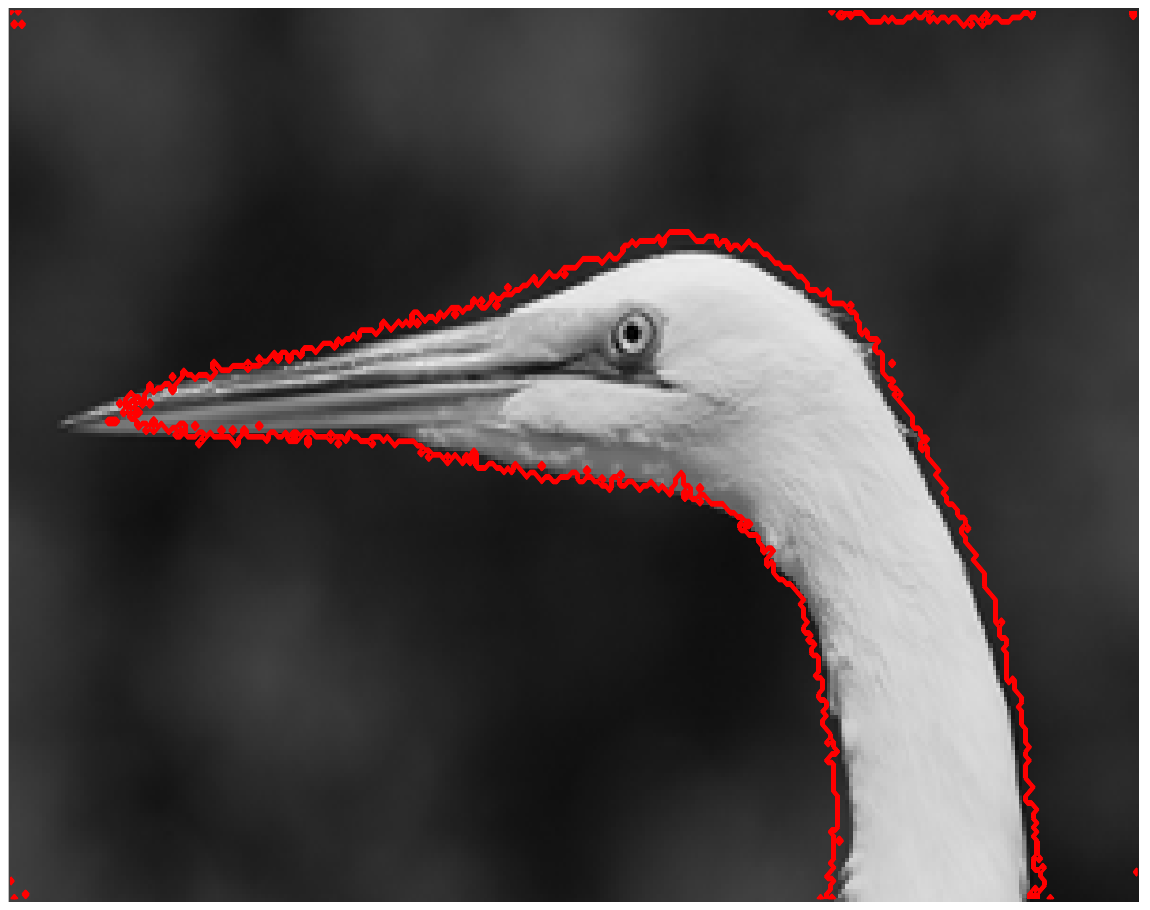}} &
		\captionsetup[subfigure]{justification=centering}
\subcaptionbox{ICTM}{\includegraphics[width = 1.50in]{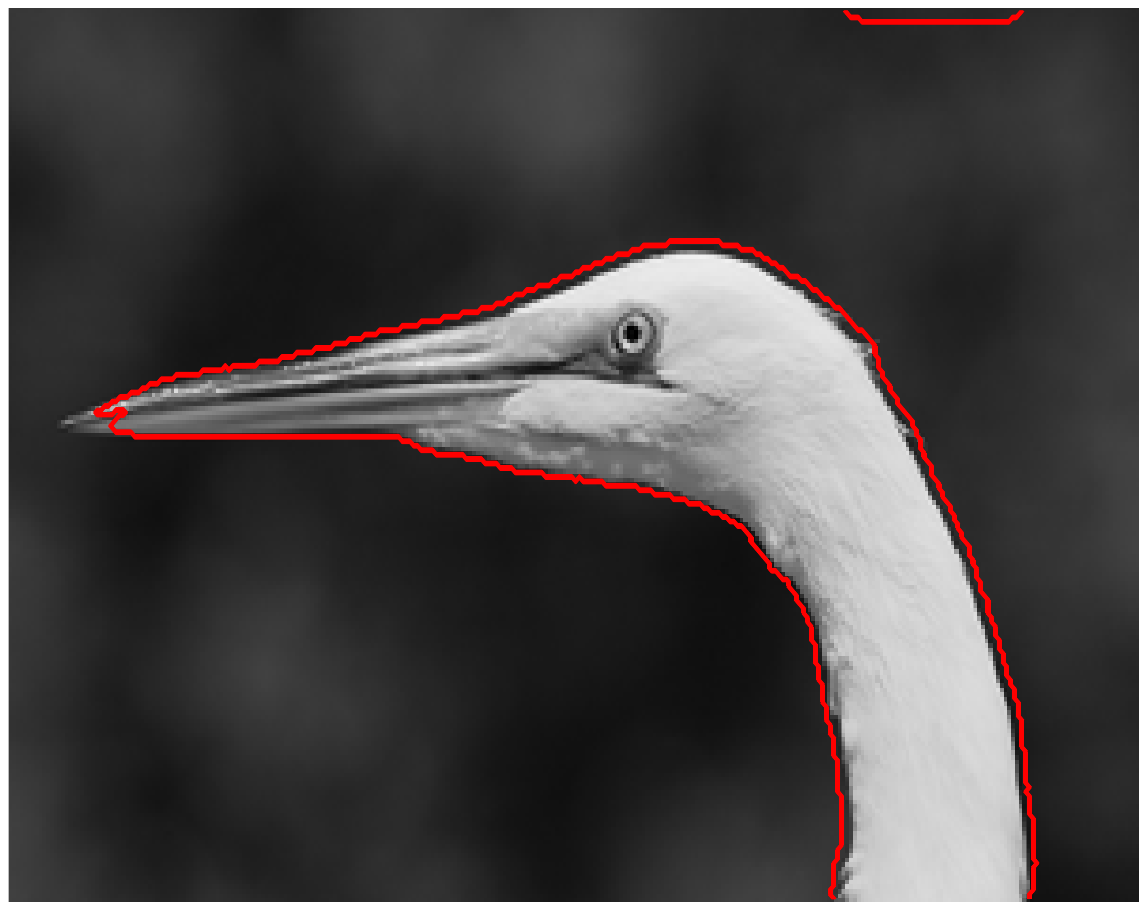}}&		\captionsetup[subfigure]{justification=centering}
\subcaptionbox{Convex Potts}{\includegraphics[width = 1.50in]{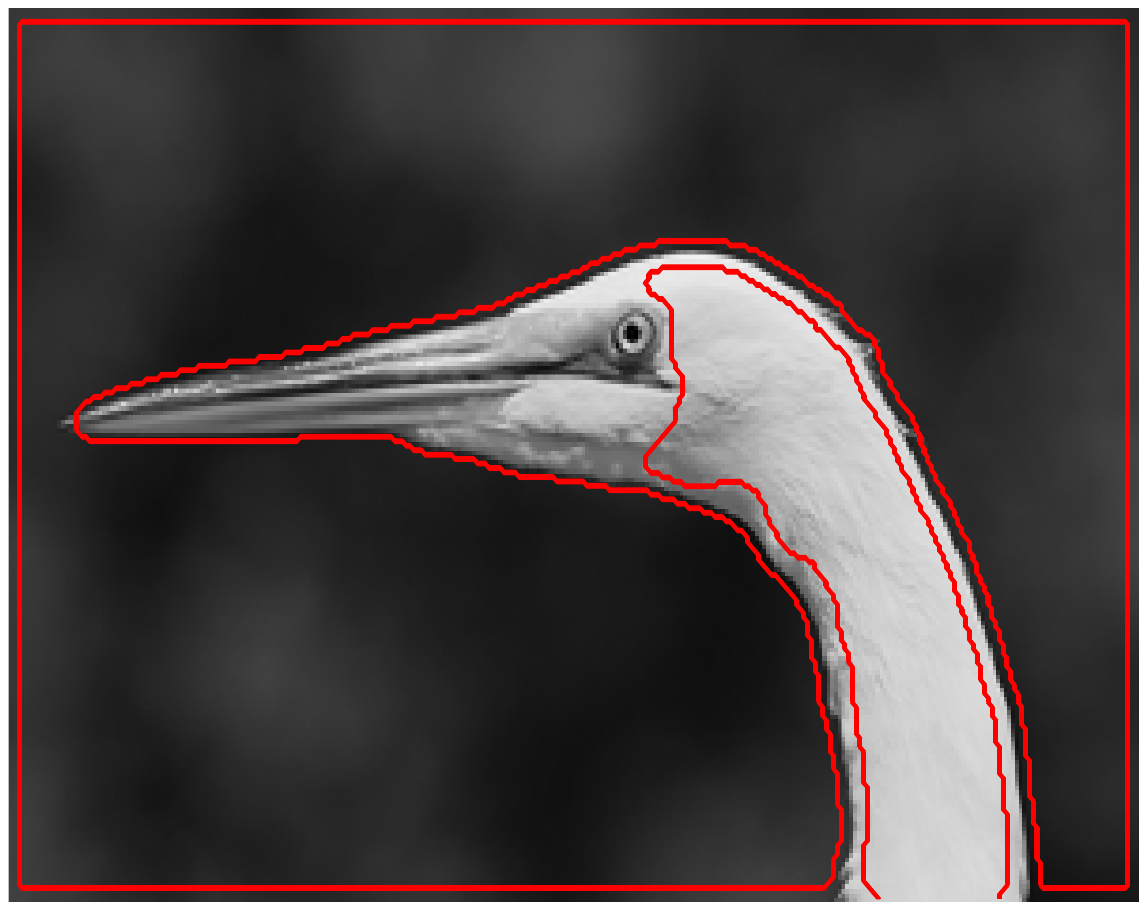}} &		\captionsetup[subfigure]{justification=centering}

\subcaptionbox{SaT-Potts}{\includegraphics[width = 1.50in]{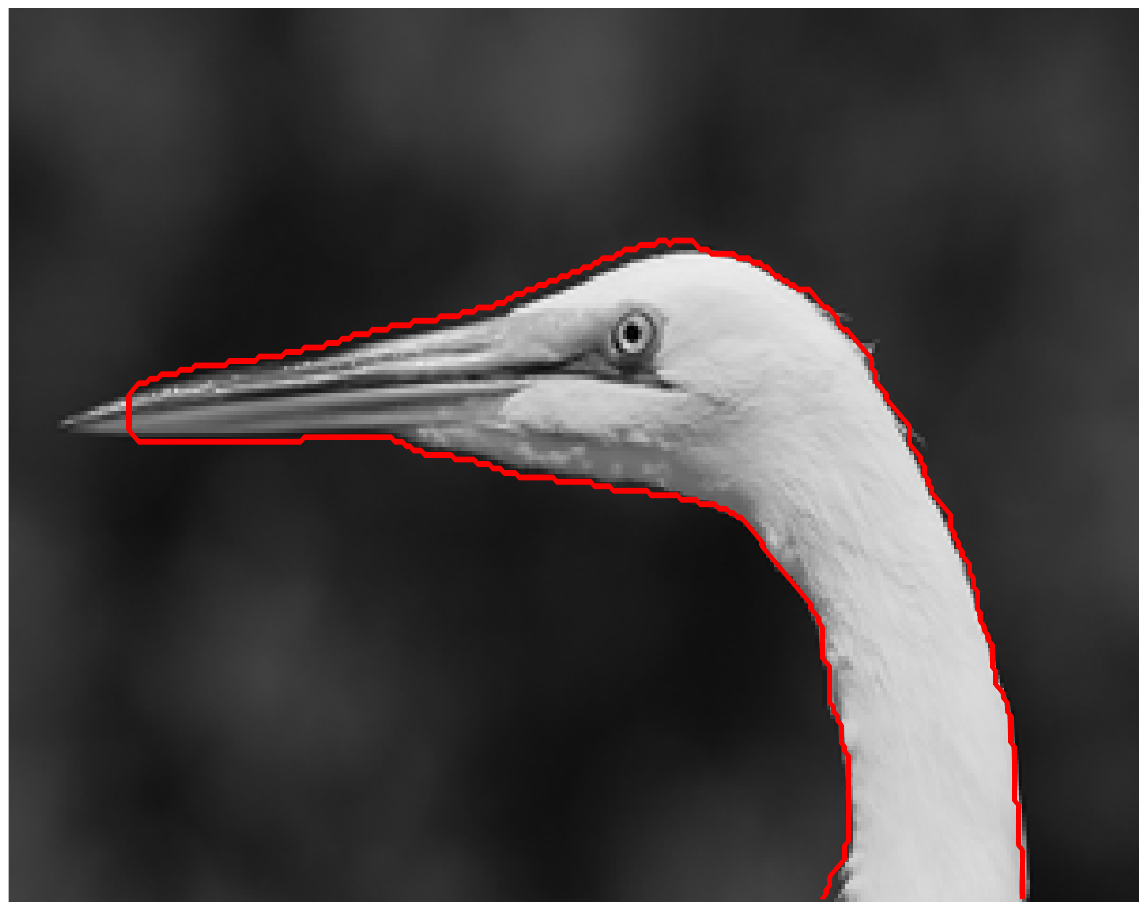}} 
		\end{tabular}}
  		\caption{Segmentation results of Figures {\ref{fig:egret}} corrupted by motion blur followed by Gaussian noise.}
		\label{fig:egret_result}
\end{figure*}

\begin{figure*}[t!]
\resizebox{\textwidth}{!}{
\begin{tabular}{cccccc}
		\stepcounter{figure}
        \setcounter{caption@flags}{4}
        \setcounter{subfigure}{0}
\captionsetup[subfigure]{justification=centering}
\subcaptionbox{Noisy and blurry.}{\includegraphics[width = 1.50in]{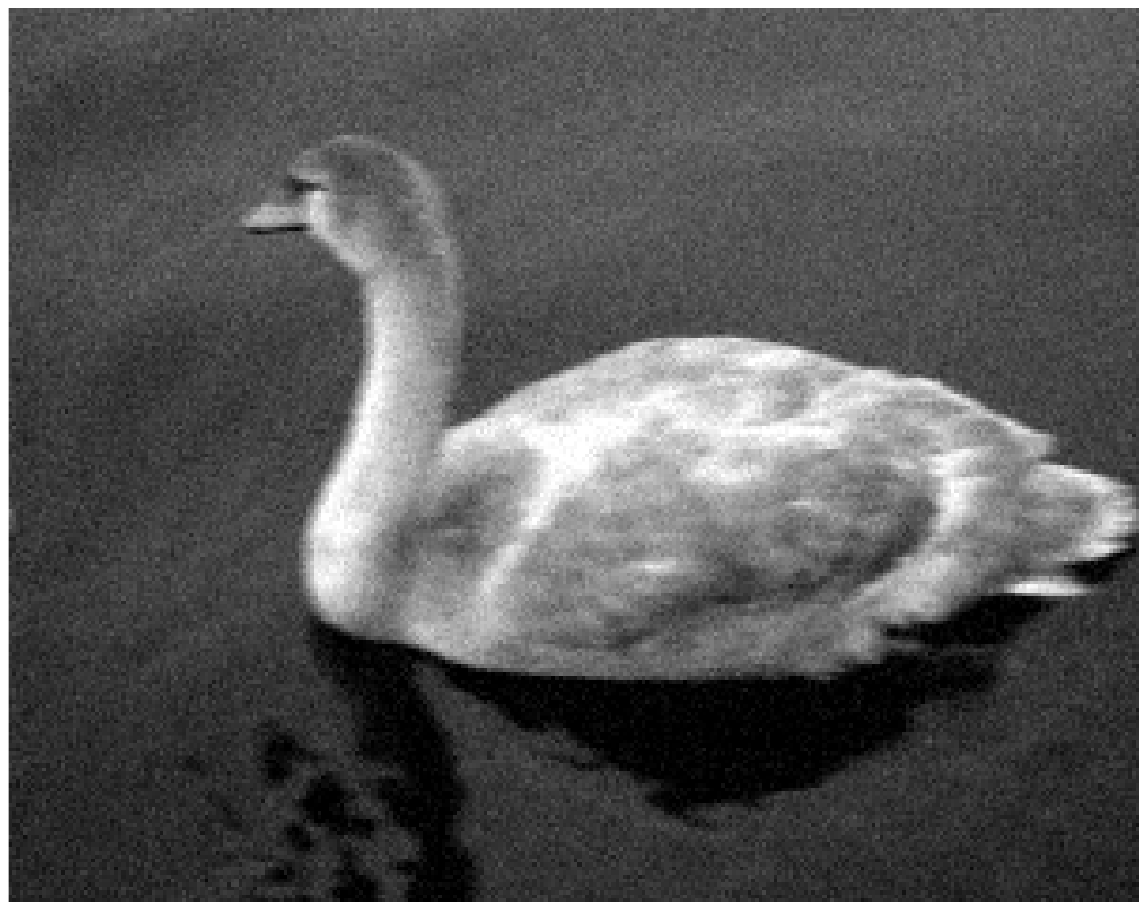}} & \captionsetup[subfigure]{justification=centering}
\subcaptionbox{Ground truth.}{\includegraphics[width = 1.50in]{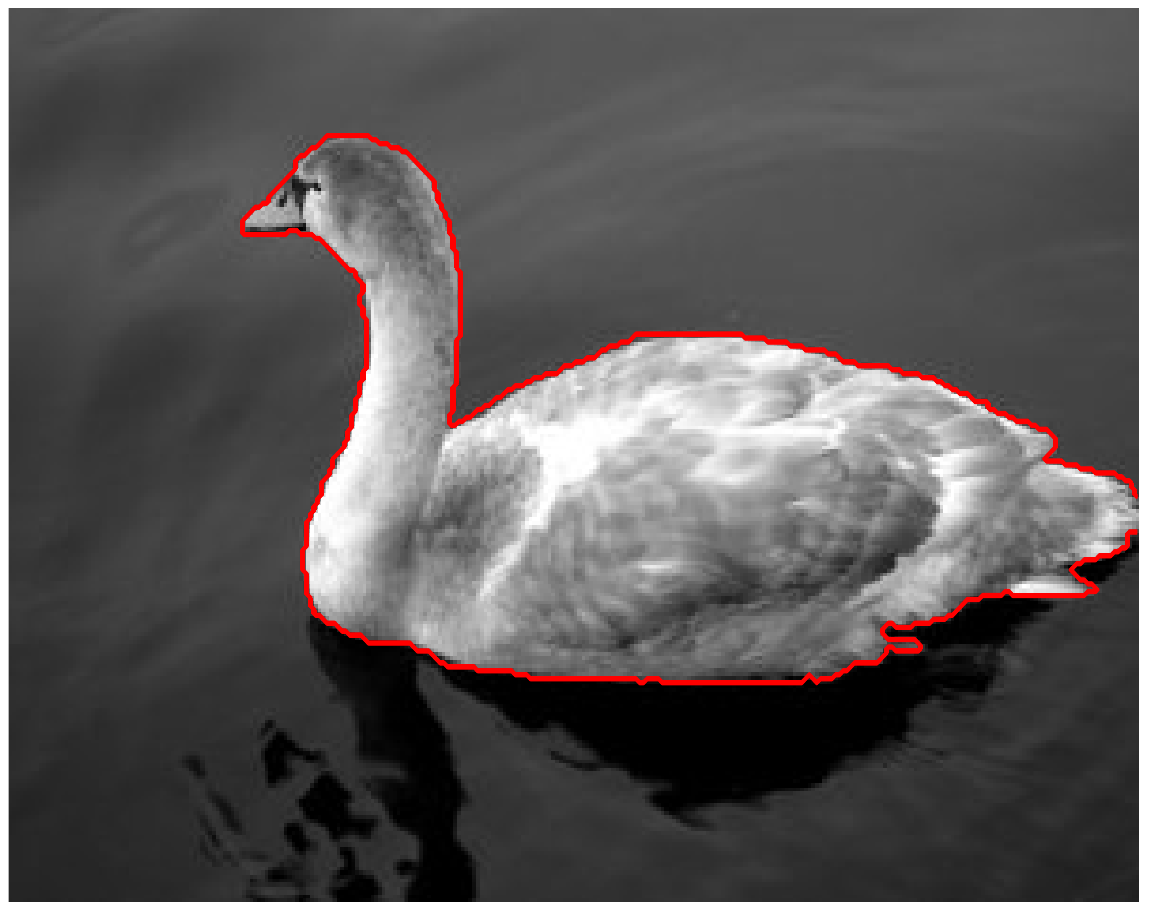}} &
\captionsetup[subfigure]{justification=centering}
\subcaptionbox{(original) SaT}{\includegraphics[width = 1.50in]{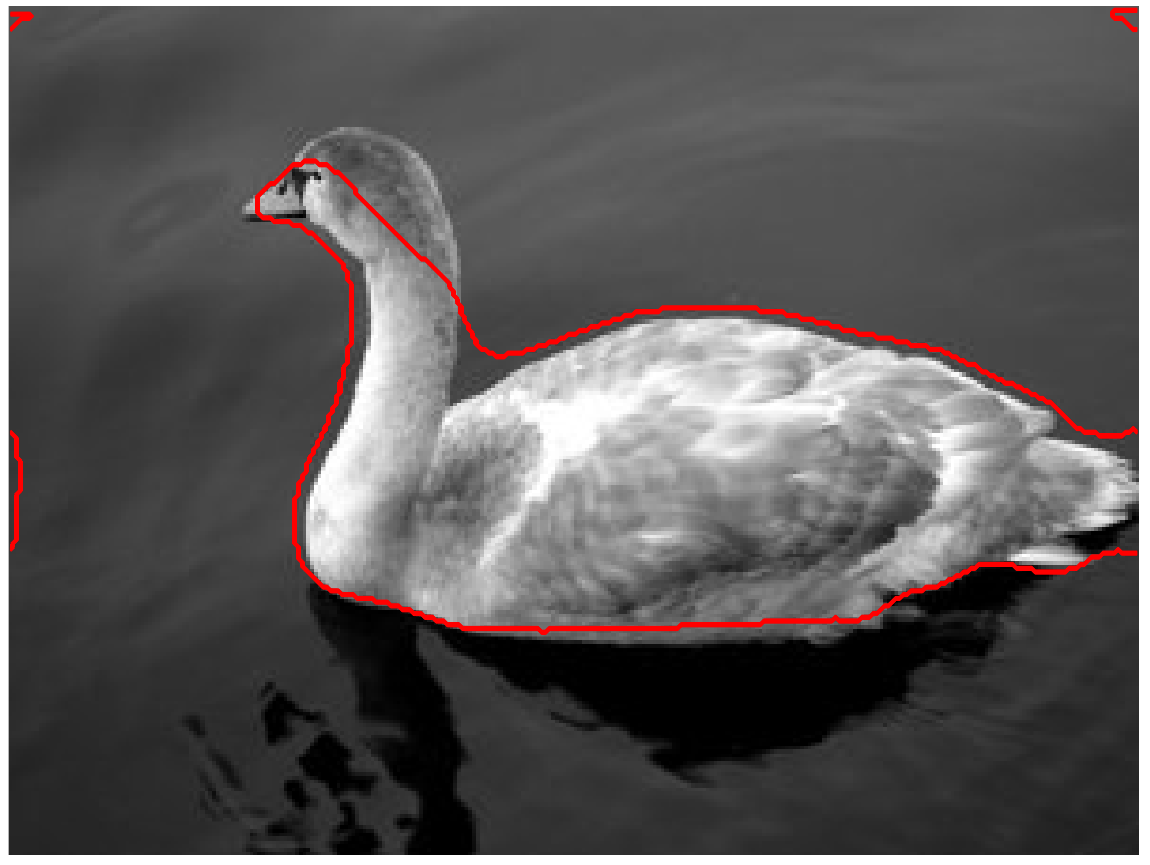}} & 		\captionsetup[subfigure]{justification=centering}
\subcaptionbox{TV$^{p}$ SaT}{\includegraphics[width = 1.50in]{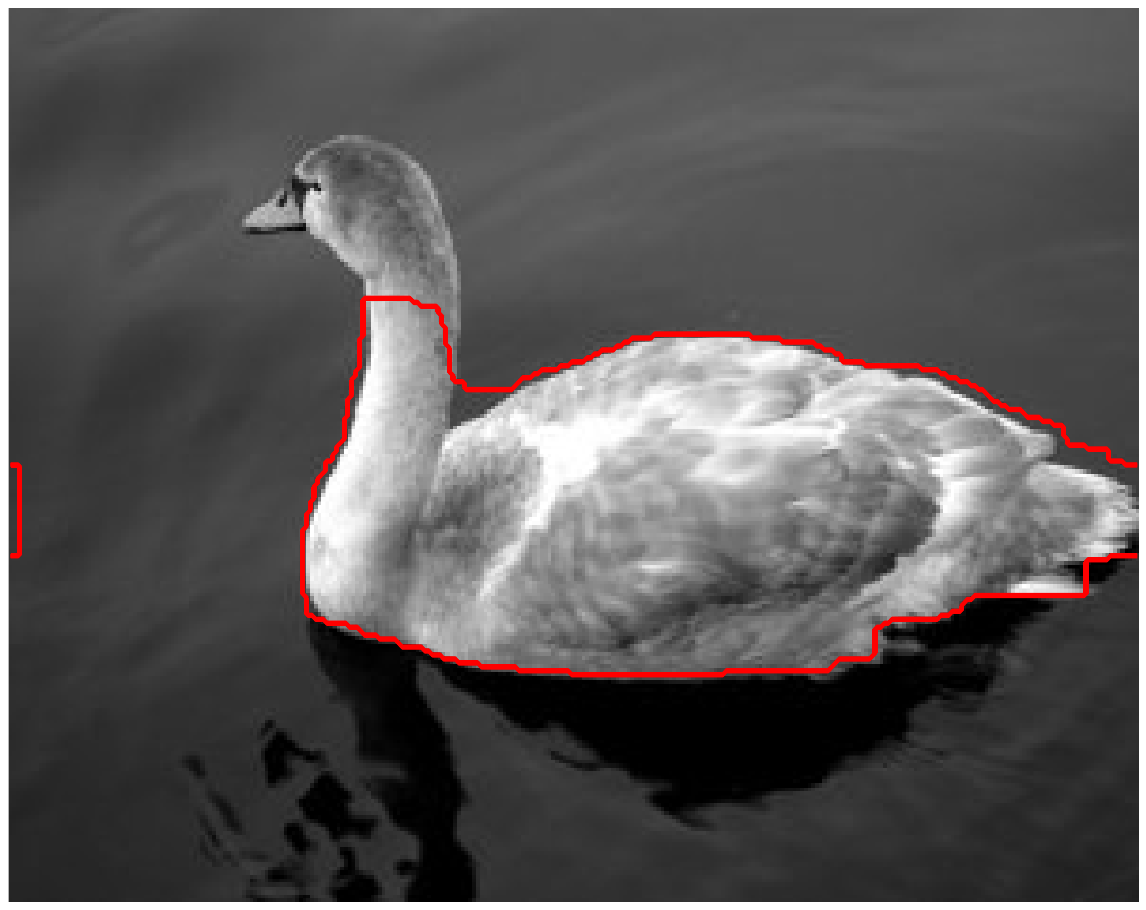}} &		\captionsetup[subfigure]{justification=centering}
\subcaptionbox{AITV SaT (ADMM)}{\includegraphics[width = 1.50in]{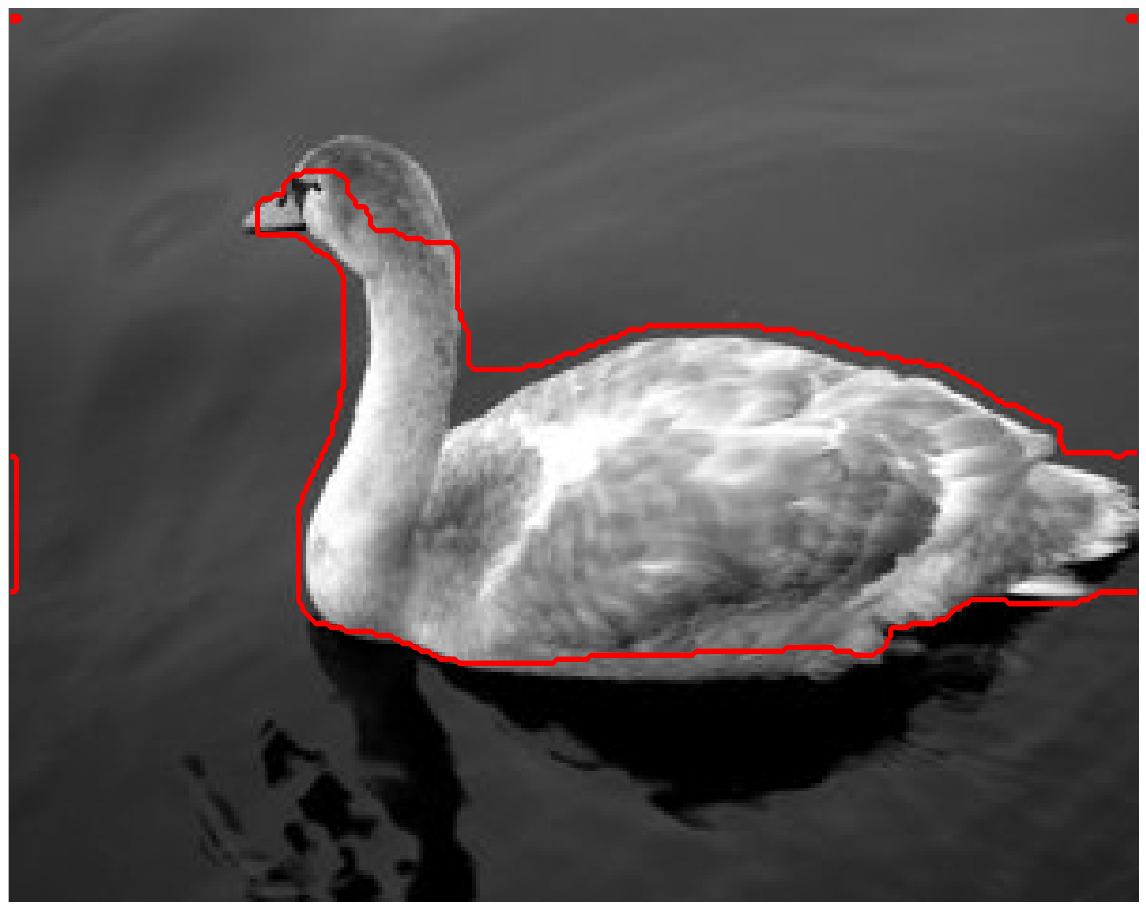}} 
& \captionsetup[subfigure]{justification=centering}
\subcaptionbox{AITV   SaT (DCA)}{\includegraphics[width = 1.50in]{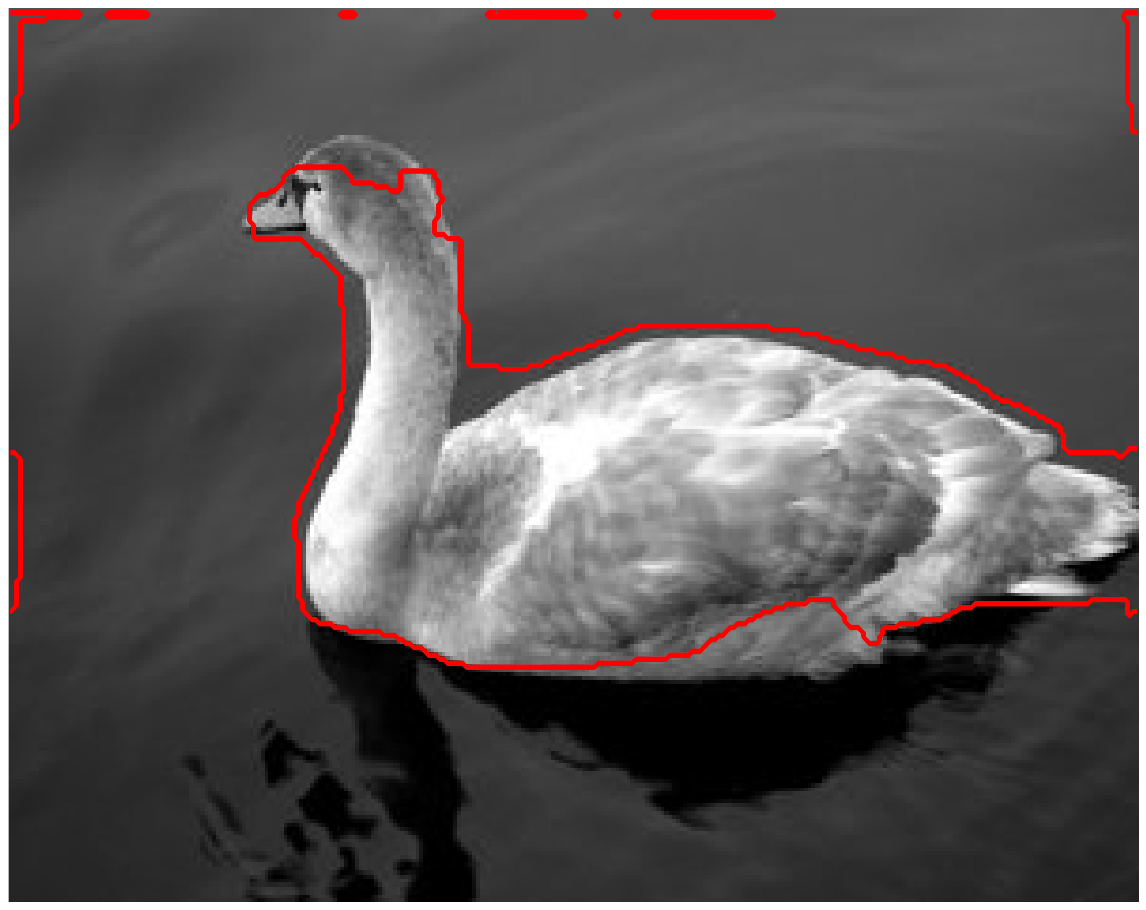}} \\
&
		\captionsetup[subfigure]{justification=centering}
\subcaptionbox{AITV CV}{\includegraphics[width = 1.50in]{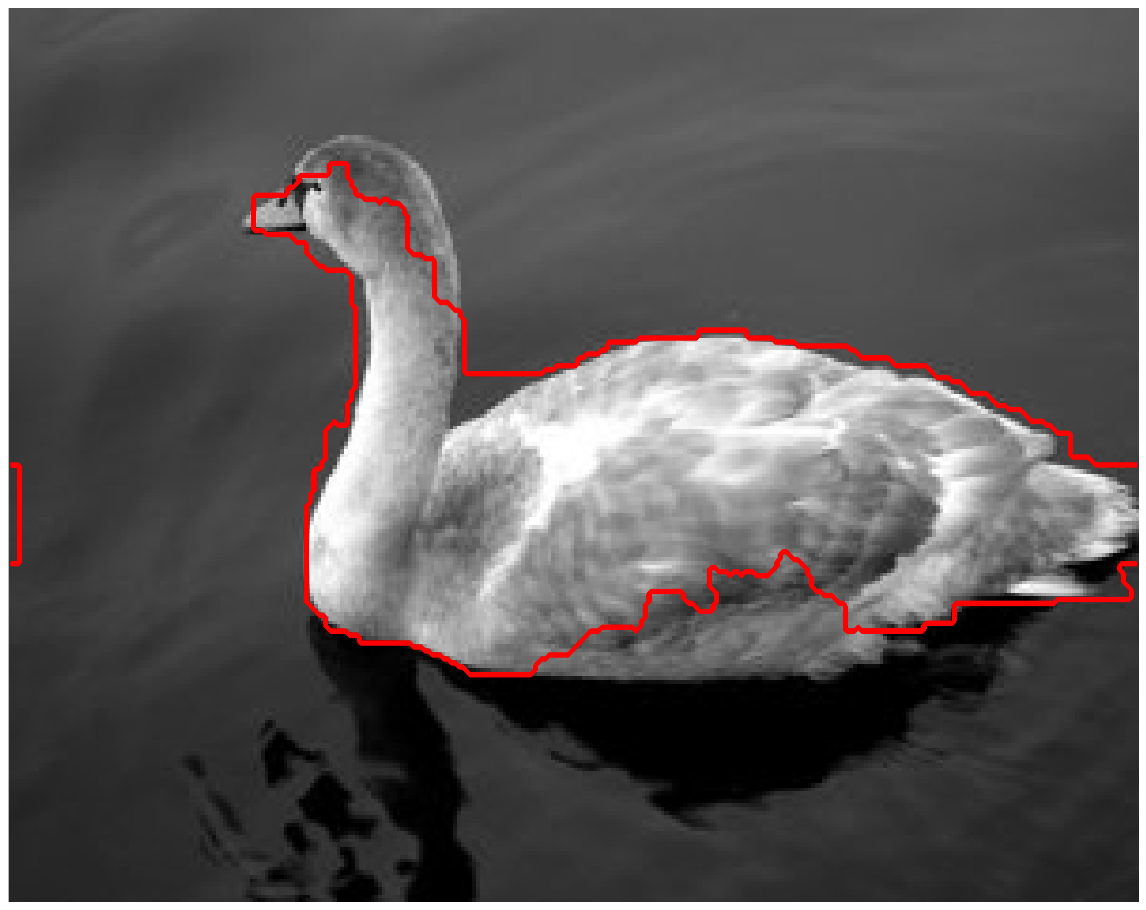}}&
		\captionsetup[subfigure]{justification=centering}
\subcaptionbox{TV$^p$ MS}{\includegraphics[width = 1.50in]{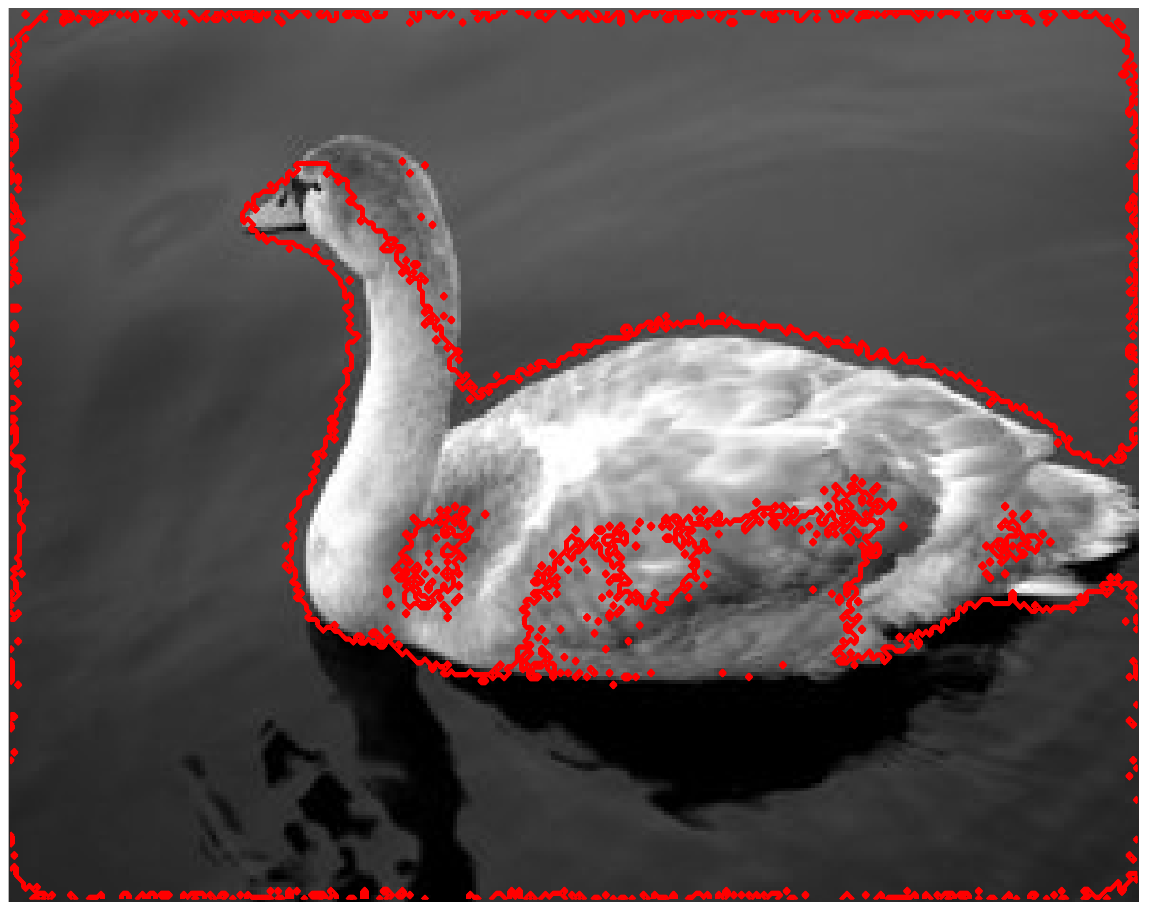}} &
		\captionsetup[subfigure]{justification=centering}
\subcaptionbox{ICTM}{\includegraphics[width = 1.50in]{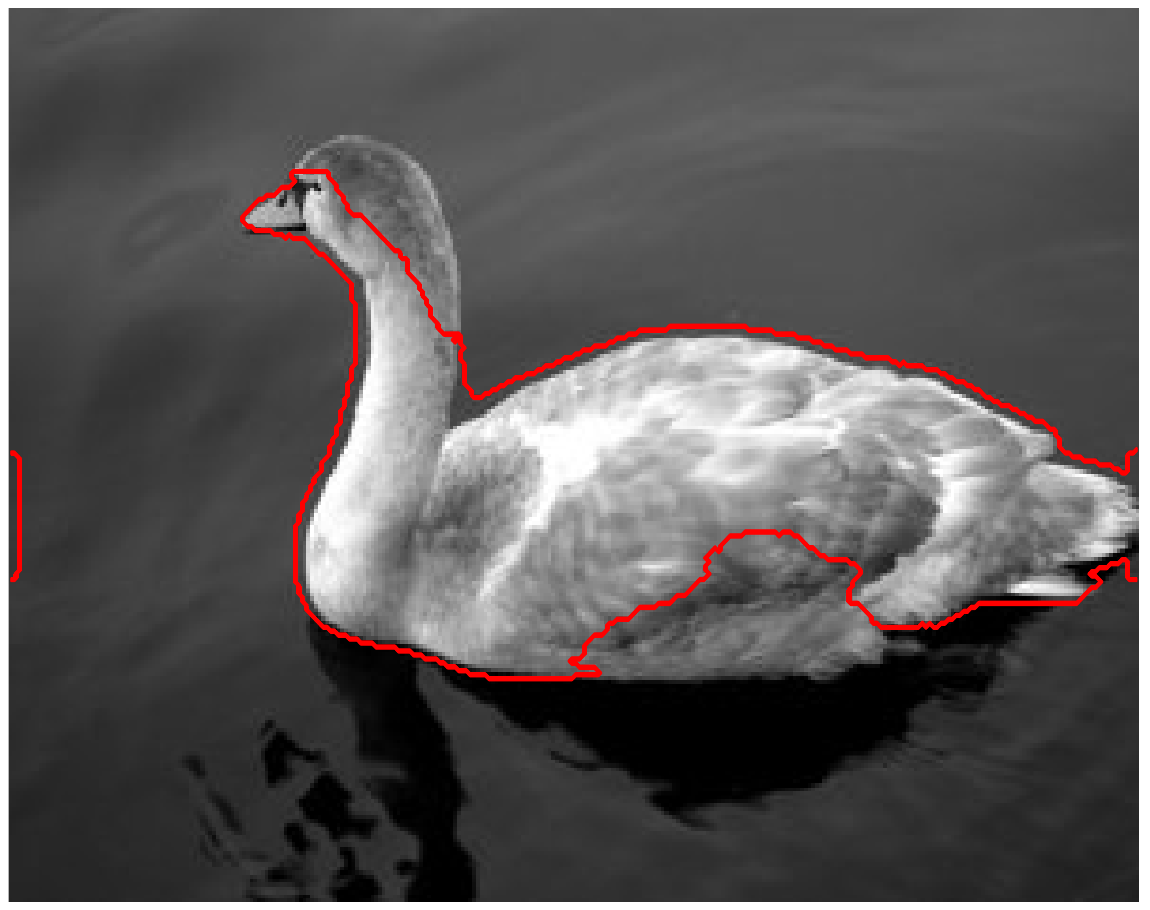}}&		\captionsetup[subfigure]{justification=centering}
\subcaptionbox{Convex Potts}{\includegraphics[width = 1.50in]{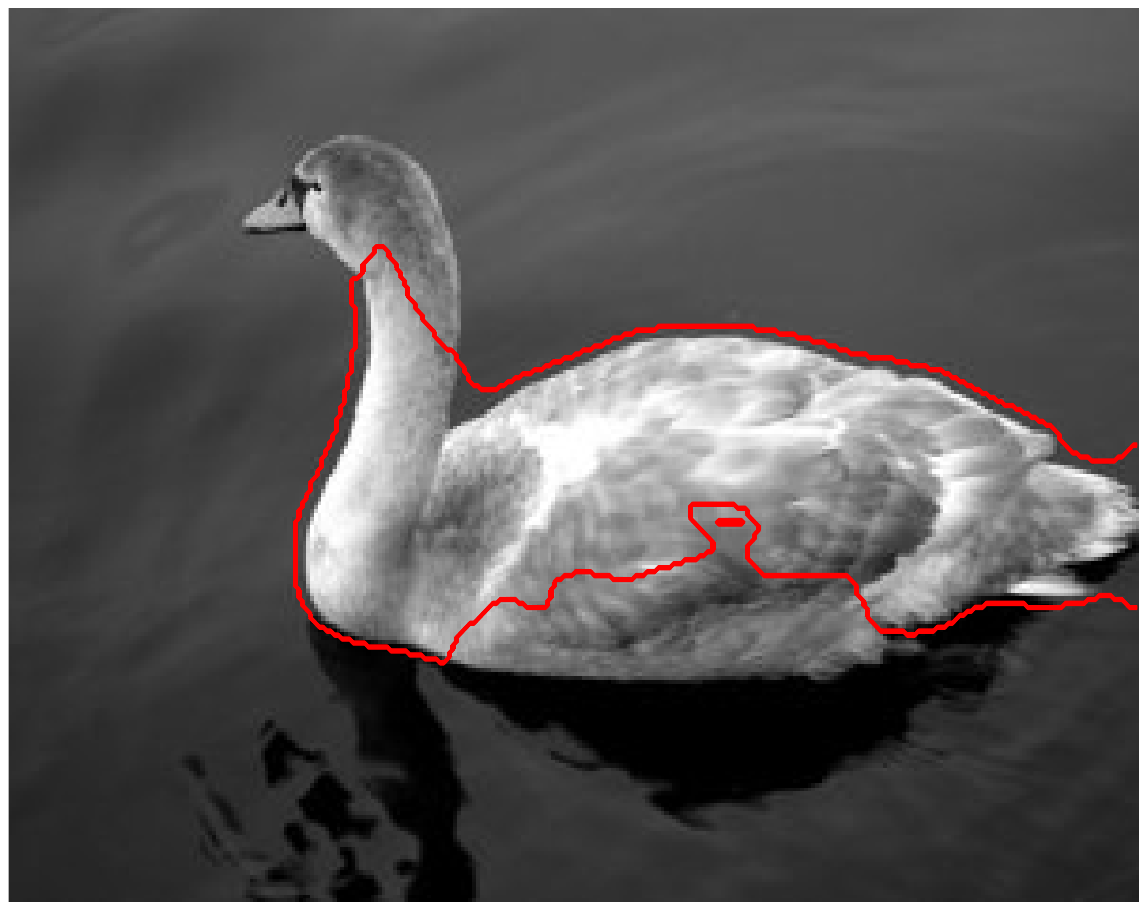}} &		\captionsetup[subfigure]{justification=centering}

\subcaptionbox{SaT-Potts}{\includegraphics[width = 1.50in]{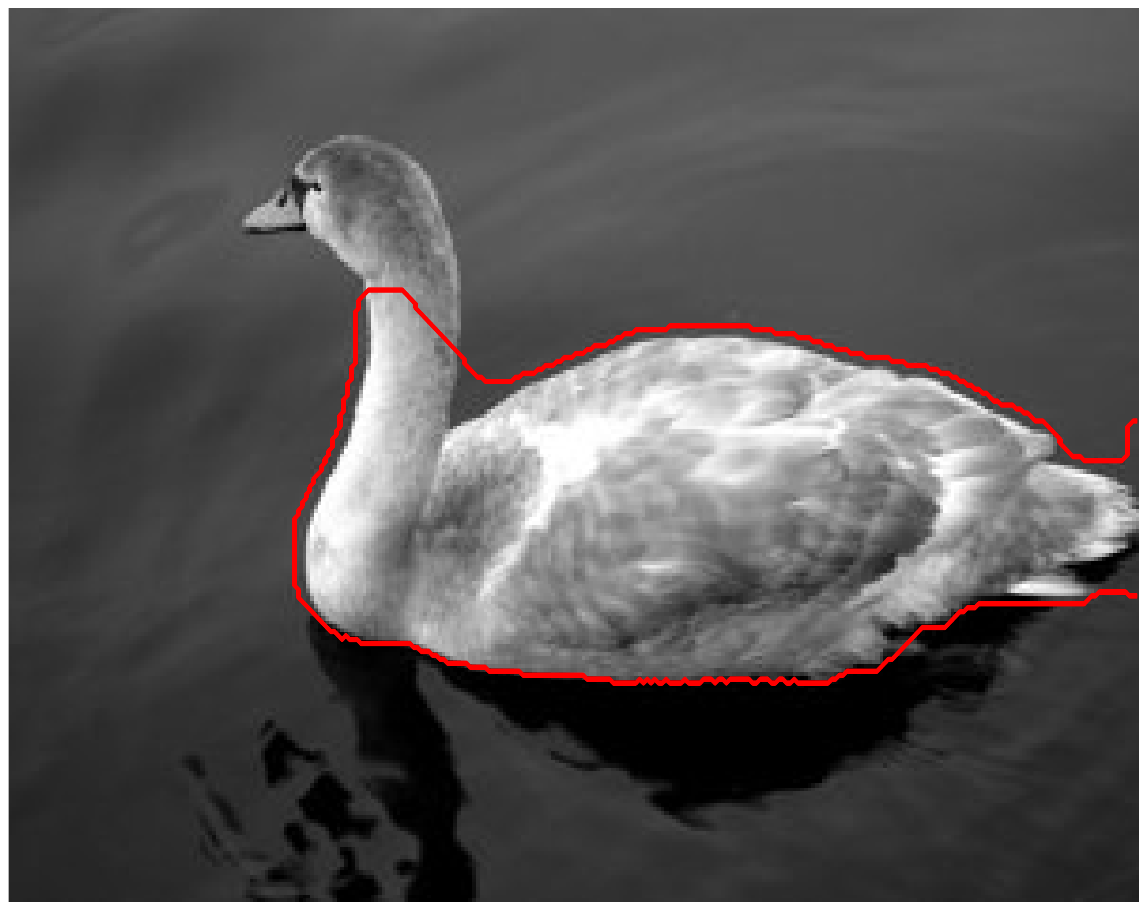}} 
		\end{tabular}}
  		\caption{Segmentation results of Figures {\ref{fig:swan}} corrupted by motion blur followed by Gaussian noise. }
		\label{fig:swan_result}
\end{figure*}

\begin{figure*}[t!]
\resizebox{\textwidth}{!}{
\begin{tabular}{cccccc}
		\stepcounter{figure}
        \setcounter{caption@flags}{4}
        \setcounter{subfigure}{0}
\captionsetup[subfigure]{justification=centering}
\subcaptionbox{Noisy and blurry.}{\includegraphics[width = 1.50in]{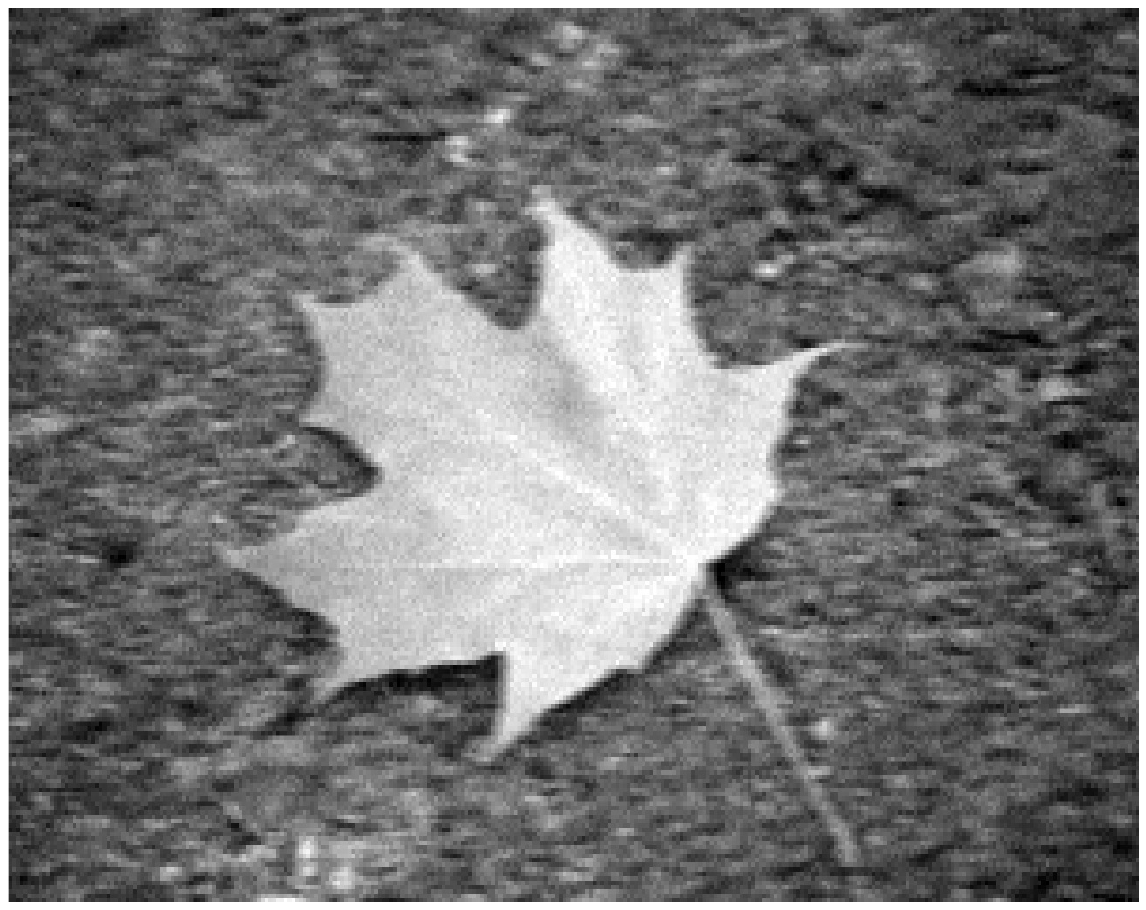}} & \captionsetup[subfigure]{justification=centering}
\subcaptionbox{Ground truth.}{\includegraphics[width = 1.50in]{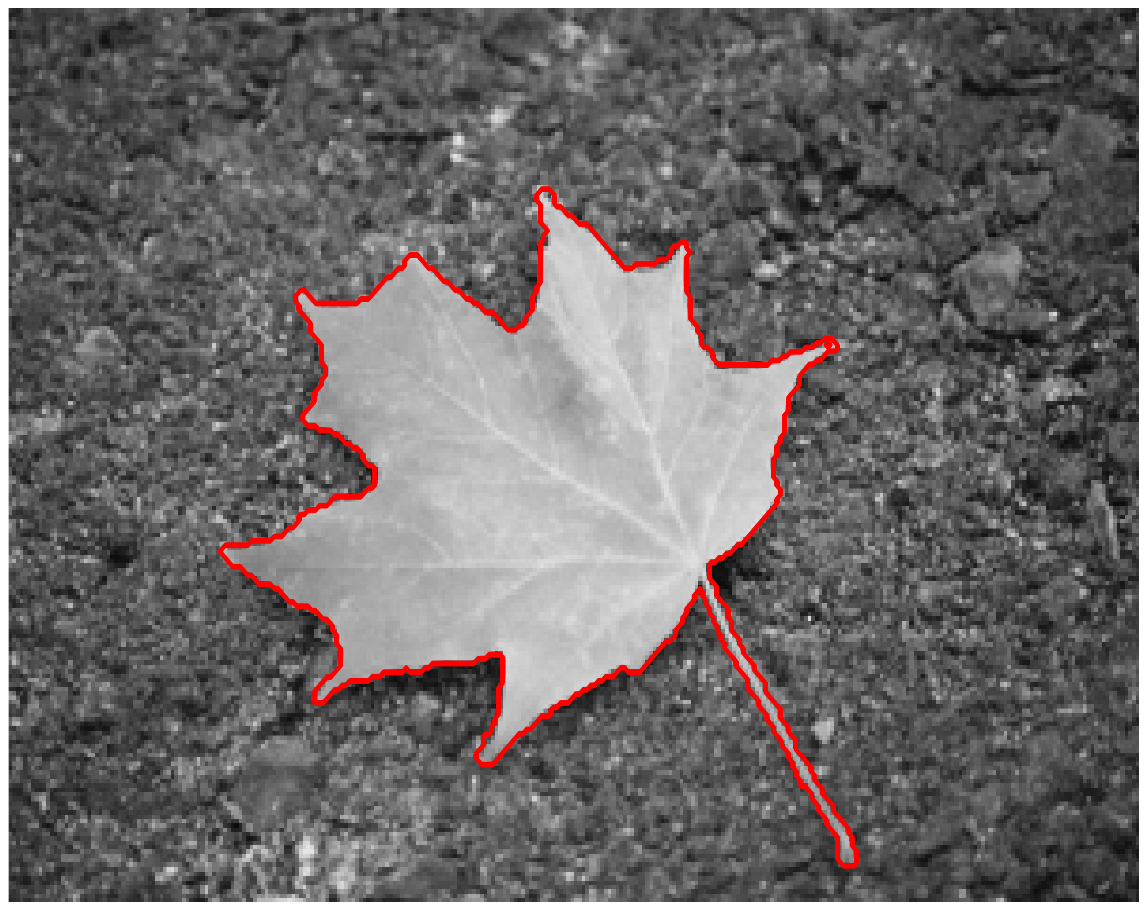}} &
\captionsetup[subfigure]{justification=centering}
\subcaptionbox{(original) SaT}{\includegraphics[width = 1.50in]{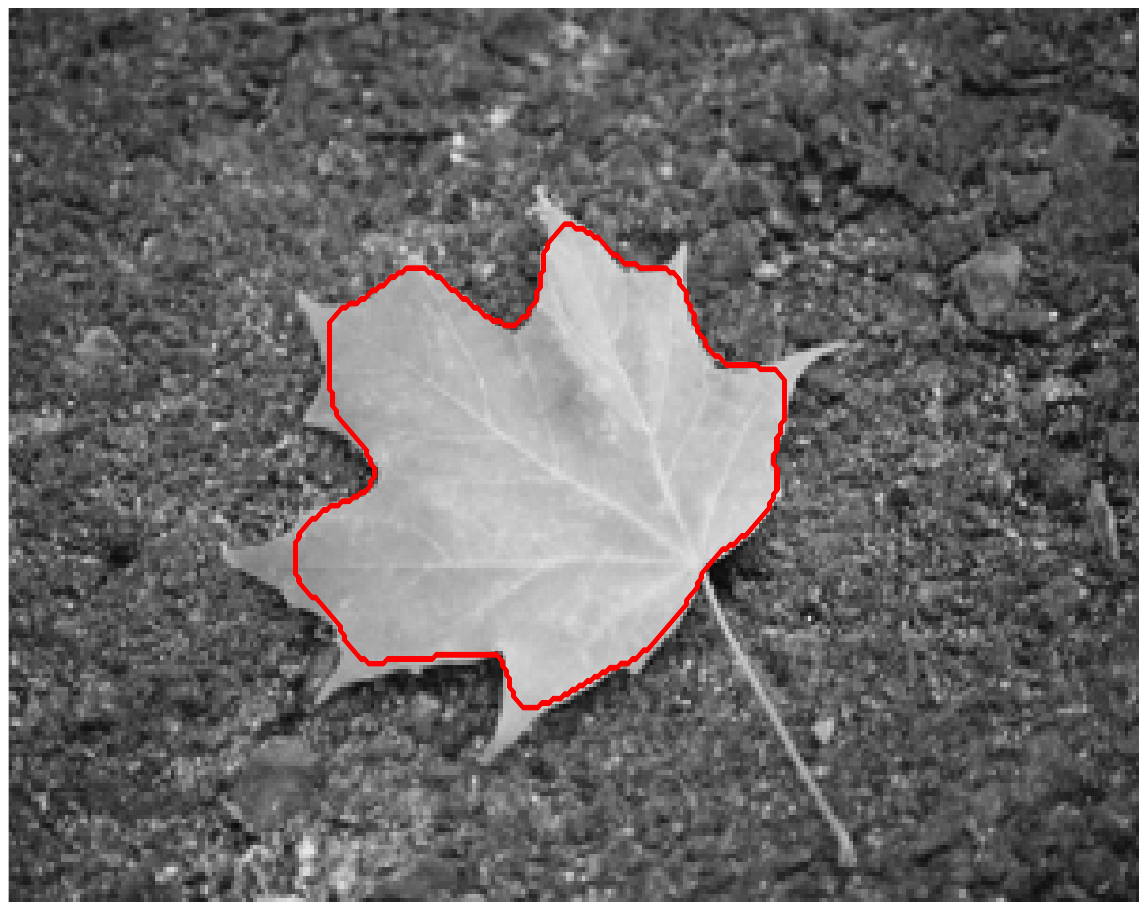}} & 		\captionsetup[subfigure]{justification=centering}
\subcaptionbox{TV$^{p}$ SaT}{\includegraphics[width = 1.50in]{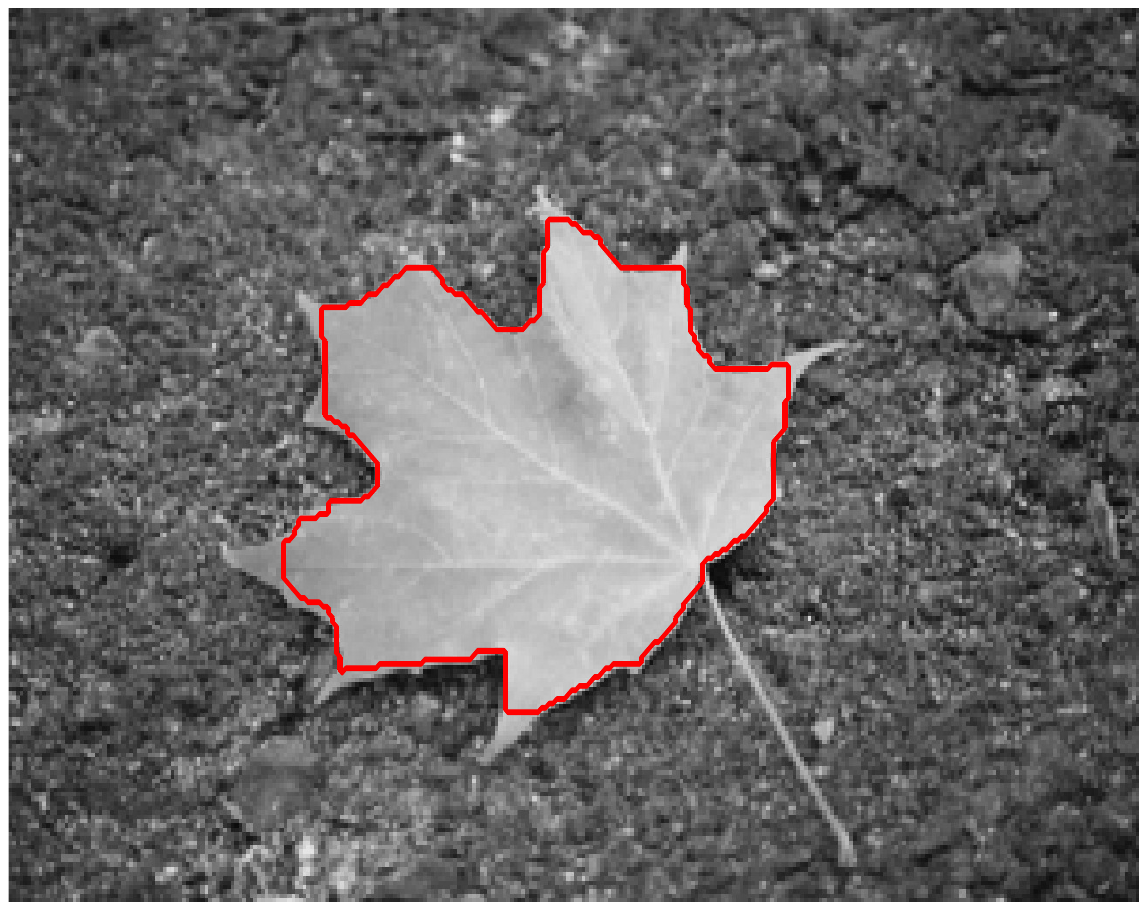}} &		\captionsetup[subfigure]{justification=centering}
\subcaptionbox{AITV SaT (ADMM)}{\includegraphics[width = 1.50in]{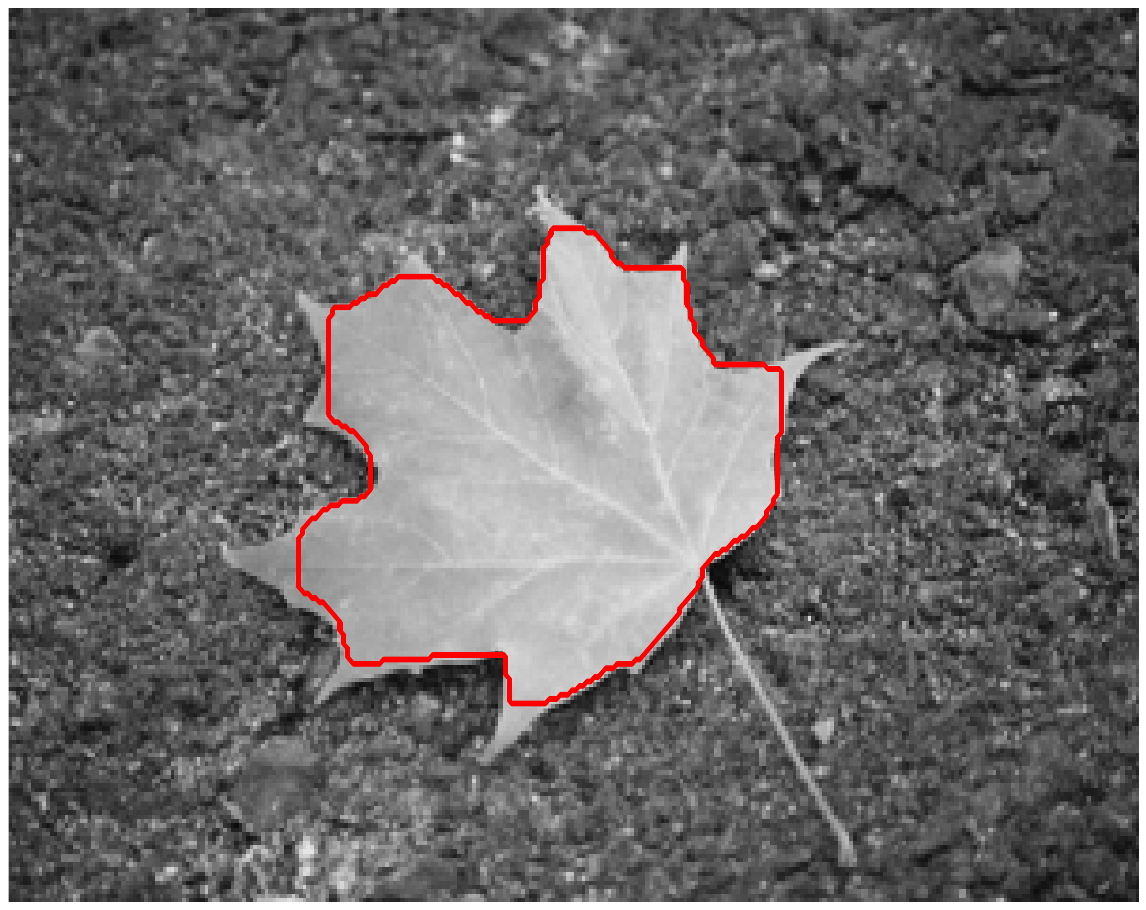}} 
& \captionsetup[subfigure]{justification=centering}
\subcaptionbox{AITV   SaT (DCA)}{\includegraphics[width = 1.50in]{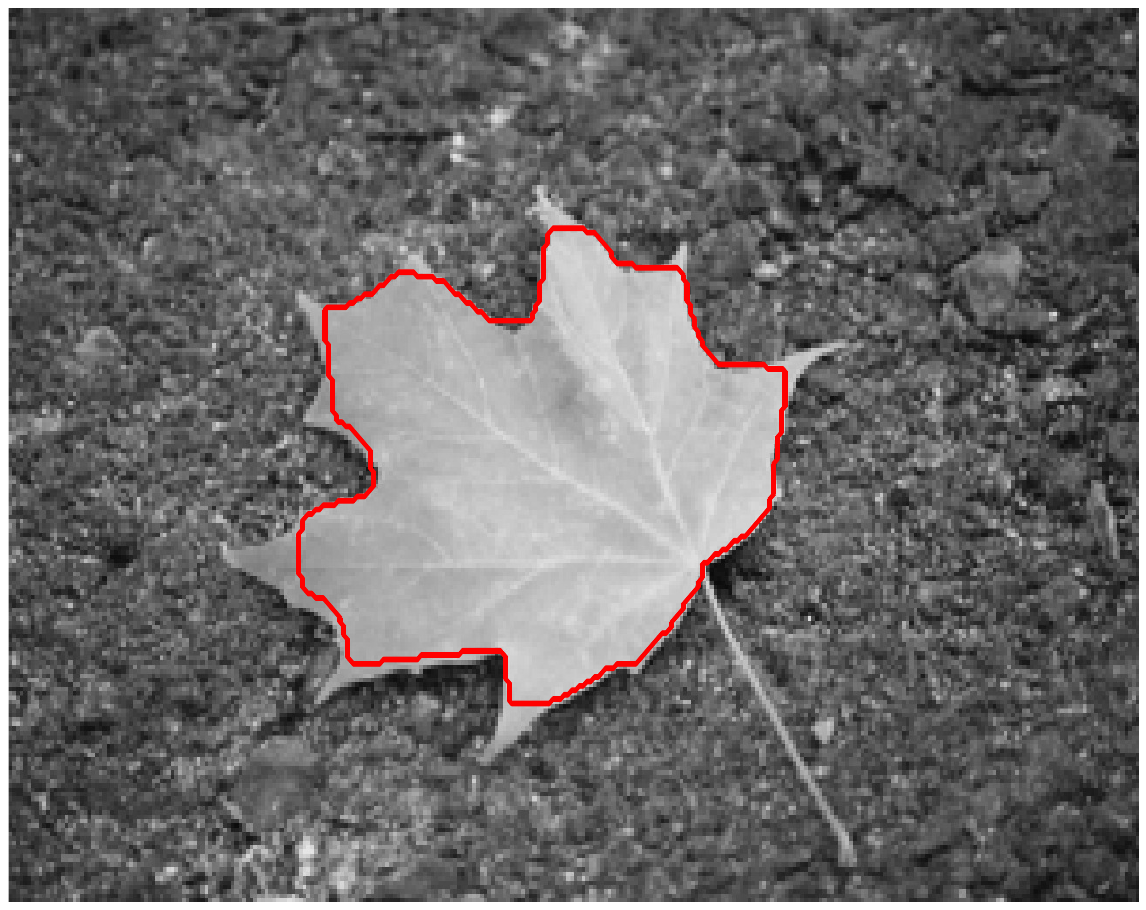}} \\
&
		\captionsetup[subfigure]{justification=centering}
\subcaptionbox{AITV CV}{\includegraphics[width = 1.50in]{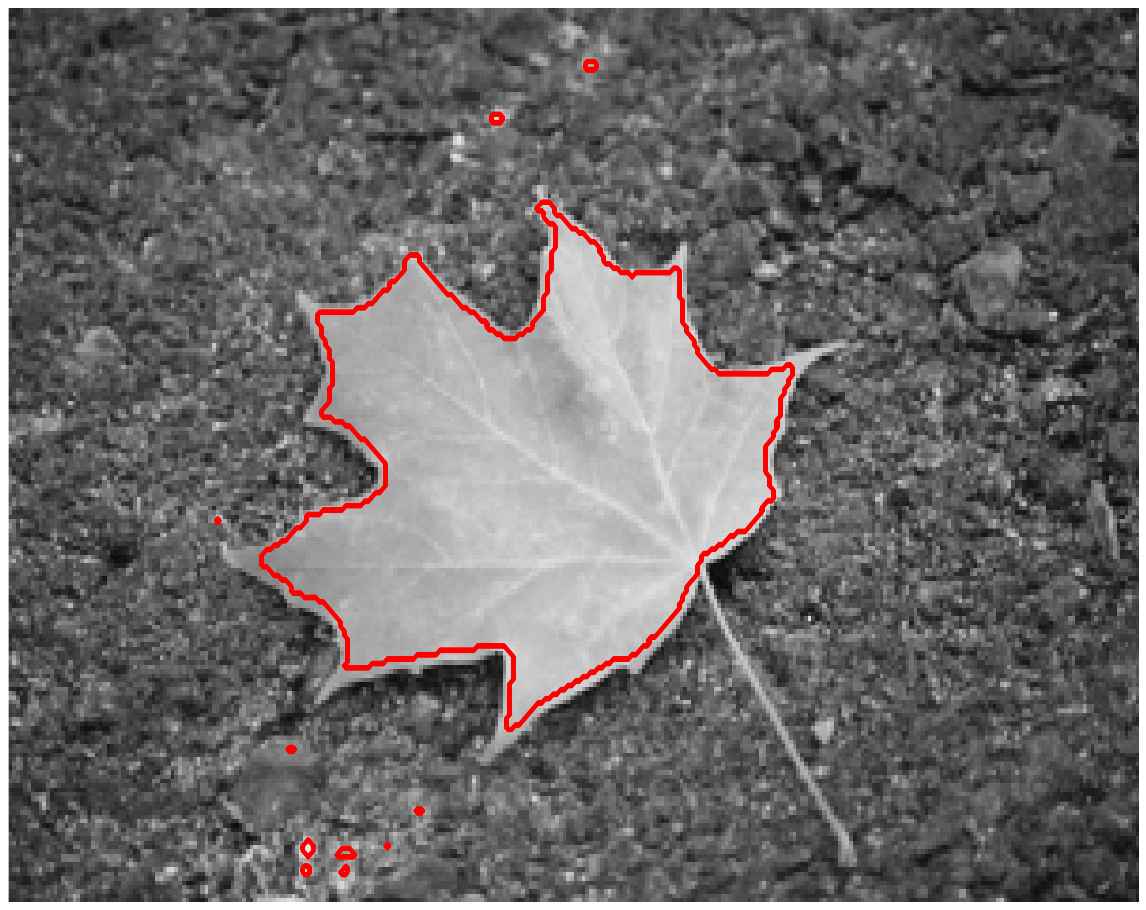}}&
		\captionsetup[subfigure]{justification=centering}
\subcaptionbox{TV$^p$ MS}{\includegraphics[width = 1.50in]{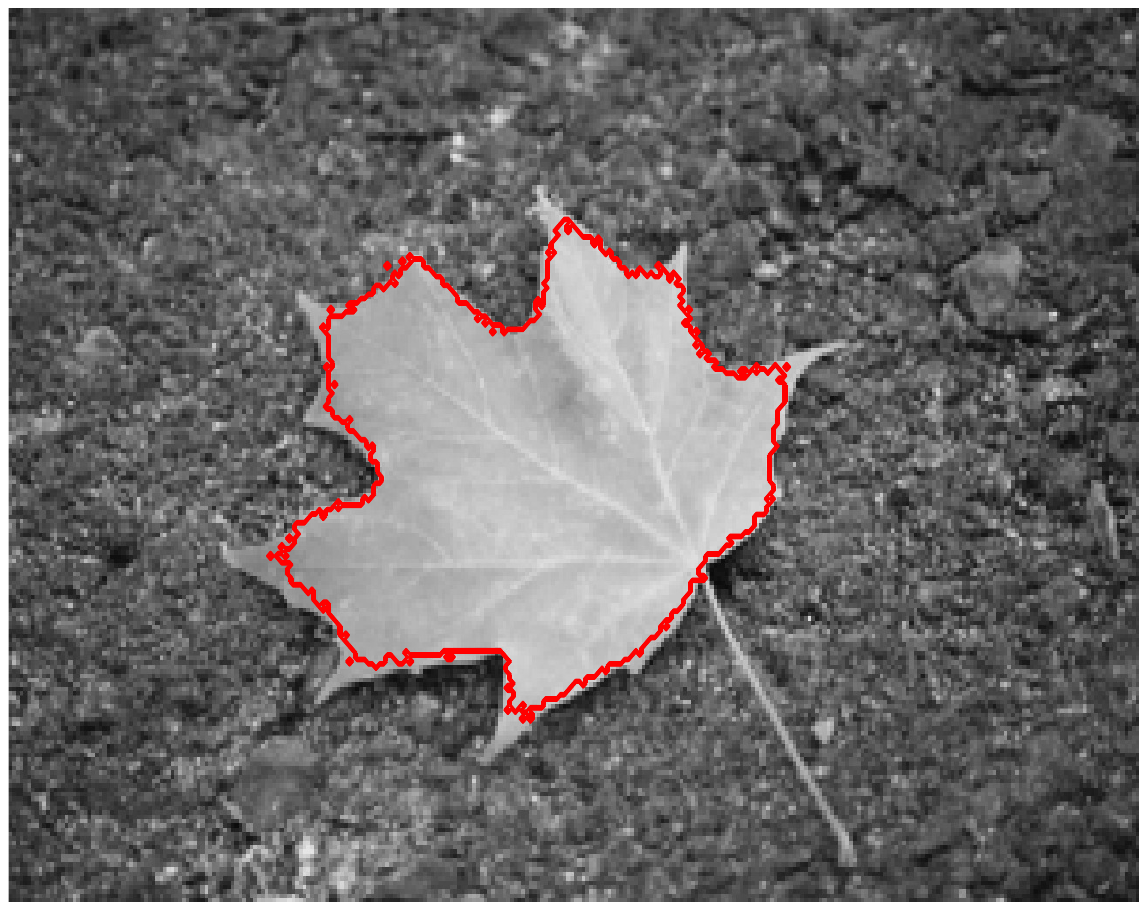}} &
		\captionsetup[subfigure]{justification=centering}
\subcaptionbox{ICTM}{\includegraphics[width = 1.50in]{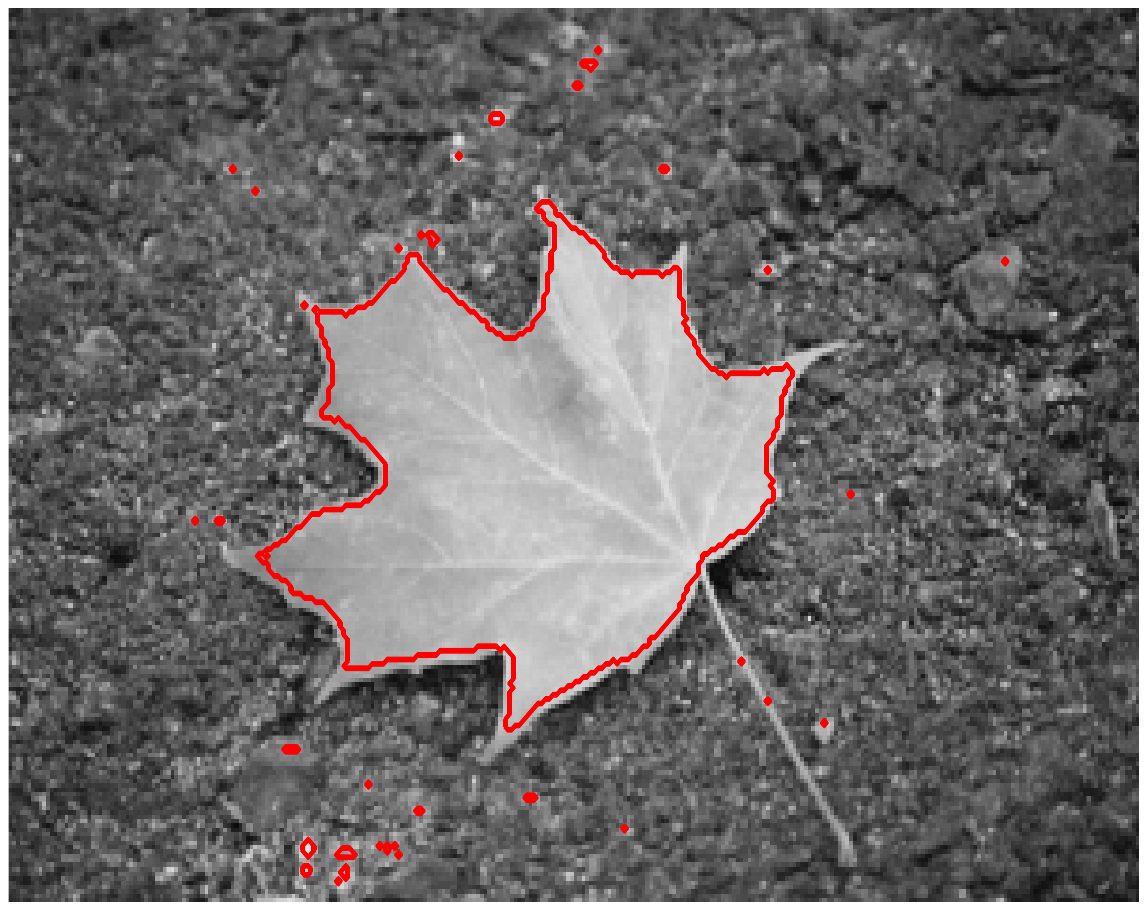}}&		\captionsetup[subfigure]{justification=centering}
\subcaptionbox{Convex Potts}{\includegraphics[width = 1.50in]{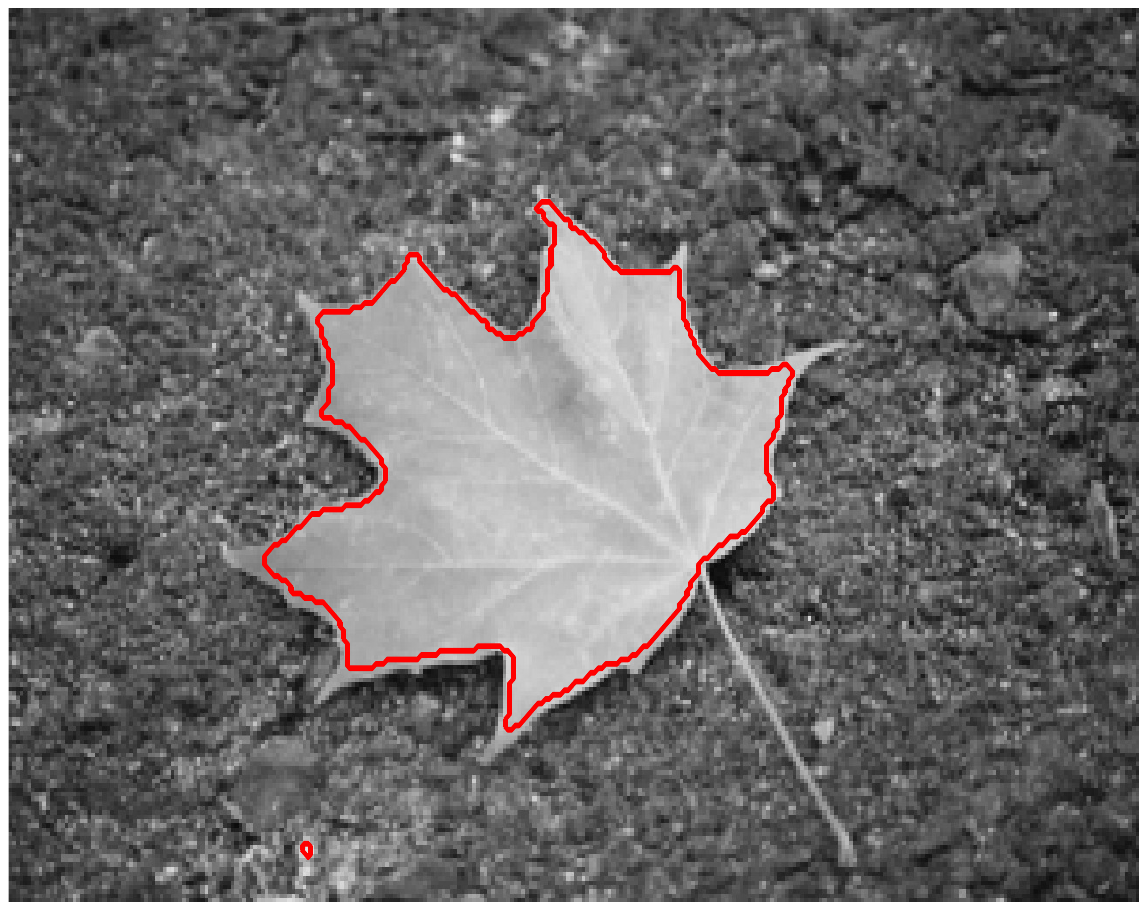}} &		\captionsetup[subfigure]{justification=centering}

\subcaptionbox{SaT-Potts}{\includegraphics[width = 1.50in]{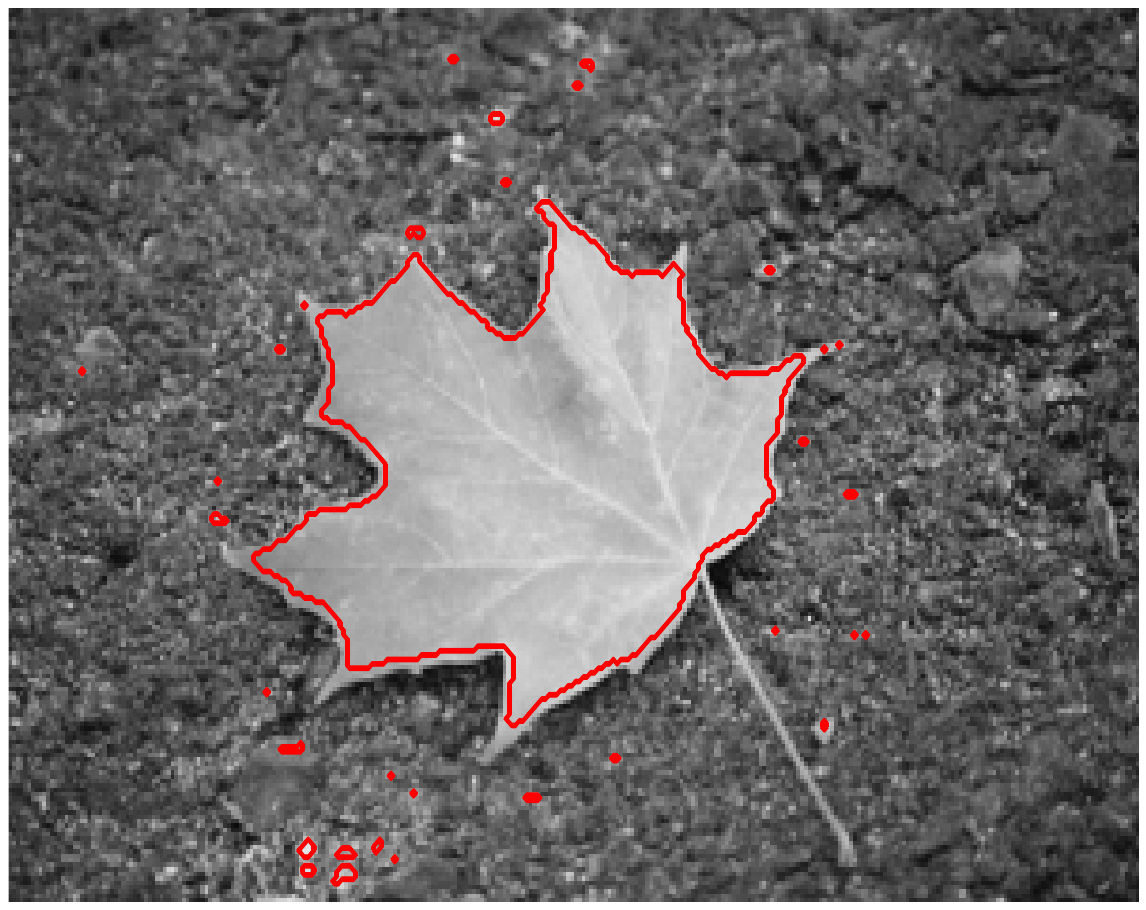}} 
		\end{tabular}}
  		\caption{Segmentation results of Figures {\ref{fig:leaf}} corrupted by motion blur followed by Gaussian noise.}
		\label{fig:leaf_result}
\end{figure*}

When the IIH image is added as a channel to the grayscale image, we smooth each channel and then apply $K$-means clustering for the SaT methods. For the other segmentation methods, we consider their multichannel extensions to process the two channels that are composed of grayscale and IIH.

For the images in Figure {\ref{fig:real_grayscale}} (after rescaling the pixel intensities to $[0,1]$), we corrupt them with motion blur \texttt{fspecial(`motion', 5,0)} followed by Gaussian noise with mean 0 and variance 0.001. We tune the parameters $\lambda \in [0.25, 5]$ and $\mu \in [5,60]$ for the SaT methods. For each image, the ground truth is determined from the segmentation results by three human subjects. A pixel is declared an object of interest in the ground truth if at least two subjects agree \cite{AlpertGBB07}. The DICE indices and computational times of the segmentation algorithms are recorded in Table \ref{tab:grayscale_real_results} while the segmentation results and their ground truths are presented in Figures {\ref{fig:caterpillar_result}}-{\ref{fig:leaf_result}}. Note that some segmentation results have the image border segmented because of the boundary artifacts created by the IIH image (see Figures {\ref{fig:iih_caterpiller}} and {\ref{fig:iih_egret}}). For all four images, AITV SaT (ADMM) is among the top three methods with the highest DICE indices. It provides satisfactory results in about two seconds. Moreover, it outperforms its DCA counterpart in terms of DICE indices and computational times, especially for Figure {\ref{fig:egret}}. For Figures {\ref{fig:caterpillar_result}}-{\ref{fig:swan_result}}, although both algorithms solve the same model {\eqref{eq:AITV_MS}}, they output different results. As {\eqref{eq:AITV_MS}} is nonconvex, it is possible that ADMM and DCA  attain different solutions.

\subsection{ Real Color Images} \label{sec:real_color}
\afterpage{
 \begin{figure}[t!!]
\centering
\begin{tabular}{c@{}c@{}c@{}}
		\subcaptionbox{Garden. Size: $321 \times 481$.\label{fig:circle}}{\includegraphics[scale=0.325]{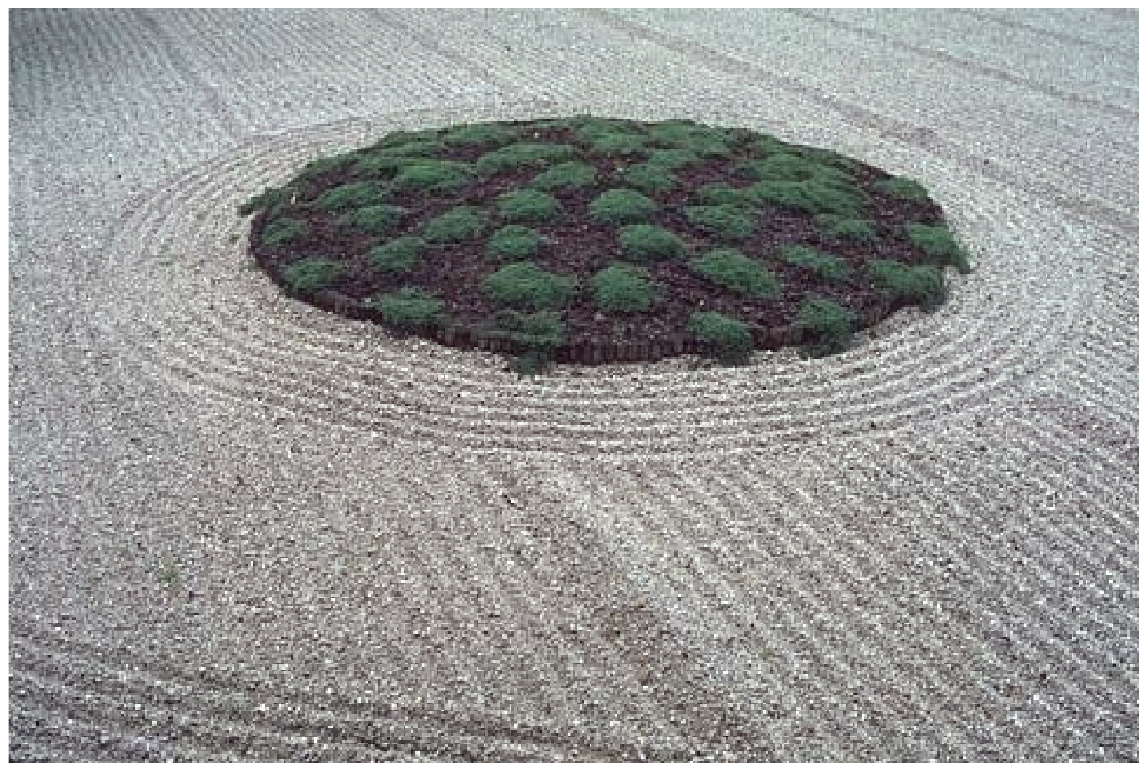}} &
		 \subcaptionbox{ Man. Size: $321 \times 481.$\label{fig:man}}{\includegraphics[scale=0.325]{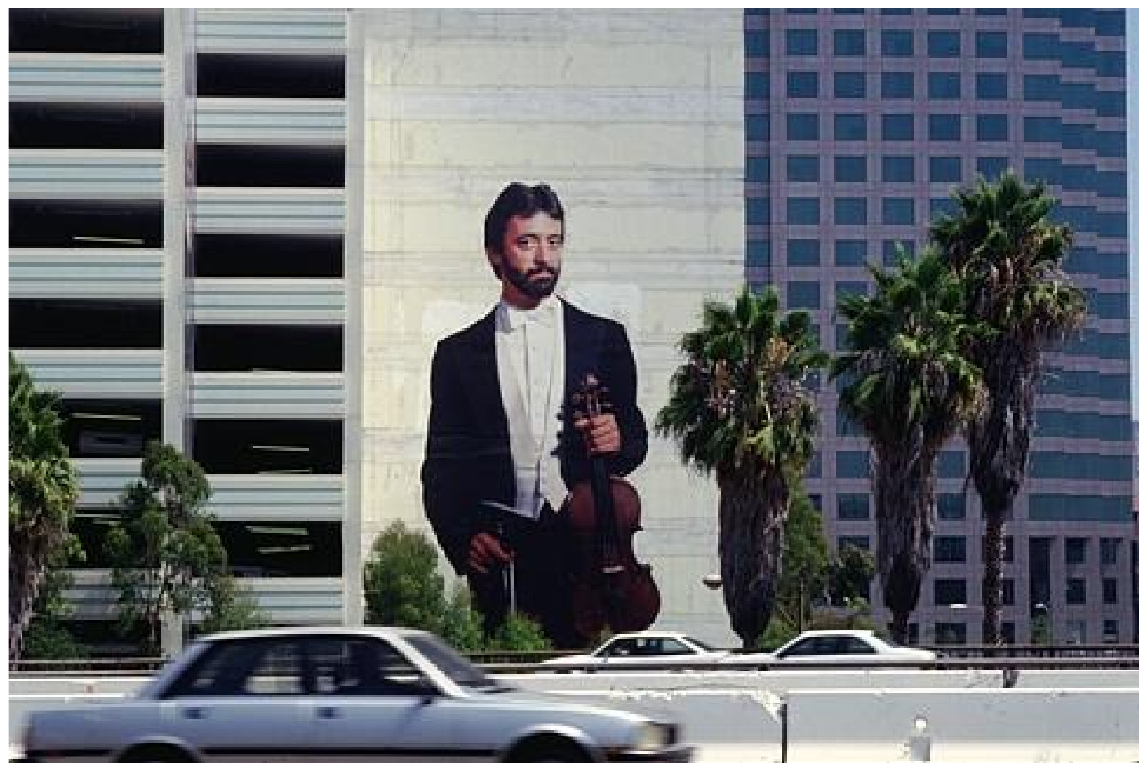}} &
		 \subcaptionbox{House. Size: $321 \times 481$.\label{fig:house}}{\includegraphics[scale=0.325]{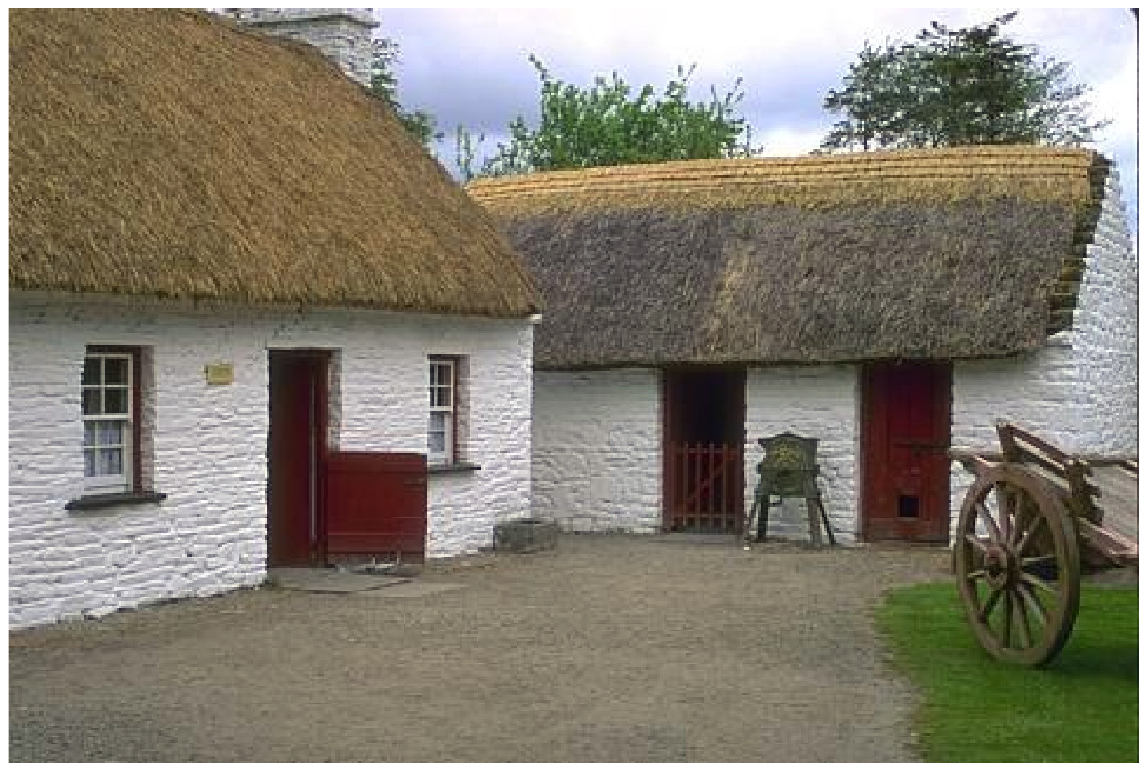}} \\
   &\captionsetup[subfigure]{justification=centering}
\subcaptionbox{Building.\\ Size: $481 \times 321$. \label{fig:building}}{\includegraphics[scale=0.35]{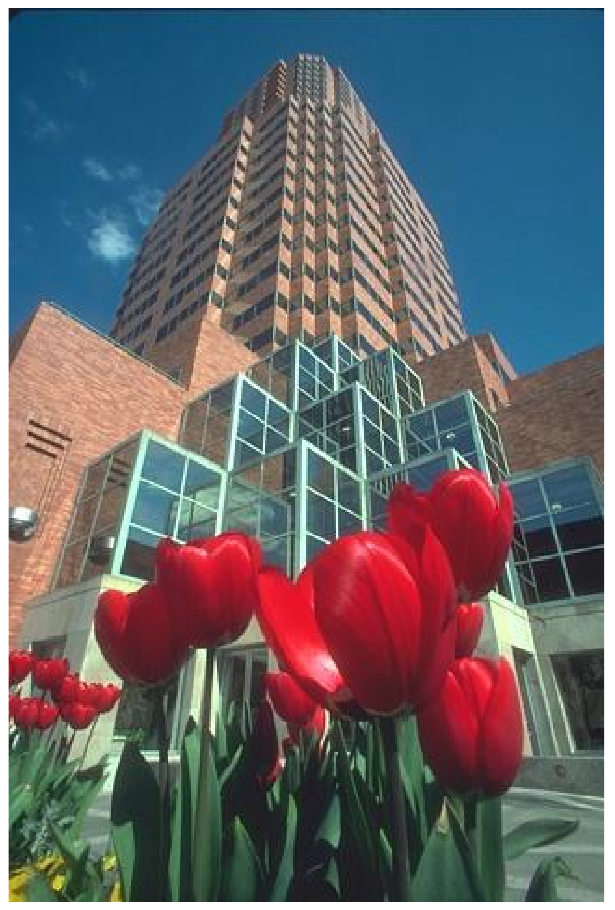}}&
	\end{tabular}
	\caption{Real color images for image segmentation. }
	\label{fig:real_color}
\end{figure}

  \begin{table}[t!]
 	\caption{Comparison of the PSNRs and computational times (seconds) between the segmentation methods applied to the images in Figure \ref{fig:real_color} corrupted with either Gaussian noise with mean zero and variance 0.025 or 10\% SP noise. Number in \textbf{bold} indicates either the highest PSNR or the fastest time among the segmentation methods for a given image. }\label{tab:color_real}
 	
 	\begin{subtable}[h]{\textwidth}
 		\centering
 		\begin{adjustbox}{width=1\textwidth}
 			\begin{tabular}{l|cc||cc||cc||cc|}
 				\hhline{~|--------}
 				& \multicolumn{2}{c||}{ \makecell{garden (Figure \ref{fig:circle}) \\
 						$k=3$}}    & \multicolumn{2}{c||}{ \makecell{man (Figure \ref{fig:man})\\
 						$k=5$}}    & \multicolumn{2}{c||}{\makecell{house (Figure \ref{fig:house})\\ $k=6$}}    & \multicolumn{2}{c|}{\makecell{building (Figure \ref{fig:building}) \\ $k=8$}}    \\ \hhline{~|--------} 
 				& \multicolumn{1}{c|}{PSNR} & \multicolumn{1}{c||}{Time (s) } & \multicolumn{1}{c|}{PSNR} & \multicolumn{1}{c||}{Time (s) }  & \multicolumn{1}{c|}{PSNR} & \multicolumn{1}{c||}{Time (s) }  & \multicolumn{1}{c|}{PSNR} & \multicolumn{1}{c|}{Time (s) }   \\ \hline
 				\multicolumn{1}{|l|}{(Original) SLaT} & \multicolumn{1}{c|}{\textbf{20.45}} & \multicolumn{1}{c||}{7.09} & \multicolumn{1}{c|}{21.11} & \multicolumn{1}{c||}{13.68} & \multicolumn{1}{c|}{21.88} & \multicolumn{1}{c||}{13.56} & \multicolumn{1}{c|}{21.76} & \multicolumn{1}{c|}{17.60} \\ \hline
 				\multicolumn{1}{|l|}{TV$^{p}$ SLaT} & \multicolumn{1}{c|}{20.26} & \multicolumn{1}{c||}{11.58} & \multicolumn{1}{c|}{22.11} & \multicolumn{1}{c||}{13.59} & \multicolumn{1}{c|}{21.93} & \multicolumn{1}{c||}{9.83} & \multicolumn{1}{c|}{21.63} & \multicolumn{1}{c|}{12.16}  \\ \hline
 				\multicolumn{1}{|l|}{AITV SLaT (ADMM)} & \multicolumn{1}{c|}{20.42} & \multicolumn{1}{c||}{7.52} & \multicolumn{1}{c|}{22.19} & \multicolumn{1}{c||}{15.51} & \multicolumn{1}{c|}{21.85} & \multicolumn{1}{c||}{10.99} & \multicolumn{1}{c|}{\textbf{21.78}} & \multicolumn{1}{c|}{11.90}  \\ \hline
 				\multicolumn{1}{|l|}{AITV SLaT (DCA)} & \multicolumn{1}{c|}{20.42} & \multicolumn{1}{c||}{38.27} & \multicolumn{1}{c|}{\textbf{22.21}} & \multicolumn{1}{c||}{53.51} & \multicolumn{1}{c|}{21.81} & \multicolumn{1}{c||}{47.54} & \multicolumn{1}{c|}{\textbf{21.78}} & \multicolumn{1}{c|}{52.73}  \\ \hline
 				\multicolumn{1}{|l|}{AITV FR} & \multicolumn{1}{c|}{18.51} & \multicolumn{1}{c||}{148.62} & \multicolumn{1}{c|}{21.47} & \multicolumn{1}{c||}{354.32} & \multicolumn{1}{c|}{20.35} & \multicolumn{1}{c||}{447.50} & \multicolumn{1}{c|}{19.91} & \multicolumn{1}{c|}{583.13} \\ \hline
 				\multicolumn{1}{|l|}{ICTM} & \multicolumn{1}{c|}{20.28} & \multicolumn{1}{c||}{\textbf{6.36}} & \multicolumn{1}{c|}{20.29} & \multicolumn{1}{c||}{11.05} & \multicolumn{1}{c|}{19.89} & \multicolumn{1}{c||}{\textbf{5.22}} & \multicolumn{1}{c|}{18.63} & \multicolumn{1}{c|}{12.41} \\ \hline
 				\multicolumn{1}{|l|}{TV$^p$ MS} & \multicolumn{1}{c|}{19.76} & \multicolumn{1}{c||}{88.64} & \multicolumn{1}{c|}{22.10} & \multicolumn{1}{c||}{52.55} & \multicolumn{1}{c|}{\textbf{22.10}} & \multicolumn{1}{c||}{169.34} & \multicolumn{1}{c|}{21.35} & \multicolumn{1}{c|}{53.62} \\ \hline
 				\multicolumn{1}{|l|}{Convex Potts} & \multicolumn{1}{c|}{18.50} & \multicolumn{1}{c||}{14.93} & \multicolumn{1}{c|}{21.30} & \multicolumn{1}{c||}{52.22} & \multicolumn{1}{c|}{20.64} & \multicolumn{1}{c||}{76.07} & \multicolumn{1}{c|}{21.17} & \multicolumn{1}{c|}{152.46} \\ \hline
 				\multicolumn{1}{|l|}{SaT-Potts} & \multicolumn{1}{c|}{19.34} & \multicolumn{1}{c||}{9.78} & \multicolumn{1}{c|}{21.71} & \multicolumn{1}{c||}{\textbf{7.60}} & \multicolumn{1}{c|}{21.62} & \multicolumn{1}{c||}{7.31} & \multicolumn{1}{c|}{21.56} & \multicolumn{1}{c|}{\textbf{7.85}} \\ \hline
 			\end{tabular}
 		\end{adjustbox}
 		\caption{Gaussian noise}
 	\end{subtable}
 	
 	\begin{subtable}[h]{\textwidth}
 		\centering
 		\begin{adjustbox}{width=1\textwidth}
 			\begin{tabular}{l|cc||cc||cc||cc|}
 				\hhline{~|--------}
 				& \multicolumn{2}{c||}{ \makecell{garden (Figure \ref{fig:circle}) \\
 						$k=3$}}    & \multicolumn{2}{c||}{ \makecell{man (Figure \ref{fig:man})\\
 						$k=5$}}    & \multicolumn{2}{c||}{\makecell{house (Figure \ref{fig:house})\\ $k=6$}}    & \multicolumn{2}{c|}{\makecell{building (Figure \ref{fig:building}) \\ $k=8$}}    \\ \hhline{~|--------} 
 				& \multicolumn{1}{c|}{PSNR} & \multicolumn{1}{c||}{Time (s) } & \multicolumn{1}{c|}{PSNR} & \multicolumn{1}{c||}{Time (s) }  & \multicolumn{1}{c|}{PSNR} & \multicolumn{1}{c||}{Time (s) }  & \multicolumn{1}{c|}{PSNR} & \multicolumn{1}{c|}{Time (s) }   \\ \hline
 				\multicolumn{1}{|l|}{(Original) SLaT} & \multicolumn{1}{c|}{19.22} & \multicolumn{1}{c||}{\textbf{6.53}} & \multicolumn{1}{c|}{19.83} & \multicolumn{1}{c||}{16.26} & \multicolumn{1}{c|}{20.95} & \multicolumn{1}{c||}{13.71} & \multicolumn{1}{c|}{20.08} & \multicolumn{1}{c|}{18.79} \\ \hline
 				\multicolumn{1}{|l|}{TV$^{p}$ SLaT} & \multicolumn{1}{c|}{18.34} & \multicolumn{1}{c||}{9.06} & \multicolumn{1}{c|}{19.56} & \multicolumn{1}{c||}{9.57} & \multicolumn{1}{c|}{20.43} & \multicolumn{1}{c||}{9.19} & \multicolumn{1}{c|}{19.90} & \multicolumn{1}{c|}{11.77}  \\ \hline
 				\multicolumn{1}{|l|}{AITV SLaT (ADMM)} & \multicolumn{1}{c|}{19.29} & \multicolumn{1}{c||}{9.30} & \multicolumn{1}{c|}{20.13} & \multicolumn{1}{c||}{12.57} & \multicolumn{1}{c|}{21.08} & \multicolumn{1}{c||}{9.59} & \multicolumn{1}{c|}{20.53} & \multicolumn{1}{c|}{16.10}  \\ \hline
 				\multicolumn{1}{|l|}{AITV SLaT (DCA)} & \multicolumn{1}{c|}{18.99} & \multicolumn{1}{c||}{47.62} & \multicolumn{1}{c|}{20.09} & \multicolumn{1}{c||}{76.38} & \multicolumn{1}{c|}{20.91} & \multicolumn{1}{c||}{84.83} & \multicolumn{1}{c|}{19.97} & \multicolumn{1}{c|}{96.29}  \\ \hline
 				\multicolumn{1}{|l|}{AITV FR} & \multicolumn{1}{c|}{18.33} & \multicolumn{1}{c||}{144.43} & \multicolumn{1}{c|}{19.96} & \multicolumn{1}{c||}{307.03} & \multicolumn{1}{c|}{20.35} & \multicolumn{1}{c||}{565.65} & \multicolumn{1}{c|}{19.15} & \multicolumn{1}{c|}{759.58} \\ \hline
 				\multicolumn{1}{|l|}{ICTM} & \multicolumn{1}{c|}{17.73} & \multicolumn{1}{c||}{27.95} & \multicolumn{1}{c|}{18.44} & \multicolumn{1}{c||}{24.06} & \multicolumn{1}{c|}{17.95} & \multicolumn{1}{c||}{\textbf{9.34}} & \multicolumn{1}{c|}{16.57} & \multicolumn{1}{c|}{20.40} \\ \hline
 				\multicolumn{1}{|l|}{TV$^p$ MS} & \multicolumn{1}{c|}{\textbf{19.44}} & \multicolumn{1}{c||}{97.90} & \multicolumn{1}{c|}{\textbf{20.25}} & \multicolumn{1}{c||}{44.43} & \multicolumn{1}{c|}{\textbf{21.27}} & \multicolumn{1}{c||}{113.61} & \multicolumn{1}{c|}{\textbf{20.60}} & \multicolumn{1}{c|}{101.82} \\ \hline
 				\multicolumn{1}{|l|}{Convex Potts} & \multicolumn{1}{c|}{18.86} & \multicolumn{1}{c||}{14.95} & \multicolumn{1}{c|}{19.26} & \multicolumn{1}{c||}{45.86} & \multicolumn{1}{c|}{19.62} & \multicolumn{1}{c||}{76.83} & \multicolumn{1}{c|}{19.18} & \multicolumn{1}{c|}{154.43} \\ \hline
 				\multicolumn{1}{|l|}{SaT-Potts} & \multicolumn{1}{c|}{18.27} & \multicolumn{1}{c||}{8.58} & \multicolumn{1}{c|}{18.94} & \multicolumn{1}{c||}{\textbf{6.48}} & \multicolumn{1}{c|}{20.20} & \multicolumn{1}{c||}{9.77} & \multicolumn{1}{c|}{19.45} & \multicolumn{1}{c|}{\textbf{7.97}} \\ \hline
 			\end{tabular}
 		\end{adjustbox}
 		\caption{SP noise}
 	\end{subtable}
 	
 \end{table}\clearpage}

\afterpage{\clearpage

\begin{figure}[t!]
\centering
\resizebox{\textwidth}{!}{
\begin{tabular}{cccccc}
\centering
	\parbox[t]{2mm}{\multirow{2}{*}{\rotatebox[origin=c]{90}{Gaussian Noise}}} &	\captionsetup[subfigure]{justification=centering}
\subcaptionbox{Noisy image.}{\includegraphics[width = 1.50in]{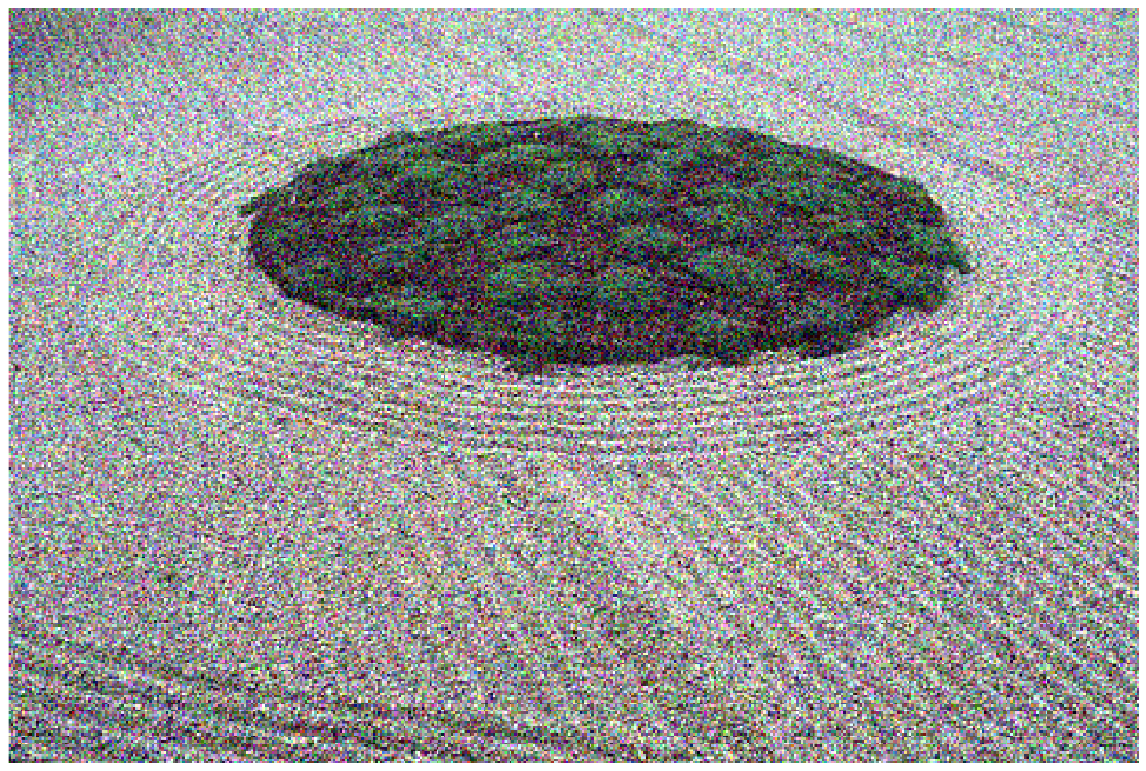}} &
		\captionsetup[subfigure]{justification=centering}
\subcaptionbox{(original) SLaT\\
PSNR: 20.45}{\includegraphics[width = 1.50in]{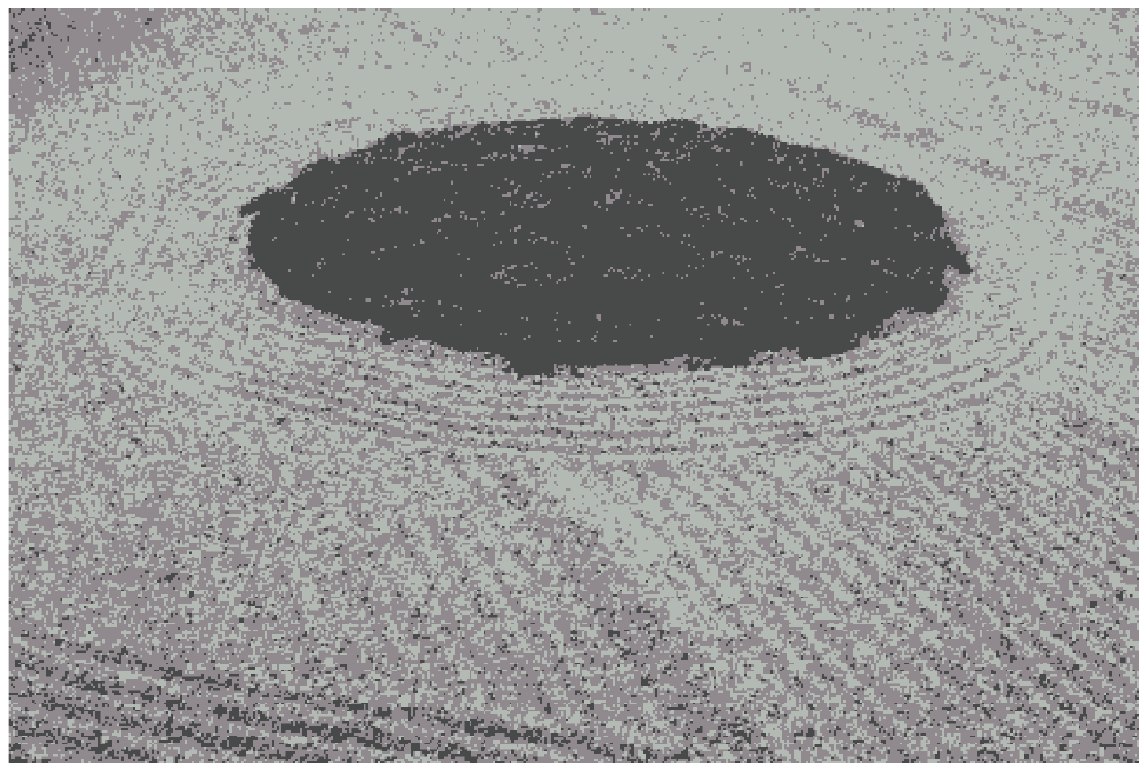}} & 		\captionsetup[subfigure]{justification=centering}
\subcaptionbox{TV$^{p}$ SLaT\\
PSNR: 20.26}{\includegraphics[ width = 1.50in]{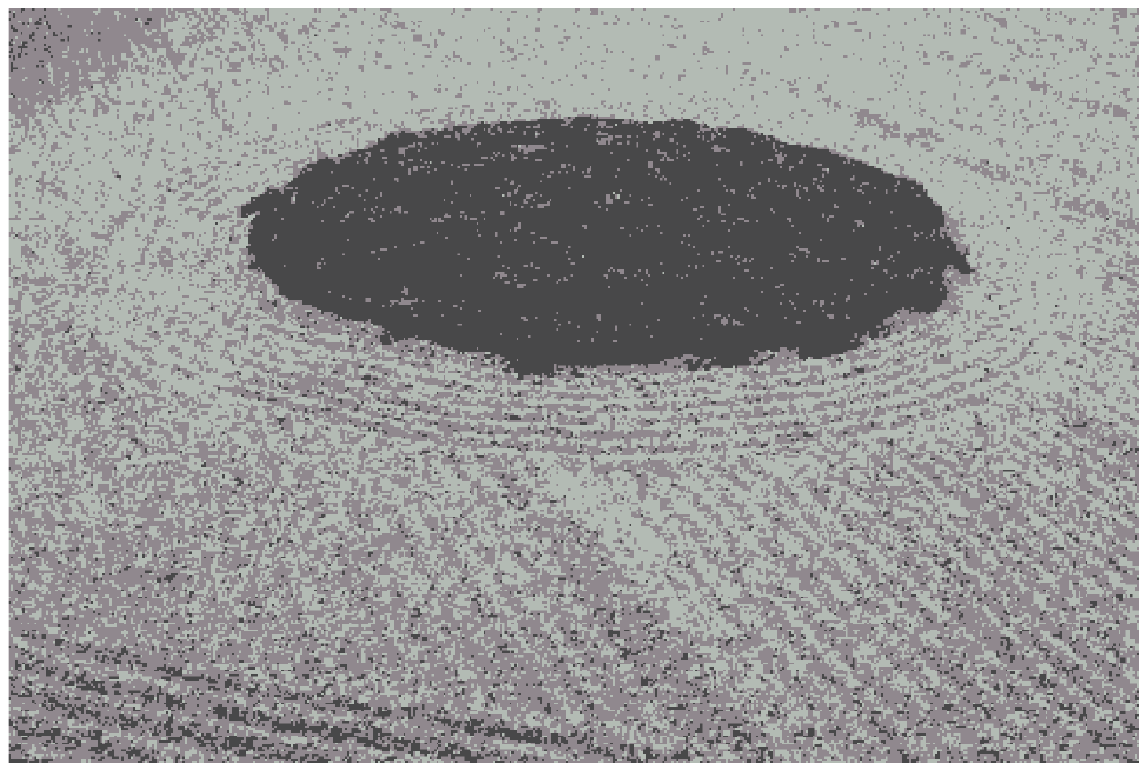}} &		\captionsetup[subfigure]{justification=centering}
\subcaptionbox{AITV   SLaT (ADMM)\\
PSNR: 20.42}{\includegraphics[width = 1.50in]{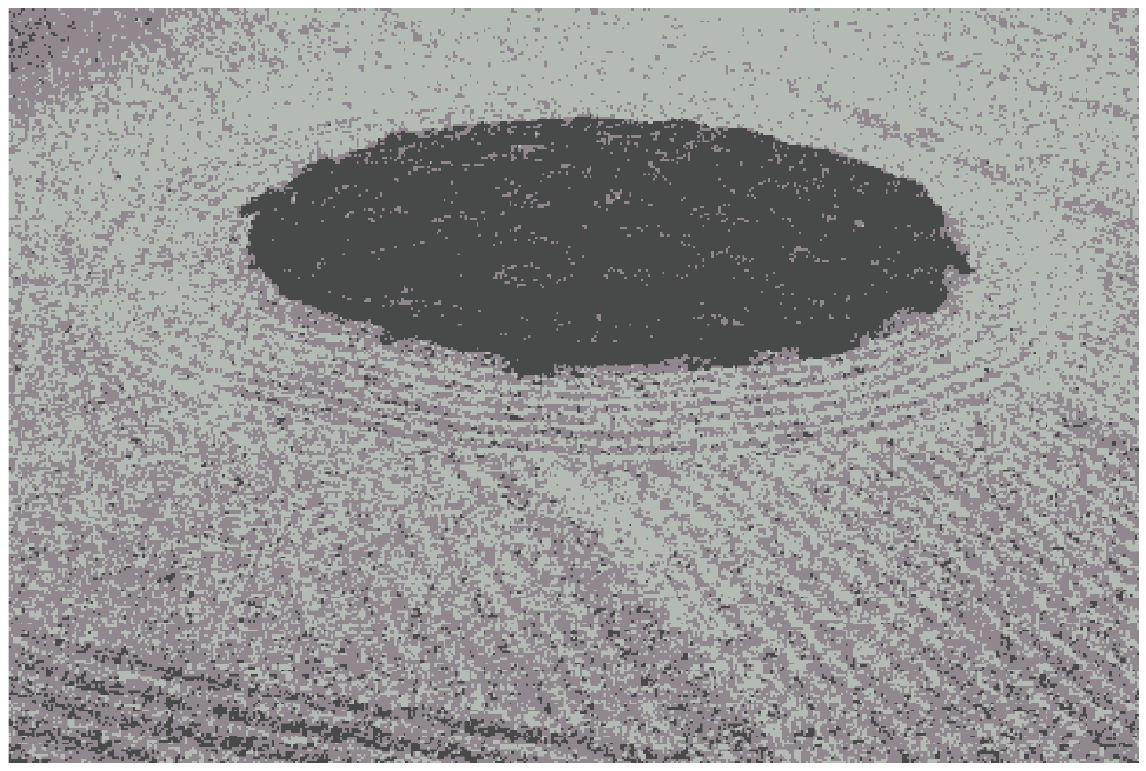}} & 
	\captionsetup[subfigure]{justification=centering}
\subcaptionbox{AITV   SLaT (DCA)\\
PSNR: 20.42}{\includegraphics[width = 1.50in]{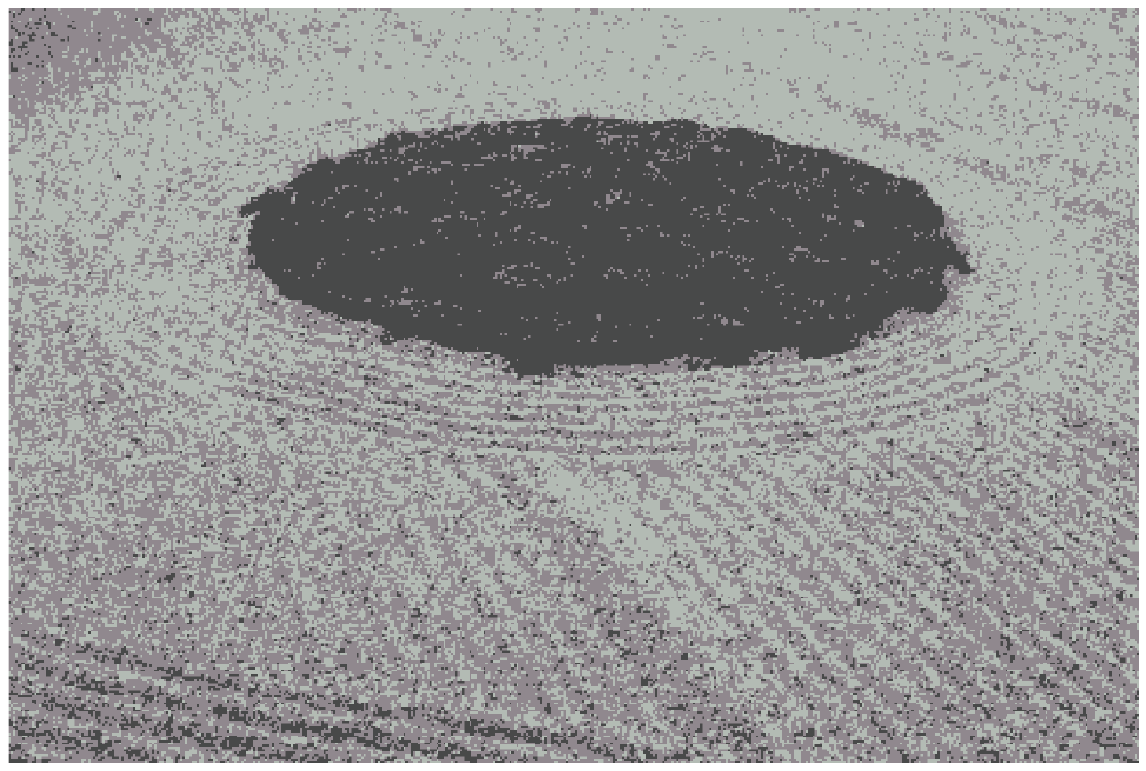}}\\

	&	\captionsetup[subfigure]{justification=centering}
\subcaptionbox{AITV   FR \\ PSNR: 18.51}{\includegraphics[width = 1.50in]{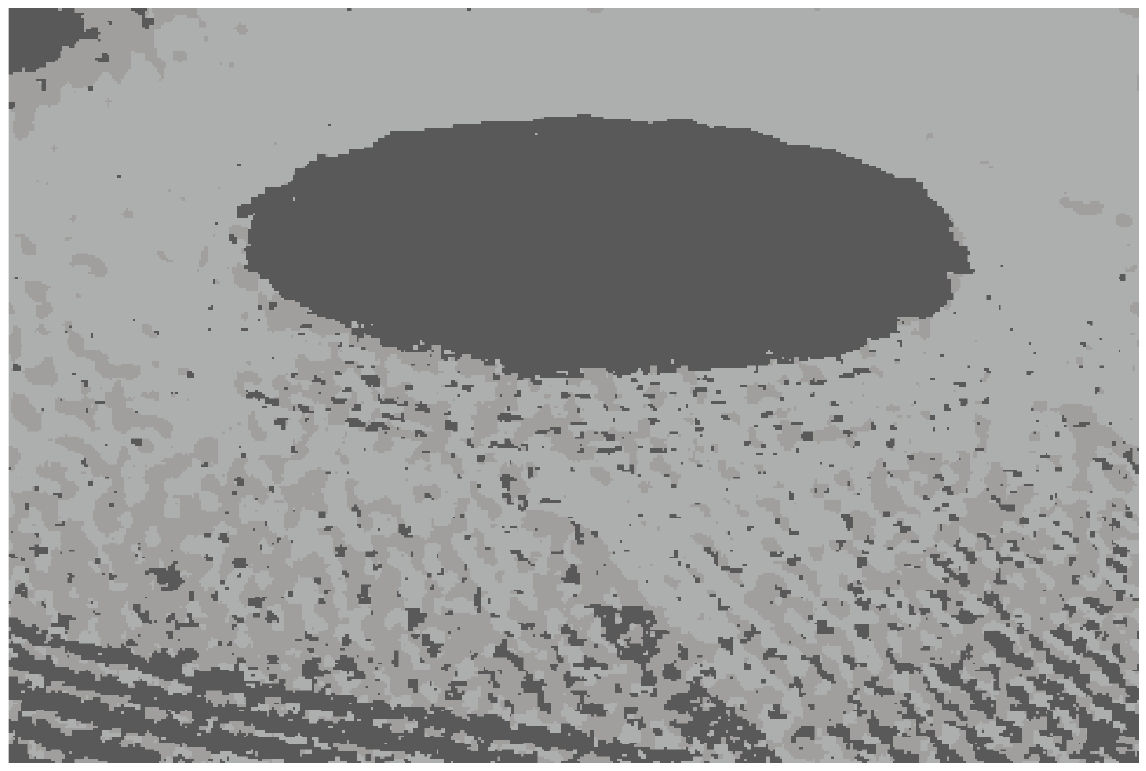}} &		\captionsetup[subfigure]{justification=centering}

\subcaptionbox{ICTM\\
PSNR: 20.28}{\includegraphics[width = 1.50in]{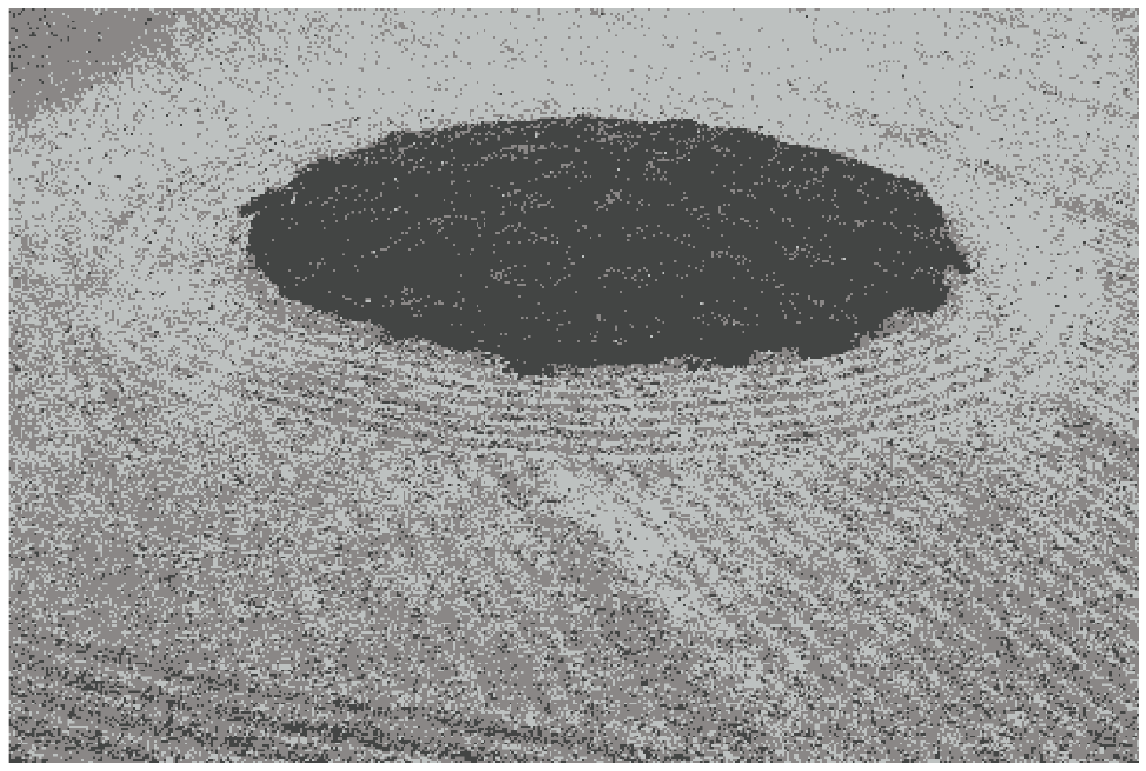}} &		\captionsetup[subfigure]{justification=centering}

\subcaptionbox{TV$^p$ MS\\
PSNR: 19.76}{\includegraphics[ width = 1.50in]{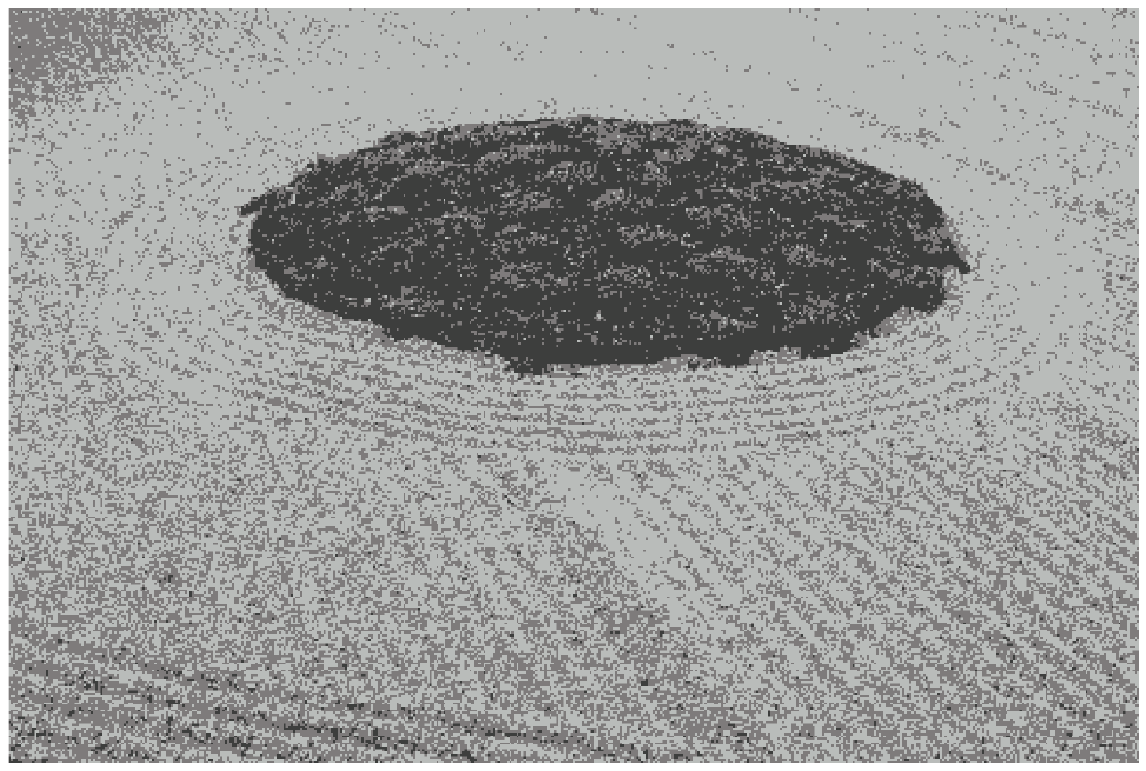}} &		\captionsetup[subfigure]{justification=centering}

\subcaptionbox{Convex Potts\\
PSNR: 18.50}{\includegraphics[ width = 1.50in]{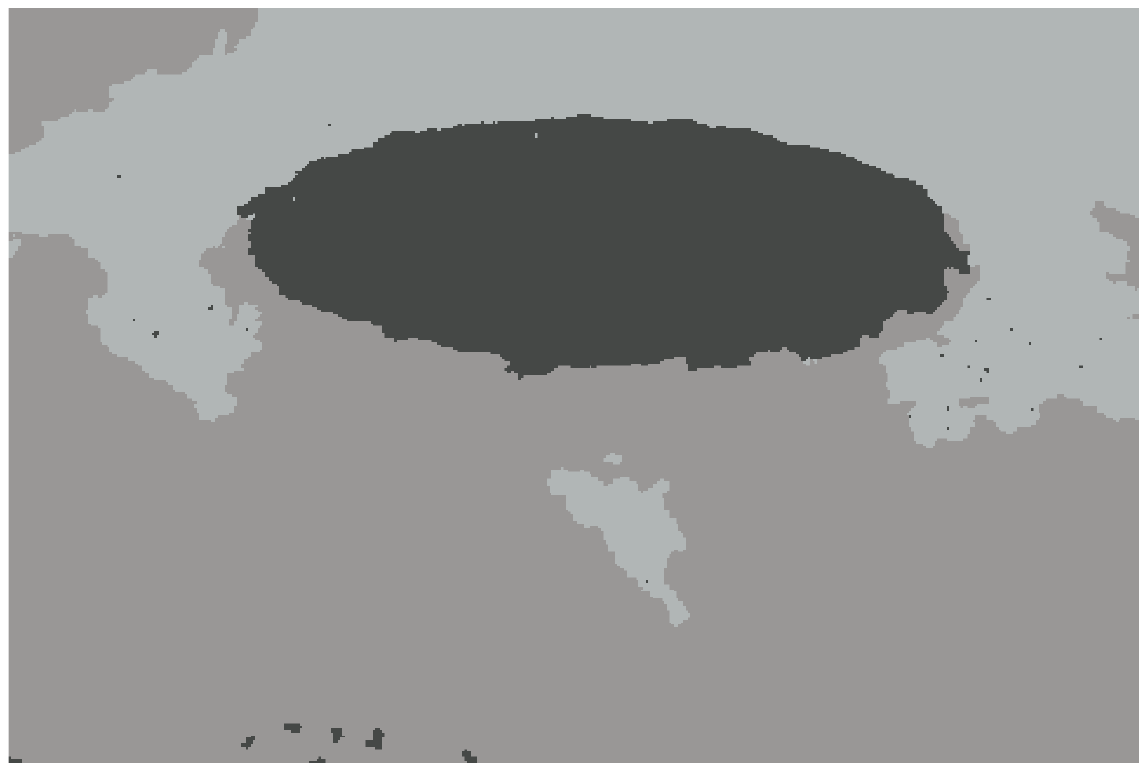}} 
&		\captionsetup[subfigure]{justification=centering}

\subcaptionbox{SaT-Potts\\
PSNR: 19.34}{\includegraphics[ width = 1.50in]{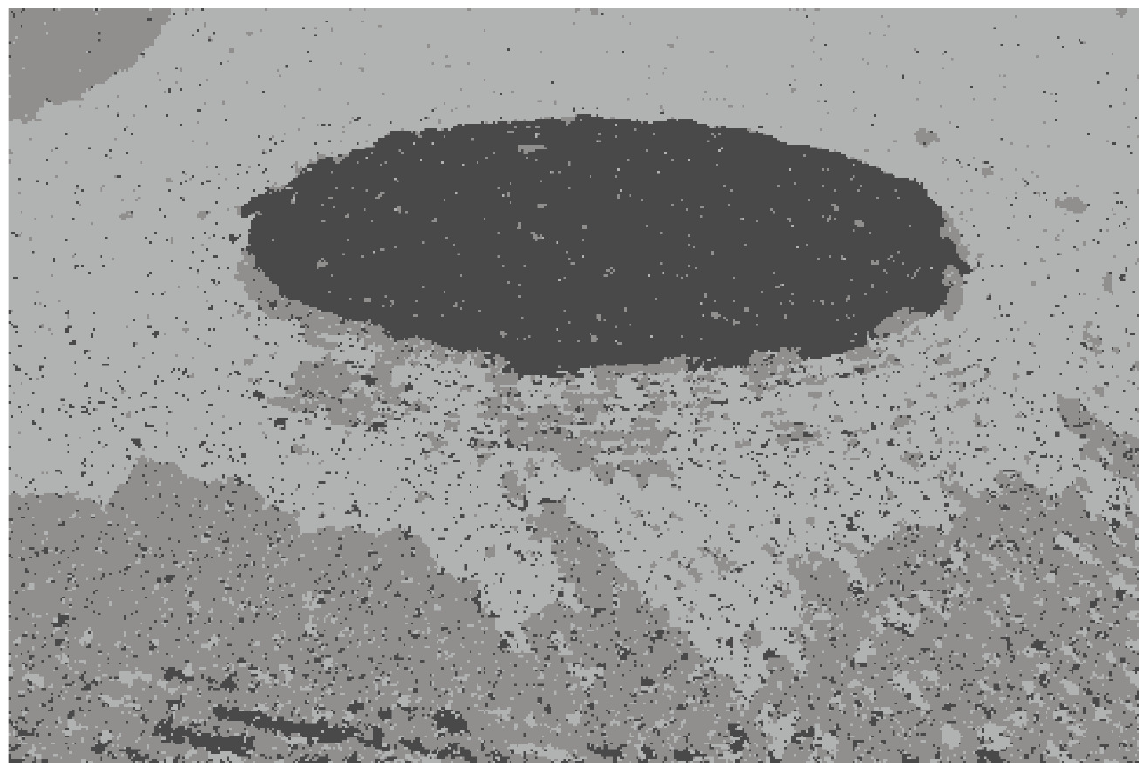}}  \\ \hline \\
	\parbox[t]{2mm}{\multirow{2}{*}{\rotatebox[origin=c]{90}{Salt \& Pepper Noise}}} &	\captionsetup[subfigure]{justification=centering}
\subcaptionbox{Noisy image.}{\includegraphics[width = 1.50in]{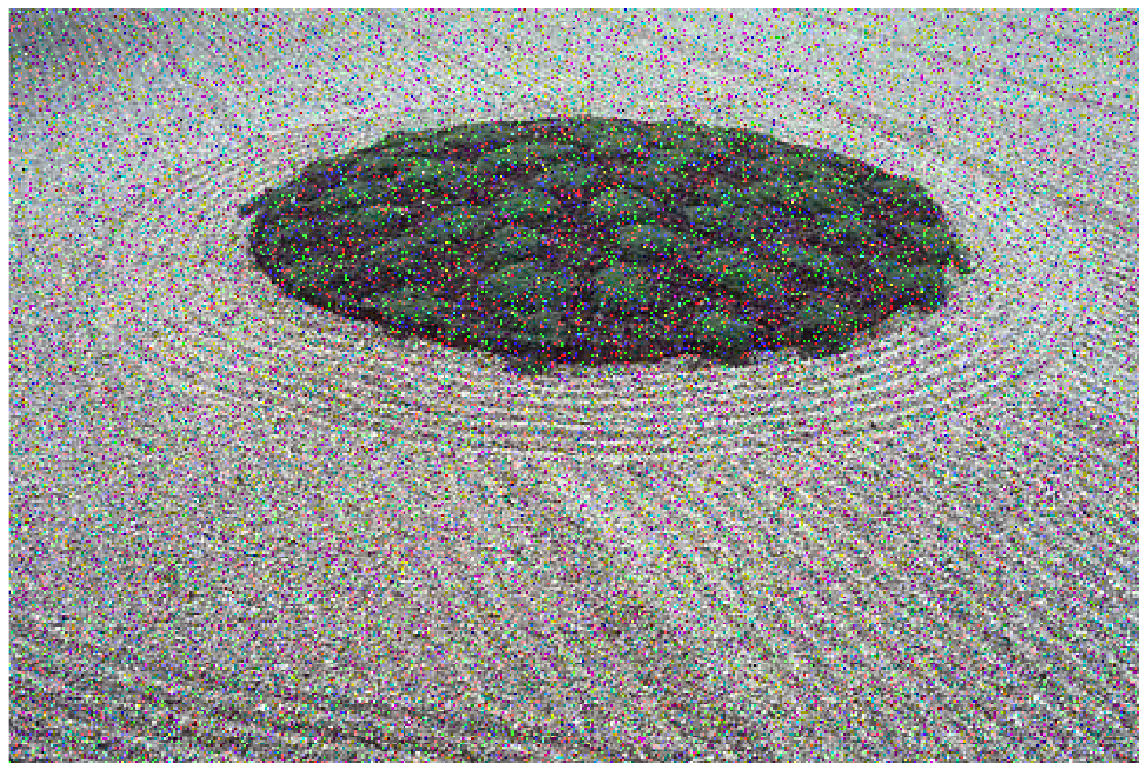}} &
		\captionsetup[subfigure]{justification=centering}
\subcaptionbox{(original) SLaT\\
PSNR: 19.22}{\includegraphics[width = 1.50in]{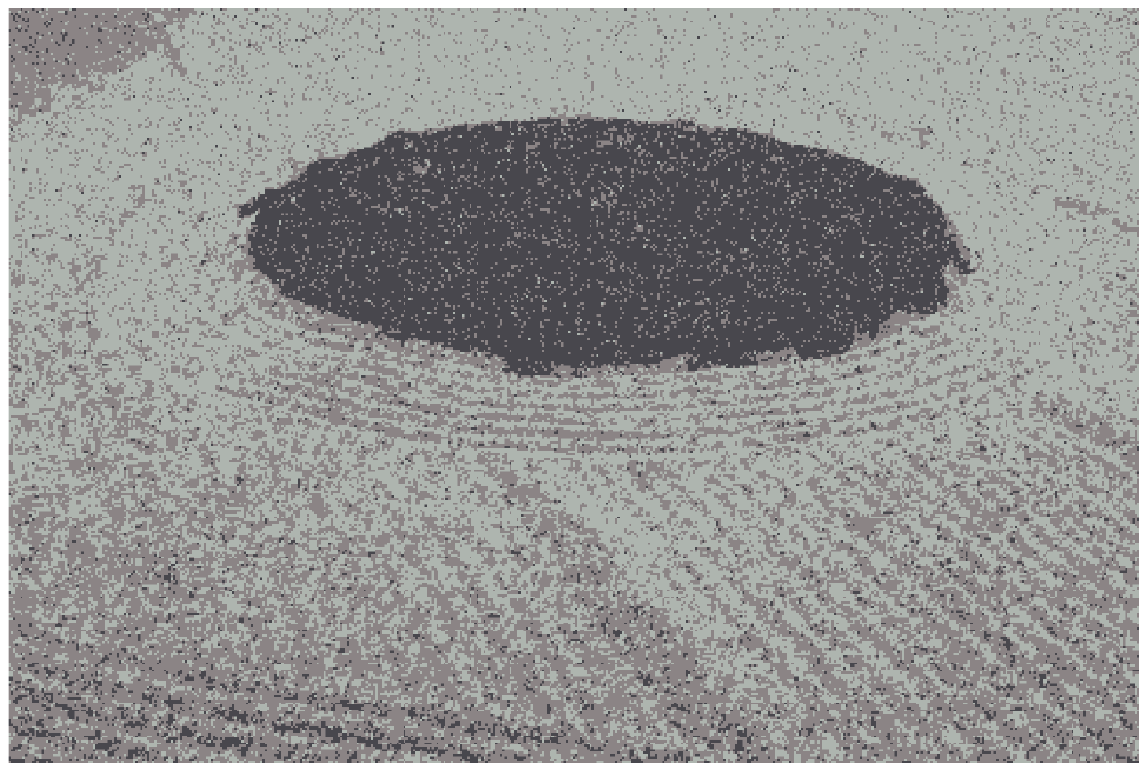}} & 		\captionsetup[subfigure]{justification=centering}
\subcaptionbox{TV$^{p}$ SLaT\\
PSNR: 18.34}{\includegraphics[ width = 1.50in]{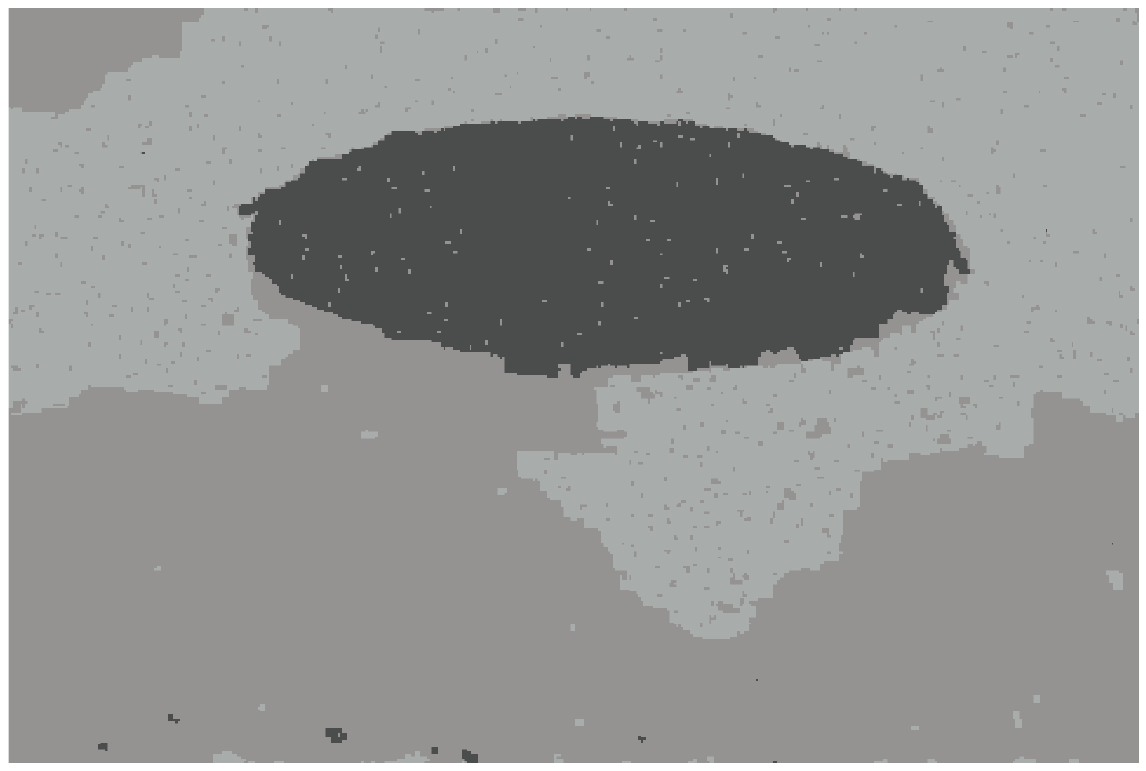}} &		\captionsetup[subfigure]{justification=centering}
\subcaptionbox{AITV   SLaT (ADMM)\\
PSNR: 19.29}{\includegraphics[width = 1.50in]{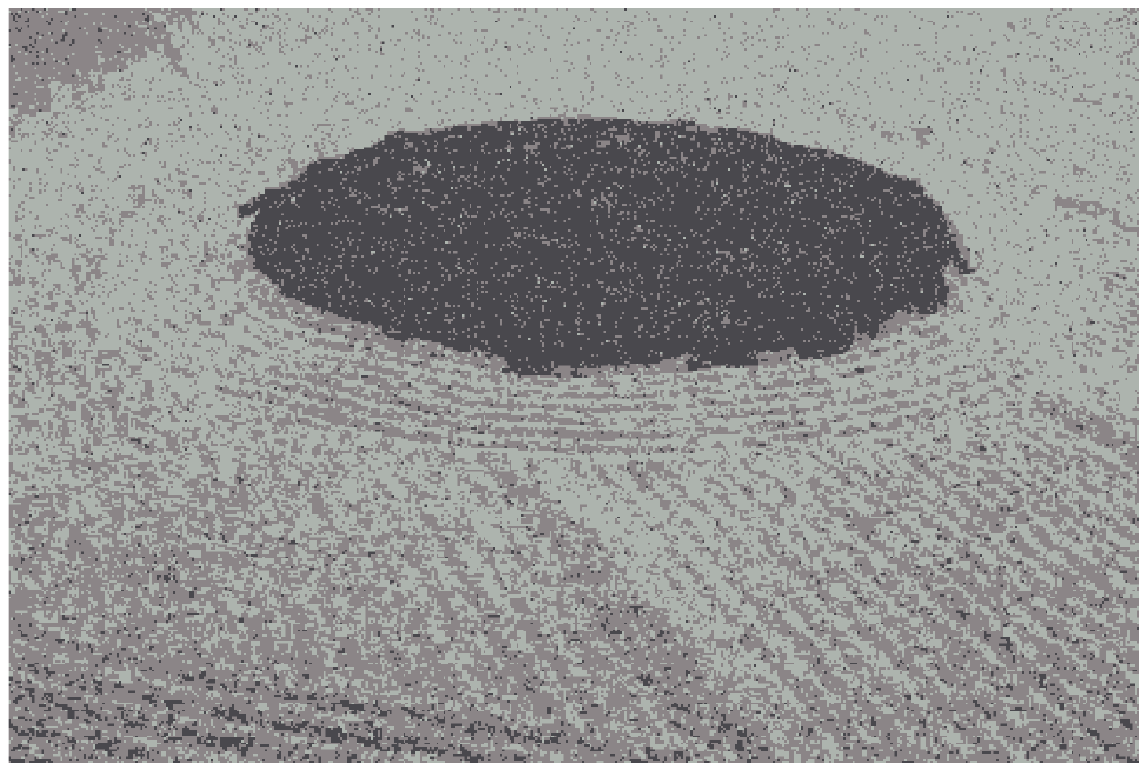}} & 
	\captionsetup[subfigure]{justification=centering}
\subcaptionbox{AITV   SLaT (DCA)\\
PSNR: 18.99}{\includegraphics[width = 1.50in]{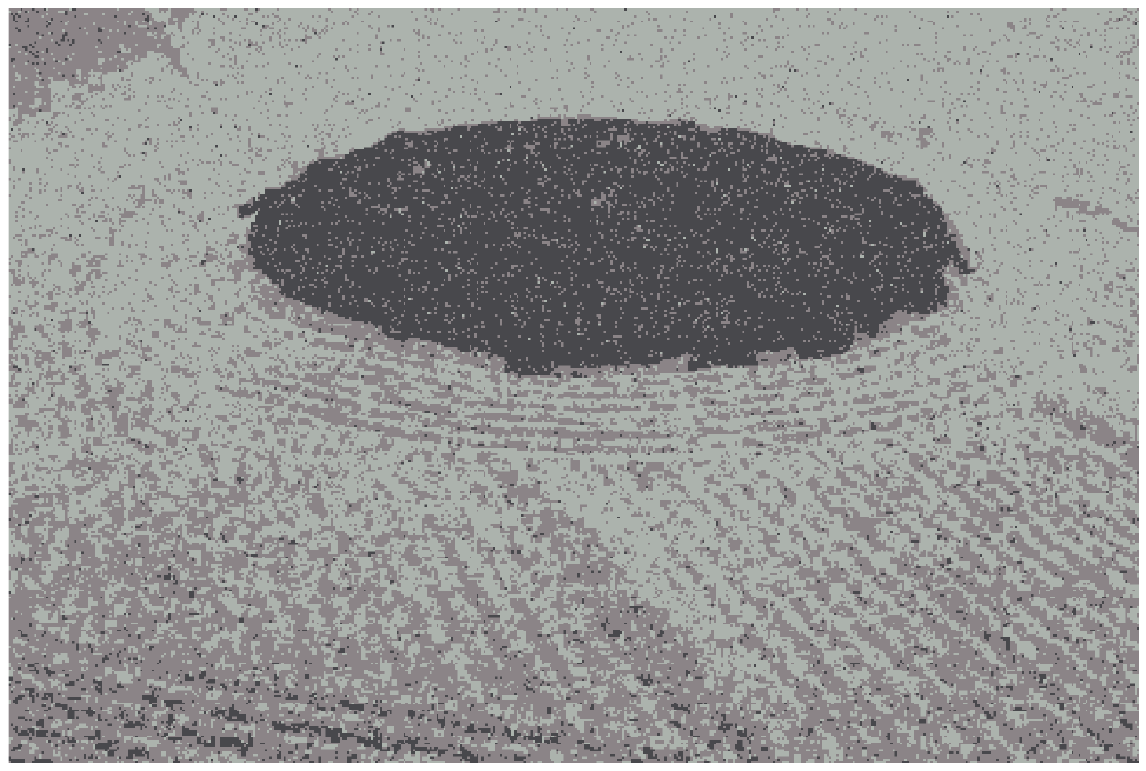}}\\

	&	\captionsetup[subfigure]{justification=centering}
\subcaptionbox{AITV   FR \\ PSNR: 18.33}{\includegraphics[width = 1.50in]{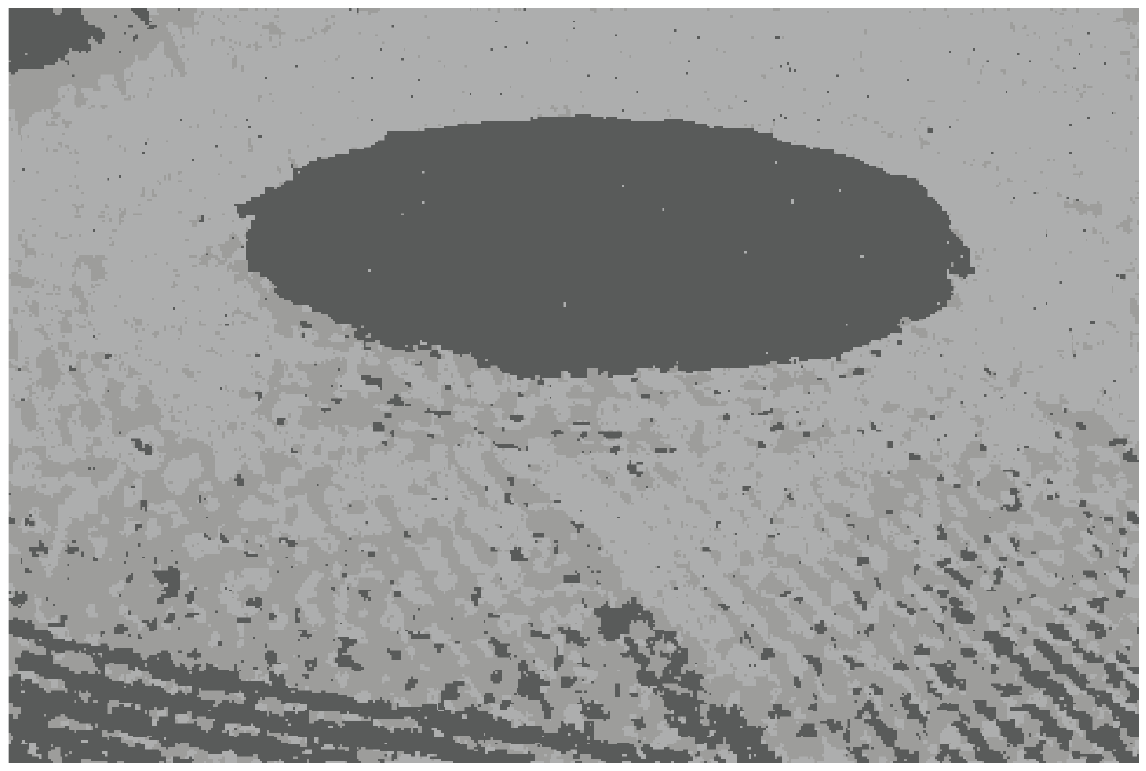}} &		\captionsetup[subfigure]{justification=centering}

\subcaptionbox{ICTM\\
PSNR: 17.73}{\includegraphics[width = 1.50in]{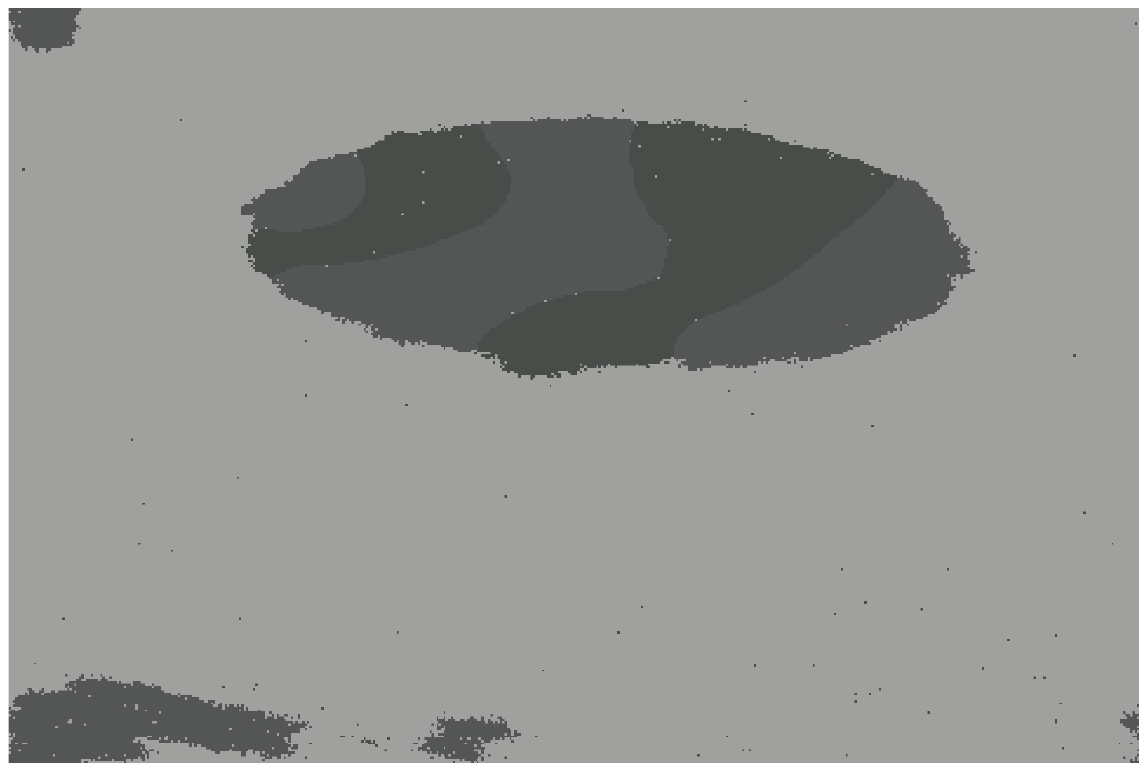}} &		\captionsetup[subfigure]{justification=centering}

\subcaptionbox{TV$^p$ MS\\
PSNR: 19.44}{\includegraphics[ width = 1.50in]{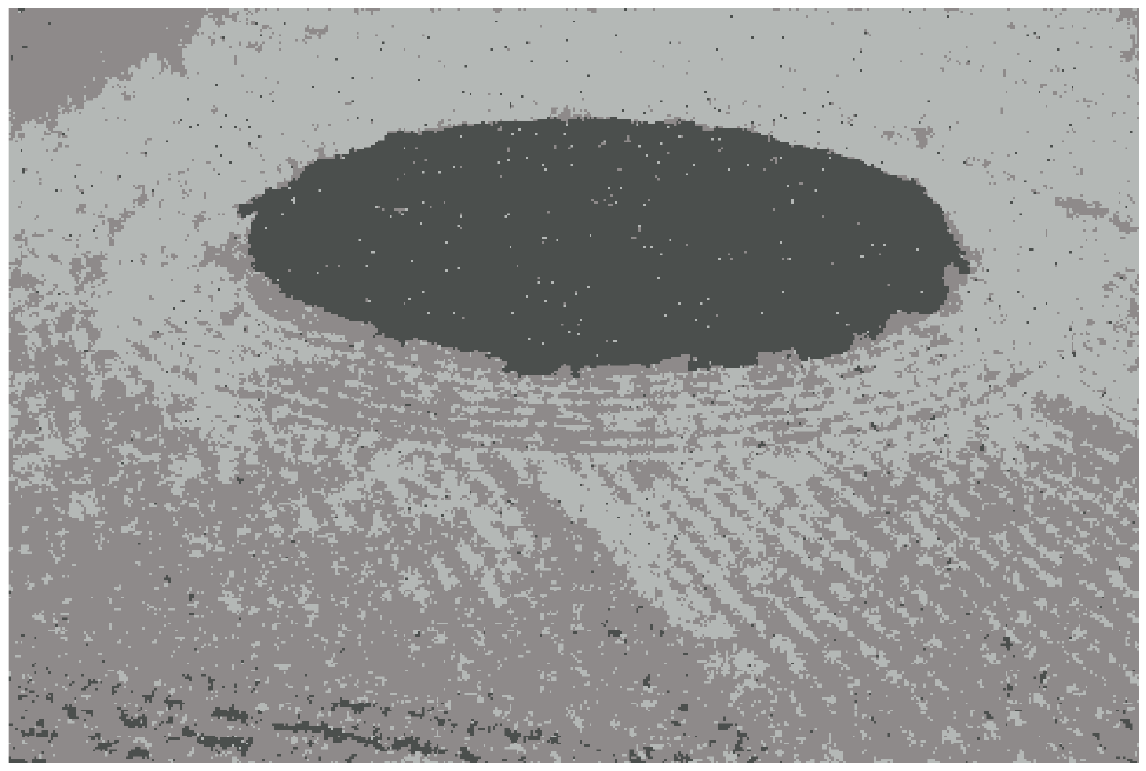}} &		\captionsetup[subfigure]{justification=centering}

\subcaptionbox{Convex Potts\\
PSNR: 18.86}{\includegraphics[ width = 1.50in]{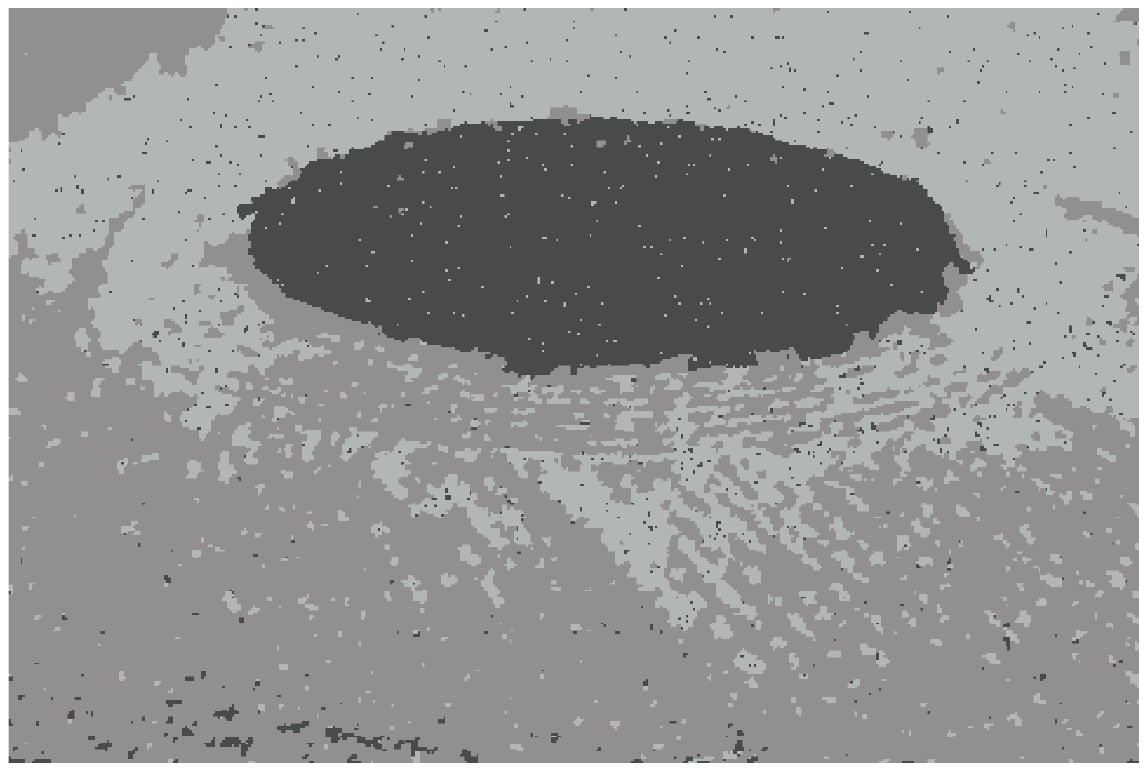}} 
&		\captionsetup[subfigure]{justification=centering}

\subcaptionbox{SaT-Potts\\
PSNR: 18.27}{\includegraphics[ width = 1.50in]{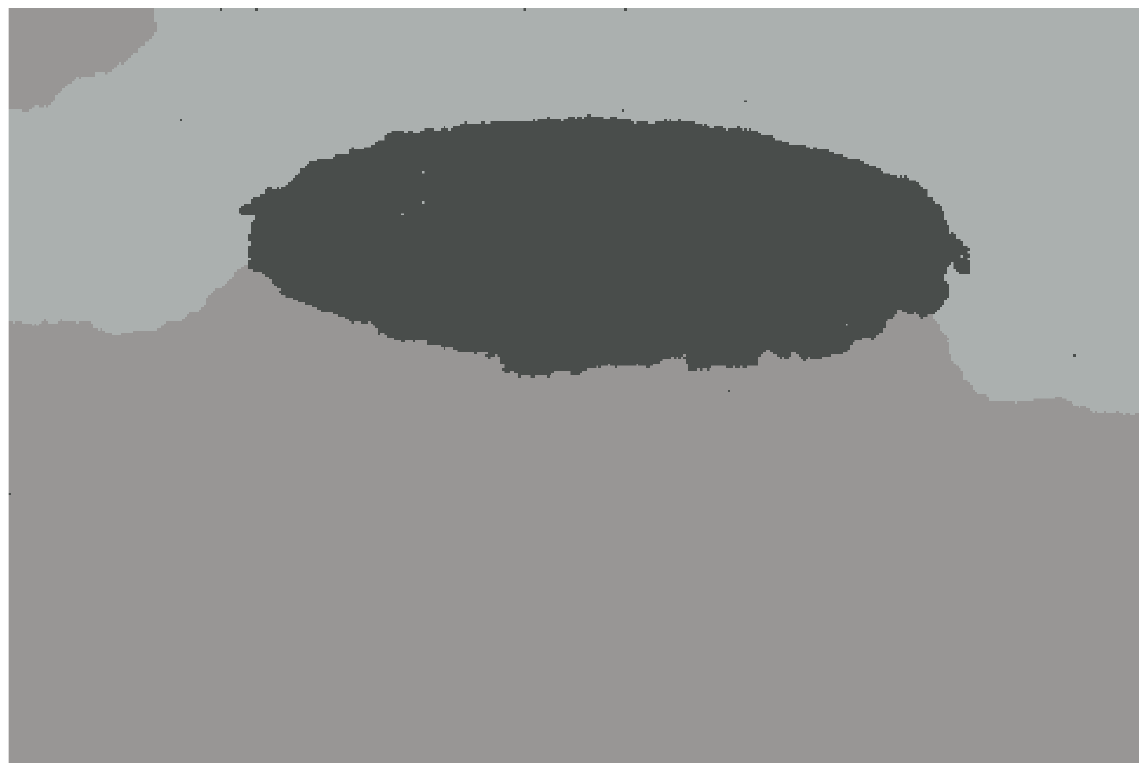}}
		\end{tabular}
}

		\caption{Segmentation results into $k=3$ regions of Figure  {\ref{fig:circle}} corrupted by either Gaussian noise of mean zero and variance 0.025 or 10\% SP noise.}
		\label{fig:circle_result}		
\end{figure}

\begin{figure}
\centering
\resizebox{\textwidth}{!}{\begin{tabular}{cccccc}
\centering
	\parbox[t]{2mm}{\multirow{2}{*}{\rotatebox[origin=c]{90}{Gaussian Noise}}} &
		\captionsetup[subfigure]{justification=centering}
\subcaptionbox{Noisy image.}{\includegraphics[width = 1.50in]{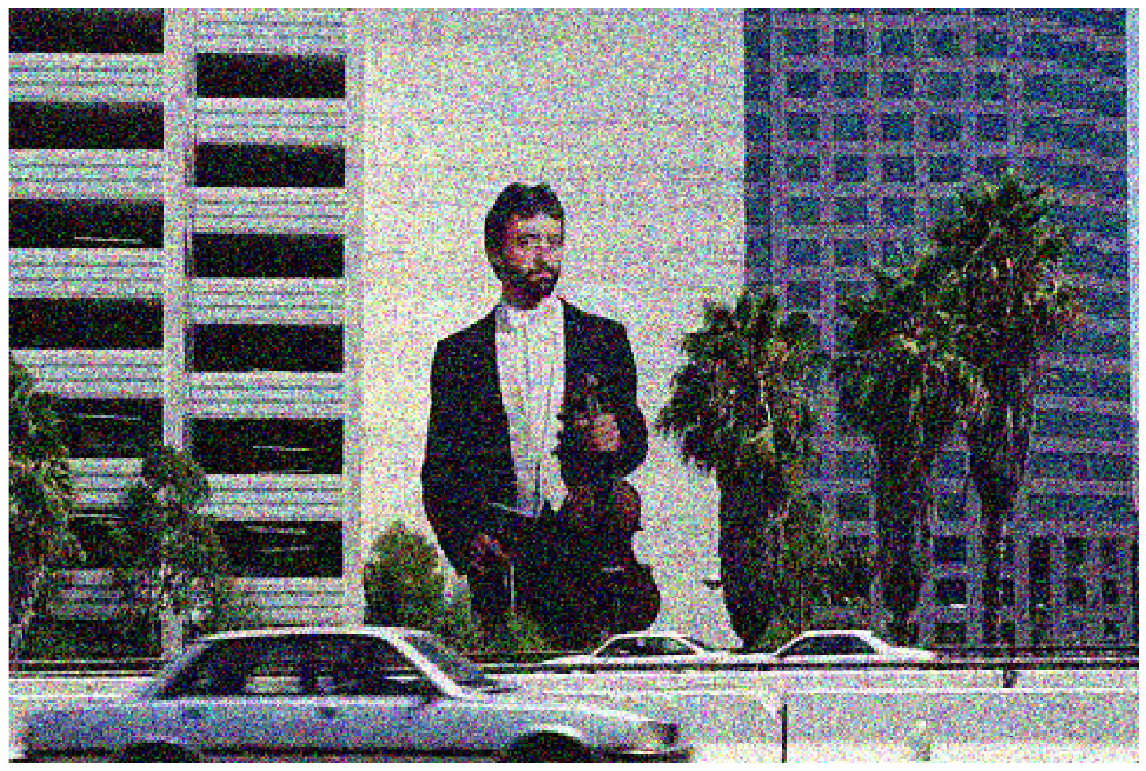}} &
		\captionsetup[subfigure]{justification=centering}

\subcaptionbox{(original) SLaT\\
PSNR: 21.11}{\includegraphics[width = 1.50in]{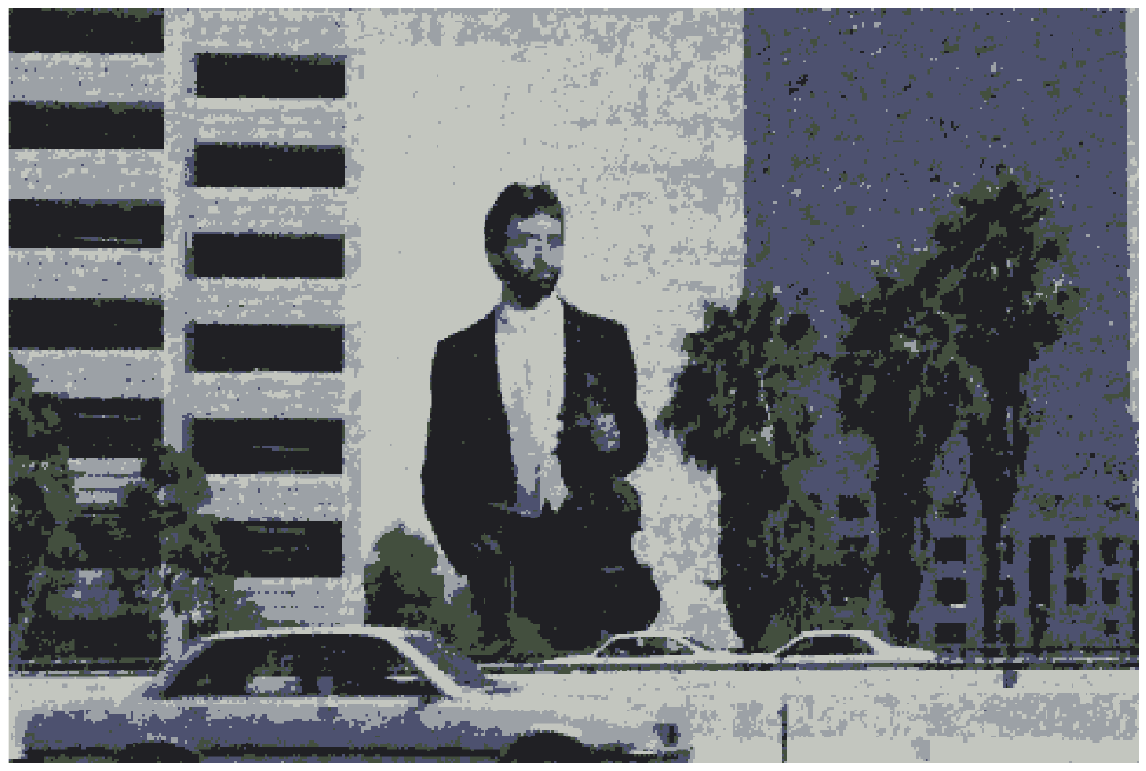}} & 		\captionsetup[subfigure]{justification=centering}
\subcaptionbox{TV$^{p}$ SLaT\\
PSNR: 22.11}{\includegraphics[ width = 1.50in]{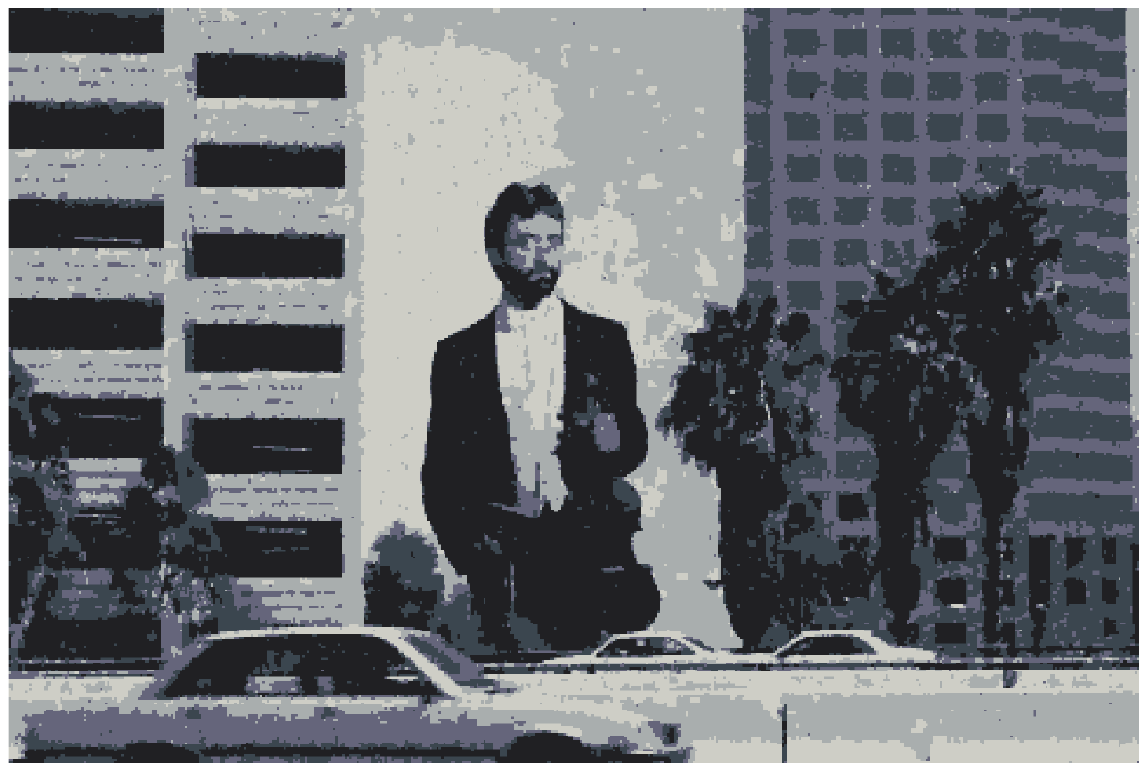}} &		\captionsetup[subfigure]{justification=centering}
\subcaptionbox{AITV   SLaT (ADMM)\\
PSNR: 22.19}{\includegraphics[width = 1.50in]{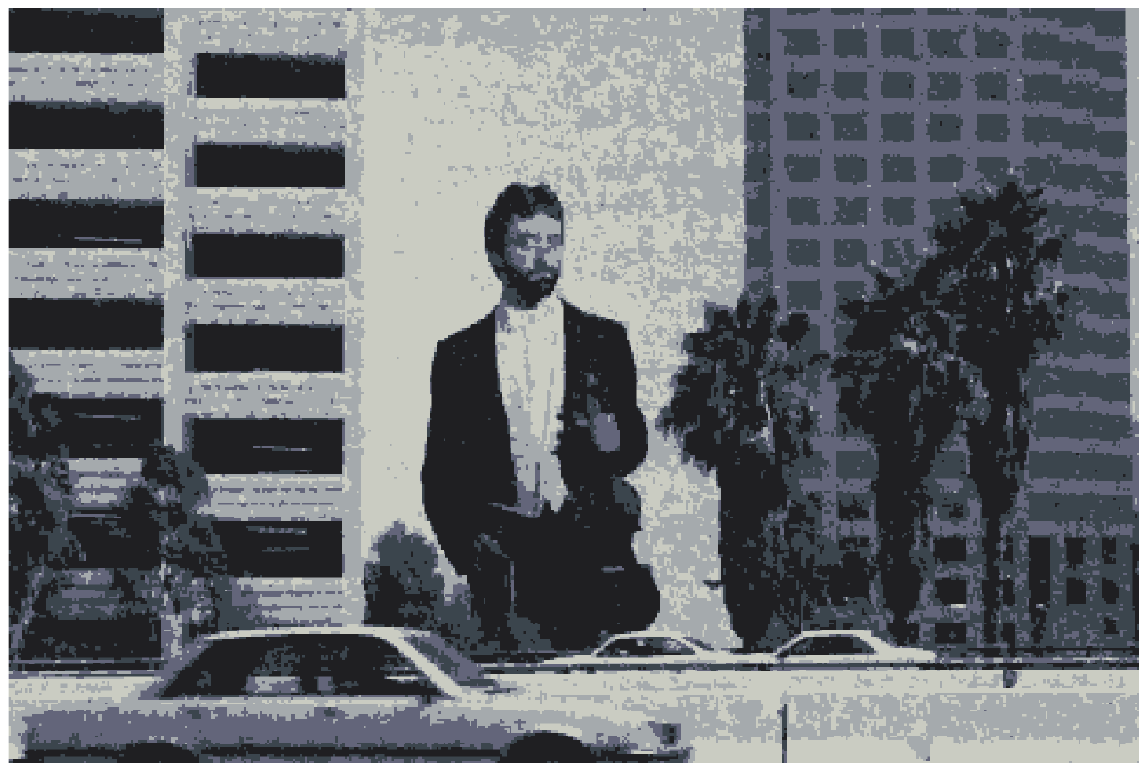}} & \captionsetup[subfigure]{justification=centering}
\subcaptionbox{AITV   SLaT (DCA)\\
PSNR: 22.21}{\includegraphics[width = 1.50in]{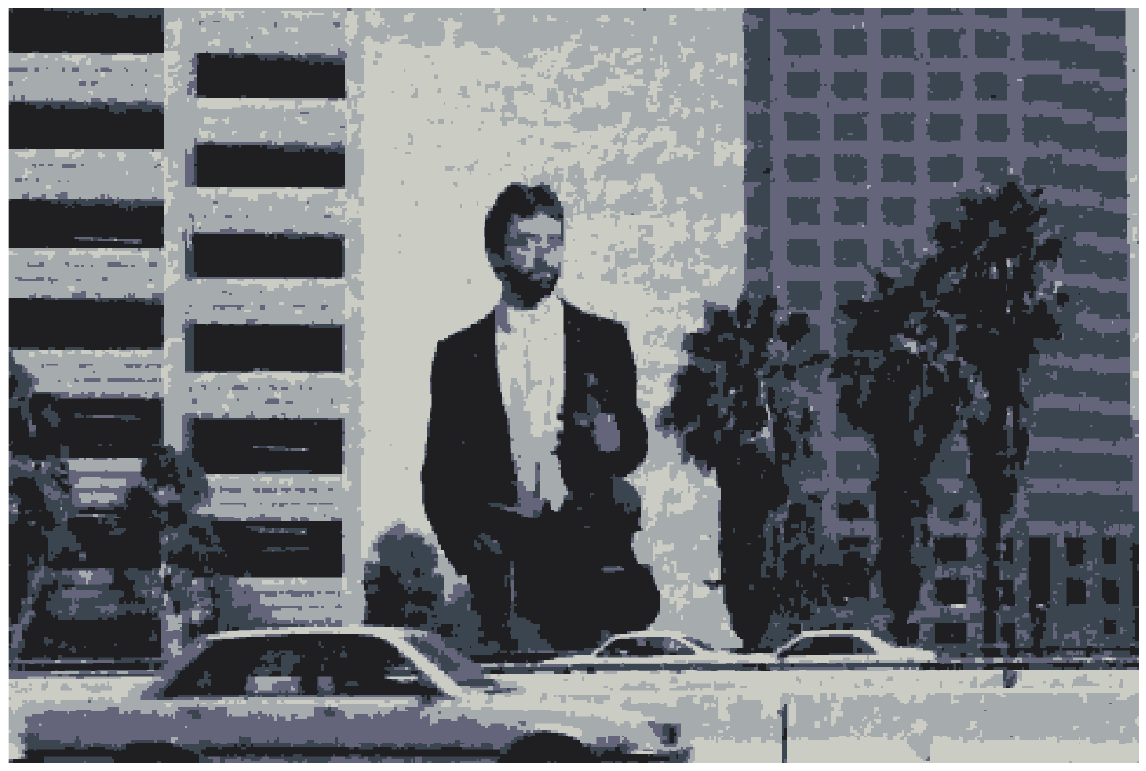}}
\\ &
		\captionsetup[subfigure]{justification=centering}
\subcaptionbox{AITV   FR \\ PSNR: 21.47}{\includegraphics[width = 1.50in]{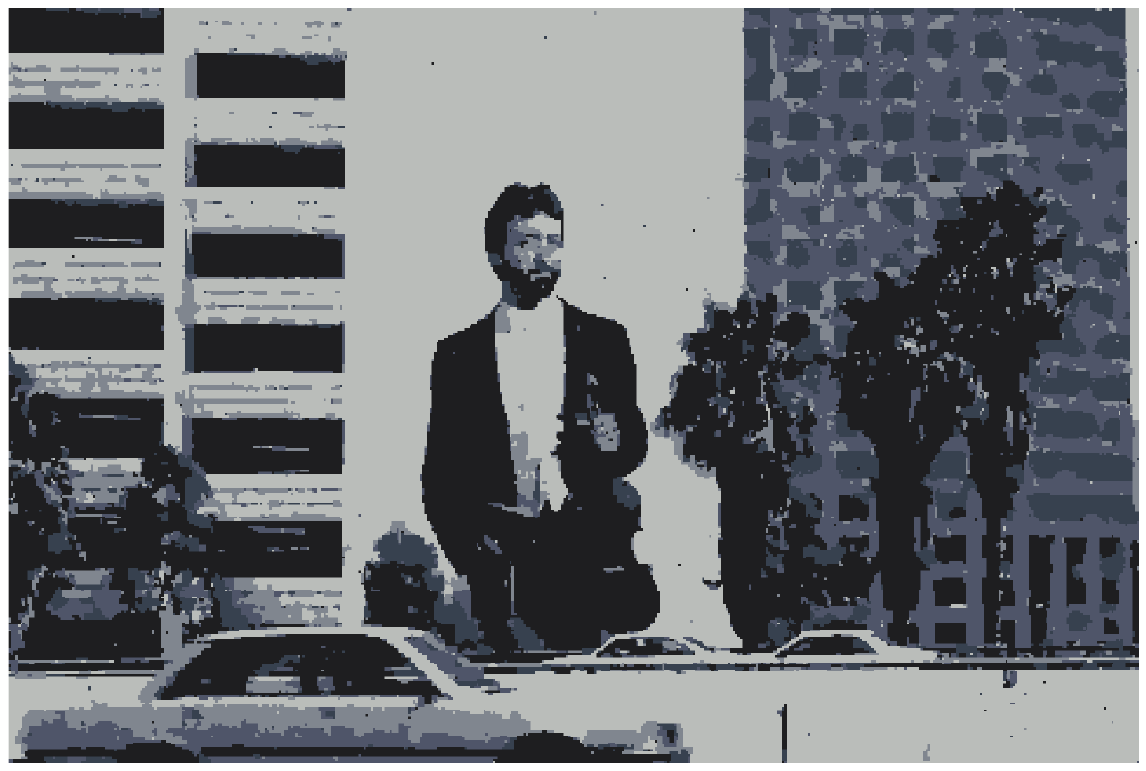}} &		\captionsetup[subfigure]{justification=centering}

\subcaptionbox{ICTM\\
PSNR: 20.29}{\includegraphics[width = 1.50in]{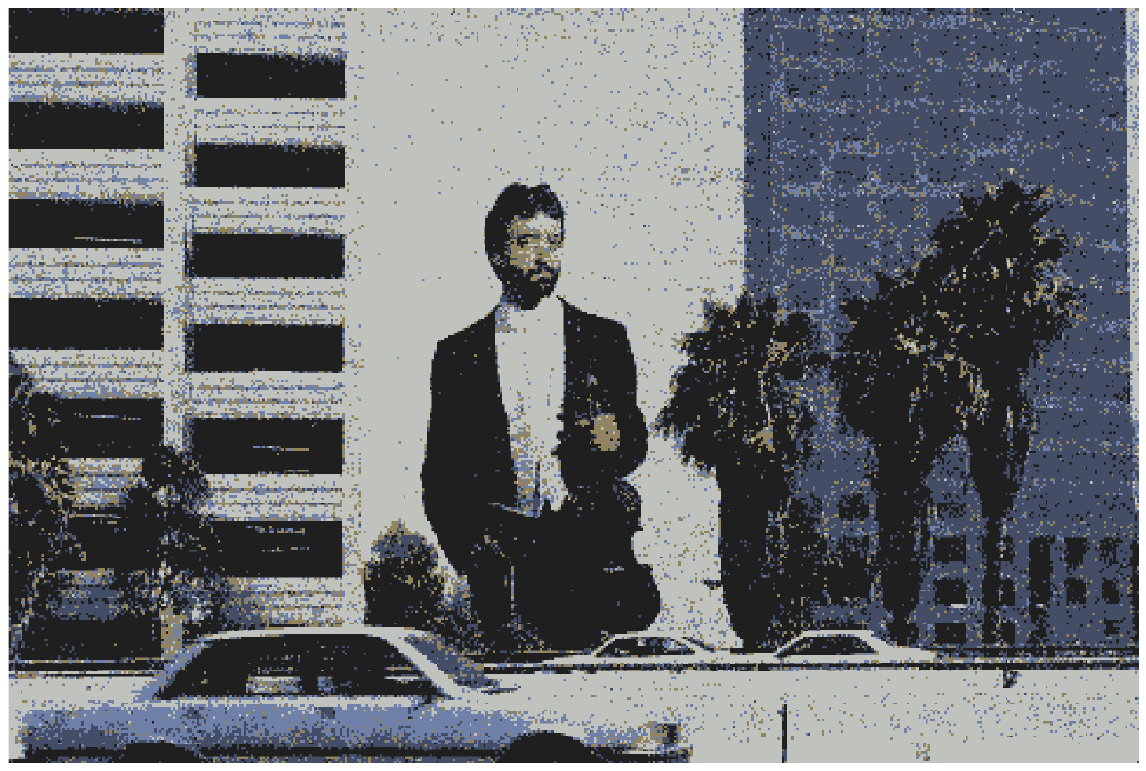}} &	\captionsetup[subfigure]{justification=centering}
\subcaptionbox{TV$^p$ MS\\
PSNR: 22.10}{\includegraphics[ width = 1.50in]{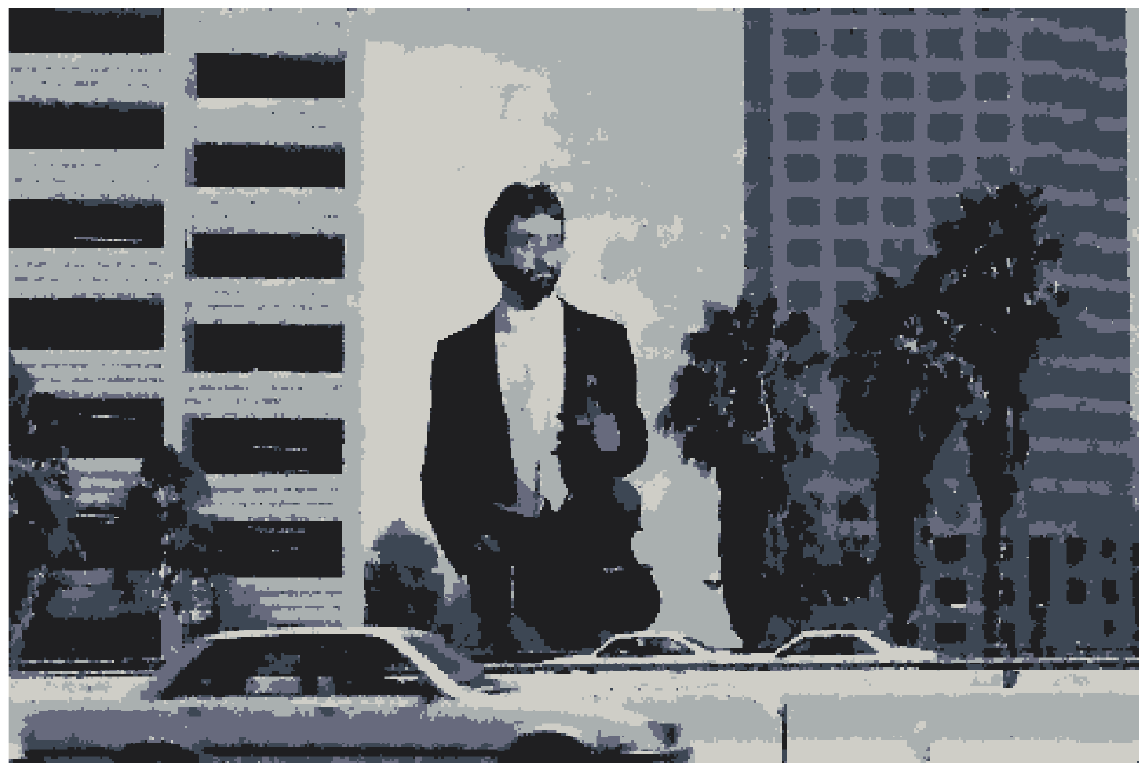}}  &	\captionsetup[subfigure]{justification=centering}
\subcaptionbox{Convex Potts\\
PSNR: 21.30}{\includegraphics[ width = 1.50in]{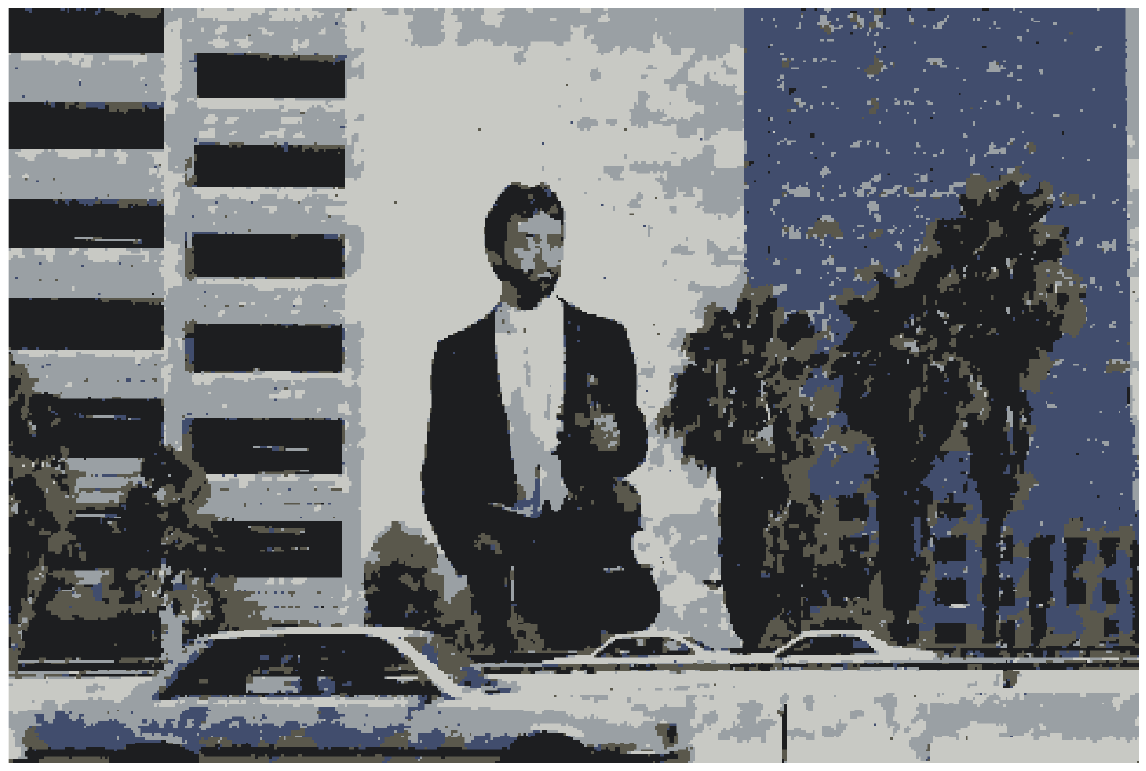}}  &	\captionsetup[subfigure]{justification=centering}
\subcaptionbox{SaT-Potts\\
PSNR: 21.71}{\includegraphics[ width = 1.50in]{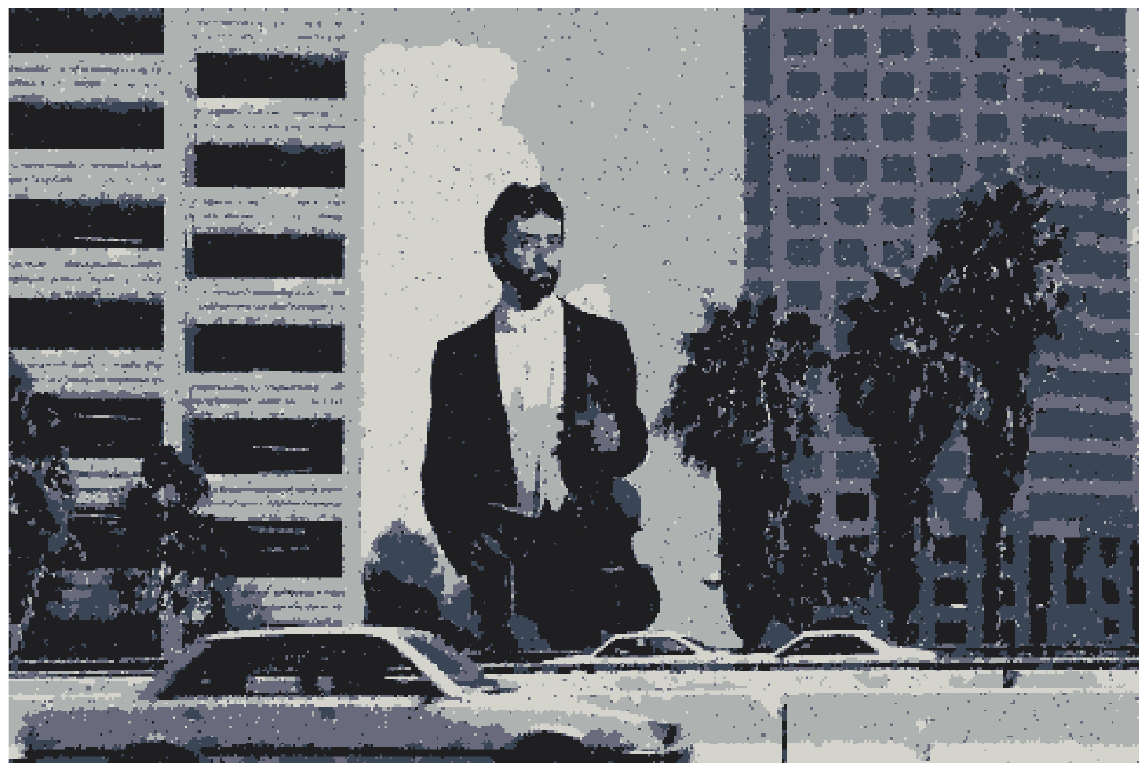}} \\ \hline \\
	\parbox[t]{2mm}{\multirow{2}{*}{\rotatebox[origin=c]{90}{Salt \& Pepper Noise}}} &
		\captionsetup[subfigure]{justification=centering}
\subcaptionbox{Noisy image.}{\includegraphics[width = 1.50in]{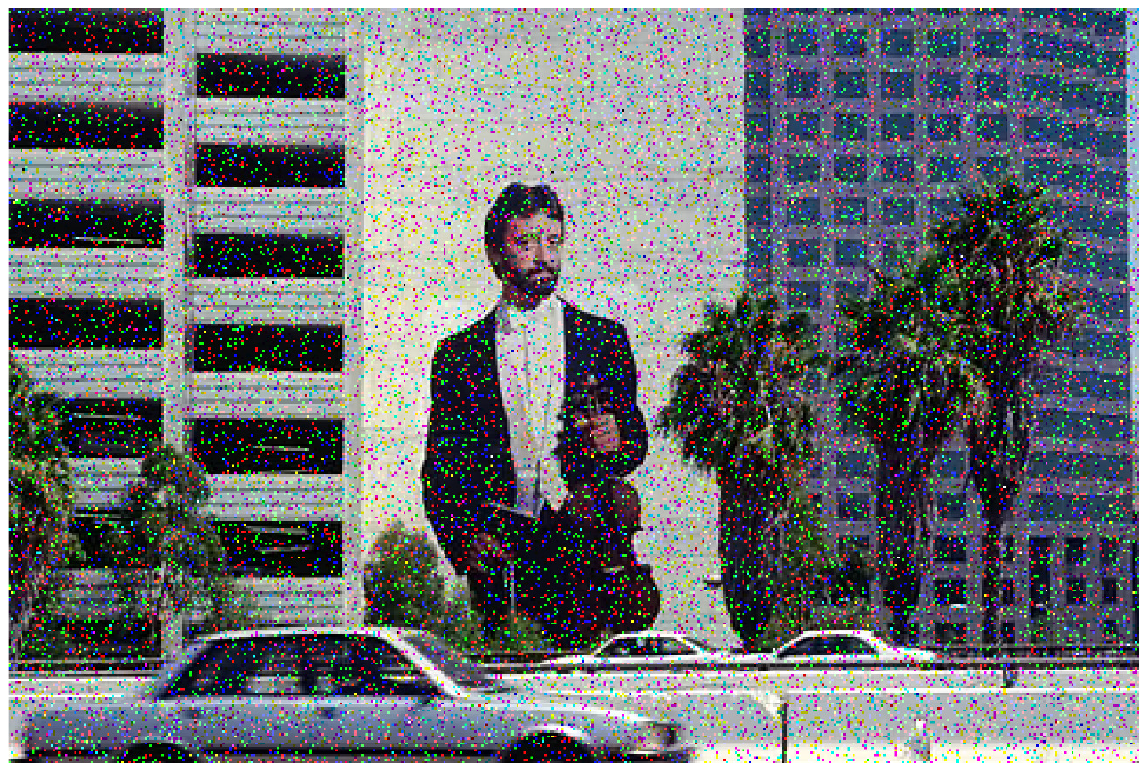}} &
		\captionsetup[subfigure]{justification=centering}

\subcaptionbox{(original) SLaT\\
PSNR: 19.83}{\includegraphics[width = 1.50in]{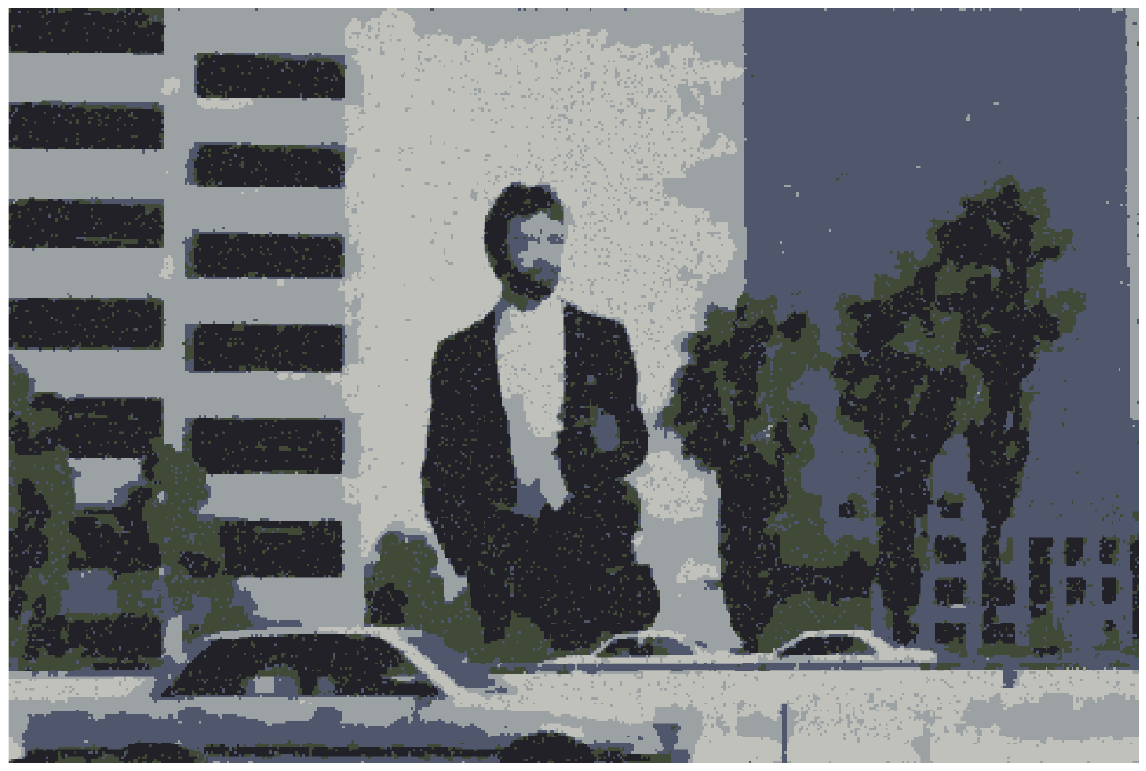}} & 		\captionsetup[subfigure]{justification=centering}
\subcaptionbox{TV$^{p}$ SLaT\\
PSNR: 19.56}{\includegraphics[ width = 1.50in]{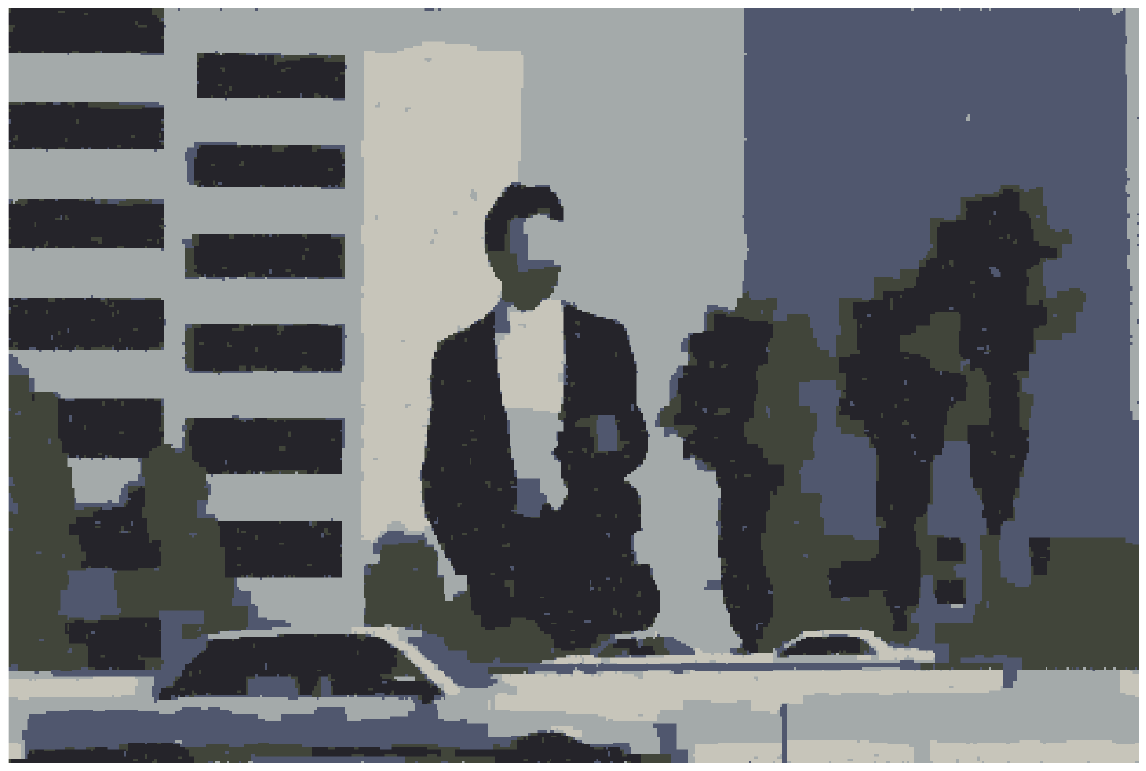}} &		\captionsetup[subfigure]{justification=centering}
\subcaptionbox{AITV   SLaT (ADMM)\\
PSNR: 20.13}{\includegraphics[width = 1.50in]{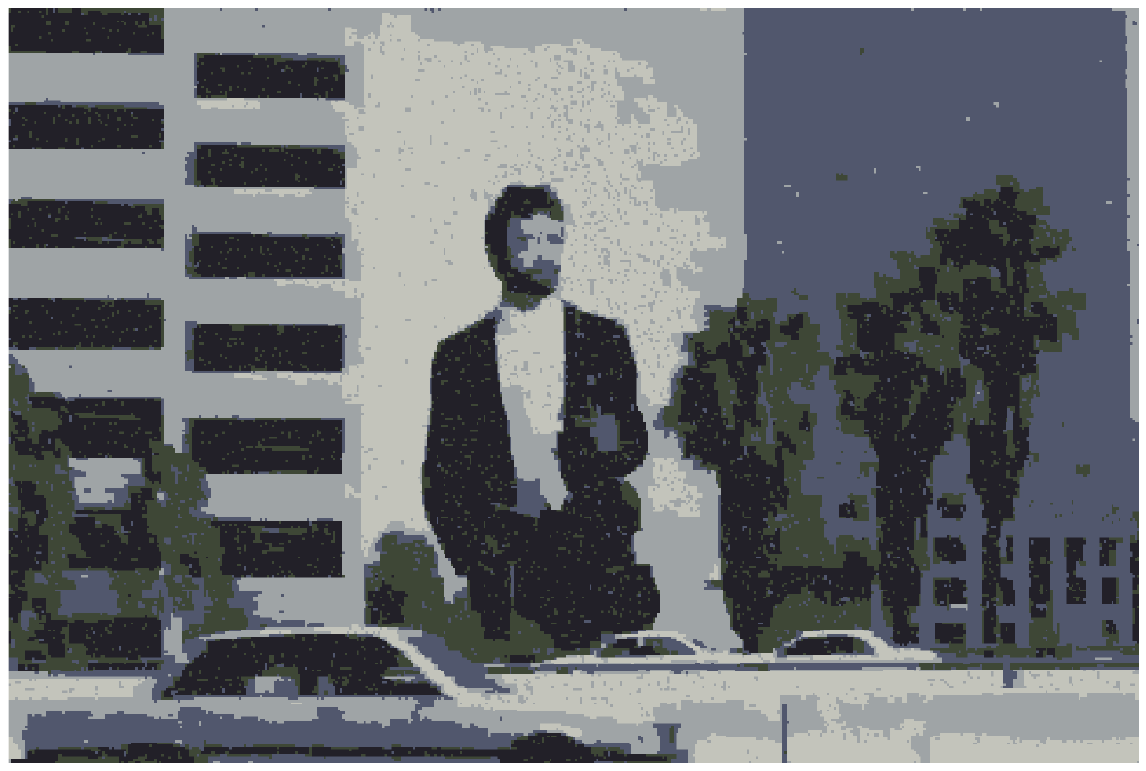}} & \captionsetup[subfigure]{justification=centering}
\subcaptionbox{AITV   SLaT (DCA)\\
PSNR: 20.09}{\includegraphics[width = 1.50in]{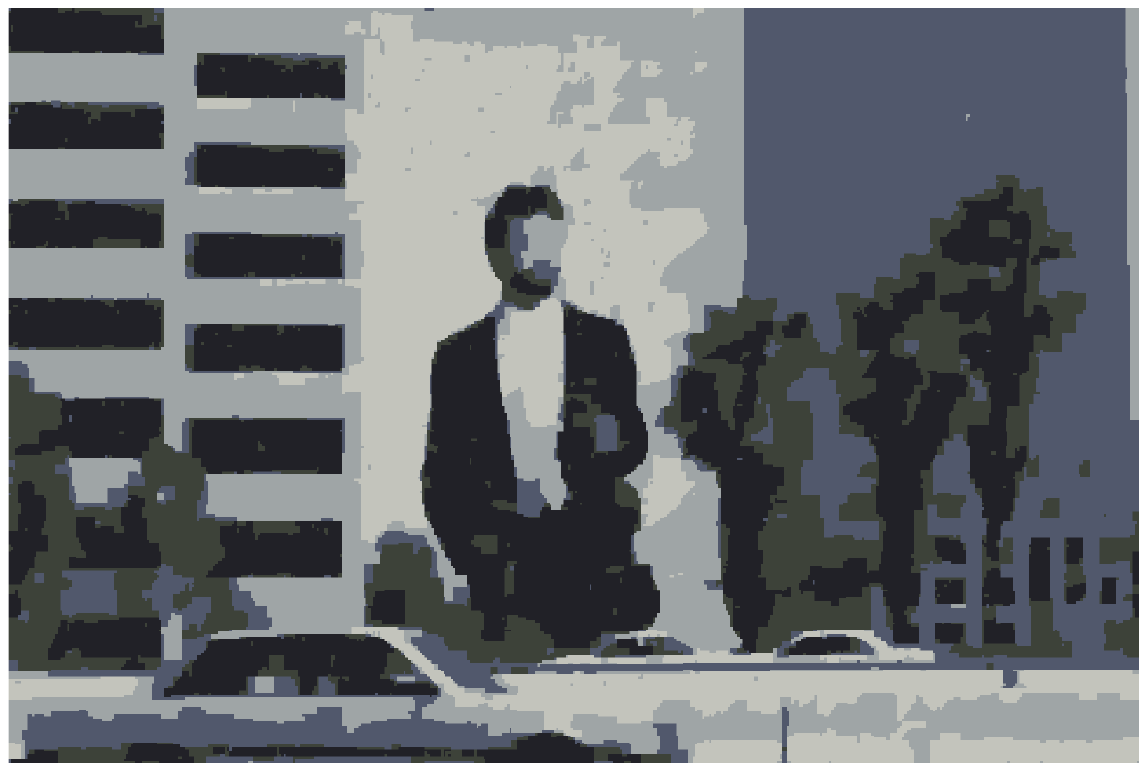}}
\\ &
		\captionsetup[subfigure]{justification=centering}
\subcaptionbox{AITV   FR \\ PSNR: 19.96}{\includegraphics[width = 1.50in]{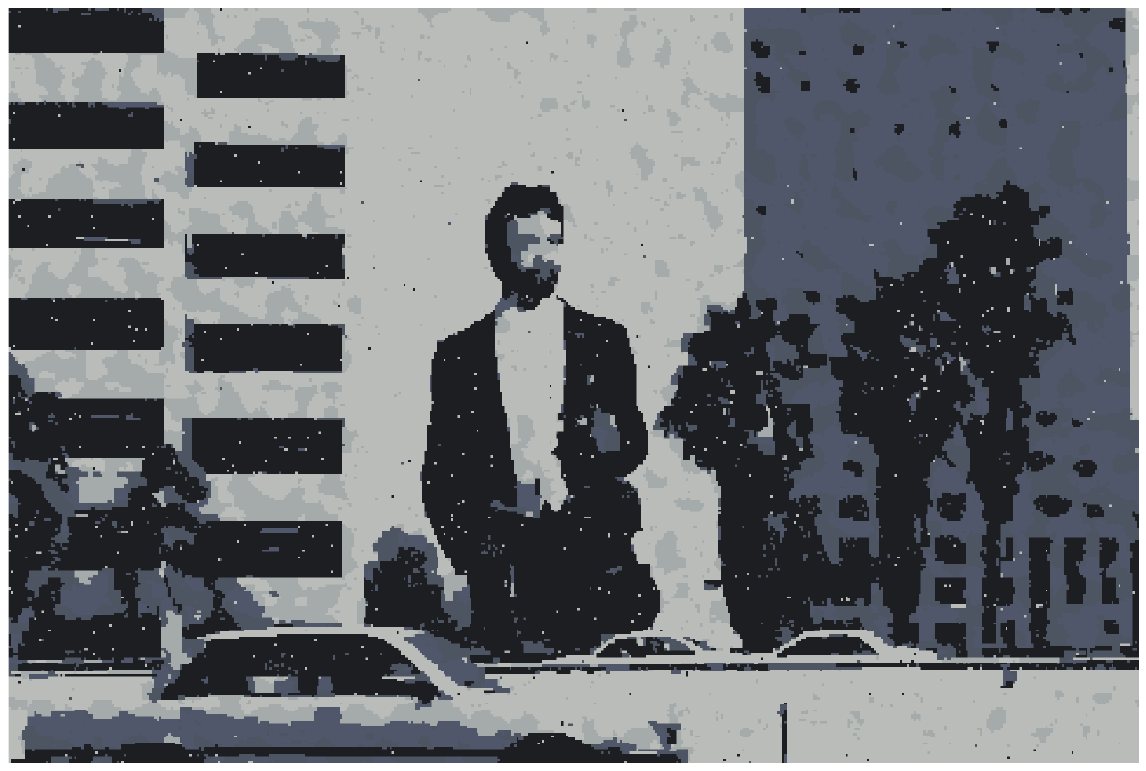}} &		\captionsetup[subfigure]{justification=centering}

\subcaptionbox{ICTM\\
PSNR: 18.44}{\includegraphics[width = 1.50in]{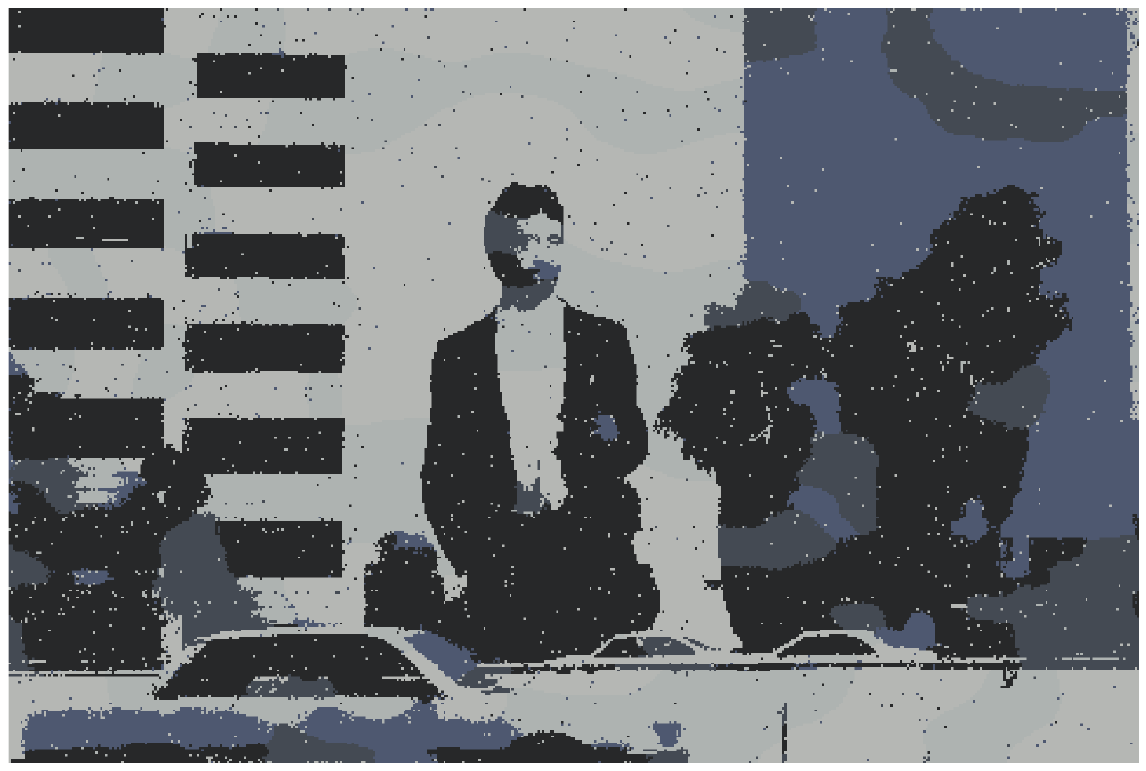}} &	\captionsetup[subfigure]{justification=centering}
\subcaptionbox{TV$^p$ MS\\
PSNR: 20.25}{\includegraphics[ width = 1.50in]{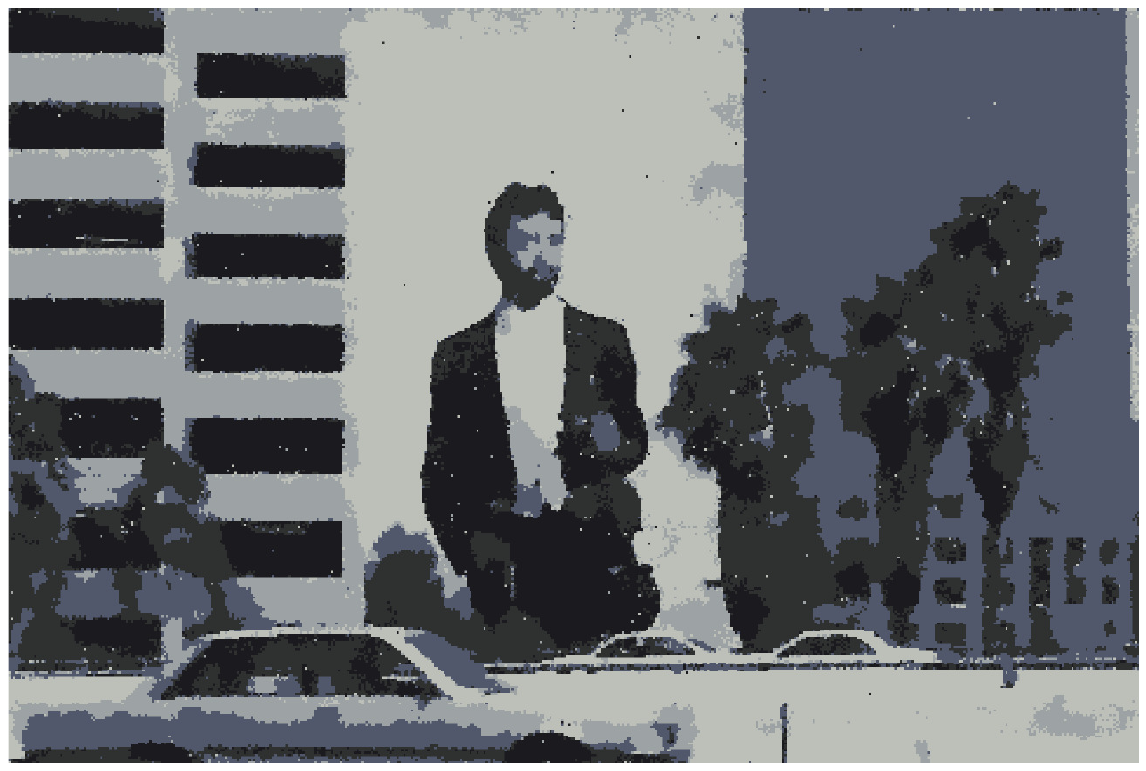}}  &	\captionsetup[subfigure]{justification=centering}
\subcaptionbox{Convex Potts\\
PSNR: 19.26}{\includegraphics[ width = 1.50in]{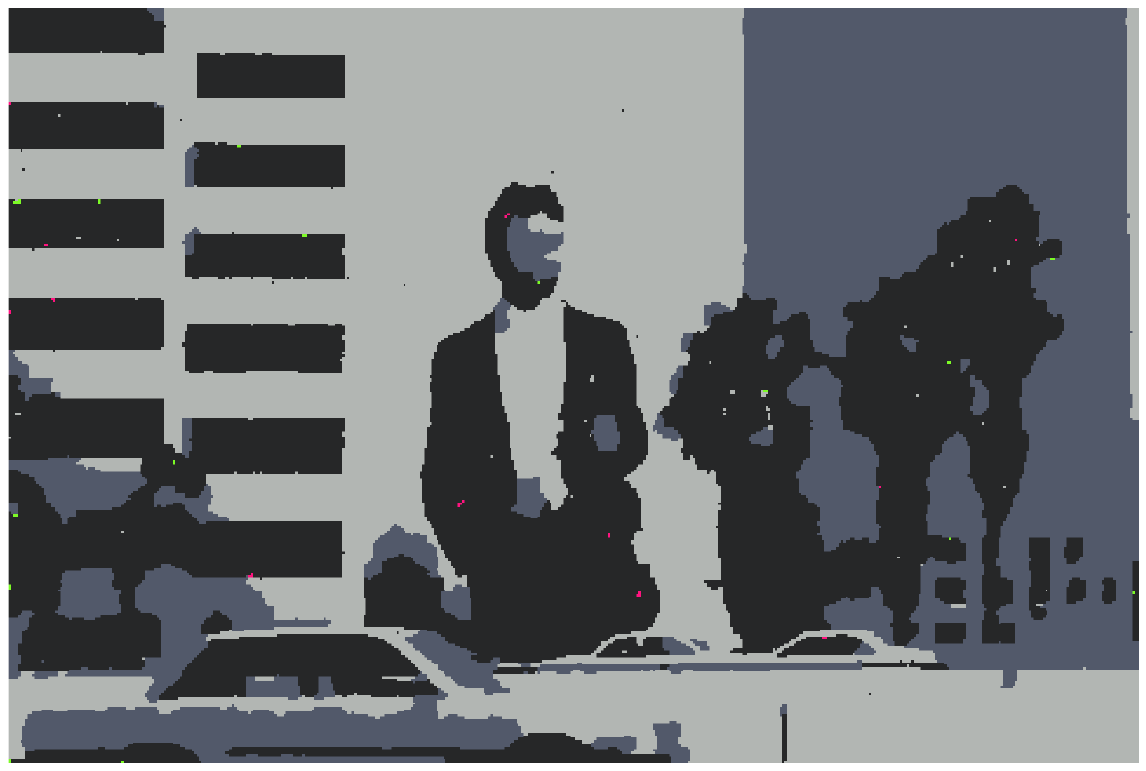}}  &	\captionsetup[subfigure]{justification=centering}
\subcaptionbox{SaT-Potts\\
PSNR: 18.94}{\includegraphics[ width = 1.50in]{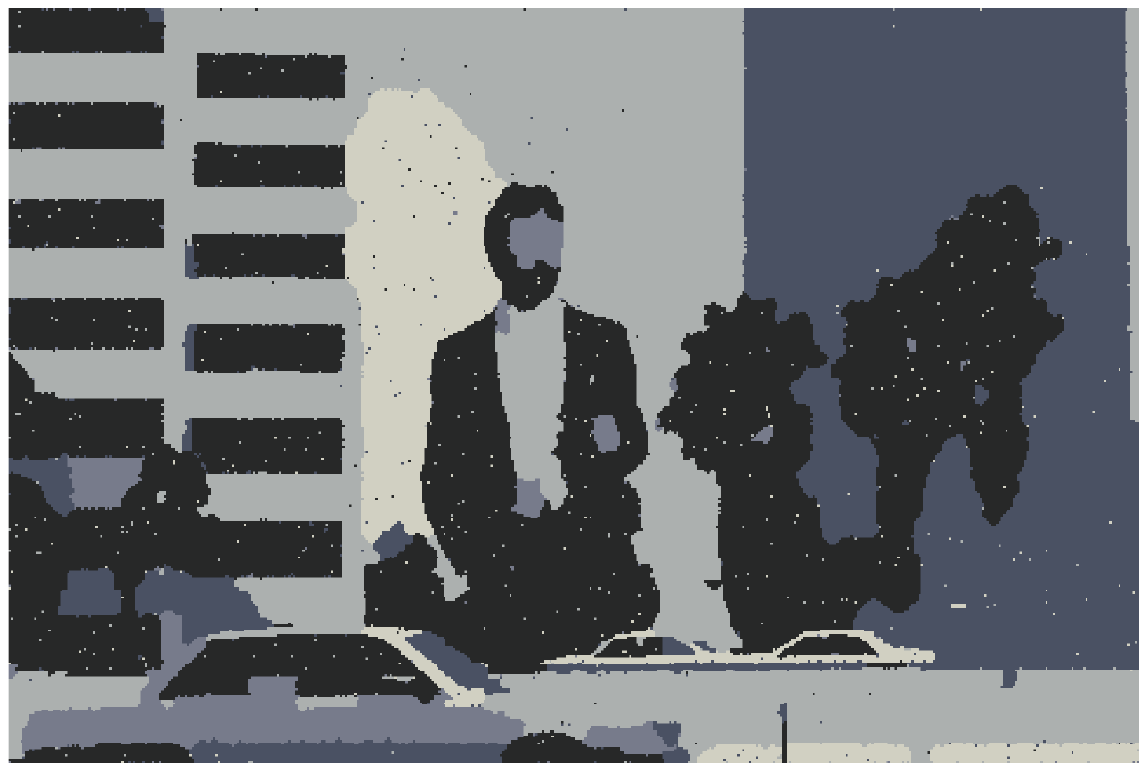}} 
		\end{tabular}}
		\caption{Segmentation results into $k=5$ regions of Figure  {\ref{fig:man}} corrupted by either Gaussian noise of mean zero and variance 0.025 or 10\% SP noise.}
		\label{fig:man_result}		
\end{figure}
\begin{figure}[th!!]
	\centering
	\resizebox{\textwidth}{!}{\begin{tabular}{cccccc}
			\centering
			\parbox[t]{2mm}{\multirow{2}{*}{\rotatebox[origin=c]{90}{Gaussian Noise}}}&\captionsetup[subfigure]{justification=centering}
			\subcaptionbox{Noisy image.}{\includegraphics[width = 1.50in]{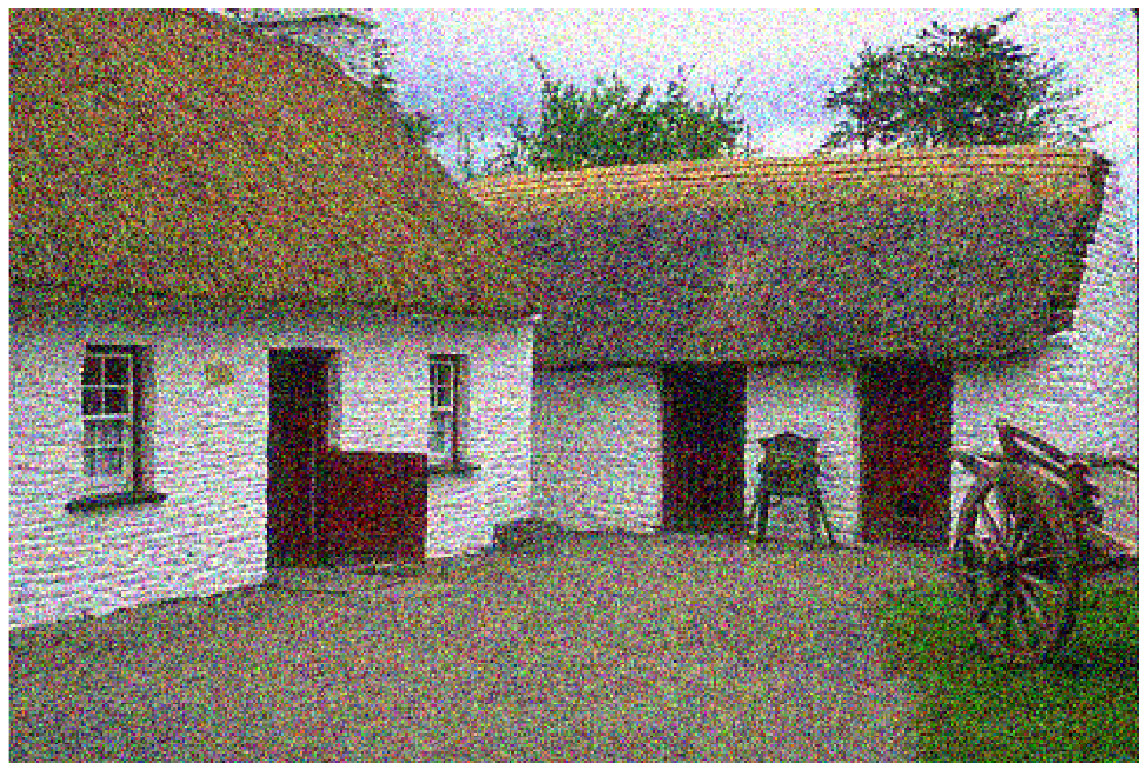}} &
			\captionsetup[subfigure]{justification=centering}
			
			\subcaptionbox{(original) SLaT\\
				PSNR: 21.88}{\includegraphics[width = 1.50in]{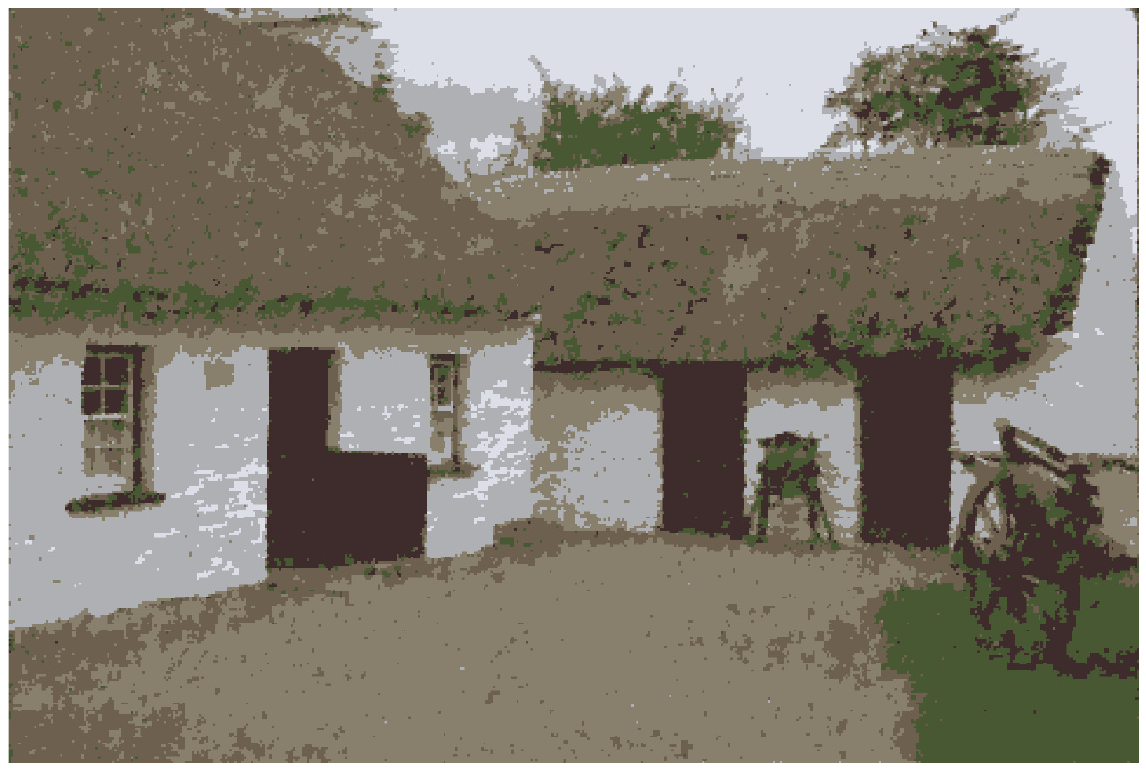}} & 		\captionsetup[subfigure]{justification=centering}
			\subcaptionbox{TV$^{p}$ SLaT\\
				PSNR: 21.93}{\includegraphics[ width = 1.50in]{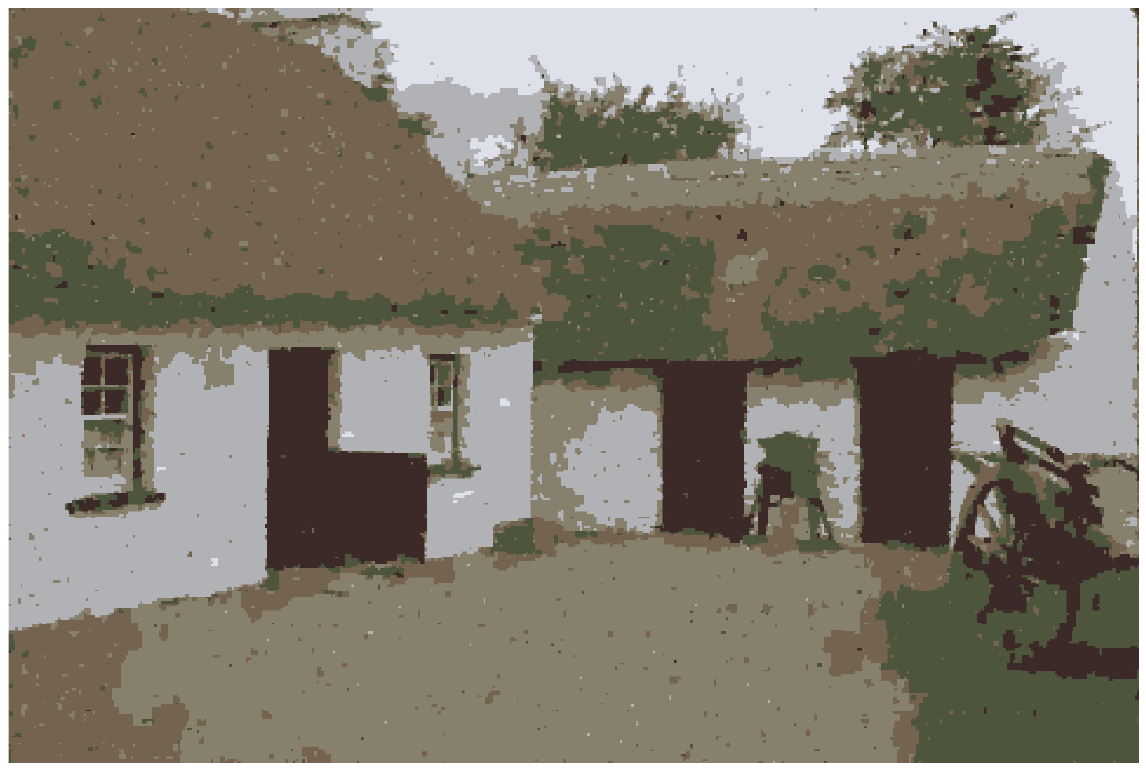}} &		\captionsetup[subfigure]{justification=centering}
			\subcaptionbox{AITV  SLaT (ADMM)\\
				PSNR: 21.85}{\includegraphics[width = 1.50in]{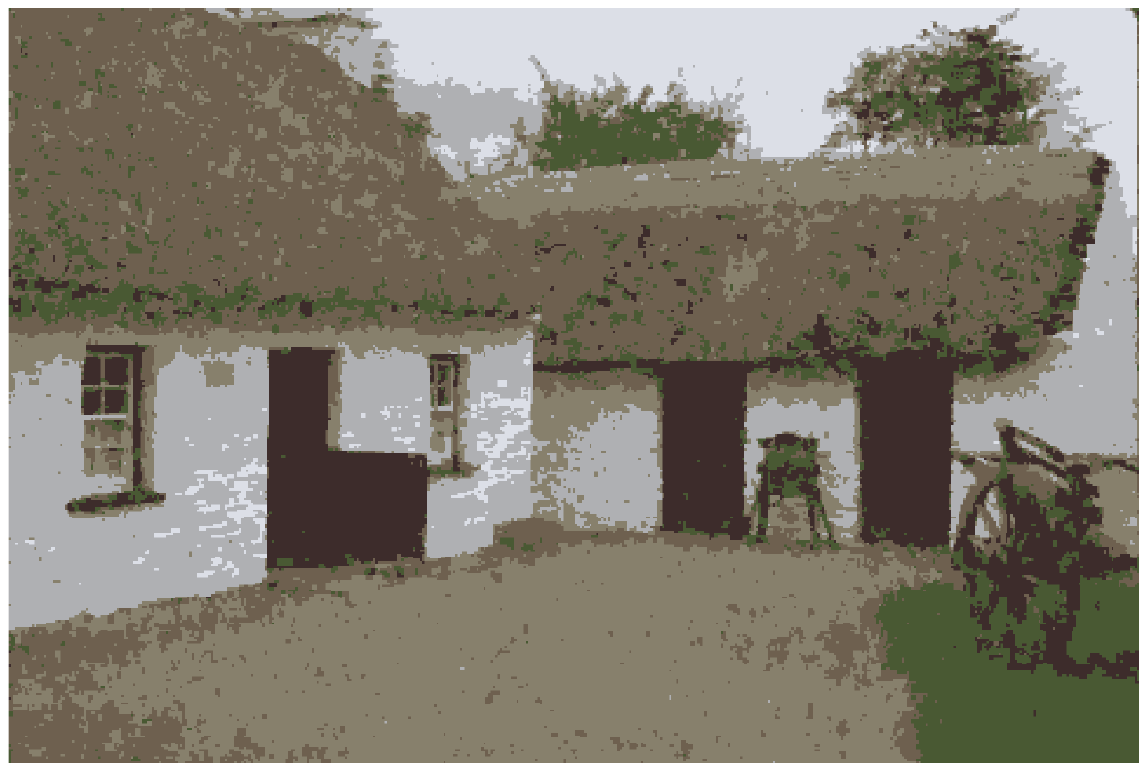}} & \captionsetup[subfigure]{justification=centering}
			\subcaptionbox{AITV  SLaT (DCA)\\
				PSNR: 21.81}{\includegraphics[width = 1.50in]{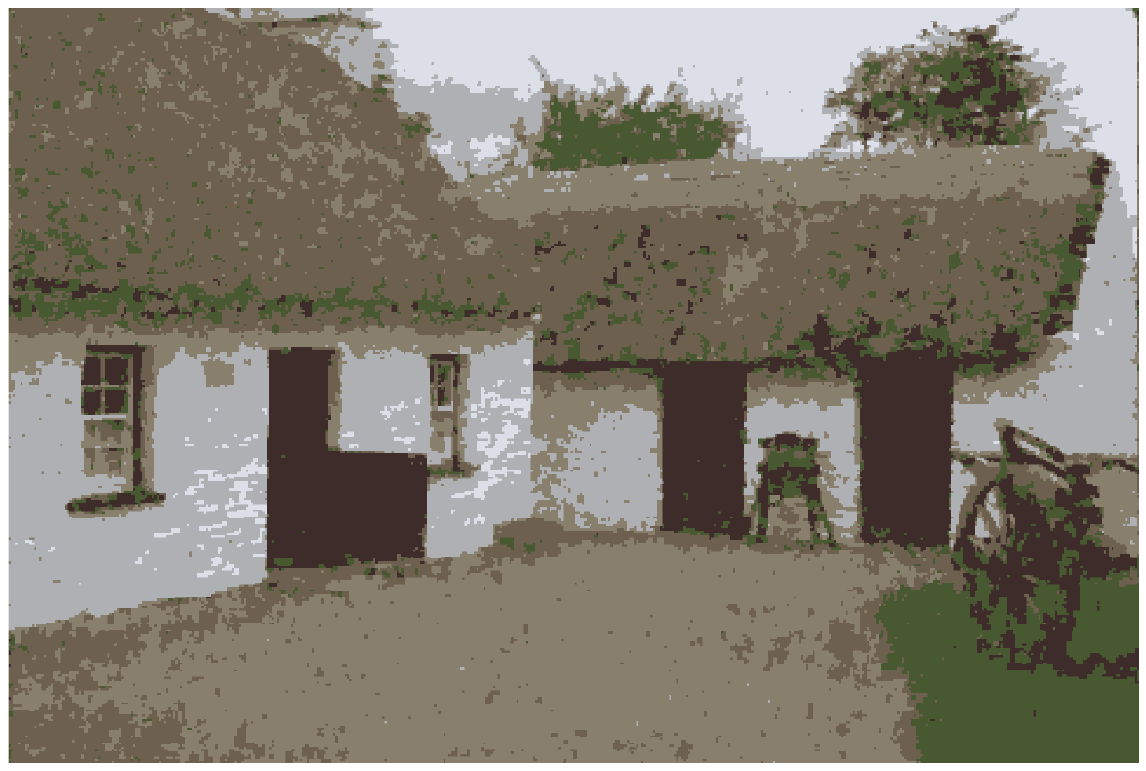}} \\ &
			\captionsetup[subfigure]{justification=centering}
			\subcaptionbox{AITV   FR \\ PSNR: 20.35}{\includegraphics[width = 1.50in]{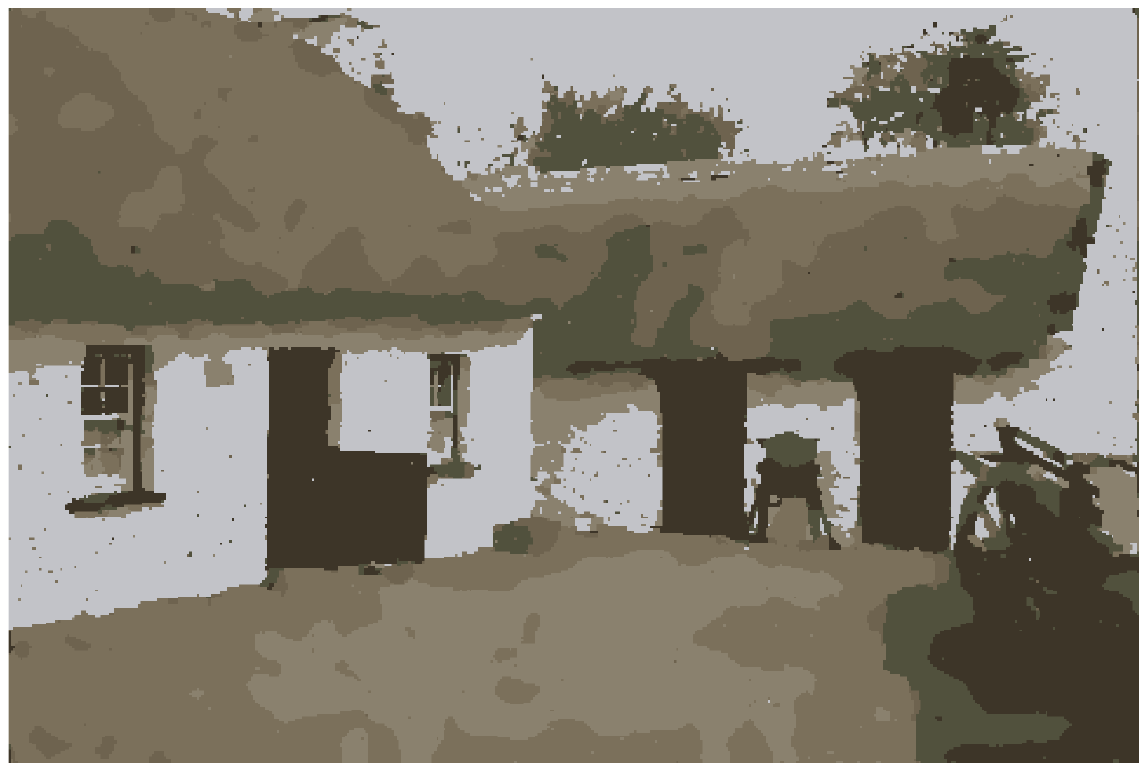}} 
			&		\captionsetup[subfigure]{justification=centering}
			\subcaptionbox{ICTM\\
				PSNR: 19.89}{\includegraphics[width = 1.50in]{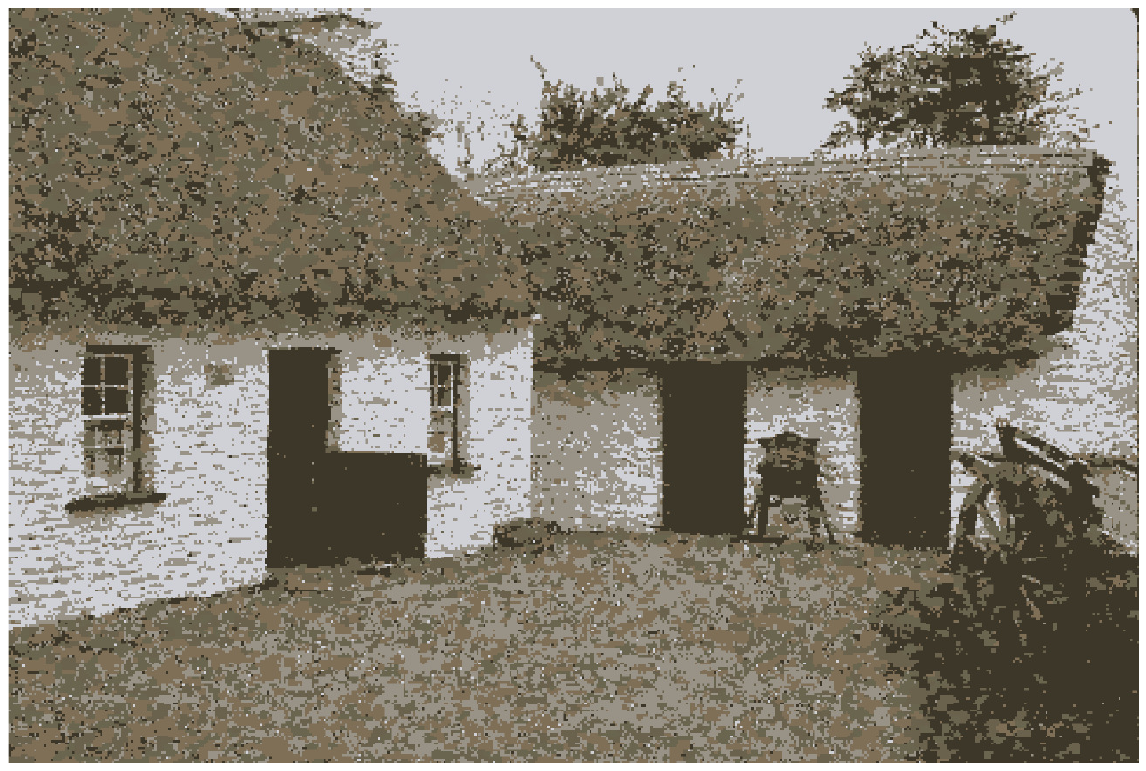}} &		\captionsetup[subfigure]{justification=centering}
			\subcaptionbox{TV$^p$ MS\\
				PSNR: 22.10}{\includegraphics[width = 1.50in]{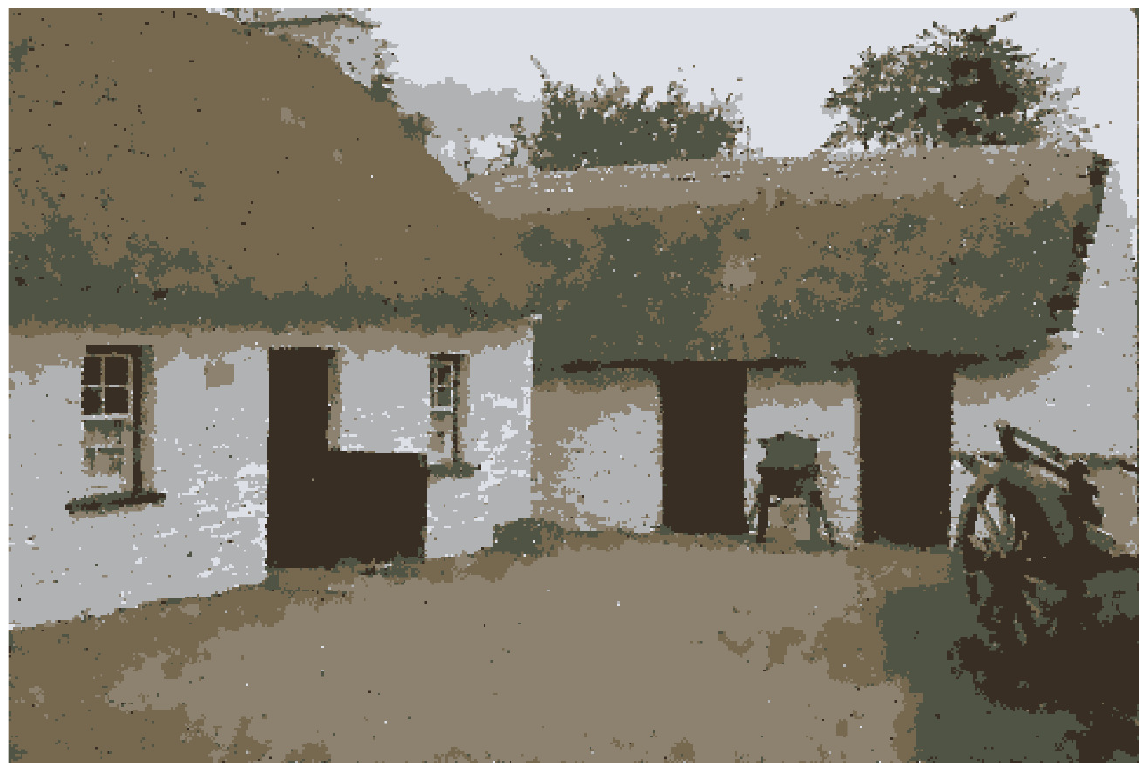}} &		\captionsetup[subfigure]{justification=centering}
			
			\subcaptionbox{Convex Potts\\
				PSNR: 20.64}{\includegraphics[width = 1.50in]{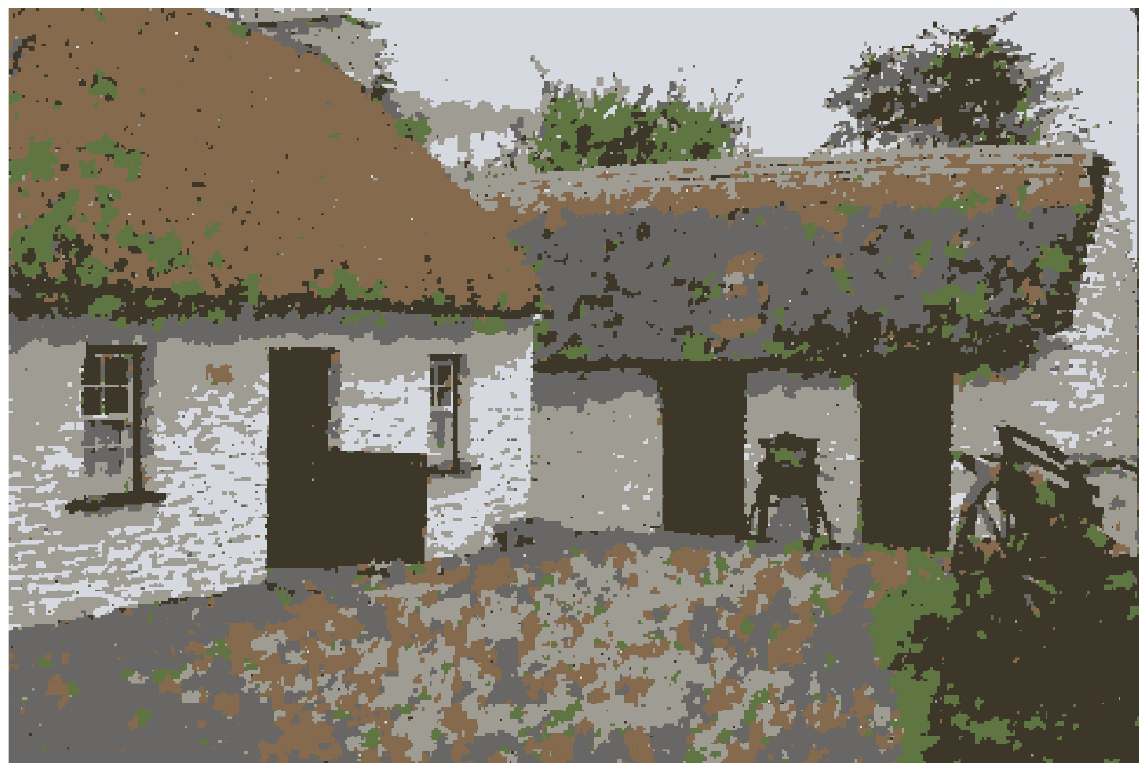}} &		\captionsetup[subfigure]{justification=centering}
			
			\subcaptionbox{SaT-Potts\\
				PSNR: 21.62}{\includegraphics[ width = 1.50in]{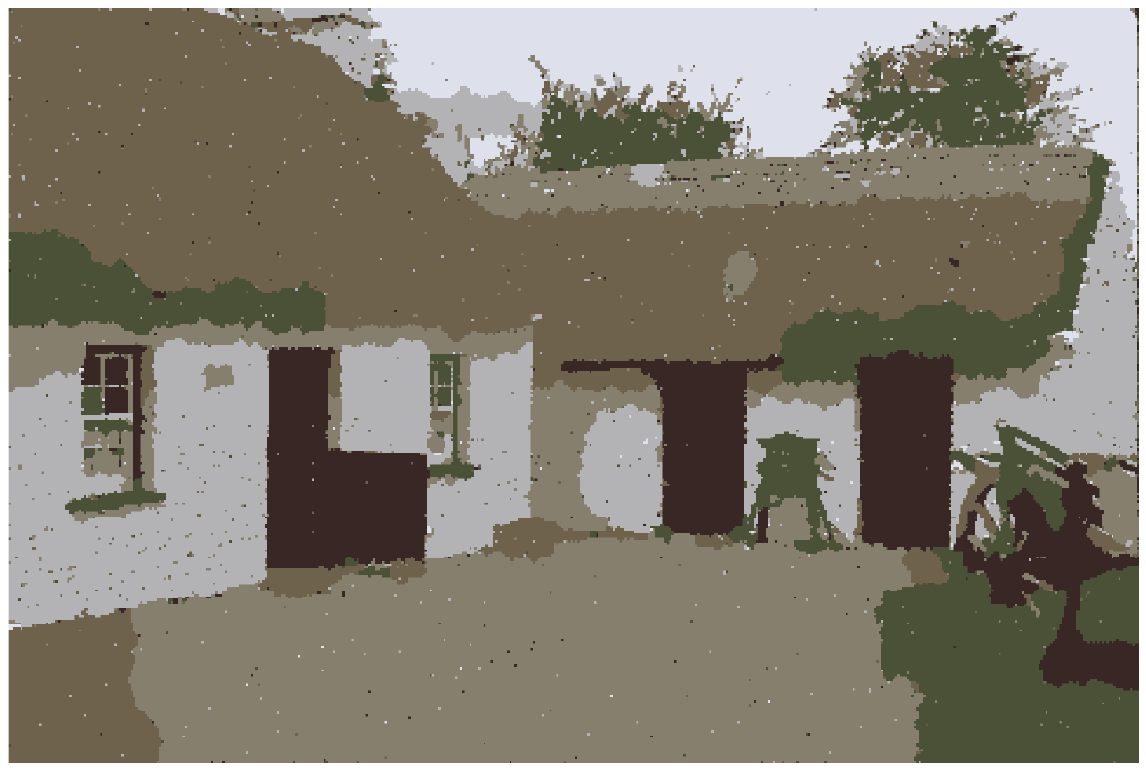}} \\ \hline \\
			\parbox[t]{2mm}{\multirow{2}{*}{\rotatebox[origin=c]{90}{Salt \& Pepper Noise}}}&\captionsetup[subfigure]{justification=centering}
			\subcaptionbox{Noisy image.}{\includegraphics[width = 1.50in]{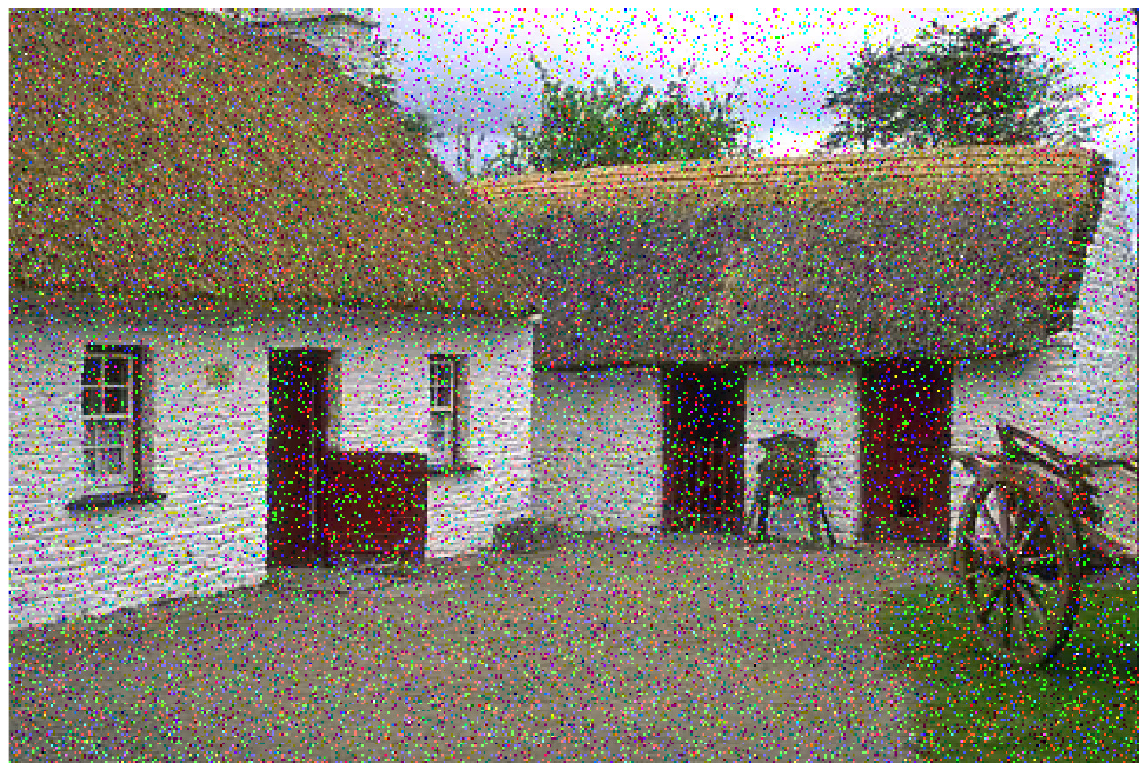}} &
			\captionsetup[subfigure]{justification=centering}
			
			\subcaptionbox{(original) SLaT\\
				PSNR: 20.95}{\includegraphics[width = 1.50in]{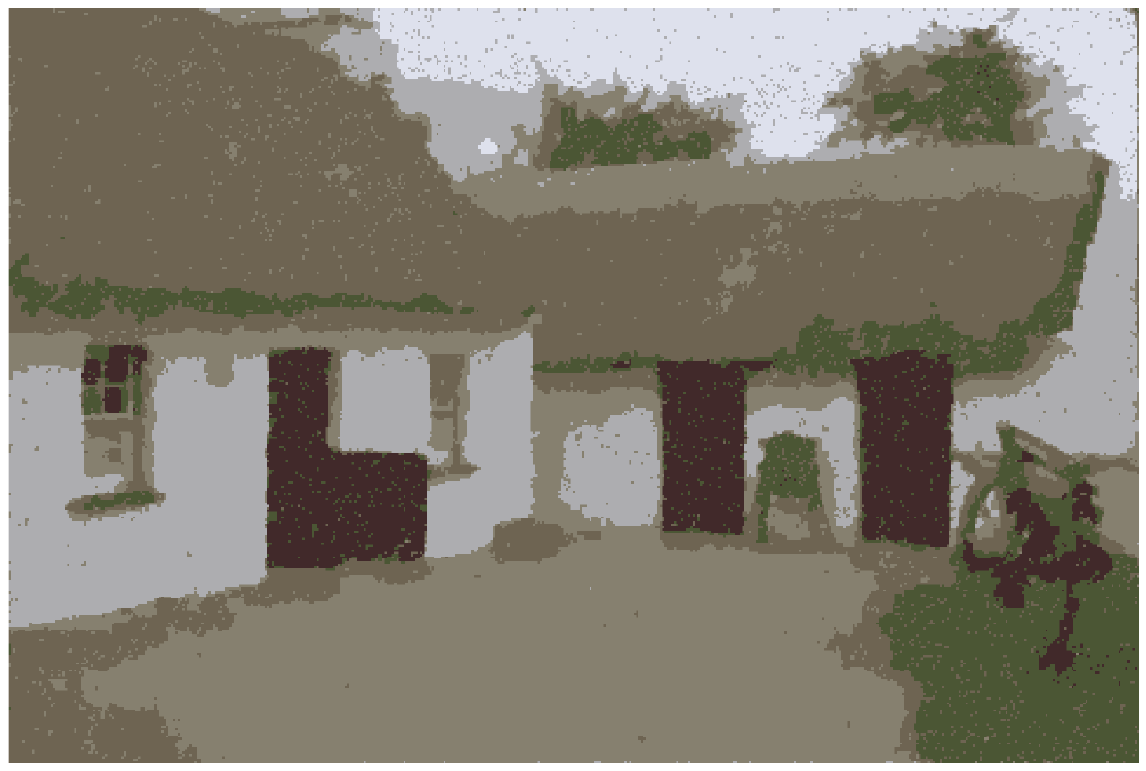}} & 		\captionsetup[subfigure]{justification=centering}
			\subcaptionbox{TV$^{p}$ SLaT\\
				PSNR: 20.43}{\includegraphics[ width = 1.50in]{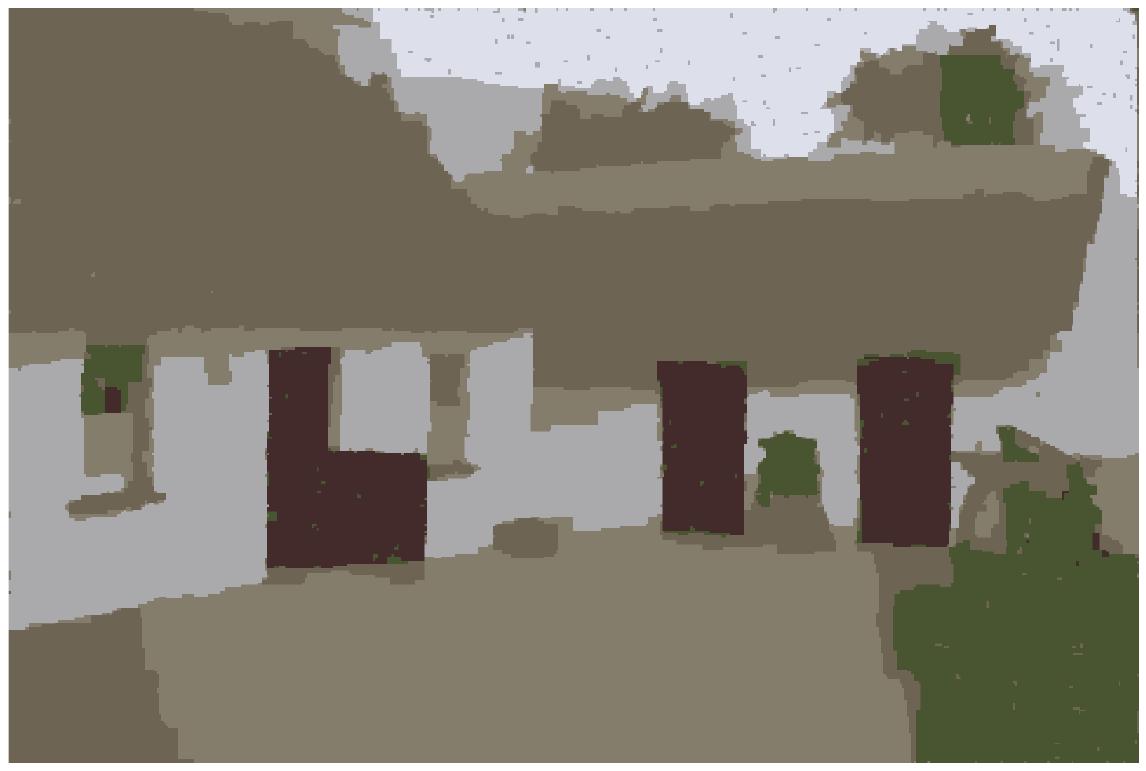}} &		\captionsetup[subfigure]{justification=centering}
			\subcaptionbox{AITV  SLaT (ADMM)\\
				PSNR: 21.08}{\includegraphics[width = 1.50in]{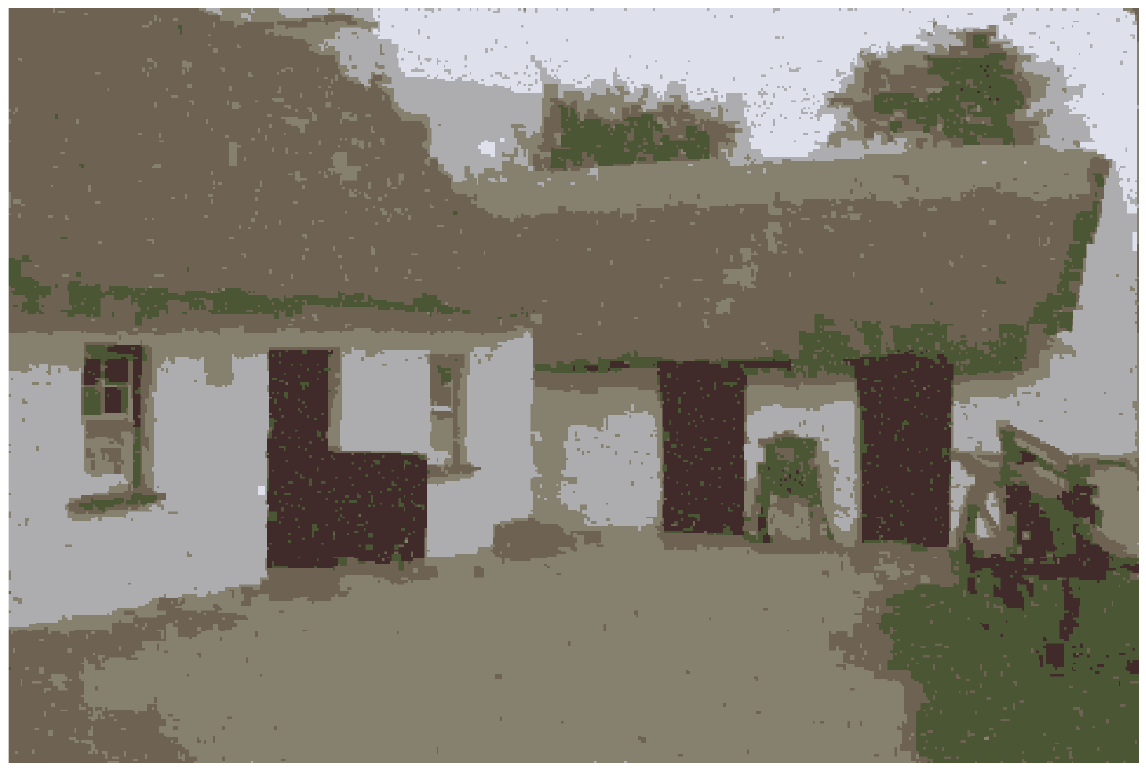}} & \captionsetup[subfigure]{justification=centering}
			\subcaptionbox{AITV  SLaT (DCA)\\
				PSNR: 20.91}{\includegraphics[width = 1.50in]{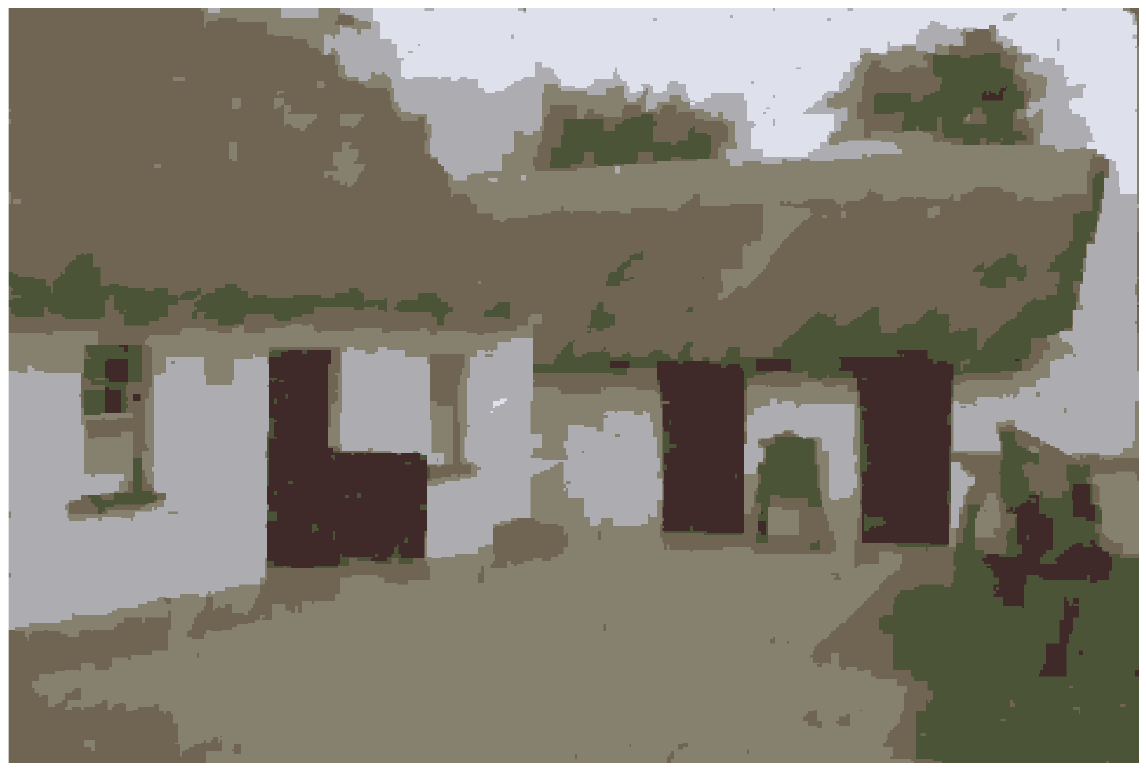}} \\ &
			\captionsetup[subfigure]{justification=centering}
			\subcaptionbox{AITV   FR \\ PSNR: 19.25}{\includegraphics[width = 1.50in]{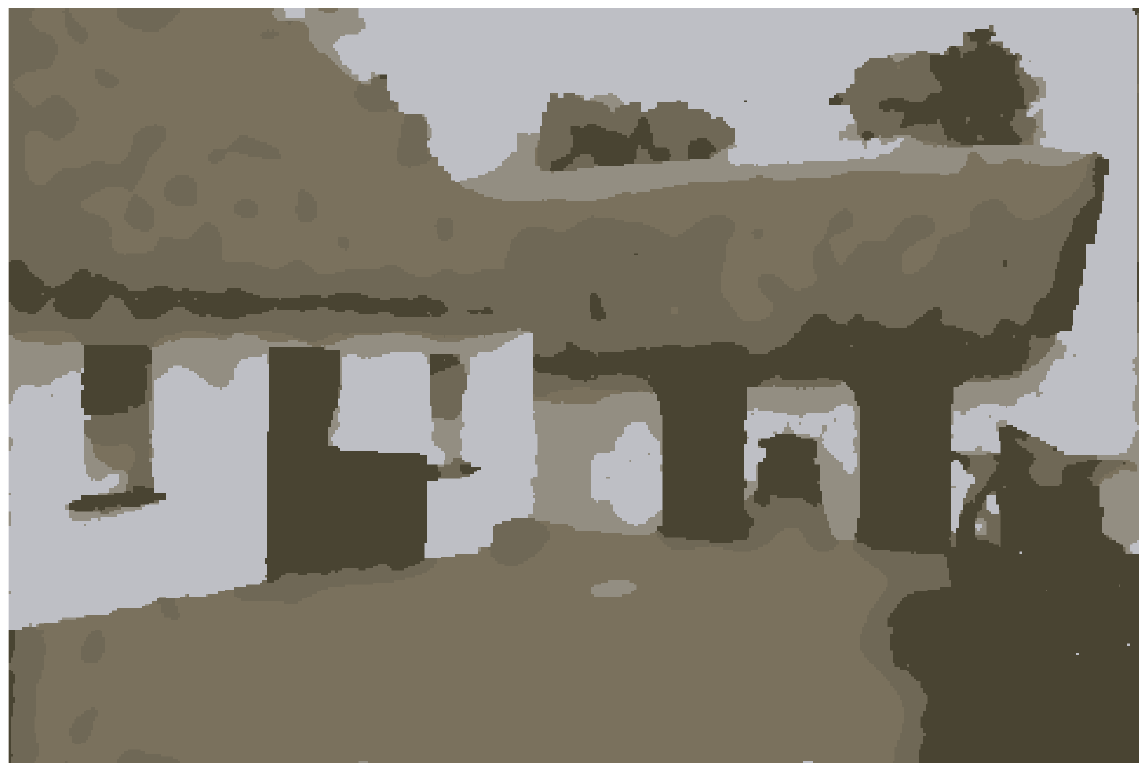}} 
			&		\captionsetup[subfigure]{justification=centering}
			\subcaptionbox{ICTM\\
				PSNR: 17.95}{\includegraphics[width = 1.50in]{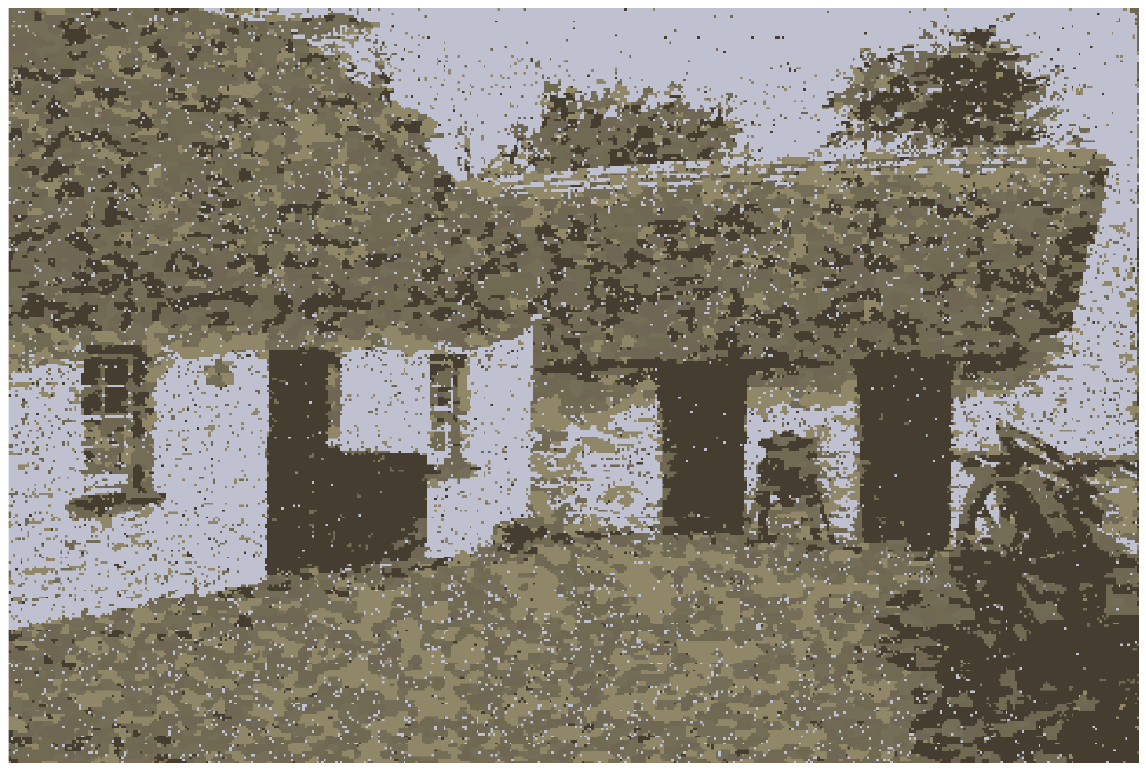}} &		\captionsetup[subfigure]{justification=centering}
			\subcaptionbox{TV$^p$ MS\\
				PSNR: 21.27}{\includegraphics[width = 1.50in]{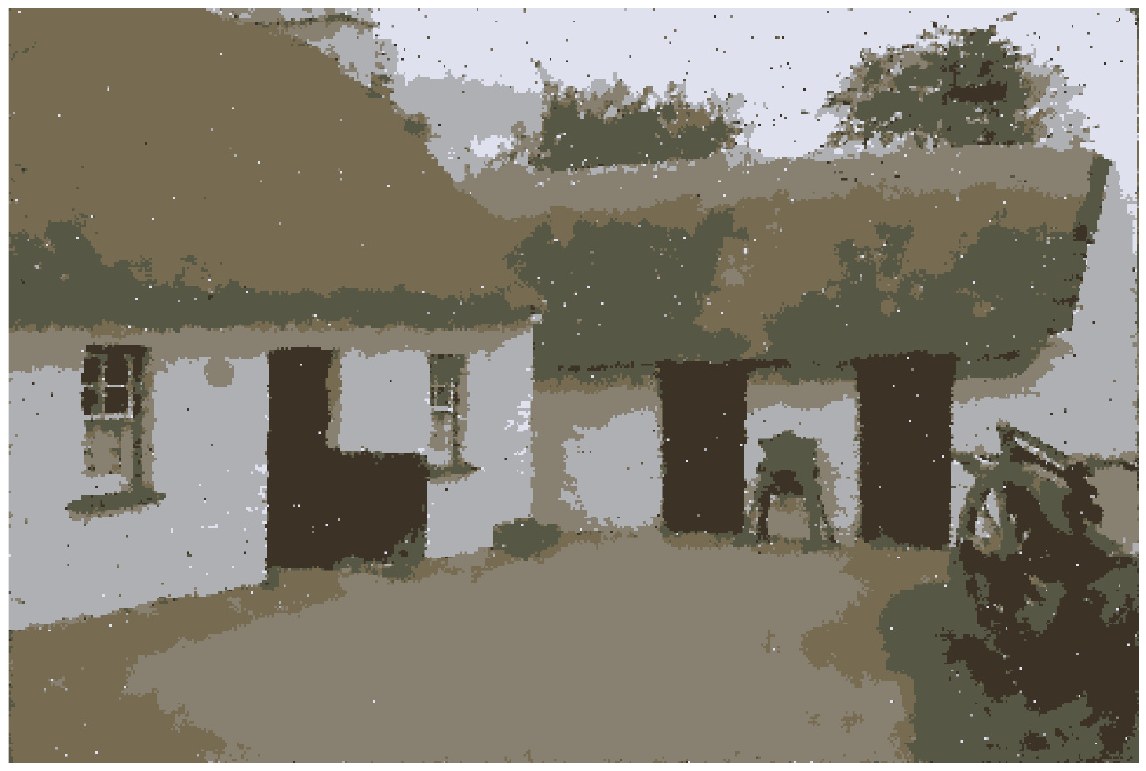}} &		\captionsetup[subfigure]{justification=centering}
			
			\subcaptionbox{Convex Potts\\
				PSNR: 19.62}{\includegraphics[width = 1.50in]{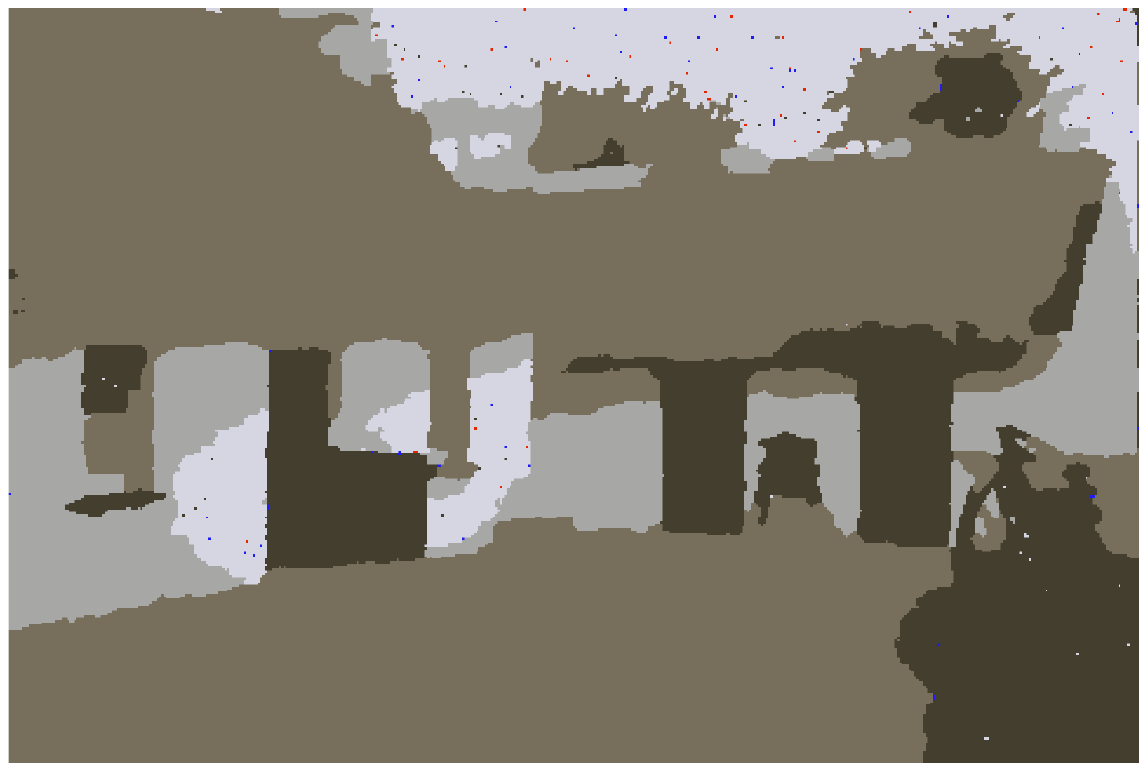}} &		\captionsetup[subfigure]{justification=centering}
			
			\subcaptionbox{SaT-Potts\\
				PSNR: 20.20}{\includegraphics[ width = 1.50in]{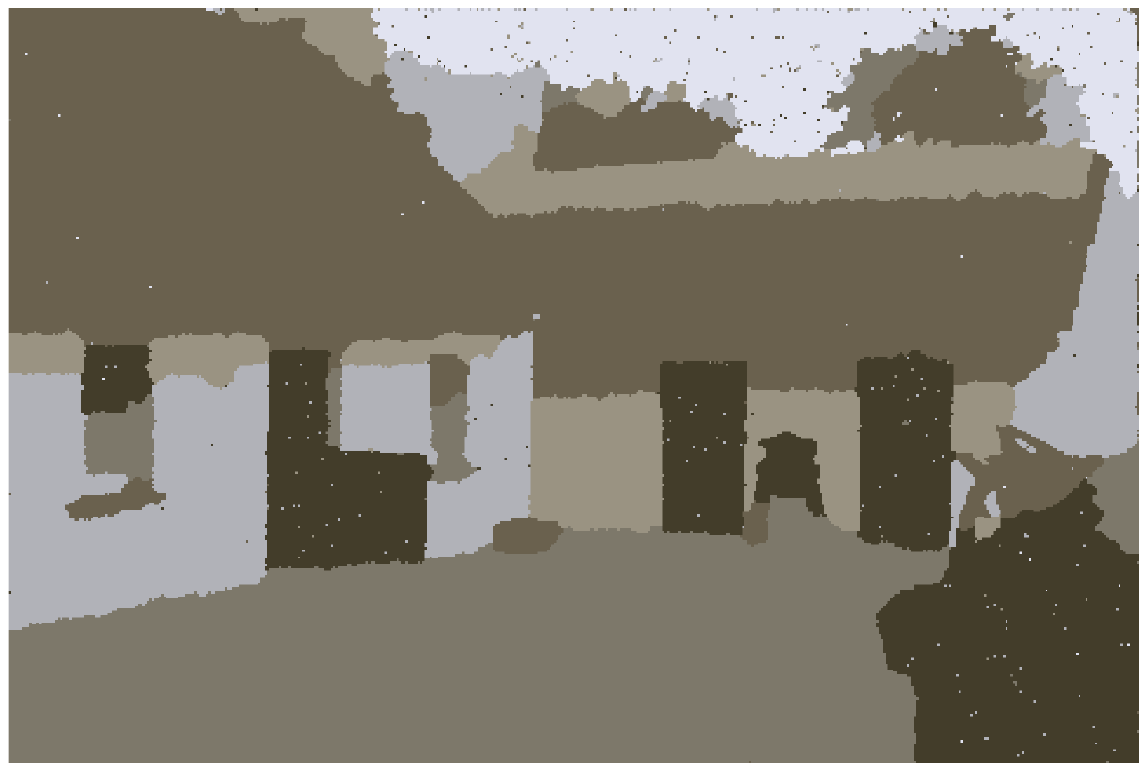}}
	\end{tabular}}
\end{figure}}
Four real color images taken from {\cite{martin-2001}} are presented in Figure {\ref{fig:real_color}} for segmentation. The images are corrupted with either Gaussian noise of mean zero and variance 0.025 or 10\% SP noise. We segment Figure {\ref{fig:circle}} with $k=3$ regions, Figure {\ref{fig:man}} with $k=5$ regions, Figure {\ref{fig:house}} with $k=6$ regions, and Figure {\ref{fig:building}} with $k=8$ regions. Because ground truth is unavailable, we use PSNR to evaluate the segmentation result as a piecewise-constant approximation of the original image. For the SLaT methods, we tune the parameters $\lambda \in [2,30]$ and $\mu \in [0.05, 1.0]$ for all the images.

Table {\ref{tab:color_real}} records the PSNR values and computational times in seconds of the segmentation algorithms while Figures {\ref{fig:circle_result}}-{\ref{fig:building_result}} present the visual results. Overall,   AITV SLaT (ADMM) is generally among the top three methods with the best PSNR values for both noise cases. In fact, for SP noise, AITV SLaT (ADMM) has the second best PSNRs while being significantly faster than TV$^p$ MS that has the best PSNRs. 

\afterpage{
\clearpage
\begin{figure}[th!]
\centering
\resizebox{\textwidth}{!}{\begin{tabular}{cccccc}
\centering
	\parbox[t]{2mm}{\multirow{2}{*}{\rotatebox[origin=c]{90}{Gaussian Noise}}} &
		\captionsetup[subfigure]{justification=centering}
\subcaptionbox{Noisy image.}{\includegraphics[width = 1.50in, height = 2.00in]{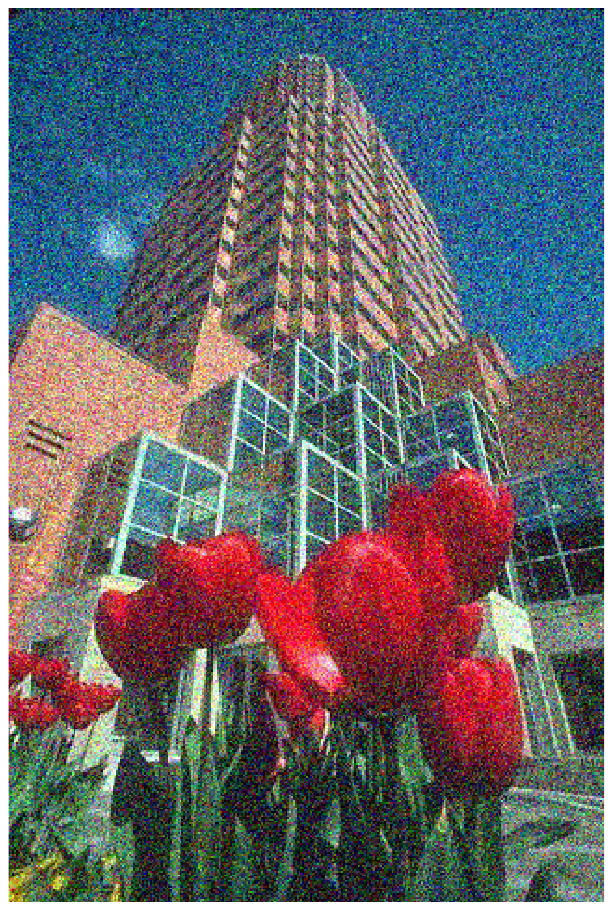}} &
		\captionsetup[subfigure]{justification=centering}

\subcaptionbox{(original) SLaT\\
PSNR: 21.76}{\includegraphics[width = 1.50in, height = 2.00in]{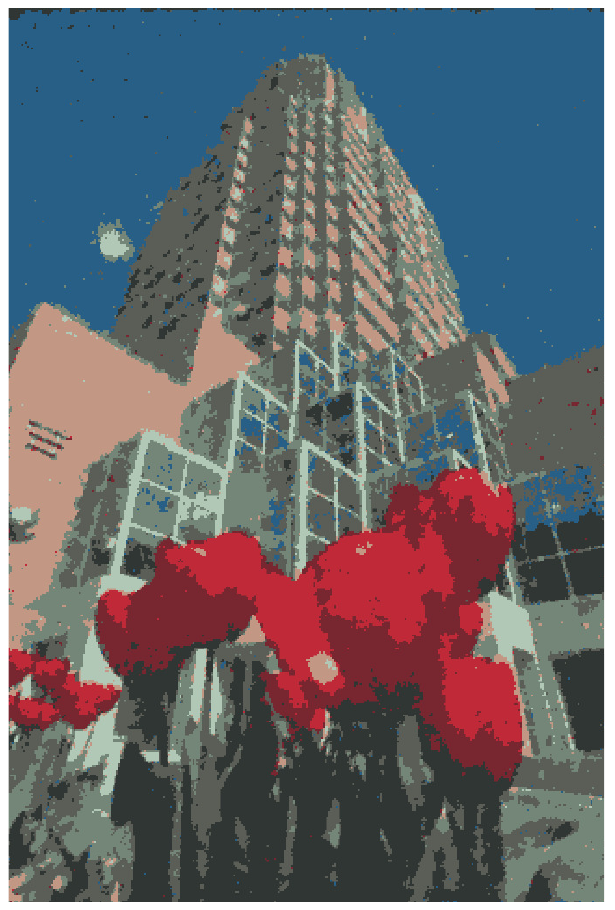}} & 		\captionsetup[subfigure]{justification=centering}
\subcaptionbox{TV$^{p}$ SLaT\\
PSNR: 21.63}{\includegraphics[ width = 1.50in, height = 2.00in]{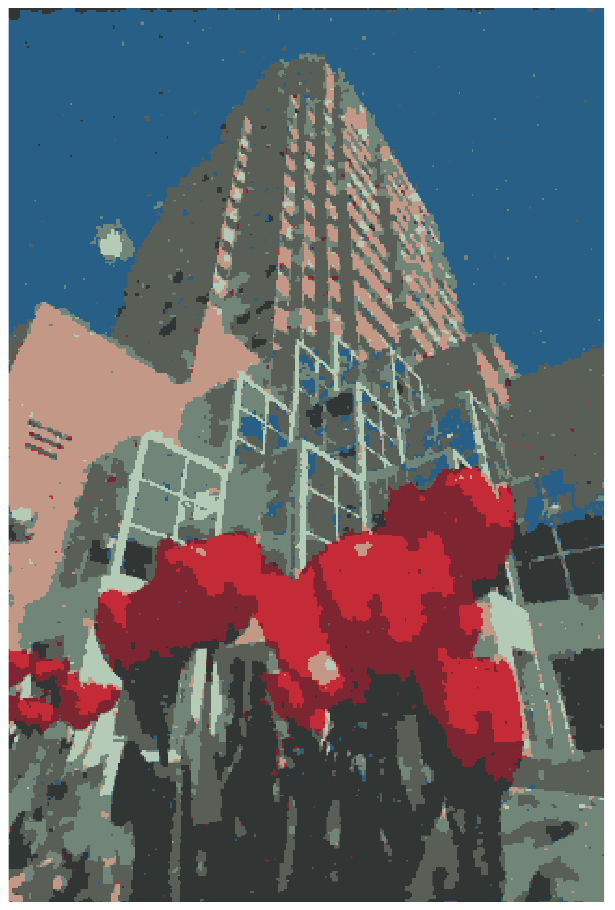}} &		\captionsetup[subfigure]{justification=centering}
\subcaptionbox{AITV SLaT (ADMM)\\
PSNR: 21.78}{\includegraphics[width = 1.50in, height = 2.00in]{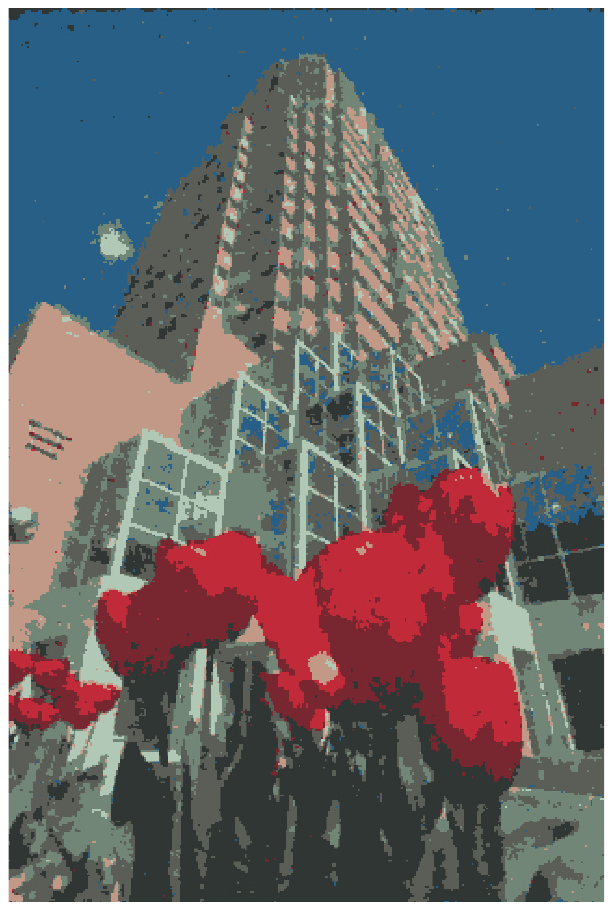}} &
	\captionsetup[subfigure]{justification=centering}
\subcaptionbox{AITV SLaT (DCA)\\
PSNR: 21.78}{\includegraphics[width = 1.50in, height = 2.00in]{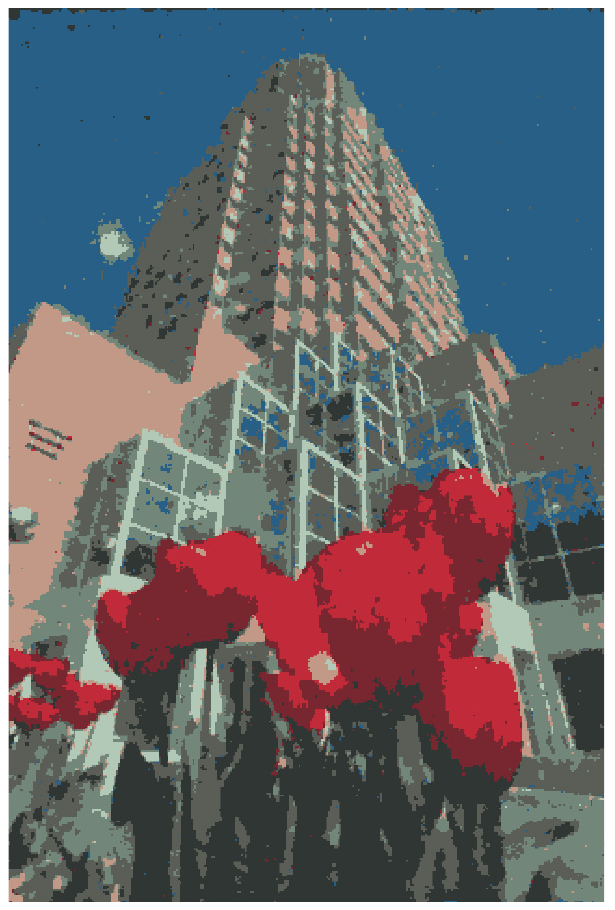}} 
\\ &
		\captionsetup[subfigure]{justification=centering}
\subcaptionbox{AITV FR \\ PSNR: 19.91}{\includegraphics[width = 1.50in, height = 2.00in]{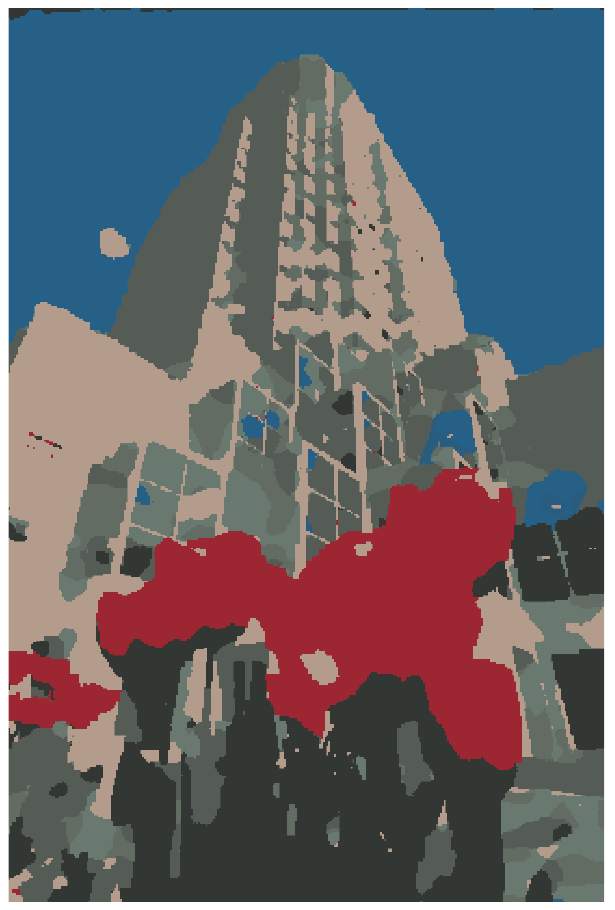}}&		\captionsetup[subfigure]{justification=centering}
\subcaptionbox{ICTM\\
PSNR: 18.63}{\includegraphics[width = 1.50in, height = 2.00in]{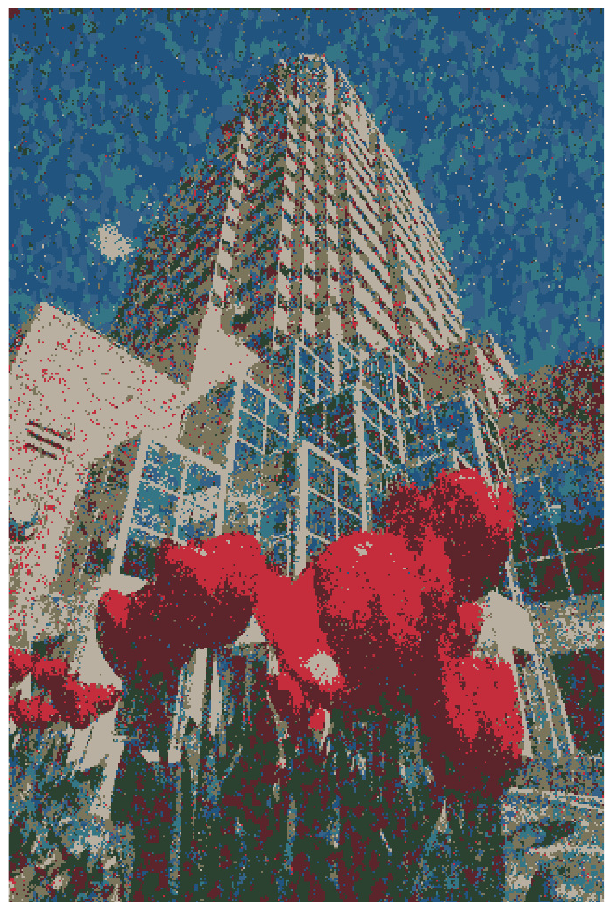}}&		\captionsetup[subfigure]{justification=centering}
\subcaptionbox{TV$^p$ MS\\
PSNR: 21.35}{\includegraphics[width = 1.50in, height = 2.00in]{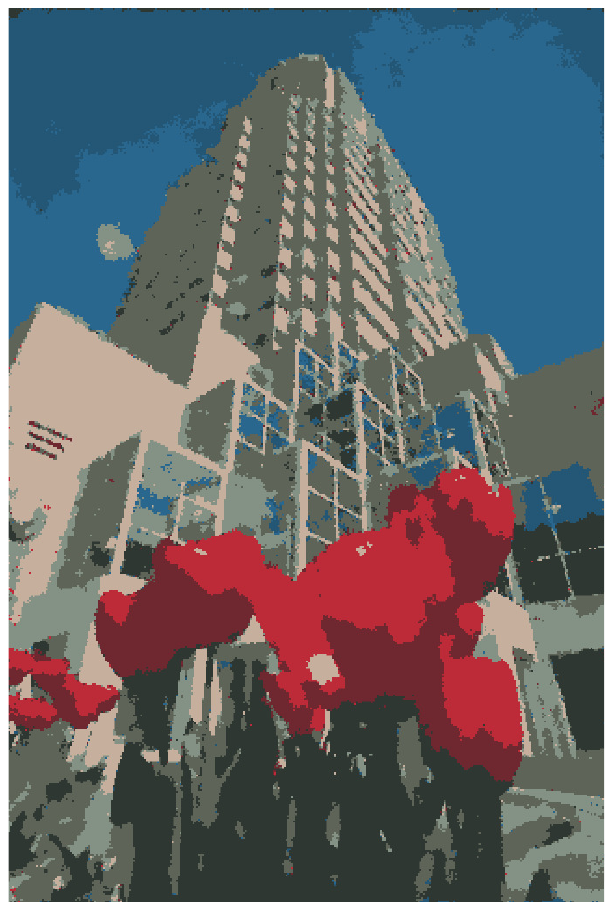}} &		\captionsetup[subfigure]{justification=centering}
\subcaptionbox{Convex Potts\\
PSNR: 21.17}{\includegraphics[width = 1.50in, height = 2.00in]{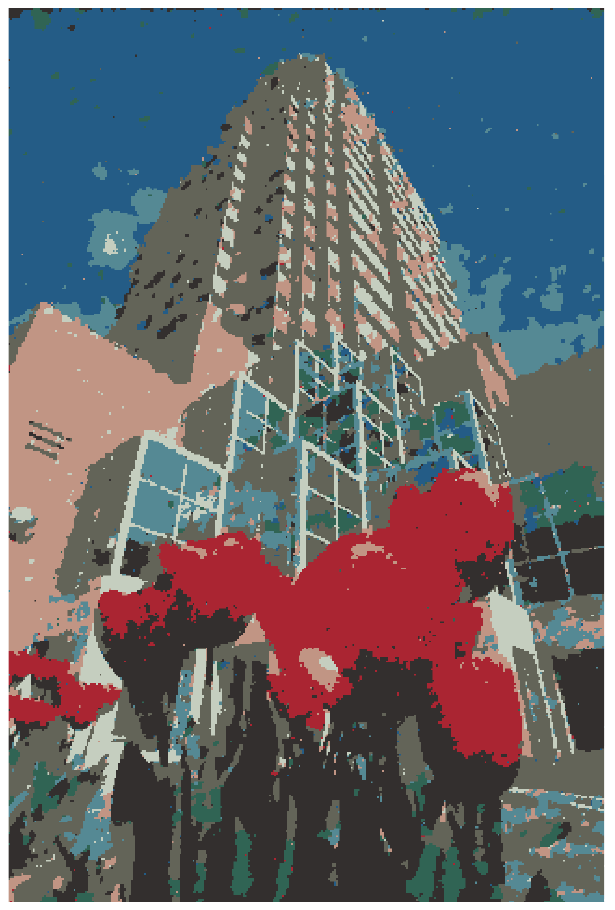}} &		\captionsetup[subfigure]{justification=centering}

\subcaptionbox{SaT-Potts\\
PSNR: 21.56}{\includegraphics[ width = 1.50in, height = 2.00in]{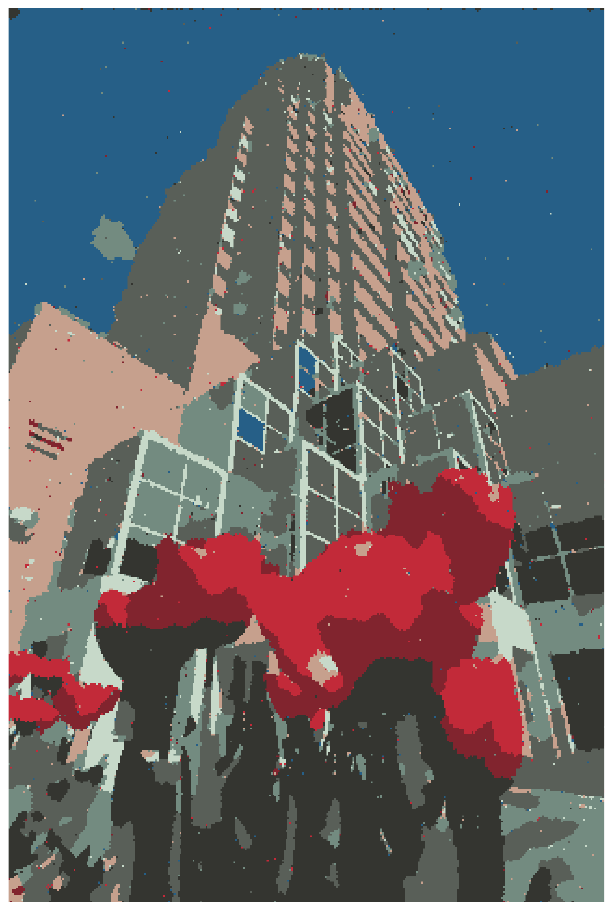}} \\ \hline \\
	\parbox[t]{2mm}{\multirow{2}{*}{\rotatebox[origin=c]{90}{Salt \& Pepper Noise}}} &
		\captionsetup[subfigure]{justification=centering}
\subcaptionbox{Noisy image.}{\includegraphics[width = 1.50in, height = 2.00in]{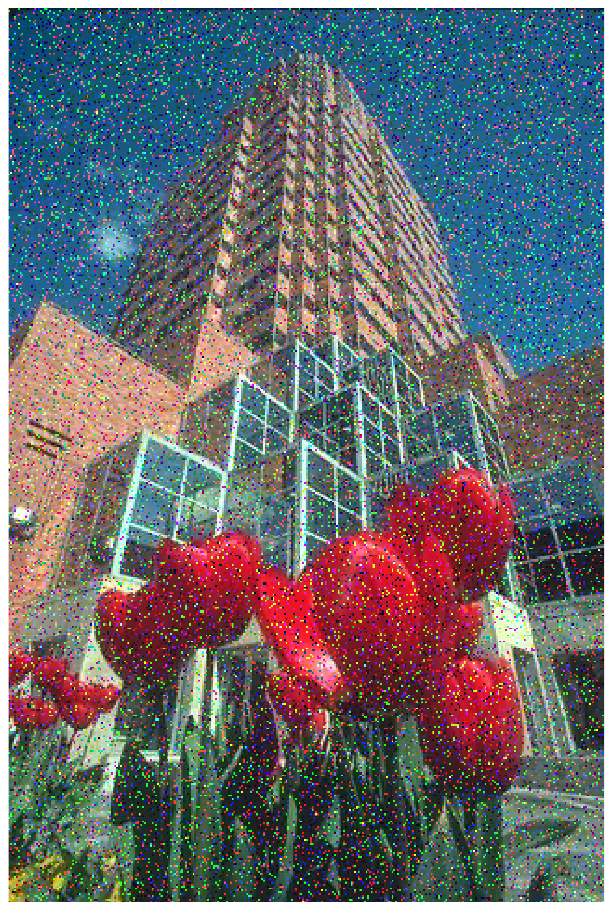}} &
		\captionsetup[subfigure]{justification=centering}

\subcaptionbox{(original) SLaT\\
PSNR:20.08}{\includegraphics[width = 1.50in, height = 2.00in]{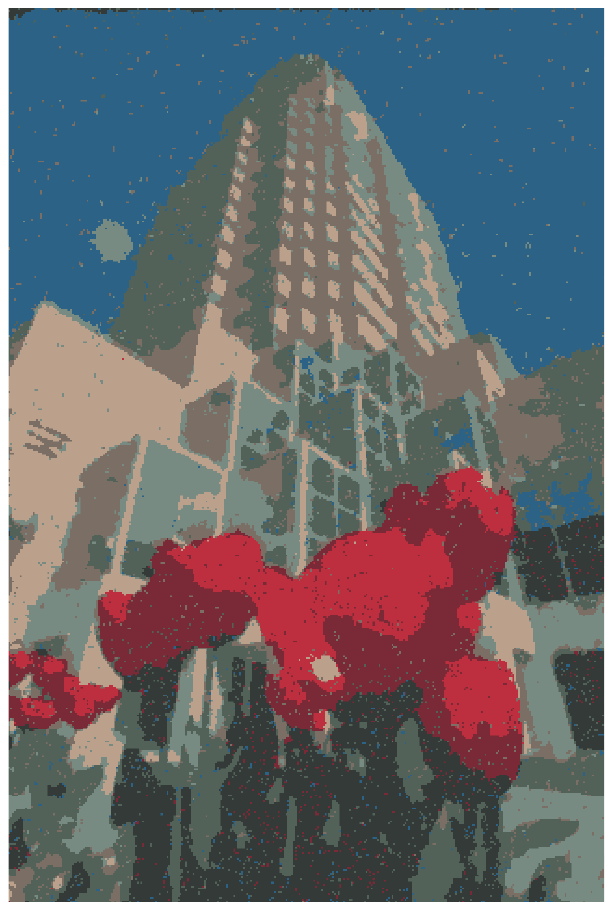}} & 		\captionsetup[subfigure]{justification=centering}
\subcaptionbox{TV$^{p}$ SLaT\\
PSNR: 19.90}{\includegraphics[ width = 1.50in, height = 2.00in]{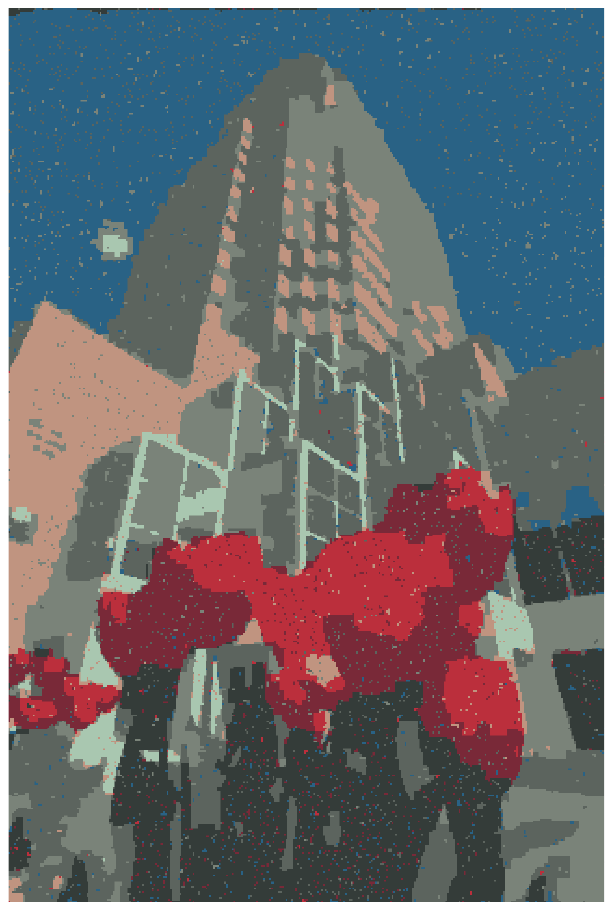}} &		\captionsetup[subfigure]{justification=centering}
\subcaptionbox{AITV SLaT (ADMM)\\
PSNR: 20.53}{\includegraphics[width = 1.50in, height = 2.00in]{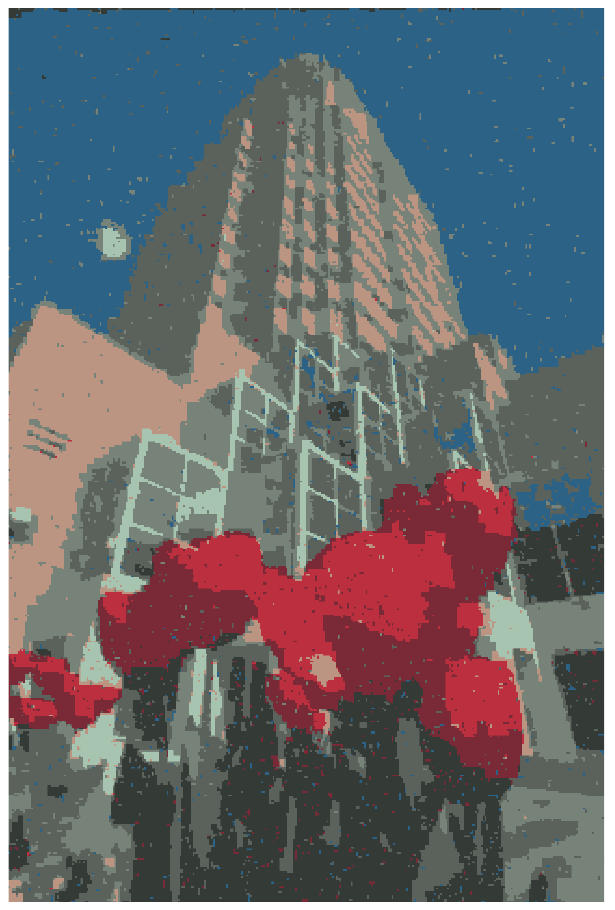}} &
	\captionsetup[subfigure]{justification=centering}
\subcaptionbox{AITV SLaT (DCA)\\
PSNR: 19.97}{\includegraphics[width = 1.50in, height = 2.00in]{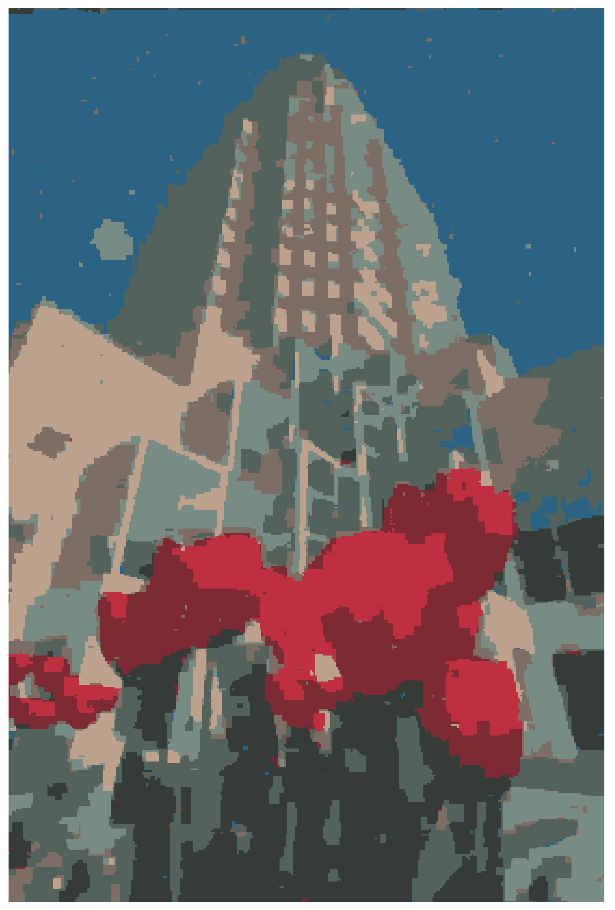}} 
\\ &
		\captionsetup[subfigure]{justification=centering}
\subcaptionbox{AITV FR \\ PSNR: 19.15}{\includegraphics[width = 1.50in, height = 2.00in]{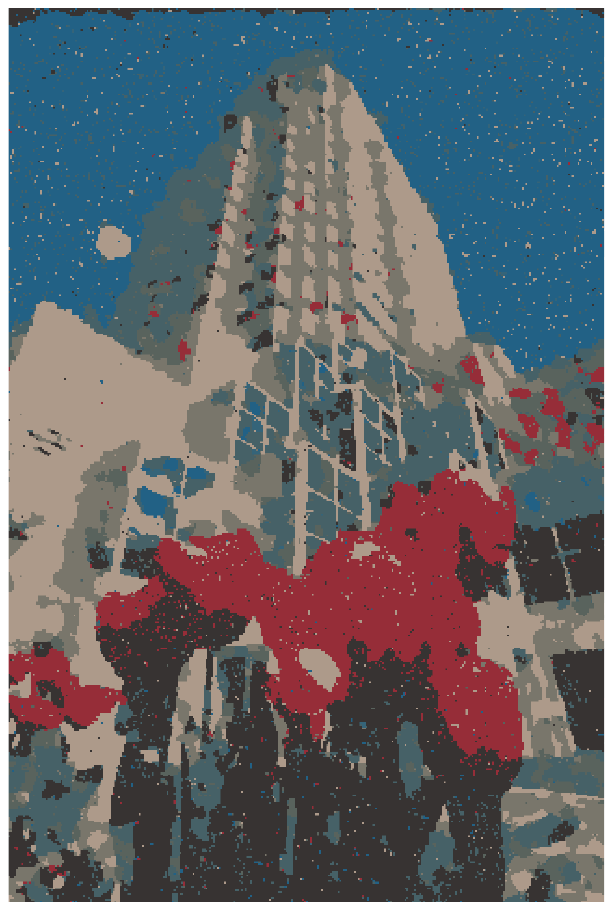}}&		\captionsetup[subfigure]{justification=centering}
\subcaptionbox{ICTM\\
PSNR: 16.57}{\includegraphics[width = 1.50in, height = 2.00in]{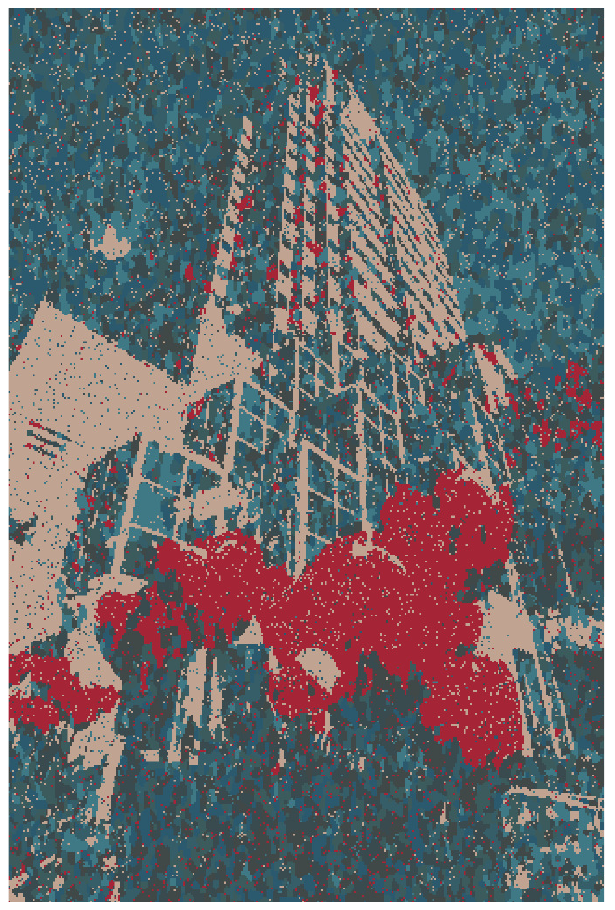}}&		\captionsetup[subfigure]{justification=centering}
\subcaptionbox{TV$^p$ MS\\
PSNR: 20.60}{\includegraphics[width = 1.50in, height = 2.00in]{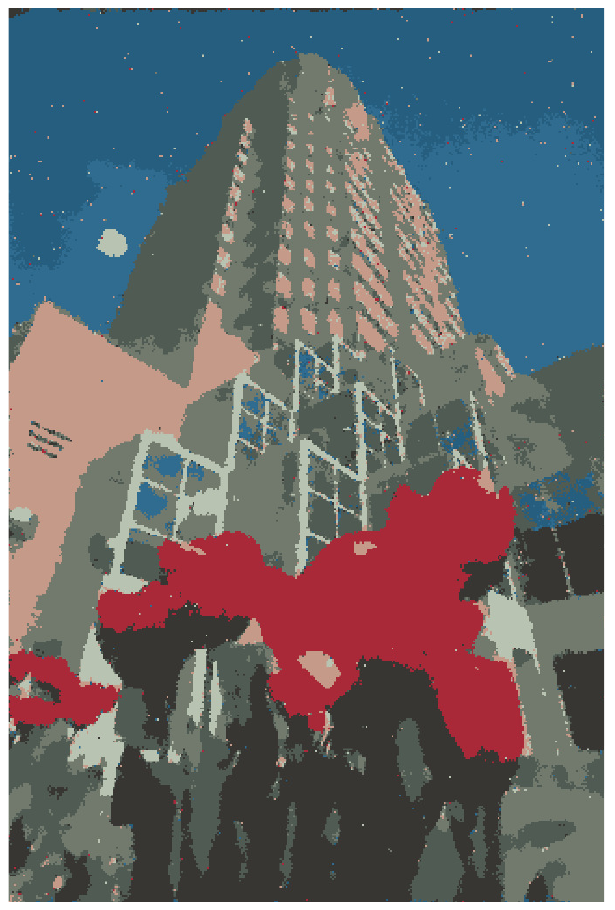}} &		\captionsetup[subfigure]{justification=centering}
\subcaptionbox{Convex Potts\\
PSNR: 19.18}{\includegraphics[width = 1.50in, height = 2.00in]{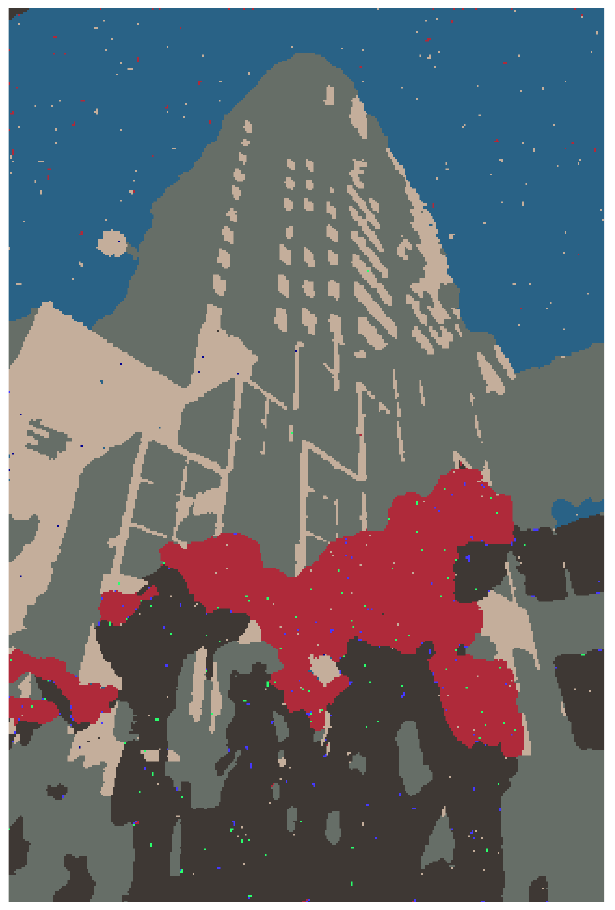}} &		\captionsetup[subfigure]{justification=centering}

\subcaptionbox{SaT-Potts\\
PSNR: 19.45}{\includegraphics[ width = 1.50in, height = 2.00in]{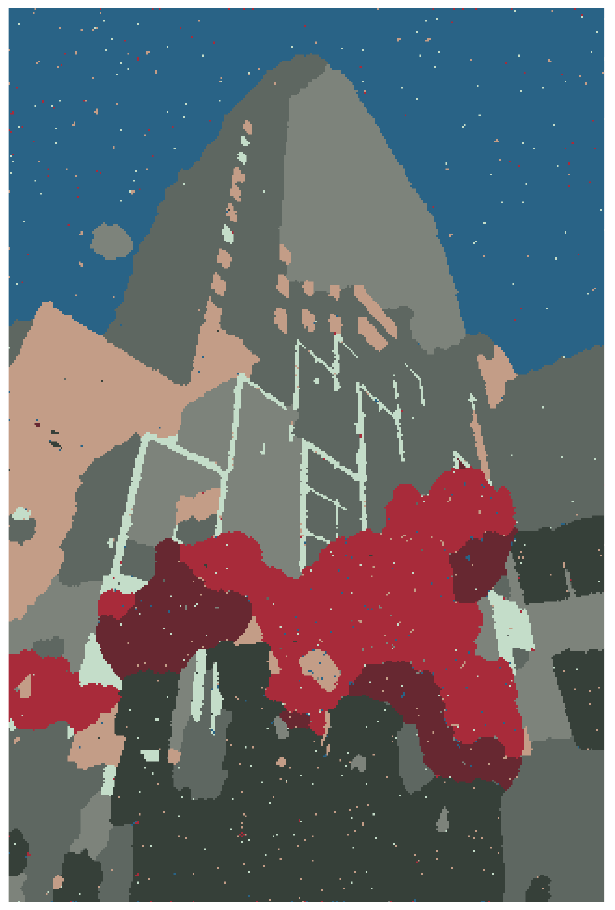}}
		\end{tabular}}
		\caption{Segmentation results into $k=8$ regions of Figure {\ref{fig:building}}  corrupted by either Gaussian noise of mean zero and variance 0.025 or 10\% SP noise.} 
		\label{fig:building_result}		
\end{figure}
\begin{figure}[ht!]
\resizebox{\textwidth}{!}{
\begin{tabular}{ccccc}
\rotatebox{90}{Original}&
\captionsetup[subfigure]{justification=centering}
\subcaptionbox{Garden \label{fig:garden2}}{\includegraphics[width = 1.50in]{circle.pdf}} & \captionsetup[subfigure]{justification=centering}
\subcaptionbox{Man \label{fig:man2}}{\includegraphics[width = 1.50in]{man.pdf}} &
\captionsetup[subfigure]{justification=centering}
\subcaptionbox{House\label{fig:house2}}{\includegraphics[width = 1.50in]{house.pdf}} & 		\captionsetup[subfigure]{justification=centering}
\subcaptionbox{Building\label{fig:building2}}{\includegraphics[width = 1.50in, height=2in]{building.pdf}} \\
\rotatebox{90}{AITV SLaT with Lab space}&
\captionsetup[subfigure]{justification=centering}
\subcaptionbox{PSNR: 20.42}{\includegraphics[width = 1.50in]{circle_AITV_SLAT_ADMM.pdf}} & \captionsetup[subfigure]{justification=centering}
\subcaptionbox{PSNR: 22.19}{\includegraphics[width = 1.50in]{man_AITV_SLAT_ADMM.pdf}} &
\captionsetup[subfigure]{justification=centering}
\subcaptionbox{PSNR: 21.85}{\includegraphics[width = 1.50in]{house_AITV_SLAT_ADMM.pdf}} & 		\captionsetup[subfigure]{justification=centering}
\subcaptionbox{PSNR: 21.78}{\includegraphics[width = 1.50in, height=2in]{building_AITV_SLAT_ADMM.pdf}} \\
\rotatebox{90}{AITV SLaT with HSV space}&
\captionsetup[subfigure]{justification=centering}
\subcaptionbox{PSNR: 18.25}{\includegraphics[width = 1.50in]{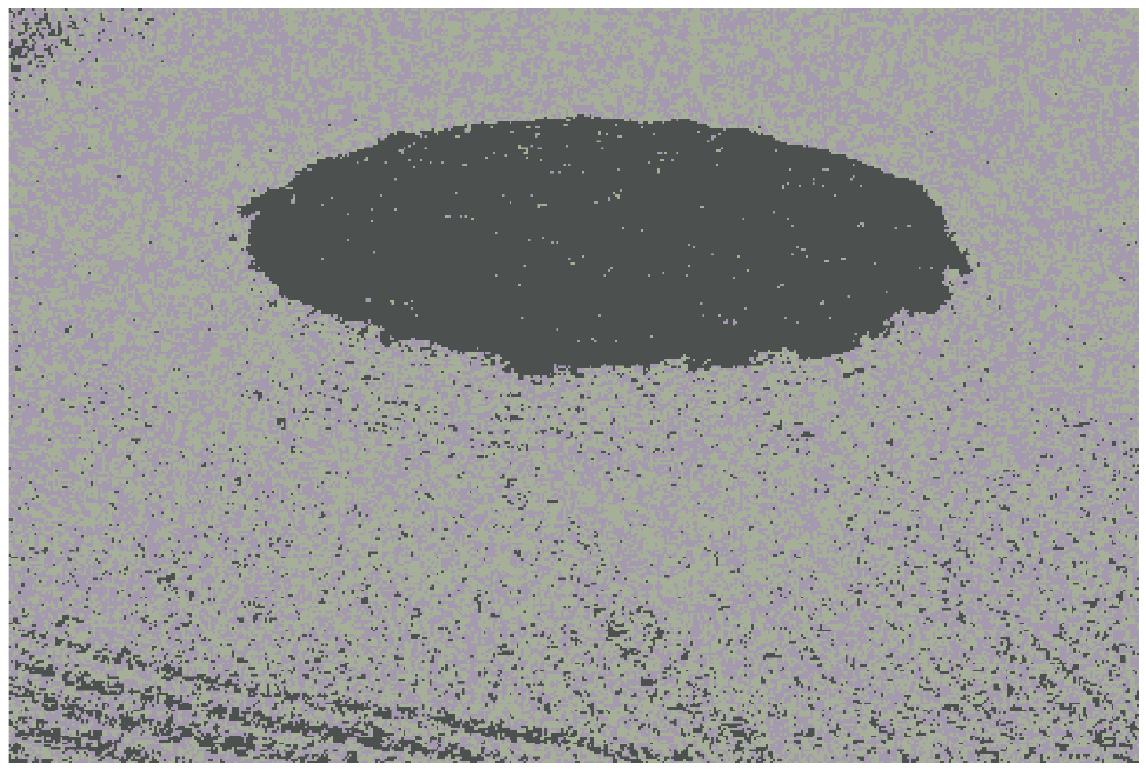}} & \captionsetup[subfigure]{justification=centering}
\subcaptionbox{PSNR: 19.93}{\includegraphics[width = 1.50in]{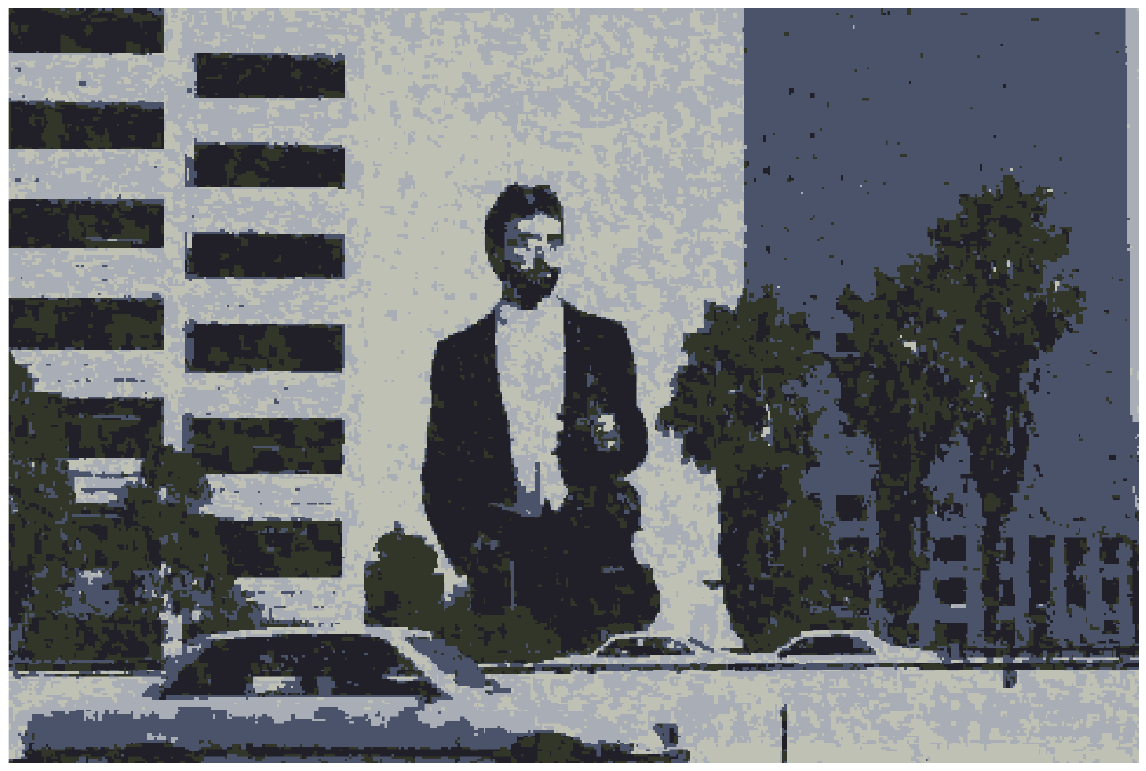}} &
\captionsetup[subfigure]{justification=centering}
\subcaptionbox{PSNR: 19.62}{\includegraphics[width = 1.50in]{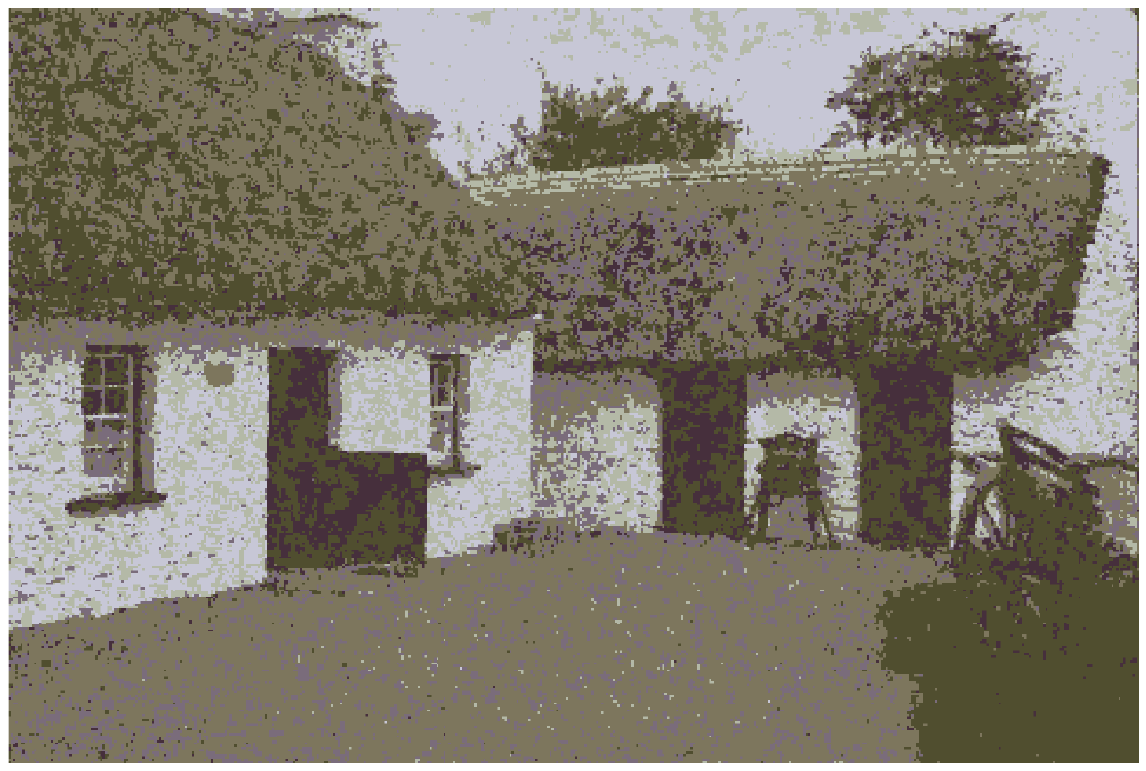}} & 		\captionsetup[subfigure]{justification=centering}
\subcaptionbox{PSNR: 20.08}{\includegraphics[width = 1.50in, height=2in]{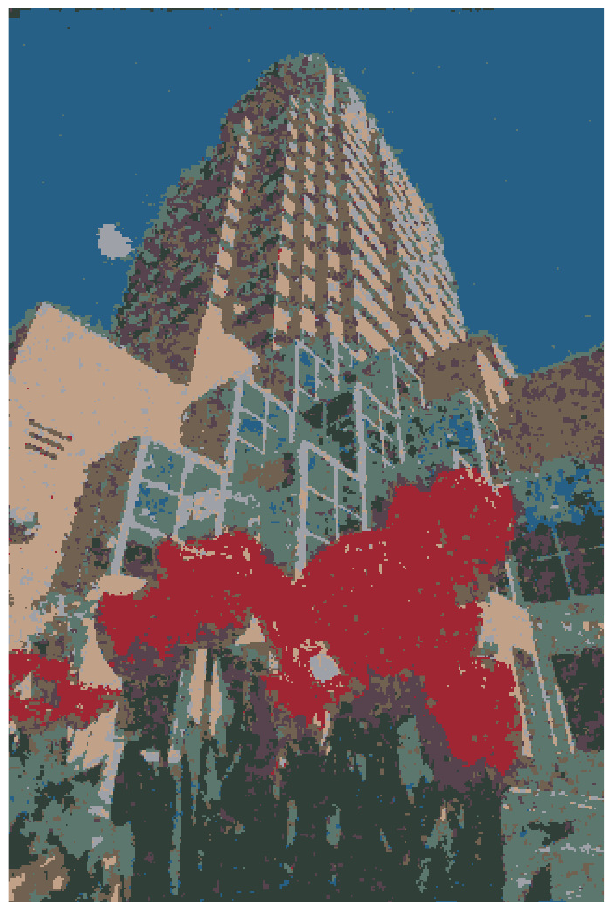}}
		\end{tabular}}
  		\caption{Comparison between using Lab space vs. HSV space for the AITV SLaT method.}
		\label{fig:lab_vs_hsv}
\end{figure}
}

In Figure {\ref{fig:circle_result}}, the sand lines are segmented in fine details by the SLaT methods, ICTM, and TV$^p$ MS in the Gaussian noise case and by the (original) SLaT, AITV SLaT, and TV$^p$ MS in the SP noise case. In Figure {\ref{fig:man_result}},   TV$^p$ SLaT, AITV SLAT, AITV FR, TV$^p$ MS, and SaT-Potts can clearly segment the multiple rows of windows on the top part of the building on the right. Under the SP noise, no algorithms succeed in the windows, but AITV SLaT (ADMM) and TV$^p$ MS are  able to preserve some parts of the man's eyes and the palm trees' green color and foliage. Despite AITV SLaT (ADMM) having a lower PSNR, the palm trees are greener in the segmentation result of AITV SLaT (ADMM) than TV$^p$ MS. In Figure {\ref{fig:house_result}}, under Gaussian noise, despite having lower PSNRs, both ADMM and DCA of AITV SLaT are able to more clearly segment the bottom half of the wheel at the lower right corner than TV$^p$ SLaT and TV$^p$ MS. Moreover, the roofs in the segmentation results of AITV SLaT are mostly brown while they have a considerable amount of green in the results of the TV$^p$ models. Under the SP noise, AITV SLaT (ADMM) provides the most visually appealing segmentation result even though its PSNR is not the best.  TV$^p$ MS identifies the green color of the grass and most of the wheels on the bottom right corner compared to any other methods. Lastly, for Figure {\ref{fig:building_result}}, under Gaussian noise, the SLaT methods, the Potts methods, and TV$^p$ MS produce visually similar segmentation results. Under SP noise, AITV SLaT (ADMM) and TV$^p$ MS segment more windows at the top of the building than any other methods. For all four figures, AITV SLaT (ADMM) and AITV SLaT (DCA) produce segmentation results with similar PSNR values, but the former is up to five times faster than the latter. Although   ICTM and SaT-Potts are the fastest methods, their segmentation results are less satisfactory.

\subsubsection{HSV vs. Lab}
The HSV (hue, saturation, and value) space is another popular, approximately uniform color space that could be used instead of Lab space for the SLaT methods. It was used to derived features for improving image segmentation algorithms \mbox{\cite{benninghoff2014efficient,burdescu2009new, chen2008fast, huang2007segmentation,paschos2001perceptually, sural2002segmentation}}. Some works \mbox{\cite{paschos2001perceptually, toure2018best}} claim that HSV space is better than Lab space for image segmentation. However, we provide numerical evidence to show that HSV space may not be as effective as Lab space for the SLaT methods.

To compare the segmentation results between HSV and Lab for the AITV SLaT method, we replace Lab with HSV in Algorithm {\ref{alg:sat_slat}} and apply the HSV-based algorithm to the images in Figure {\ref{fig:real_color}} corrupted with Gaussian noise with mean zero and variance 0.025. Figure {\ref{fig:lab_vs_hsv}} compares the segmentation results and the PSNR values between using HSV and Lab spaces. Overall, we observe that using Lab space for AITV SLaT leads to higher PSNR values and more detailed segmentation. For Figure {\ref{fig:garden2}}, using Lab space identifies more of the fine sand lines than HSV space. Unlike using HSV space, AITV SLaT with Lab is able to identify the windows on the right side of Figure {\ref{fig:man2}} and the wheel on the bottom right corner of Figure {\ref{fig:house}}. Lastly, for Figure {\ref{fig:building2}}, the tulips are clearly redder and more segmented in the result of AITV SLaT with Lab than with HSV. 

\subsection{Parameter Analysis}
\begin{figure*}[t]
	\centering
 \begin{tabular}{c@{}c@{}}
   		\captionsetup[subfigure]{justification=centering}\subcaptionbox{Changes in PSNR with respect to $\lambda$ with $\mu = 0.10$ fixed. \label{fig:lambda_analysis}}{\includegraphics[width=2.50in]{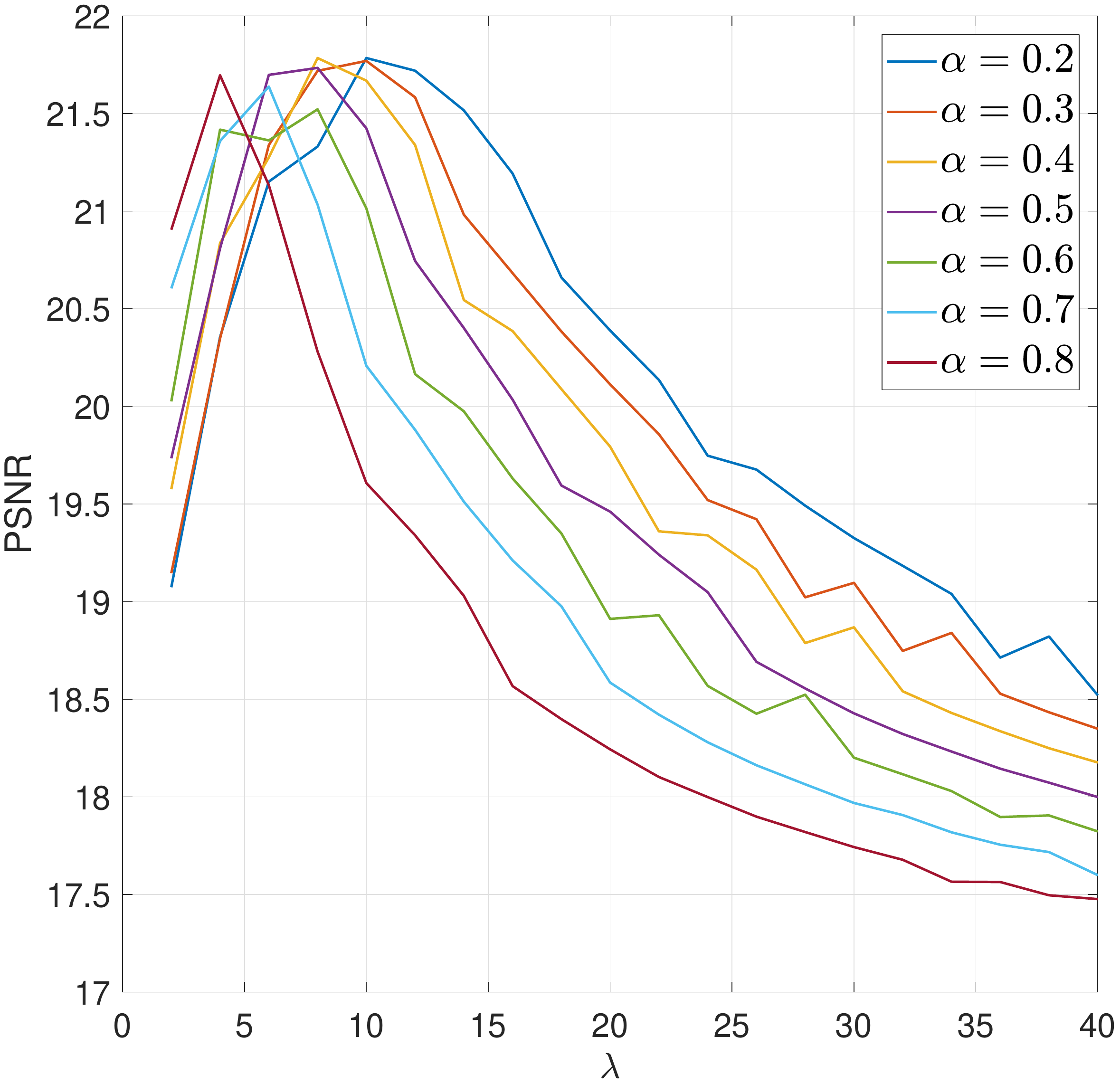}}  &\captionsetup[subfigure]{justification=centering}
		 \subcaptionbox{Changes in PSNR with respect to $\mu$ with $\lambda = 10$ fixed. \label{fig:mu_analysis}}{\includegraphics[width=2.50in]{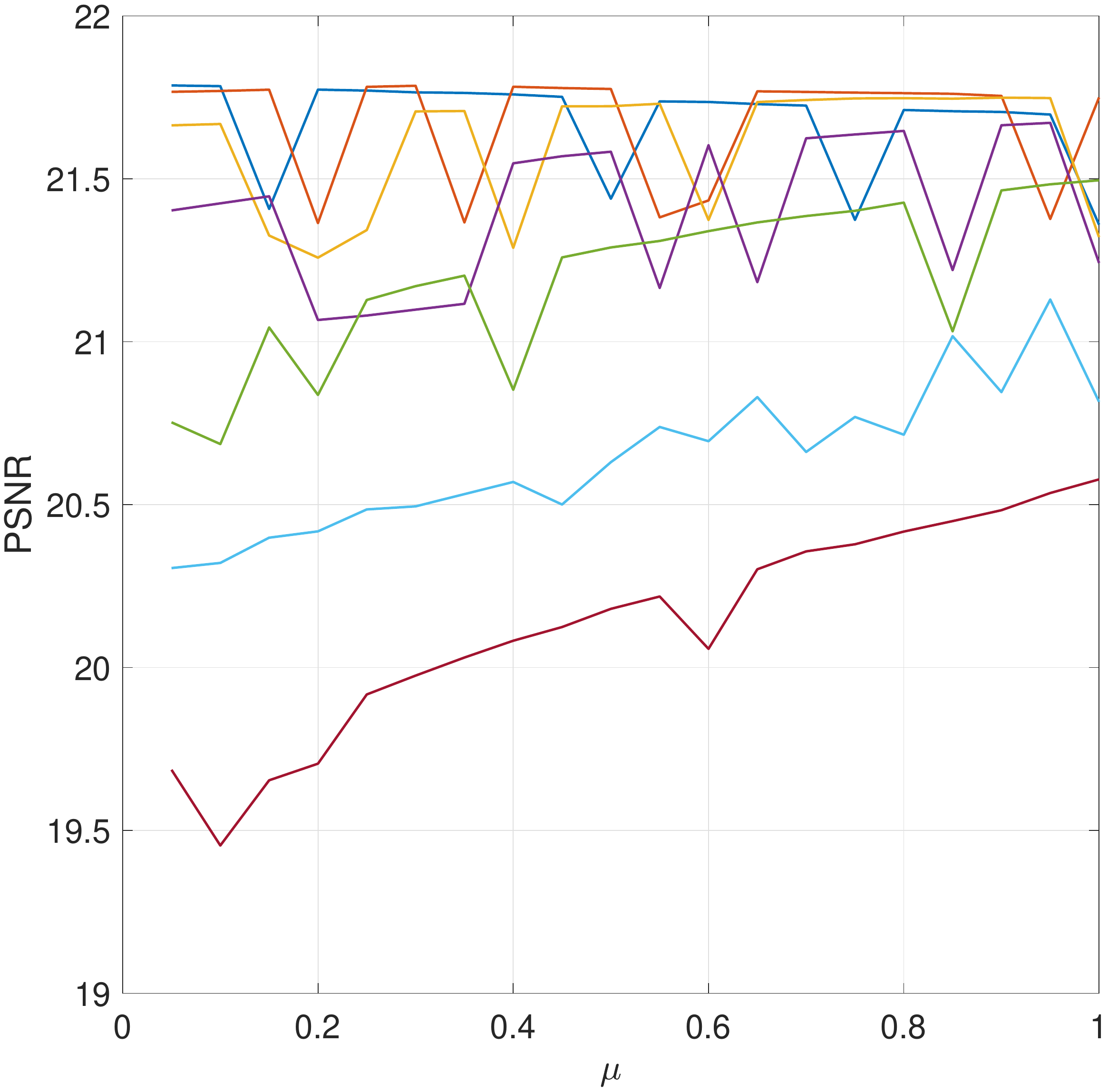}}
	\end{tabular}
	\caption{Sensitivity analysis on the model parameters $\lambda$ and $\mu$ to Figure \ref{fig:building} corrupted by Gaussian noise with mean 0 and variance 0.025. }
	\label{fig:parameter_analysis1}
\end{figure*}

\begin{figure*}[t]
	\centering
 \begin{tabular}{c@{}c@{}}
   		\captionsetup[subfigure]{justification=centering}\subcaptionbox{Changes in PSNR with respect to $\delta_0$ with $\sigma = 1.25$ fixed.\label{fig:delta_analysis1}}{\includegraphics[width=2.30in]{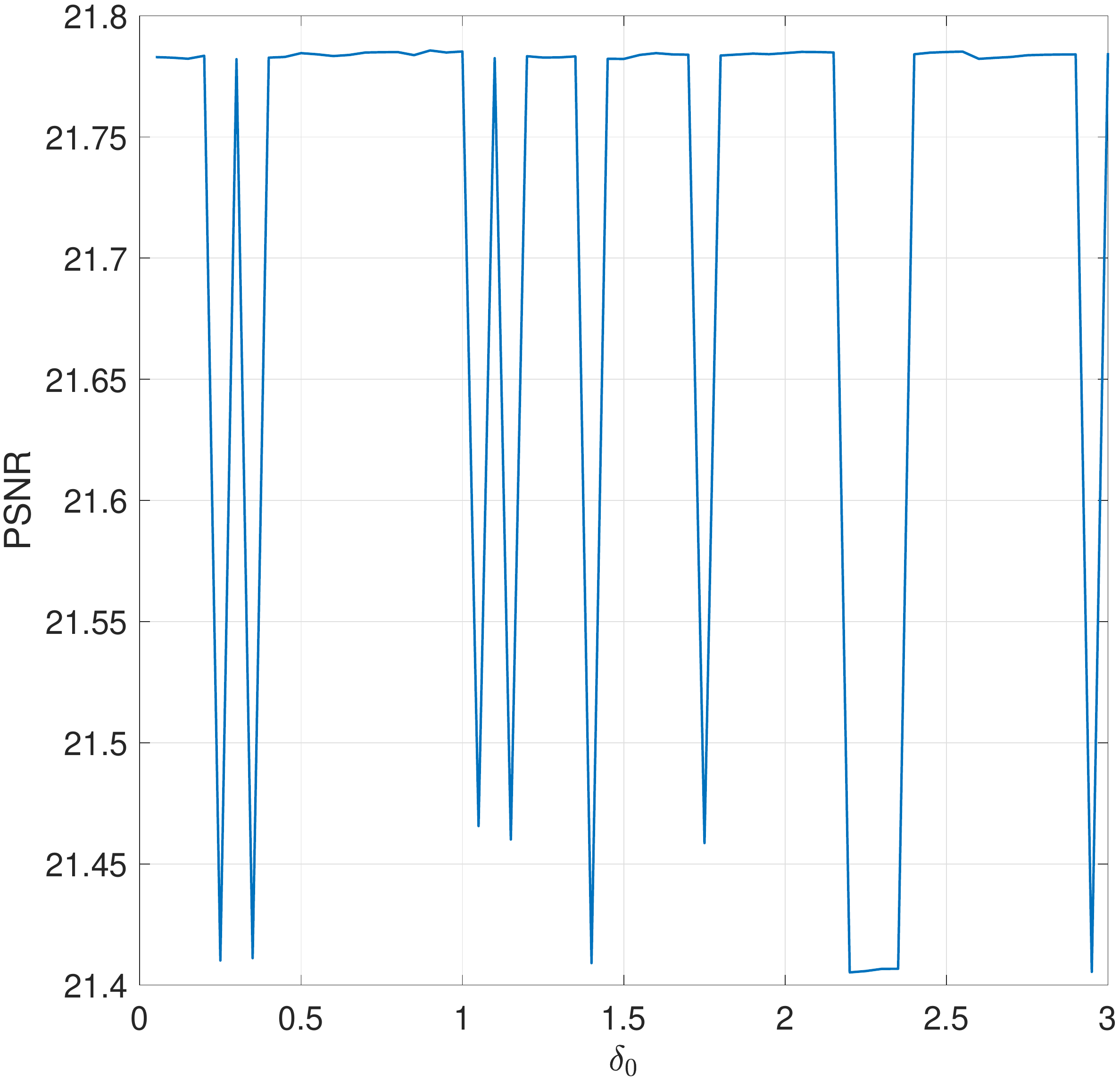}}  &\captionsetup[subfigure]{justification=centering}\subcaptionbox{Changes in computational time with respect to $\delta_0$ with $\sigma = 1.25$ fixed. \label{fig:delta_analysis2}}{\includegraphics[width=2.30in]{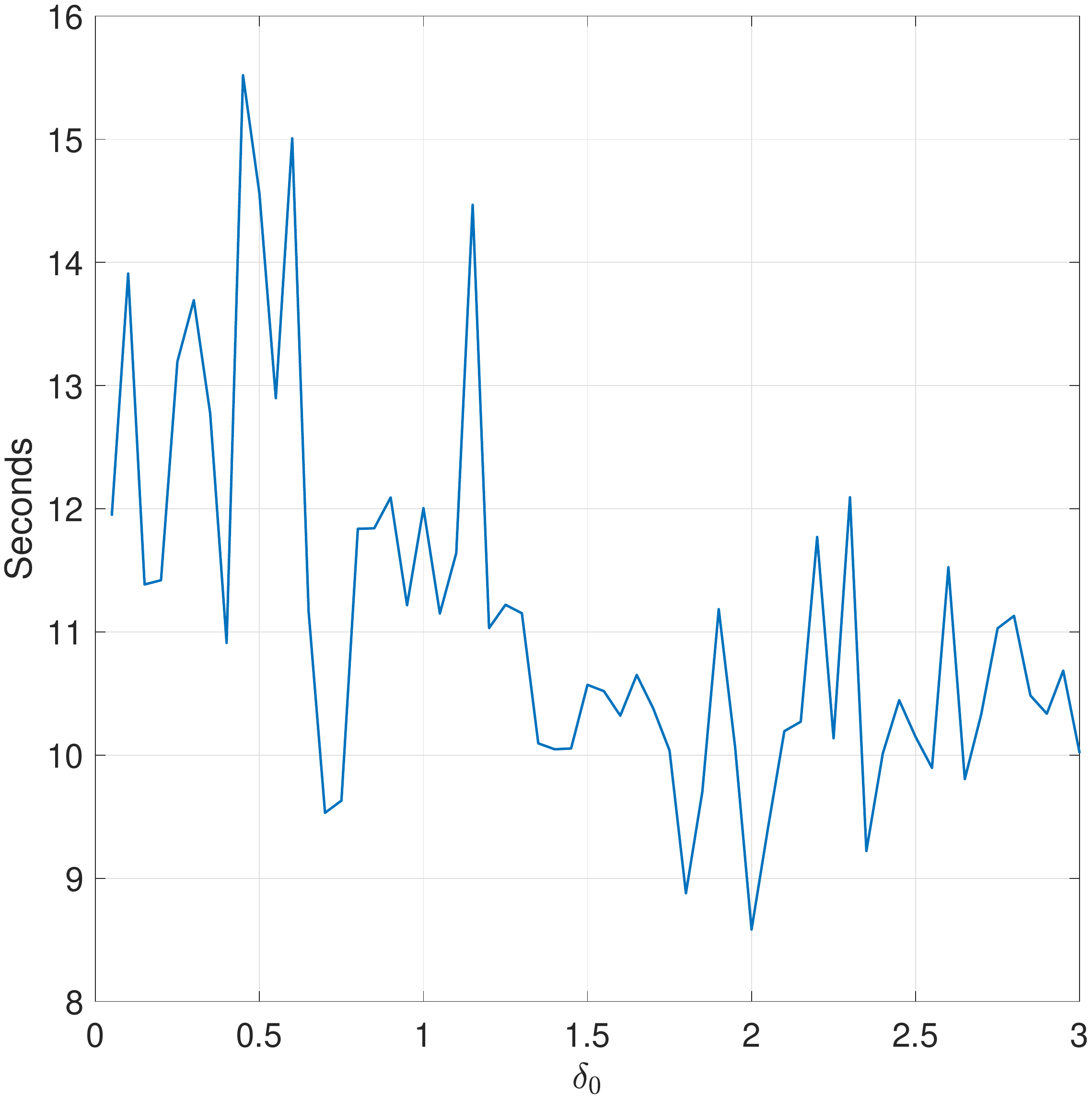}} \\ \captionsetup[subfigure]{justification=centering}
		 \subcaptionbox{Changes in PSNR with respect to $\sigma$ with $\delta_0 = 2$ fixed.\label{fig:sigma_analysis}}{\includegraphics[width=2.30in]{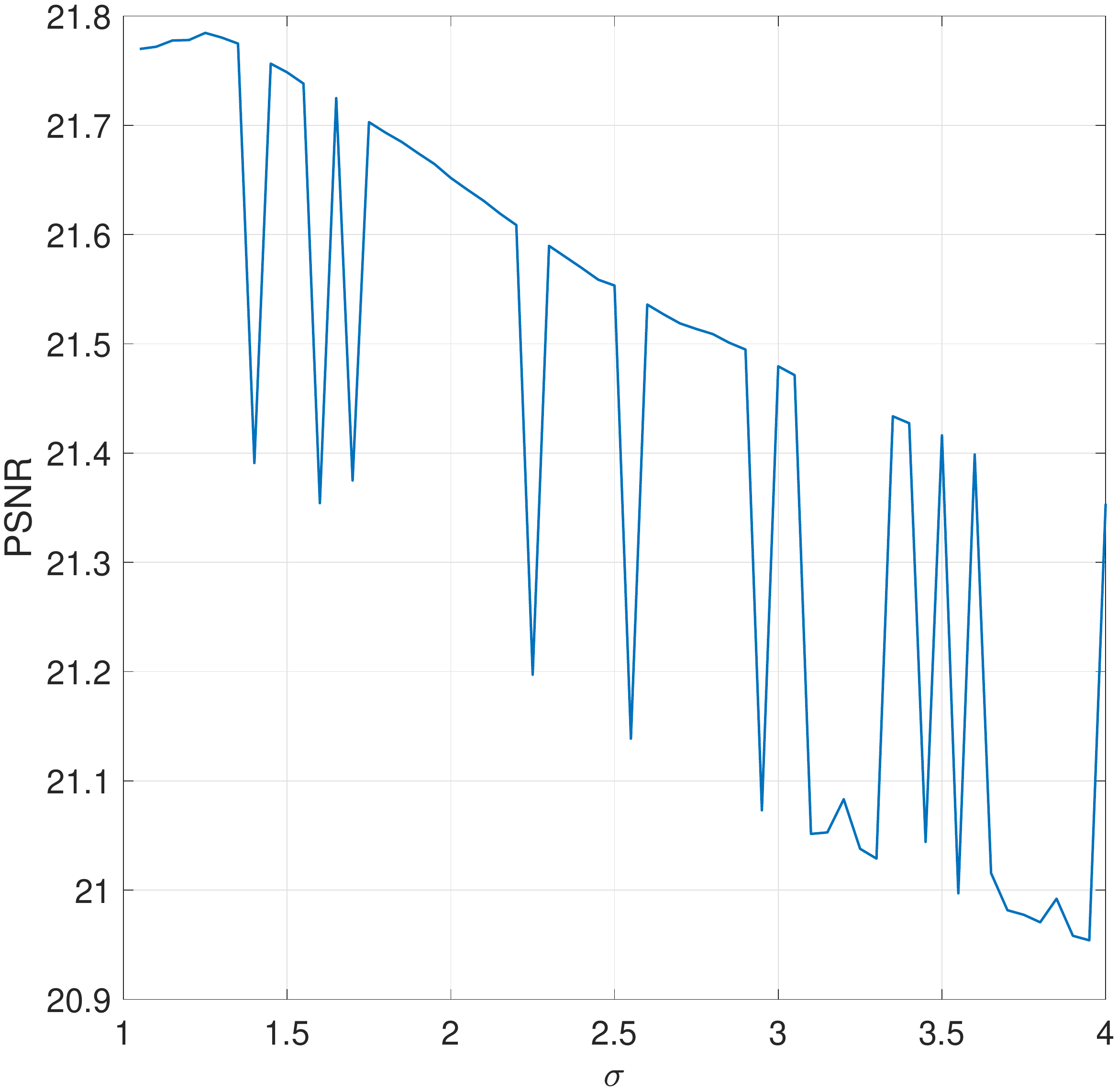}}
      		&\captionsetup[subfigure]{justification=centering}
		 \subcaptionbox{Changes in computational time with respect to $\sigma$ with $\delta_0 = 2$ fixed.\label{fig:sigma_analysis2}}{\includegraphics[width=2.30in]{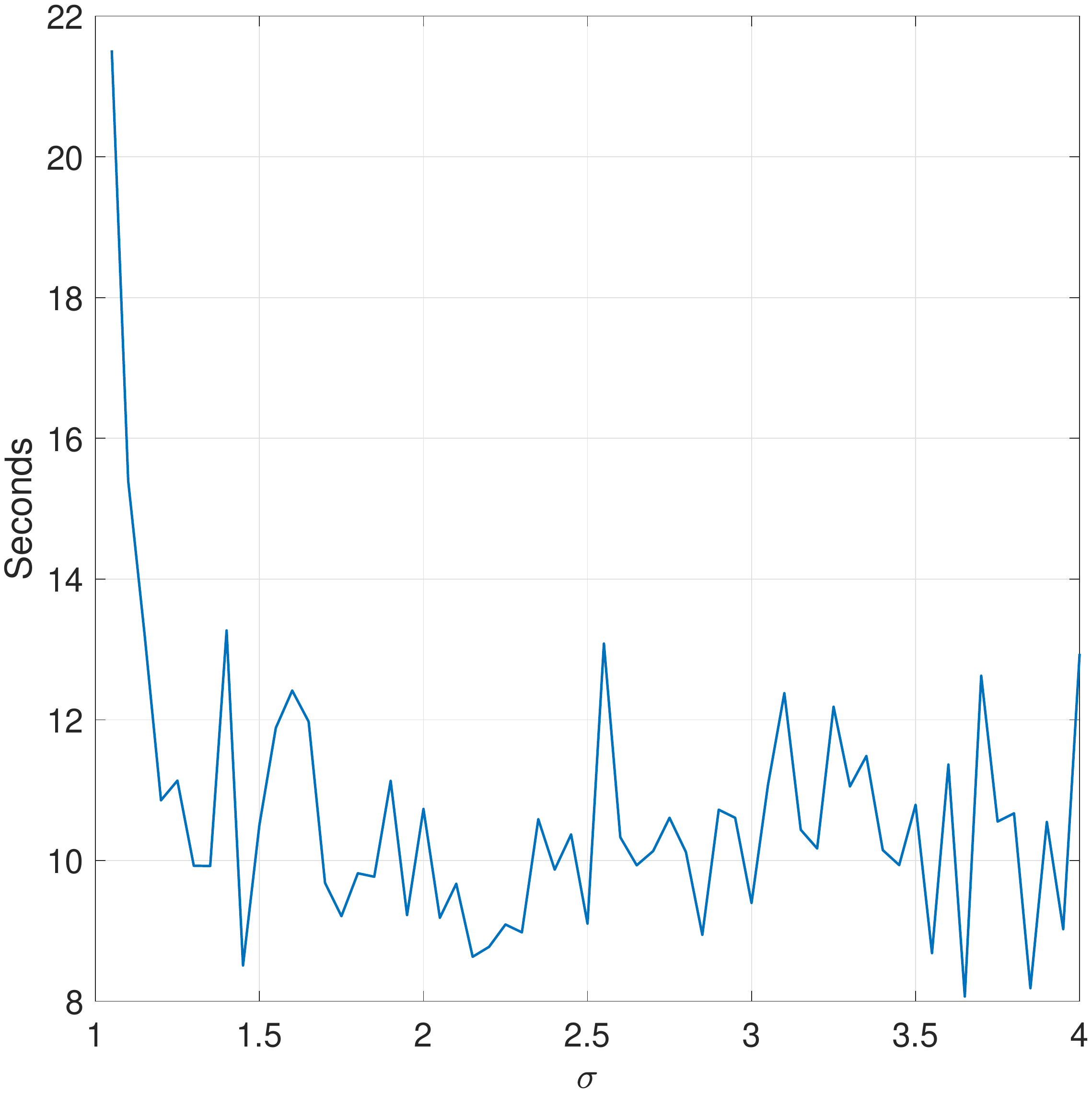}}
	\end{tabular}
	\caption{Sensitivity analysis on the ADMM algorithm parameters $\delta_0$ and $\sigma$. }
	\label{fig:parameter_analysis2}
\end{figure*}
\subsubsection{Model Parameters of \eqref{eq:AITV_MS}}
We analyze the following parameters in {\eqref{eq:AITV_MS}}:
\begin{itemize}
    \item $\lambda$: this fidelity parameter weighs how close the approximation $Au^*$ is to the original image $f$, where $u^*$ is a solution to {\eqref{eq:AITV_MS}}. When the image $f$ has a large amount of noise, choosing a small value for $\lambda$ is recommended. 
    \item $\mu$: this smoothing parameter determines the smoothness of the solution $u^*$ of {\eqref{eq:AITV_MS}}, which may help with denoising. However, choosing a large value for $\mu$ will deteriorate important edge information in $u^*$.  
    \item $\alpha \in [0,1]$: this sparsity parameter determines the gradient vector sparsity at each pixel, which is important in preserving edge information. However, choosing a large $\alpha$ may result in preserving some noise in the solution $u^*$.
\end{itemize}
To perform sensitivity analysis on the model parameters, we apply AITV SLaT (ADMM) with parameters $\delta_0 = 2$ and $\sigma =1.25$ to Figure {\ref{fig:building}} corrupted with Gaussian noise with mean 0 and variance 0.025. We examine the sparsity parameter $\alpha \in \{0.2, 0.3, \ldots, 0.8\}$ while we vary either the fidelity parameter $\lambda$ with $\mu = 0.10$ fixed or the smoothing parameter $\mu$ with $\lambda = 10$ fixed. The sensitivity analysis is visualized in Figure {\ref{fig:parameter_analysis1}}. 

Figure {\ref{fig:lambda_analysis}} shows that the PSNR has a concave relationship with respect to the fidelity parameter $\lambda$ for each value of $\alpha$. We observe that larger value of $\alpha$ leads to higher PSNR for smaller value of $\lambda$. More specifically, when $\lambda \leq 5.0$, the order of the PSNR curves follows the increasing value of $\alpha$. However, the order is reversed when $\lambda$ becomes large enough, such as when $\lambda \geq 10$.  Figure {\ref{fig:mu_analysis}} shows that with respect to the smoothing parameter, PSNR is generally increasing when $0.4 \leq \alpha \leq 0.8$ while it appears to be robust for $\alpha = 0.2, 0.3$.
\subsubsection{Algorithm Parameters of Algorithm \ref{alg:admm}}
We analyze the following parameters introduced in the ADMM algorithm that solves {\eqref{eq:AITV_MS}}:
\begin{itemize}
    \item $\delta_0$: this penalty parameter weighs the quadratic difference between the original variable $\nabla u$ and the auxiliary variable $w$.
    \item $\sigma$: this penalty multiplier determines the numerical convergence speed of the ADMM algorithm. 
\end{itemize}
We perform sensitivity analysis on AITV SLaT (ADMM) with model parameters $\lambda = 10, \mu = 0.1$, and $\alpha = 0.2$ to Figure {\ref{fig:building}} corrupted with Gaussian noise with mean 0 and variance 0.025. When varying $\delta_0$, we fix $\sigma = 1.25$ while when varying $\sigma$, we fix $\delta_0 =2$. Figure {\ref{fig:parameter_analysis2}} visualizes the sensitivity analysis of the algorithm parameters. 

According to Figures {\ref{fig:delta_analysis1}}-{\ref{fig:delta_analysis2}}, the penalty parameter $\delta_0$ does not have much influence on the PSNR, but it does affect the speed of the ADMM algorithm. When $\delta_0 < 1.5$, the computational time is between 9.5 to 15.5 seconds, but when $\delta_0 \geq 1.5$, it decreases to between about 9 to 12 seconds. As shown in Figures {\ref{fig:sigma_analysis}}-{\ref{fig:sigma_analysis2}}, the penalty multiplier $\sigma$ does have an impact on both the PSNR and the algorithm's numerical convergence. As $\sigma$ increases, the PSNR generally decreases. When $\sigma < 1.25$, the algorithm can be as slow as up to 22 seconds, but when $\sigma \geq 1.25$, it does speed up to between 8 and 13 seconds.

\section{Conclusion} \label{sec:conclusion}
In this paper, we proposed an efficient ADMM algorithm for the  SaT/SLaT framework that utilizes AITV regularization. When designing the ADMM algorithm, we incorporated the  proximal operator  for the $\ell_1 - \alpha \ell_2$ regularization \cite{louY18}. We provided convergence analysis of ADMM  to demonstrate that the algorithm subsequentially converges to an KKT point under certain conditions. In our numerical experiments, the AITV SaT/SLaT using our ADMM algorithm produces high-quality segmentation results within a few seconds. In addition, this work shows the effectiveness of using nonconvex regularizations in image processing. As for future works, we will explore other nonconvex regularizations, such as transformed $\ell_1$ \cite{zhang2014minimization, zhang2018minimization}, as alternative options to AITV and TV$^p (0 < p < 1)$ under the SaT/SLaT framework.  To simplify the SLaT framework for color images, we plan to apply these nonconvex regularizations in quaternion space to complement the quaternion-based SaT model {\cite{wu2022efficient}} with  $\ell_1/\ell_2$ regularization \mbox{\cite{rahimi2019scale, wang2019accelerated, wang2021limited, wang2021minimizing}}.
\backmatter

\bmhead{Acknowledgments}

The authors thank Xu Li for providing code and answering questions about the IIH image in \cite{li2020three}. The authors also thank Elisha Dayag for writing the initial code for the TV$^p$-regularized Mumford-Shah model described in \cite{li2020tv}.
The work was partially supported by NSF grants  DMS-1854434, DMS-1952644, DMS-2151235, DMS-2219904, and CAREER 1846690. We thank the two reviewers for their valuable feedback in improving the quality of the manuscript.

\section*{Declarations}
\textbf{Conflict of Interest} The authors declare that they have no known competing financial interests or personal relationships that could have appeared to influence the work reported in this paper.\\

\noindent\textbf{Availability of Data and Material} The images in Section \ref{sec:synthetic_image} are available from the corresponding author on reasonable request. The images in Section \ref{sec:real_grayscale} are available at \url{https://www.wisdom.weizmann.ac.il/~vision/Seg_Evaluation_DB/}. The images in Section \ref{sec:real_color} are available at \url{https://www2.eecs.berkeley.edu/Research/Projects/CS/vision/bsds/}.\\

\noindent\textbf{Code Availability} Code generated is available at \url{https://github.com/kbui1993/Official_AITV_SaT_SLaT}.

\appendix
\section{Proofs of Section \ref{sect:convergence}}
\subsection{Proof of Lemma \ref{lemma:strong_convexity_ineq}}
\begin{proof}
It is straightforward that $u^{\top} A^{\top} A u = \|Au\|_2^2 \geq 0$ and $u^{\top} \nabla^{\top} \nabla u = \|\nabla u\|_2^2 \geq 0$ for any $u \in X$, so $\zeta\geq 0$. If $\zeta=0,$ then there exists a nonzero vector $x \in X$ such that 
$
\lambda \|Ax\|_2^2 + (\mu+\delta_0) \|\nabla x\|_2^2 = \lambda x^{\top} A^{\top} A x +(\mu+\delta_0)x^{\top} \nabla^{\top} \nabla x = 0.
$
Then we shall have $x \in \text{ker}(A) \cap \text{ker}(\nabla)$, contradicting that $\text{ker}(A) \cap \text{ker}(\nabla) = \{0\}$. Therefore, $\zeta > 0$ and hence we get
\begin{align*}
     \lambda \|Au\|_2^2 + (\mu+\delta_0)\|\nabla u\|_2^2 \geq  \zeta \|u\|_2^2 \quad \forall u \in X.
\end{align*}
As  $\delta_{t+1}\geq \delta_t$ ($\sigma \geq 1$),  $\mathcal{L}_{\delta_t}(u, w_{t}, z_t)$ is a strongly convex function of $u$ with parameter $\zeta>0$. 
Fixing $w_t, z_t$, the minimizer $u_{t+1}$  of $\mathcal{L}_{\delta_t}(u, w_t, z_t)$ in  \eqref{eq:u_update} satisfies the following inequality \cite[Theorem 5.25]{beck2017first},
\begin{align}\label{eq:u_ineq}
\begin{split}
    &\mathcal{L}_{\delta_{t}} (u_{t+1}, w_t ,z_t) - \mathcal{L}_{\delta_{t}} (u_t, w_t ,z_t)
    \leq - \frac{  \zeta}{2} \|u_{t+1} - u_{t}\|_2^2.
\end{split}
\end{align}

As $w_{t+1}$ is the optimal solution to \eqref{eq:w_update}, we have
\begin{align}\label{eq:w_ineq}
     \mathcal{L}_{\delta_t}(u_{t+1}, w_{t+1}, z_t) - \mathcal{L}_{\delta_t}(u_{t+1}, w_{t}, z_t) \leq 0.
\end{align}
It follows from the update  \eqref{eq:z_update} that
\begin{gather}
\begin{aligned}\label{eq:z_ineq}
    \mathcal{L}_{\delta_t}(u_{t+1}, w_{t+1}, z_{t+1}) - \mathcal{L}_{\delta_t} (u_{t+1}, w_{t+1}, z_t) &= \langle z_{t+1}-z_t, \nabla u_{t+1} - w_{t+1}  \rangle\\ &= \frac{1}{\delta_t} \|z_{t+1} - z_t\|_2^2.
\end{aligned}
\end{gather}
Similarly, we get
\begin{gather}
\begin{aligned}\label{eq:delta_ineq}
    \mathcal{L}_{\delta_{t+1}}(u_{t+1}, w_{t+1}, z_{t+1})  - \mathcal{L}_{\delta_t}(u_{t+1}, w_{t+1}, z_{t+1})
    =&\frac{\delta_{t+1} - \delta_{t}}{2} \|\nabla u_{t+1} - w_{t+1}\|_2^2\\
    =& \frac{\delta_{t+1} - \delta_{t}}{2 \delta_t^2} \|z_{t+1}-z_t\|_2^2.
\end{aligned}
\end{gather}
Combining \eqref{eq:u_ineq}-\eqref{eq:delta_ineq} leads to the desired inequality
\begin{align*} 
    \mathcal{L}_{\delta_{t+1}}(u_{t+1}, w_{t+1}, z_{t+1}) -  \mathcal{L}_{\delta_{t}}(u_{t}, w_{t}, z_{t}) &\leq \frac{\delta_{t+1} - \delta_{t}}{2 \delta_t^2} \|z_{t+1}-z_t\|_2^2 + \frac{1}{\delta_t} \|z_{t+1} - z_t\|_2^2\\ &\quad - \frac{\zeta}{2} \|u_{t+1} - u_t\|_2^2 \\
    &= \frac{ \sigma+1}{2\sigma^t\delta_0} \|z_{t+1}- z_t\|_2^2- \frac{\zeta}{2} \|u_{t+1} - u_t\|_2^2.
\end{align*}
\end{proof}
\subsection{Proof of Proposition \ref{prop:partial_conv1}}
\begin{proof}
(a) We  start by proving the boundedness of $\{z_t\}_{t=1}^{\infty}.$ The optimality condition of \eqref{eq:w_update} at iteration $t$ is expressed by
\begin{align}
\begin{split}
    0 &\in \partial \left( \|w_{t+1}\|_1 - \alpha \|w_{t+1}\|_{2,1} \right)- \delta_t \left( \nabla u_{t+1} -w_{t+1}\right) - z_t.
    \end{split}
\end{align}
Together with \eqref{eq:z_update}, we have
\begin{align}\label{eq:subgradient_ineq}
\begin{split}
    z_{t+1} &\in \partial \left( \|w_{t+1}\|_1 - \alpha \|w_{t+1}\|_{2,1}\right)\subset \partial \|w_{t+1}\|_1 - \alpha \partial \|w_{t+1}\|_{2,1},
    \end{split}
\end{align}
which implies that there exist two vectors $v_{1} \in \partial \|w_{t+1}\|_1$ and $v_{2} \in \partial \|w_{t+1}\|_{2,1}$ such that $z_{t+1} = v_{1}  - \alpha v_{2}$.
For any $v \in \partial \|w\|_1$, we have
\begin{align} \label{eq:l1_subeq}
    (v_x)_{i,j} = 
    \sign((w_x)_{i,j}) \text { and }(v_y)_{i,j} = 
    \sign((w_y)_{i,j}),
\end{align}
which guarantees that $\|v\|_\infty\leq 1.$
If $z \in \partial \|w\|_{2,1}$, then 
\begin{align} \label{eq:l2_subeq}
    z_{i,j} = \begin{cases} \displaystyle
    \frac{w_{i,j}}{\|w_{i,j}\|_2} &\text{ if } \|w_{i,j}\|_2 \neq 0, \\
    \in \{z_{i,j} \in \mathbb{R}^2: \|z_{i,j}\|_2 \leq 1 \} &\text{ if } \|w_{i,j}\|_2 = 0.
    \end{cases}
\end{align}
By \eqref{eq:l2_subeq}, we have $\|(v_{2})_{i,j}\|_2 \leq 1$, which means that $\|v_{2}\|_{\infty} \leq1$. As a result, $\|z_{t+1}\|_{\infty} \leq \|v_{ 1}\|_{\infty} + \alpha \|v_{2}\|_{\infty} \leq 2$. Altogether, we arrive at an upper bound, i.e.,
\begin{align}\label{eq:z_bound}
\begin{split}
    \|z_{t+1}\|_2 &= \sqrt{\sum_{i,j} \left(|(z_{t+1,x})_{i,j}|^2+ |(z_{t+1,y})_{i,j}|^2\right)}\\ &\leq \sqrt{ 2^2 (2 MN)} = 2 \sqrt{2 MN}. 
    \end{split}
\end{align}

By telescoping summation of \eqref{eq:strong_convexity_ineq}, we have for all $t$ that
\begin{align*}
    \mathcal{L}_{\delta_{t+1}}(u_{t+1}, w_{t+1}, z_{t+1})   \leq& \mathcal{L}_{\delta_{0}}(u_{0}, w_{0}, z_{0}) + \frac{(\sigma+1)}{2 \delta_0} \sum_{i=0}^t \frac{1}{\sigma^i} \|z_{i+1} - z_i\|_2^2 \\ \leq &\mathcal{L}_{\delta_{0}}(u_{0}, w_{0}, z_{0}) + \frac{(\sigma+1)}{2 \delta_0} \sum_{i=0}^{\infty} \frac{1}{\sigma^i} \|z_{i+1} - z_i\|_2^2. 
\end{align*}
Now that $\{z_t\}_{t=1}^{\infty}$ is bounded, then $\{\|z_{t+1}-z_t\|_2^2\}_{t=1}^{\infty}$ is  bounded. Denote  $C :=  \displaystyle \sup_{t \in \mathbb{N}} \|z_{t+1}-z_t\|_2^2$. 
If $\sigma = 1$ and $\displaystyle \sum_{i=0}^{\infty} \|z_{i+1} - z_i\|_2^2 < \infty$, then $\{\mathcal{L}_{\delta_t}(u_t, w_t, z_t)\}_{t=1}^{\infty}$ is uniformly bounded above. On the other hand, if $\sigma >1$,  then we get
\begin{align*}
    \mathcal{L}_{\delta_{t+1}}(u_{t+1}, w_{t+1}, z_{t+1}) 
   \leq \mathcal{L}_{\delta_{0}}(u_{0}, w_{0}, z_{0}) + \frac{C(\sigma+1)}{2\delta_0} \sum_{i=0}^{\infty} \frac{1}{\sigma^i} < \infty,
\end{align*}
where the infinite sum  converges for $\sigma > 1$. In either case, we have that $\{\mathcal{L}_{\delta_t}(u_t, w_t, z_t)\}_{t=1}^{\infty}$ is uniformly bounded above, and hence there exists a constant $\tilde{C} > 0$ such that $ \mathcal{L}_{\delta_t}(u_t, w_t, z_t) < \tilde{C}$.

Since $\|w\|_{2,1} \leq \|w\|_1$, we have
\begin{align*}
\frac{\mu}{2} \|\nabla u_t\|_2^2 - \frac{1}{2 \delta_t} \|z_t\|_2^2 \leq \mathcal{L}_{\delta_t} (u_t, w_t,z_t) \leq \tilde{C}.
\end{align*}
This suggests an upper bound  of $\|\nabla u_{t}\|_2,$ i.e.,
\begin{align*}
    \|\nabla u_t\|_2 \leq  \sqrt{\frac{2}{\mu}\left(\tilde{C} + \frac{1}{2 \delta_t} \|z_t\|_2^2 \right)} \leq \sqrt{\frac{2}{\mu}\left(\tilde{C}+ \frac{4MN}{\delta_0}\right)}.
\end{align*}
Moreover, we observe that
\begin{align*}
    \frac{\lambda}{2} \|f-Au_t\|_2^2 - \frac{1}{2\delta_t} \|z_t\|_2^2 \leq \mathcal{L}_{\delta_t}(u_t, w_t, z_t) \leq \tilde{C}.
\end{align*}
As $\{z_t\}_{t=1}^{\infty}$ is bounded, then $\{\|f-Au_t\|_2^2\}_{t=1}^{\infty}$ is bounded as well.  Altogether $\{F(u_t)\}_{t=1}^{\infty}$ is a bounded sequence, and hence we conclude that $\{u_t\}_{t=1}^{\infty}$ is bounded by coercivity  in Lemma \ref{lemma:coercive}. Lastly, from {\eqref{eq:z_update}}, we have
\begin{align*}
   \|w_t\|_2 \leq \left \| \nabla u_t - \frac{z_{t} - z_{t-1}}{\delta_{t-1}}  \right\|_2^2 \leq \frac{4 \sqrt{2 MN}}{\delta_0} +  \sqrt{\frac{2}{\mu}\left(\tilde{C}+ \frac{4MN}{\delta_0}\right)},
\end{align*}
proving that $\{w_t\}_{t=1}^{\infty}$ is bounded.

(b) By Lemma \ref{lemma:strong_convexity_ineq}, we can derive
\begin{align*}
    \mathcal{L}_{\delta_{t+1}}(u_{t+1}, w_{t+1}, z_{t+1}) \leq& \mathcal{L}_{0}(u_{0}, w_{0}, z_{0}) + \frac{(\sigma+1)}{2 \delta_0} \sum_{i=0}^t \frac{1}{\sigma^i} \|z_{i+1} - z_i\|_2^2\\ & - \frac{\zeta}{2} \sum_{i=0}^t \|u_{i+1}-u_i\|_2^2. 
\end{align*}
By \eqref{eq:z_bound}, we have
\begin{align} \label{eq:lagrange_lower_bound}
\begin{split}
    \mathcal{L}_{\delta_{t+1}}(u_{t+1}, w_{t+1}, z_{t+1}) &\geq - \frac{1}{2 \delta_{t+1}}\|z_{t+1}\|_2^2 \geq -\frac{4MN}{\delta_{0}}, \; \forall t \in \mathbb{N}.
    \end{split}
\end{align}
Combining the two inequalities gives us
\begin{align*}
     -\frac{4MN}{\delta_{0}} +\frac{\zeta}{2} \sum_{i=0}^t \|u_{i+1} - u_i\|_2^2   \leq&\mathcal{L}_{\delta_{t+1}}(u_{t+1}, w_{t+1}, z_{t+1})+\frac{\zeta}{2}\sum_{i=0}^t \|u_{i+1} - u_i\|_2^2 \\ \leq&\mathcal{L}_{0}(u_{0}, w_{0}, z_{0}) + \frac{(\sigma+1)}{2 \delta_0} \sum_{i=0}^t \frac{1}{\sigma^i} \|z_{i+1} - z_i\|_2^2. 
\end{align*}
As $t \rightarrow \infty$, we obtain
{\
\begin{align*}
    0 &\leq \frac{ \zeta}{2}\sum_{i=0}^{\infty} \|u_{i+1} - u_i\|_2^2 \leq \mathcal{L}_{\delta_0}(u_0, w_0, z_0) +  \frac{(\sigma+1)}{2 \delta_0} \sum_{i=0}^{\infty} \frac{1}{\sigma^i} \|z_{i+1} - z_i\|_2^2 + \frac{4MN}{\delta_0}.
\end{align*}}%
Earlier in proving the boundedness of $\{\mathcal{L}_{\delta_t}(u_t, w_t, z_t)\}_{t=1}^{\infty}$, we show that the summation $\displaystyle \sum_{i=0}^{\infty} \frac{1}{\sigma^i} \|z_{i+1} - z_i\|_2^2$ converges. As a result, the summation $\displaystyle \sum_{i=0}^{\infty} \|u_{i+1} - u_i\|_2^2$ converges, which implies that $u_{t+1} -u_t \rightarrow 0$.
\end{proof}

\bibliography{references}


\end{document}